\documentclass[onecolumn,10pt,journal,compsoc]{IEEEtran}
\usepackage{epsfig}
\usepackage{amsmath}
\usepackage{amssymb}
\usepackage{algorithm}
\usepackage{algorithmic}
\usepackage{multirow}
\usepackage{url}
\usepackage[tight,footnotesize]{subfigure}
\usepackage{amsthm}
\usepackage{mathrsfs}
\usepackage[pagebackref=false,breaklinks=true,letterpaper=true,colorlinks,citecolor=blue,bookmarks=true]{hyperref}
\usepackage{enumerate}

\newtheorem{theorem}{Theorem}
\newtheorem{lemma}{Lemma}

\newtheorem{corollary}{Corollary}
\newtheorem{assumption}{Assumption}

\begin{document}

\vspace{3mm}
\title{Larger is Better: The Effect of Learning Rates Enjoyed by Stochastic Optimization with Progressive Variance Reduction}

\author{~\\
\large{Fanhua~Shang}\\
\vspace{3mm}
\small{Department of Computer Science and Engineering, The Chinese University of Hong Kong}\\
\small{fhshang@cse.cuhk.edu.hk}\\
~\\
\vspace{-3mm}
\today
\IEEEcompsocitemizethanks{\IEEEcompsocthanksitem All the codes of VR-SGD and some related variance reduced stochastic methods can be downloaded from the author's website: \url{https://sites.google.com/site/fanhua217/publications}.}
}

\IEEEtitleabstractindextext{
\begin{abstract}
In this paper, we propose a simple variant of the original stochastic \emph{variance reduction} gradient (SVRG)~\cite{johnson:svrg}, where hereafter we refer to as the \emph{variance reduced stochastic gradient descent} (VR-SGD). Different from the choices of the \emph{snapshot point} and \emph{starting point} in SVRG and its proximal variant, Prox-SVRG~\cite{xiao:prox-svrg}, the two vectors of each epoch in VR-SGD are set to the \emph{average} and \emph{last iterate} of the previous epoch, respectively. This setting allows us to use much larger learning rates or step sizes than SVRG, e.g., $3/(7L)$ for VR-SGD vs.\ $1/(10L)$ for SVRG, and also makes our convergence analysis more challenging. In fact, a larger learning rate enjoyed by VR-SGD means that the variance of its stochastic gradient estimator asymptotically approaches zero more rapidly. Unlike common stochastic methods such as SVRG and proximal stochastic methods such as Prox-SVRG, we design two different update rules for \emph{smooth} and \emph{non-smooth} objective functions, respectively. In other words, VR-SGD can tackle non-smooth and/or non-strongly convex problems directly without using any reduction techniques such as quadratic regularizers. Moreover, we analyze the \emph{convergence properties} of VR-SGD for \emph{strongly convex} problems, which show that VR-SGD attains a \emph{linear} convergence rate. We also provide the \emph{convergence guarantees} of VR-SGD for \emph{non-strongly convex} problems. Experimental results show that the performance of VR-SGD is significantly better than its counterparts, SVRG and Prox-SVRG, and it is also much better than the \emph{best known} stochastic method, Katyusha~\cite{zhu:Katyusha}.
\end{abstract}

}

\maketitle

\newpage
\section{Introduction}
In this paper, we focus on the following \emph{composite} convex optimization problem:
\vspace{-1mm}
\begin{equation}\label{equ01}
\min_{x\in\mathbb{R}^{d}} F(x)\stackrel{\rm{def}}{=}\frac{1}{n}\sum^{n}_{i=1}f_{i}(x)+g(x)
\end{equation}
where $f(x)\!:=\!\frac{1}{n}\!\sum^{n}_{i=1}\!f_{i}(x)$, $f_{i}(x)\!:\!\mathbb{R}^{d}\!\rightarrow\!\mathbb{R},\,i\!=\!1,\ldots,n$ are the smooth convex functions, and $g(x)$ is a relatively simple (but possibly non-differentiable) convex function (referred to as a regularizer). The formulation (\ref{equ01}) arises in many places in machine learning, signal processing, data science, statistics and operations research, such as \emph{regularized empirical risk minimization} (ERM). For instance, one popular choice of the component function $f_{i}(\cdot)$ in binary classification problems is the logistic loss, i.e., $f_{i}(x)\!=\!\log(1+\exp(-b_{i}a^{T}_{i}x))$, where $\{(a_{1},b_{1}),\ldots,(a_{n},b_{n})\}$ is a collection of training examples, and $b_{i}\!\in\!\{\pm1\}$. Some popular choices for the regularizer include the $\ell_{2}$-norm regularizer (i.e., $g(x)\!=\!(\lambda_{1}/2)\|x\|^{2}$), the $\ell_{1}$-norm regularizer (i.e., $g(x)\!=\!\lambda_{2}\|x\|_{1}$), and the elastic-net regularizer (i.e., $g(x)\!=\!(\lambda_{1}/2)\|x\|^{2}\!+\!\lambda_{2}\|x\|_{1}$), where $\lambda_{1}\!\geq\!0$ and $\lambda_{2}\!\geq\!0$ are two regularization parameters. So far examples of some other applications include deep neural networks~\cite{johnson:svrg,krizhevsky:deep,sutskever:sgd,zhu:vrnc}, group Lasso~\cite{ouyang:sadmm,liu:avrrg}, sparse learning and coding~\cite{qu:svrg,paquette:catalyst}, phase retrieval~\cite{duchi:ssgd}, matrix completion~\cite{recht:psgd,zhang:svrgd}, conditional random fields~\cite{schmidt:crf}, eigenvector computation~\cite{shamir:pca,garber:svd} such as principal component analysis (PCA) and singular value decomposition (SVD), generalized eigen-decomposition and canonical correlation analysis (CCA)~\cite{zhu:cca}.

\subsection{Stochastic Gradient Descent}
In this paper, we are especially interested in developing efficient algorithms to solve regularized ERM problems involving a large sum of $n$ component functions. The standard and effective method for solving Problem (\ref{equ01}) is the (proximal) gradient descent (GD) method, including accelerated proximal gradient (APG)~\cite{nesterov:fast,nesterov:co,teng:apg,beck:fista}. For \emph{smooth} objective functions, the update rule of GD is
\begin{equation}\label{equ02}
x_{k+1}=x_{k}-\eta_{k}\!\left[\frac{1}{n}\sum^{n}_{i=1}\nabla\! f_{i}(x_{k})+\nabla\! g(x_{k})\right]
\end{equation}
for $k\!=\!1,2,\ldots$, where $\eta_{k}\!>\!0$ is commonly referred to as the step-size in optimization or the learning rate in machine learning. When the regularizer $g(\cdot)$ is \emph{non-smooth}, e.g., the $\ell_{1}$-norm regularizer, we need to introduce the following proximal operator into (\ref{equ02}),
\begin{equation}\label{equ03}
x_{k+1}=\textup{Prox}_{\,\eta_{k},g}(y_{k}):=\mathop{\arg\min}_{x\in\mathbb{R}^{d}}\,({1}/{2\eta_{k}})\!\cdot\!\|x-y_{k}\|^{2}+g(x)
\end{equation}
where $y_{k}\!=x_{k}\!-\!(\eta_{k}/n)\sum^{n}_{i=1}\!\nabla\! f_{i}(x_{k}).$ The GD methods mentioned above have been proven to achieve linear convergence for \emph{strongly convex} problems, and APG attains the optimal convergence rate of $\mathcal{O}(1/T^2)$ for \emph{non-strongly convex} problems, where $T$ denotes the number of iterations. However, the per-iteration cost of all the batch (or deterministic) methods is $O(nd)$, which is expensive.

Instead of evaluating the full gradient of $f(\cdot)$ at each iteration, an effective alternative is the stochastic (or incremental) gradient descent (SGD) method~\cite{robbins:sgd}. SGD only evaluates the gradient of a single component function at each iteration, thus it has much \emph{lower} per-iteration cost, $O(d)$, and has been successfully applied to many large-scale learning problems~\cite{krizhevsky:deep,zhang:sgd,hu:sgd,bubeck:sgd}. The update rule of SGD is formulated as follows:
\begin{equation}\label{equ04}
x_{k+1}=x_{k}-\eta_{k}[\nabla\!f_{i_{k}}\!(x_{k})+\nabla\!g(x_{k})]
\end{equation}
where $\eta_{k}\!\propto\!1/k$, and the index $i_{k}$ is chosen uniformly at random from $\{1,\ldots,n\}$. Although the expectation of $\nabla\!f_{i_{k}}\!(x_{k})$ is an \emph{unbiased} estimation for $\nabla\!f(x_{k})$, i.e., $\mathbb{E}[\nabla\!f_{i_{k}}\!(x_{k})]\!=\!\nabla\! f(x_{k})$, the variance of the stochastic gradient estimator $\nabla\!f_{i_{k}}\!(x_{k})$ may be large due to the variance of random sampling~\cite{johnson:svrg}. Thus, stochastic gradient estimators are also called ``noisy gradients", and we need to gradually reduce its step size, leading to slow convergence. In particular, even under the strongly convex condition, standard SGD attains a slower \emph{sub-linear} convergence rate of $\mathcal{O}(1/T)$~\cite{rakhlin:sgd,shamir:sgd}.

\subsection{Accelerated SGD}
Recently, many SGD methods with \emph{variance reduction} techniques were proposed, such as stochastic average gradient (SAG)~\cite{roux:sag}, stochastic variance reduced gradient (SVRG)~\cite{johnson:svrg}, stochastic dual coordinate ascent (SDCA)~\cite{shalev-Shwartz:sdca}, SAGA~\cite{defazio:saga}, stochastic primal-dual coordinate (SPDC)~\cite{zhang:spdc}, and their proximal variants, such as Prox-SAG~\cite{schmidt:sag}, Prox-SVRG~\cite{xiao:prox-svrg} and Prox-SDCA~\cite{shalev-Shwartz:prox-sdca}. All these accelerated SGD methods can use a constant step size $\eta$ instead of diminishing step sizes for SGD, and fall into the following three categories: \emph{primal} methods such as SVRG and SAGA, \emph{dual} methods such as SDCA, and \emph{primal-dual} methods such as SPDC. In essence, many of primal methods use the full gradient at the snapshot $\widetilde{x}$ or average gradients to progressively reduce the variance of stochastic gradient estimators, as well as dual and primal-dual methods, \emph{which leads to a revolution in the area of first-order methods}~\cite{shang:fsvrg}. Thus, they are also known as the \emph{hybrid gradient descent} method~\cite{zhang:svrg} or \emph{semi-stochastic gradient descent} method~\cite{koneeny:mini}. In particular, under the strongly convex condition, the accelerated SGD methods enjoy linear convergence rates and the overall complexity of $\mathcal{O}\!\left((n\!+\!{L}/{\mu})\log({1}/{\epsilon})\right)$ to obtain an $\epsilon$-suboptimal solution, where each $f_{i}(\cdot)$ is $L$-smooth and $g(\cdot)$ is $\mu$-strongly convex. The complexity bound shows that they converge significantly faster than deterministic APG methods, whose complexity is $\mathcal{O}((n\sqrt{L/\mu})\log({1}/{\epsilon}))$~\cite{koneeny:mini}.

SVRG~\cite{johnson:svrg} and its proximal variant, Prox-SVRG~\cite{xiao:prox-svrg}, are particularly attractive because of their \emph{low storage} requirement compared with other stochastic methods such as SAG, SAGA and SDCA, which require to store all the gradients of the $n$ component functions $f_{i}(\cdot)$ or dual variables. At the beginning of each epoch in SVRG, the full gradient $\nabla\! f(\widetilde{x})$ is computed at the \emph{snapshot point} $\widetilde{x}$, which is updated periodically. The update rule for the \emph{smooth} optimization problem (\ref{equ01}) is given by
\addtocounter{equation}{1}
\begin{align}
&\widetilde{\nabla}\! f_{i_{k}}\!(x_{k})=\nabla\! f_{i_{k}}\!(x_{k})-\nabla\! f_{i_{k}}\!(\widetilde{x})+\nabla\! f(\widetilde{x})\tag{\theequation a},\label{equ051}\\
&x_{k+1}=x_{k}-\eta[\widetilde{\nabla}\! f_{i_{k}}\!(x_{k})+\nabla\!g(x_{k})]\tag{\theequation b}.\label{equ052}
\end{align}
When $g(\cdot)\!\equiv\!0$, the update rule in (\ref{equ052}) becomes the original one in~\cite{johnson:svrg}, i.e., $x_{k+1}\!=\!x_{k}\!-\!\eta\widetilde{\nabla}\! f_{i_{k}}\!(x_{k})$. It is not hard to verify that the variance of the SVRG estimator $\widetilde{\nabla}\! f_{i_{k}}\!(x_{k})$, i.e., $\mathbb{E}\|\widetilde{\nabla}\! f_{i_{k}}\!(x_{k})\!-\!\nabla\!f(x_{k})\|^2$, can be much smaller than that of the SGD estimator $\nabla\! f_{i_{k}}\!(x_{k})$, i.e., $\mathbb{E}\|\nabla\! f_{i_{k}}\!(x_{k})\!-\!\nabla\!f(x_{k})\|^2$. However, for \emph{non-strongly convex} problems, the accelerated SGD methods mentioned above converge much slower than batch APG methods such as FISTA~\cite{beck:fista}, namely, $\mathcal{O}(1/T)$ vs.\ $\mathcal{O}(1/T^2)$.

More recently, many \emph{acceleration} techniques were proposed to further speed up those variance-reduced stochastic methods mentioned above. These techniques mainly include the Nesterov's acceleration technique used in~\cite{hu:sgd,nitanda:svrg,lan:rpdg,frostig:sgd,lin:vrsg}, reducing the number of gradient calculations in the early iterations \cite{shang:fsvrg,babanezhad:vrsg,zhu:univr}, the projection-free property of the conditional gradient method (also known as the Frank-Wolfe algorithm~\cite{frank:cg}) as in~\cite{hazan:svrf}, the stochastic sufficient decrease technique~\cite{shang:vrsgd}, and the momentum acceleration trick in~\cite{zhu:Katyusha,shang:fsvrg,hien:asmd}. More specifically, \cite{lin:vrsg} proposed an accelerating Catalyst framework and achieved the complexity of $\mathcal{O}((n\!+\!\!\sqrt{n{L}/{\mu}})\log({L}/{\mu})\log({1}/{\epsilon}))$ for strongly convex problems. \cite{zhu:Katyusha} and \cite{hien:asmd} proved that their accelerated methods can attain the best known complexity of $\mathcal{O}(n\log(1/\epsilon)\!+\!\sqrt{nL/\epsilon})$ for non-strongly convex problems. The overall complexity matches the theoretical upper bound provided in \cite{woodworth:bound}. Katyusha~\cite{zhu:Katyusha} and point-SAGA~\cite{defazio:sagab} achieve the best-known complexity of $\mathcal{O}((n\!+\!\sqrt{nL/\mu})\log(1/\epsilon))$ for strongly convex problems, which is identical to the upper complexity bound in~\cite{woodworth:bound}. That is, \emph{Katyusha} is the \emph{best known} stochastic optimization method for both \emph{strongly convex} and \emph{non-strongly convex} problems. Its proximal gradient update rules are formulated as follows:
\addtocounter{equation}{1}
\begin{align}
&x_{k+1}=w_{1}y_{k}+w_{2}\widetilde{x}+(1-w_{1}-w_{2})z_{k}\tag{\theequation a},\label{equ061}\\
&y_{k+1}=\mathop{\arg\min}_{y\in\mathbb{R}^{d}} \left\{\frac{1}{2\eta}\|y-y_{k}\|^2+y^{T}\widetilde{\nabla}\! f_{i_{k}}\!(x_{k+1})+g(y)\right\}\tag{\theequation b},\label{equ062}\\
&z_{k+1}=\mathop{\arg\min}_{z\in\mathbb{R}^{d}} \left\{\frac{3L}{2}\|z-x_{k+1}\|^2+z^{T}\widetilde{\nabla}\! f_{i_{k}}\!(x_{k+1})+g(z)\right\}\tag{\theequation c}\label{equ063}
\end{align}
where $w_{1},w_{2}\!\in\![0,1]$ are two momentum parameters. To eliminate the need for parameter tuning, $\eta$ is set to $1/(3w_{1}L)$, and $w_{2}$ is fixed to $0.5$ in~\cite{zhu:Katyusha}. Unfortunately, most of the accelerated methods mentioned above, including Katyusha, require at least two auxiliary variables and two momentum parameters, which lead to complicated algorithm design and high per-iteration complexity~\cite{shang:fsvrg}.

\subsection{Our Contributions}
From the above discussion, one can see that most of accelerated stochastic variance reduction methods such as~\cite{zhu:Katyusha,shang:fsvrg,nitanda:svrg,zhu:univr,hazan:svrf,shang:vrsgd} and applications such as~\cite{qu:svrg,paquette:catalyst,zhang:svrgd,shamir:pca,garber:svd,li:svrg} are based on the \emph{stochastic variance reduced gradient} (SVRG) method~\cite{johnson:svrg}. Thus, any key improvement on SVRG is very important for the research of stochastic optimization. In this paper, we propose a simple variant of the original SVRG~\cite{johnson:svrg}, which is referred to as the \emph{variance reduced stochastic gradient descent} (VR-SGD). The \emph{snapshot point} and \emph{starting point} of each epoch in VR-SGD are set to the \emph{average} and \emph{last iterate} of the previous epoch, respectively. Different from the settings of SVRG and Prox-SVRG~\cite{xiao:prox-svrg} (i.e., the last iterate for the two points of the former, while the average of the previous epoch for those of the latter), the two points in VR-SGD are different, which makes our convergence analysis \emph{more challenging} than SVRG and Prox-SVRG. Our empirical results show that the performance of VR-SGD is significantly better than its counterparts, SVRG and Prox-SVRG. Impressively, VR-SGD with a \emph{sufficiently large} learning rate performs much better than the \emph{best known} stochastic method, Katyusha~\cite{zhu:Katyusha}. The main contributions of this paper are summarized below.
\begin{itemize}
  \item The \emph{snapshot point} and \emph{starting point} of VR-SGD are set to two different vectors. That is, for all epochs, except the first one, $\widetilde{x}^{s}\!=\!\frac{1}{m-\!1}\!\sum^{m-\!1}_{k=1}x^{s}_{k}$ (denoted by \emph{Option I}) or $\widetilde{x}^{s}\!=\!\frac{1}{m}\!\sum^{m}_{k=1}x^{s}_{k}$ (denoted by \emph{Option II}), and $x^{s+\!1}_{0}\!=\!x^{s}_{m}$. In particular, we find that the setting of VR-SGD allows us take much \emph{larger learning rates} or step sizes than SVRG, e.g., $3/(7L)$ vs.\ $1/(10L)$, and thus significantly speeds up the convergence of SVRG and Prox-SVRG in practice. Moreover, VR-SGD has an advantage over SVRG in terms of \emph{robustness of learning rate selection}.
  \item Different from \emph{proximal stochastic gradient} methods, e.g., Prox-SVRG and Katyusha, which have a unified update rule for the two cases of \emph{smooth} and \emph{non-smooth} objectives (see Section~\ref{sec22} for details), VR-SGD employs \emph{two different update rules} for the two cases, respectively, as in (\ref{equ21}) and (\ref{equ22}) below. Empirical results show that gradient update rules as in (\ref{equ21}) for \emph{smooth} optimization problems are better choices than proximal update formulas as in \eqref{equ14}.
  \item Finally, we theoretically analyze the \emph{convergence properties} of VR-SGD with Option I or Option II for \emph{strongly convex} problems, which show that VR-SGD attains a \emph{linear} convergence rate. We also give the \emph{convergence guarantees} of VR-SGD with Option I or Option II for \emph{non-strongly convex} objective functions.
\end{itemize}

\section{Preliminary and Related Work}
Throughout this paper, we use $\|\cdot\|$ to denote the $\ell_{2}$-norm (also known as the standard Euclidean norm), and $\|\!\cdot\!\|_{1}$ is the $\ell_{1}$-norm, i.e., $\|x\|_{1}\!=\!\sum^{d}_{i=1}\!|x_{i}|$. $\nabla\!f(\cdot)$ denotes the full gradient of $f(\cdot)$ if it is differentiable, or $\partial\!f(\cdot)$ the subgradient if $f(\cdot)$ is only Lipschitz continuous. For each epoch $s\!\in\![S]$ and inner iteration $k\!\in\!\{0,1,\ldots,m\!-\!1\}$, $i^{s}_{k}\!\in\![n]$ is the random chosen index. We mostly focus on the case of Problem~\eqref{equ01} when each component function $f_{i}(\cdot)$ is $L$-smooth\footnote{Actually, we can extend all the theoretical results in this paper for the case, when the gradients of all component functions have the same Lipschitz constant $L$, to the more general case, when some $f_{i}(\cdot)$ have different degrees of smoothness.}, and $F(\cdot)$ is $\mu$-strongly convex. The two common assumptions are defined as follows.

\subsection{Basic Assumptions}
\begin{assumption}[Smoothness]\label{assum1}
Each convex function $f_{i}(\cdot)$ is $L$-smooth, that is, there exists a constant $L\!>\!0$ such that for all $x,y\!\in\!\mathbb{R}^{d}$,
\begin{equation}\label{equ11}
\|\nabla f_{i}(x)-\nabla f_{i}(y)\|\leq L\|x-y\|.
\end{equation}
\end{assumption}

\begin{assumption}[Strong Convexity]\label{assum2}
The convex function $F(x)$ is $\mu$-strongly convex, i.e., there exists a constant $\mu\!>\!0$ such that for all $x,y\!\in\! \mathbb{R}^{d}$,
\begin{equation}\label{equ12}
F(y)\geq F(x)+\langle\nabla\!F(x),\,y-x\rangle+\frac{\mu}{2}\|x-y\|^{2}.
\end{equation}
\end{assumption}
Note that when the regularizer $g(\cdot)$ is \emph{non-smooth}, the inequality in~\eqref{equ12} needs to be revised by simply replacing the gradient $\nabla\!F(x)$ in~\eqref{equ12} with an arbitrary sub-gradient of $F(\cdot)$ at $x$. In contrast, for a non-strongly convex or \emph{general convex} function, the inequality in~\eqref{equ12} can always be satisfied with $\mu\!=\!0$.

\begin{algorithm}[t]
\caption{SVRG (Option I) and Prox-SVRG (Option II)}
\label{alg1}
\renewcommand{\algorithmicrequire}{\textbf{Input:}}
\renewcommand{\algorithmicensure}{\textbf{Initialize:}}
\renewcommand{\algorithmicoutput}{\textbf{Output:}}
\renewcommand{\baselinestretch}{1.5}
\begin{algorithmic}[1]
\REQUIRE The number of epochs $S$, the number of iterations $m$ per epoch, and step size $\eta$.\\
\ENSURE $\widetilde{x}^{0}$.\\
\FOR{$s=1,2,\ldots,S$}
\STATE {$x^{s}_{0}=\widetilde{x}^{s-1}$; \hfill $\%$\:\emph{Initiate the variable $x^{s}_{0}$}}
\STATE {$\widetilde{\mu}^{s}=\frac{1}{n}\!\sum^{n}_{i=1}\!\nabla\!f_{i}(\widetilde{x}^{s-1})$; \hfill $\%$\:\emph{Compute the full gradient}}
\FOR{$k=0,1,\ldots,m-1$}
\STATE {Pick $i^{s}_{k}$ uniformly at random from $[n]$;}
\STATE {$\widetilde{\nabla}\! f_{i^{s}_{k}}(x^{s}_{k})=\nabla\! f_{i^{s}_{k}}(x^{s}_{k})-\nabla\! f_{i^{s}_{k}}(\widetilde{x}^{s-1})+\widetilde{\mu}^{s}$; \hfill $\%$\:\emph{The stochastic gradient estimator}}
\STATE {Option I:\, $x^{s}_{k+1}=x^{s}_{k}\!-\eta\!\left[\widetilde{\nabla}\!f_{i^{s}_{k}}(x^{s}_{k})+\!\nabla g(x^{s}_{k})\right]$, \hfill $\%$\:\emph{Smooth case of $g(\cdot)$}\\
\qquad\quad or \;\!$x^{s}_{k+1}=\textrm{Prox}_{\,\eta,\,g}\!\left(x^{s}_{k}\!-\eta\widetilde{\nabla}\!f_{i^{s}_{k}}(x^{s}_{k})\right)$; \hfill $\%$\:\emph{Non-smooth case of $g(\cdot)$}}
\STATE {Option II: $x^{s}_{k+1}=\arg\min_{y\in\mathbb{R}^{d}}\!\left\{g(y)+y^{T}\widetilde{\nabla}\!f_{i^{s}_{k}}(x^{s}_{k})+\!\frac{1}{2\eta}\|y-x^{s}_{k}\|^2\right\}$; \hfill $\%$\:\emph{Proximal update}}
\ENDFOR
\STATE {Option I:\, $\widetilde{x}^{s}=x^{s}_{m}$;  \hfill $\%$\:\emph{Last iterate for snapshot $\widetilde{x}$}}
\STATE {Option II:\,$\widetilde{x}^{s}\!=\!\frac{1}{m}\!\sum^{m}_{k=1}\!x^{s}_{k}$; \hfill $\%$\:\emph{Iterate averaging for snapshot $\widetilde{x}$}}
\ENDFOR
\OUTPUT {$\widetilde{x}^{S}$}
\end{algorithmic}
\end{algorithm}

\subsection{Related Work}
\label{sec22}
To speed up standard and proximal SGD methods, many variance reduced stochastic methods~\cite{roux:sag,shalev-Shwartz:sdca,defazio:saga,zhang:svrg} have been proposed for some special cases of Problem \eqref{equ01}. In the case when each $f_{i}(x)$ is $L$-smooth, $f(x)$ is $\mu$-strongly convex, and $g(x)\!\equiv\!0$, Roux \emph{et al.} \cite{roux:sag} proposed a stochastic average gradient (SAG) method, which attains a linear convergence rate. However, SAG needs to store all gradients as well as other incremental aggregated gradient methods such as SAGA~\cite{defazio:saga}, so that $O(nd)$ storage is required in general problems~\cite{babanezhad:vrsg}. Similarly, SDCA~\cite{shalev-Shwartz:sdca} requires storage of all dual variables~\cite{johnson:svrg}, which scales as $O(n)$. In contrast, SVRG~\cite{johnson:svrg}, as well as its proximal variant, Prox-SVRG~\cite{xiao:prox-svrg}, has the similar convergence rate to SAG and SDCA but without the memory requirements of all gradients and dual variables. In particular, the SVRG estimator in~\eqref{equ051} (independently introduced in~\cite{johnson:svrg,zhang:svrg}) may be the most popular choice for \emph{stochastic gradient estimators}. Besides, other stochastic gradient estimators include the SAGA estimator in~\cite{defazio:saga} and the stochastic recursive gradient estimator in~\cite{nguyen:srg}. Although the original SVRG in~\cite{johnson:svrg} only has convergence guarantees for a special case of Problem \eqref{equ01}, when each $f_{i}(x)$ is $L$-smooth, $f(x)$ is $\mu$-strongly convex, and $g(x)\!\equiv\!0$, one can extend SVRG to the proximal setting by introducing the proximal operator in~(\ref{equ03}), as shown in Line 7 of Algorithm~\ref{alg1}. In other words, when $g(\cdot)$ is \emph{non-smooth}, the update rule of SVRG becomes
\begin{equation}\label{equ13}
x^{s}_{k+1}=\mathop{\arg\min}_{x\in\mathbb{R}^{d}}\left\{({1}/{2\eta})\!\cdot\!\|x-[x^{s}_{k}-\eta\widetilde{\nabla}\!f_{i^{s}_{k}}(x^{s}_{k})]\|^{2}+g(x)\right\}.
\end{equation}

Some researchers~\cite{liu:avrrg,zhong:fsadmm,zheng:fadmm} have borrowed some \emph{variance reduction} techniques into ADMM for minimizing convex composite objective functions subject to an equality constraint. \cite{shamir:pca,garber:svd,zhu:cca} applied efficient stochastic solvers to compute leading eigenvectors of a symmetric matrix or generalized eigenvectors of two symmetric matrices. The first such method is VR-PCA by Shamir~\cite{shamir:pca}, and the convergence properties of the VR-PCA algorithm for such a non-convex problem are also provided. Garber \emph{et al.}~\cite{garber:svd} analyzed the convergence rate of SVRG when $f(\cdot)$ is a convex function that is a sum of non-convex component functions. Moreover, \cite{zhu:vrnc} and \cite{reddi:saga} proved that SVRG and SAGA with minor modifications can converge asymptotically to a stationary point of non-convex finite-sum problems. Some distributed variants~\cite{reddi:sgd,lee:dsgd} of accelerated SGD methods have also been proposed.

An important class of stochastic methods is the \emph{proximal stochastic gradient} (Prox-SG) method, such as Prox-SVRG~\cite{xiao:prox-svrg}, SAGA~\cite{defazio:saga}, and Katyusha~\cite{zhu:Katyusha}. Different from standard variance reduction SGD methods such as SVRG, which have a stochastic gradient update as in~\eqref{equ052}, the Prox-SG method has a unified update rule for both smooth and non-smooth cases of $g(\cdot)$. For instance, the update rule of Prox-SVRG~\cite{xiao:prox-svrg} is formulated as follows:
\begin{equation}\label{equ14}
x^{s}_{k+1}=\mathop{\arg\min}_{y\in\mathbb{R}^{d}}\left\{g(y)+y^{T}\widetilde{\nabla}\!f_{i^{s}_{k}}(x^{s}_{k})+\!\frac{1}{2\eta}\|y-x^{s}_{k}\|^2\right\}.
\end{equation}
For the sake of completeness, the details of Prox-SVRG~\cite{xiao:prox-svrg} are shown in Algorithm~\ref{alg1} with Option II. When $g(\cdot)$ is the widely used $\ell_{2}$-norm regularizer, i.e., $g(\cdot)=(\lambda_{1}/2)\|\cdot\|^{2}$, the proximal update formula in \eqref{equ14} becomes
\begin{equation}\label{equ15}
x^{s}_{k+1}=\frac{1}{1+\lambda_{1}\eta}\left[x^{s}_{k}-\eta\widetilde{\nabla}\!f_{i^{s}_{k}}(x^{s}_{k})\right].
\end{equation}

\section{Variance-Reduced Stochastic Gradient Descent}
In this section, we propose an efficient variance reduced stochastic gradient descent (VR-SGD) method with \emph{iterate averaging}. Different from the choices of the snapshot and starting points in SVRG~\cite{johnson:svrg} and Prox-SVRG~\cite{xiao:prox-svrg}, the two vectors of each epoch in VR-SGD are set to the average and last iterate of the previous epoch, respectively. Unlike common stochastic gradient methods such as SVRG and proximal stochastic gradient methods such as Prox-SVRG, we design two different update rules for smooth and non-smooth objective functions, respectively.

\subsection{Iterate Averaging}
Like SVRG and Katyusha, VR-SGD is also divided into $S$ epochs, and each epoch consists of $m$ stochastic gradient steps, where $m$ is usually chosen to be $\Theta(n)$ as in~\cite{johnson:svrg,xiao:prox-svrg,zhu:Katyusha}. Within each epoch, we need to compute the full gradient $\nabla\! f(\widetilde{x}^{s})$ at the snapshot $\widetilde{x}^{s}$ and use it to define the variance reduced stochastic gradient estimator $\widetilde{\nabla}\! f_{i_{k}}\!(x^{s}_{k})$ as in~\cite{johnson:svrg}. Unlike SVRG whose snapshot point is set to the last iterate of the previous epoch, the \emph{snapshot} $\widetilde{x}^{s}$ of each epoch in VR-SGD is set to the \emph{average} of the previous epoch, e.g., $\widetilde{x}^{s}\!=\!\frac{1}{m-\!1}\!\sum^{m-\!1}_{k=1}x^{s}_{k}$ in Option I of Algorithm~\ref{alg2}, which leads to better robustness to gradient noise\footnote{It should be emphasized that the noise introduced by random sampling is inevitable, and generally slows down the convergence speed in a sense. However, SGD and its variants are probably the most used optimization algorithms for deep learning~\cite{bengio:deep}. In particular, \cite{ge:sgd} has shown that by adding noise at each step, noisy gradient descent can escape the saddle points efficiently and converge to a local minimum of non-convex optimization problems, as well as the application of deep neural networks in~\cite{neelakantan:noise}.}, as also suggested in~\cite{shang:fsvrg,shang:vrsgd,flammarion:sgd}. In fact, the choice of Option II in Algorithm~\ref{alg2}, i.e., $\widetilde{x}^{s}\!=\!\frac{1}{m}\!\sum^{m}_{k=1}x^{s}_{k}$, also works well in practice, as shown in Fig.\ \ref{figs00}. Therefore, we provide the convergence guarantees for both our algorithms (including Algorithm~\ref{alg2}) with Option I and our algorithms with Option II in the next section. In particular, we find that the one of the effects of the choice in Option I or Option II of Algorithm~\ref{alg2} is to allow taking much larger learning rates or step sizes than SVRG in practice, e.g., $3/(7L)$ for VR-SGD vs.\ $1/(10L)$ for SVRG (see Fig.\ \ref{figs01} and Section~\ref{sec53} for details). This is the main reason why VR-SGD converges significantly faster than SVRG. Actually, a larger learning rate enjoyed by VR-SGD means that the variance of its stochastic gradient estimator goes asymptotically to zero faster.

Unlike Prox-SVRG~\cite{xiao:prox-svrg} whose starting point is initialized to the average of the previous epoch, the starting point $x^{s+1}_{0}$ of each epoch in VR-SGD is set to the last iterate $x^{s}_{m}$ of the previous epoch. That is, the last iterate of the previous epoch becomes the new starting point in VR-SGD, while those of Prox-SVRG are completely different, thereby leading to relatively slow convergence in general. It is clear that both the starting point and snapshot point of each epoch in the original SVRG~\cite{johnson:svrg} are set to the last iterate of the previous epoch{\footnote{Note that the theoretical convergence of the original SVRG~\cite{johnson:svrg} relies on its Option II, i.e., both $\widetilde{x}^{s}$ and $x^{s+\!1}_{0}$ are set to $x^{s}_{k}$, where $k$ is randomly chosen from $\{1,2,\ldots,m\}$. However, the empirical results in~\cite{johnson:svrg} suggest that Option I is a better choice than its Option II, and the convergence guarantee of SVRG with Option I for strongly convex objective functions is provided in~\cite{tan:sgd}.}}, while the two points of Prox-SVRG~\cite{xiao:prox-svrg} are set to the average of the previous epoch (also suggested in~\cite{johnson:svrg}). Different from the settings in SVRG and Prox-SVRG, the starting and snapshot points in VR-SGD are set to the two different vectors mentioned above, which makes the convergence analysis of VR-SGD more challenging than SVRG and Prox-SVRG, as shown in Section~\ref{sec4}.

\begin{algorithm}[t]
\caption{VR-SGD for strongly convex objectives}
\label{alg2}
\renewcommand{\algorithmicrequire}{\textbf{Input:}}
\renewcommand{\algorithmicensure}{\textbf{Initialize:}}
\renewcommand{\algorithmicoutput}{\textbf{Output:}}
\begin{algorithmic}[1]
\REQUIRE The number of epochs $S$, the number of iterations $m$ per epoch, and step size $\eta$.\\
\ENSURE $x^{1}_{0}=\widetilde{x}^{0}$.\\
\FOR{$s=1,2,\ldots,S$}
\STATE {$\widetilde{\mu}^{s}=\frac{1}{n}\!\sum^{n}_{i=1}\!\nabla\!f_{i}(\widetilde{x}^{s-1})$; \hfill $\%$\:\emph{Compute the full gradient}}
\FOR{$k=0,1,\ldots,m-1$}
\STATE {Pick $i^{s}_{k}$ uniformly at random from $[n]$;}
\STATE {$\widetilde{\nabla}\! f_{i^{s}_{k}}(x^{s}_{k})=\nabla\! f_{i^{s}_{k}}(x^{s}_{k})-\nabla\! f_{i^{s}_{k}}(\widetilde{x}^{s-1})+\widetilde{\mu}^{s}$;  \hfill $\%$\:\emph{The stochastic gradient estimator}}
\STATE {$x^{s}_{k+1}=x^{s}_{k}\!-\eta\!\left[\widetilde{\nabla}\!f_{i^{s}_{k}}(x^{s}_{k})+\nabla g(x^{s}_{k})\right]$,    \hfill $\%$\:\emph{Smooth case of $g(\cdot)$}\\
or\; $x^{s}_{k+1}=\textrm{Prox}_{\,\eta,\,g}\!\left(x^{s}_{k}-\eta\widetilde{\nabla}\!f_{i^{s}_{k}}(x^{s}_{k})\right)$;     \hfill $\%$\:\emph{Non-smooth case of $g(\cdot)$}}
\ENDFOR
\STATE {Option I:\,$\widetilde{x}^{s}\!=\frac{1}{m-\!1}\!\sum^{m-\!1}_{k=1}x^{s}_{k}$; \hfill $\%$\:\emph{Iterate averaging for snapshot $\widetilde{x}$}}
\STATE {Option II:\,$\widetilde{x}^{s}\!=\!\frac{1}{m}\!\sum^{m}_{k=1}\!x^{s}_{k}$; \hfill $\%$\:\emph{Iterate averaging for snapshot $\widetilde{x}$}}
\STATE {$x^{s+1}_{0}\!=x^{s}_{m}$; \hfill $\%$\:\emph{Initiate $x^{s+1}_{0}$ for the next epoch}}
\ENDFOR
\OUTPUT {$\widehat{x}^{S}\!=\widetilde{x}^{S}$ if $F(\widetilde{x}^{S})\leq F(\frac{1}{S}\!\sum^{S}_{s=1}\!\widetilde{x}^{s})$, and $\widehat{x}^{S}\!=\!\frac{1}{S}\!\sum^{S}_{s=1}\!\widetilde{x}^{s}$ otherwise.}
\end{algorithmic}
\end{algorithm}

\begin{figure}[h]
\centering
\subfigure[Logistic regression]{\includegraphics[width=0.469\columnwidth]{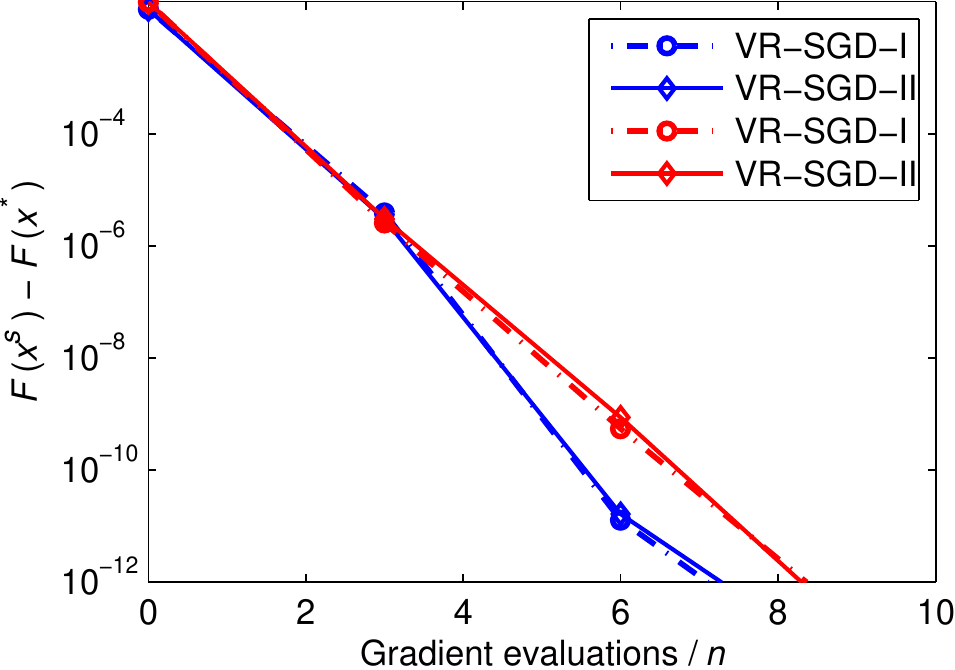}\label{figs11a}}\;\;\;
\subfigure[Ridge regression]{\includegraphics[width=0.469\columnwidth]{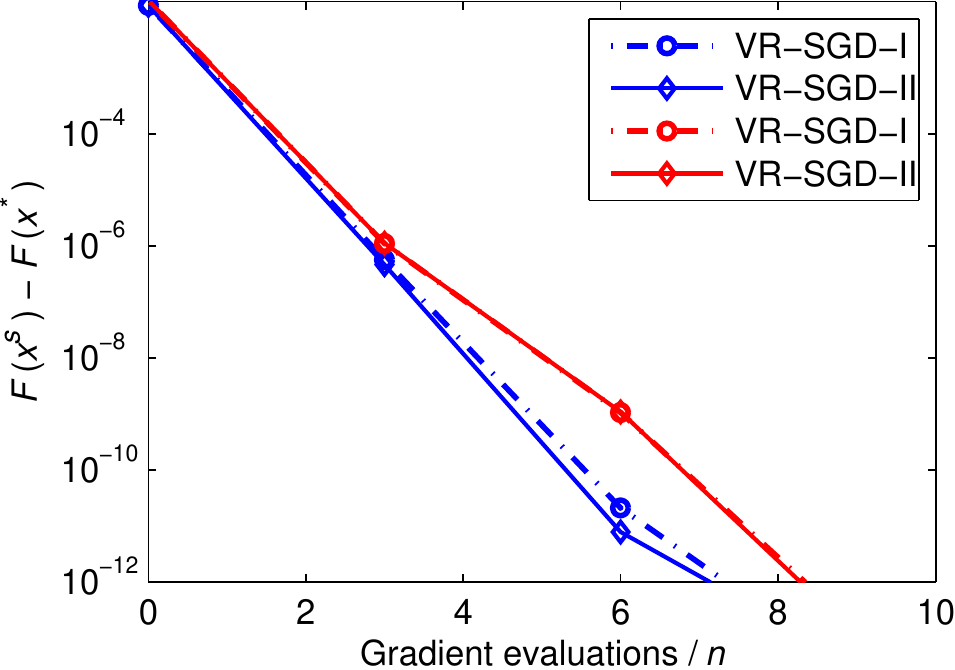}}
\caption{Comparison of VR-SGD with Option I (denoted by VR-SGD-I) and Option II (denoted by VR-SGD-II) for solving $\ell_{2}$-norm (i.e., $(\lambda/2)\|\!\cdot\!\|^{2}$) regularized logistic regression and ridge regression problems on the Covtype data set. In each plot, the vertical axis shows the objective value minus the minimum, and the horizontal axis is the number of effective passes. Note that the blue lines stand for the results where $\lambda=10^{-4}$, while the red lines correspond to the results where $\lambda=10^{-5}$ (best viewed in colors).}
\label{figs00}
\end{figure}

\subsection{The Algorithm for Strongly Convex Objectives}
In this part, we propose an efficient VR-SGD algorithm to solve \emph{strongly convex} objective functions, as outlined in \textbf{Algorithm} \ref{alg2}. It is well known that the original SVRG~\cite{johnson:svrg} only works for the case of smooth and strongly convex objective functions. However, in many machine learning applications, e.g., elastic net regularized logistic regression, the strongly convex objective function $F(x)$ is non-smooth. To solve this class of problems, the proximal variant of SVRG, Prox-SVRG~\cite{xiao:prox-svrg}, was subsequently proposed. Unlike SVRG and Prox-SVRG, VR-SGD can not only solve \emph{smooth} objective functions, but directly tackle \emph{non-smooth} ones. That is, when the regularizer $g(x)$ is smooth, e.g., the $\ell_{2}$-norm regularizer, the update rule of VR-SGD is
\begin{equation}\label{equ21}
x^{s}_{k+1}=x^{s}_{k}-\eta[\widetilde{\nabla}f_{i^{s}_{k}}(x^{s}_{k})+\nabla g(x^{s}_{k})].
\end{equation}
When $g(x)$ is non-smooth, e.g., the $\ell_{1}$-norm regularizer, the update rule of VR-SGD becomes
\begin{equation}\label{equ22}
x^{s}_{k+1}=\textrm{Prox}_{\,\eta,\,g}\!\left(x^{s}_{k}-\eta\widetilde{\nabla}\!f_{i^{s}_{k}}(x^{s}_{k})\right).
\end{equation}

Different from the proximal stochastic gradient methods such as Prox-SVRG~\cite{xiao:prox-svrg}, all of which have a unified update rule as in (\ref{equ14}) for both the smooth and non-smooth cases of $g(\cdot)$, VR-SGD has two different update rules for the two cases, as stated in (\ref{equ21}) and (\ref{equ22}). This leads to the following advantage over the Prox-SG methods: the stochastic gradient update rule in (\ref{equ21}) usually outperforms the proximal stochastic gradient update rule in (\ref{equ15}), as well as the two classes of update rules for Katyusha~\cite{zhu:Katyusha} (see Section~\ref{sec54} for details).

Fig.\ \ref{figs01} demonstrates that VR-SGD has a significant advantage over SVRG in terms of robustness of learning rate selection. That is, VR-SGD yields good performance within the range of the learning rate between $0.1/L$ and $0.4/L$, whereas the performance of SVRG is very \emph{sensitive} to the selection of learning rates. Thus, VR-SGD is convenient to apply in various real-world problems of large-scale machine learning. In fact, VR-SGD can use much larger learning rates than SVRG for logistic regression problems in practice, e.g., $6/(5L)$ for VR-SGD vs.\ $1/(10L)$ for SVRG, as shown in Fig.\ \ref{figs01a}.

\begin{figure}[t]
\centering
\subfigure[Logistic regression: $\lambda=10^{-4}$ (left) \;and\; $\lambda=10^{-5}$ (right)]{
\includegraphics[width=0.489\columnwidth]{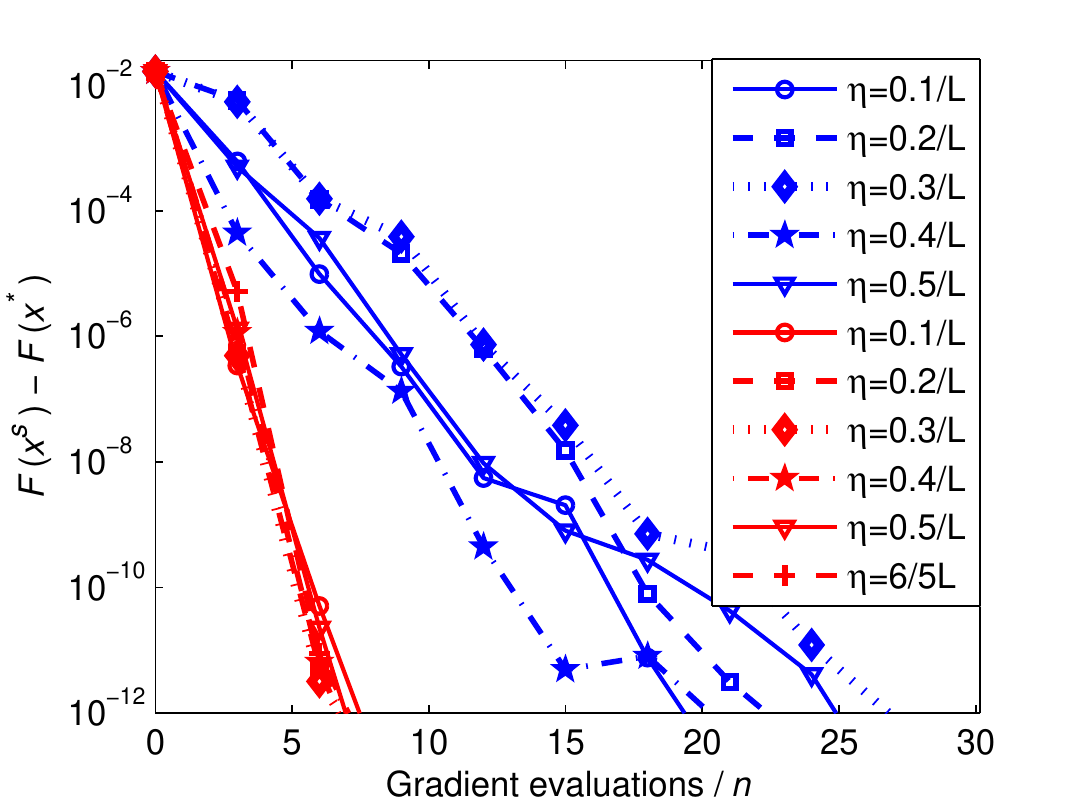}
\includegraphics[width=0.489\columnwidth]{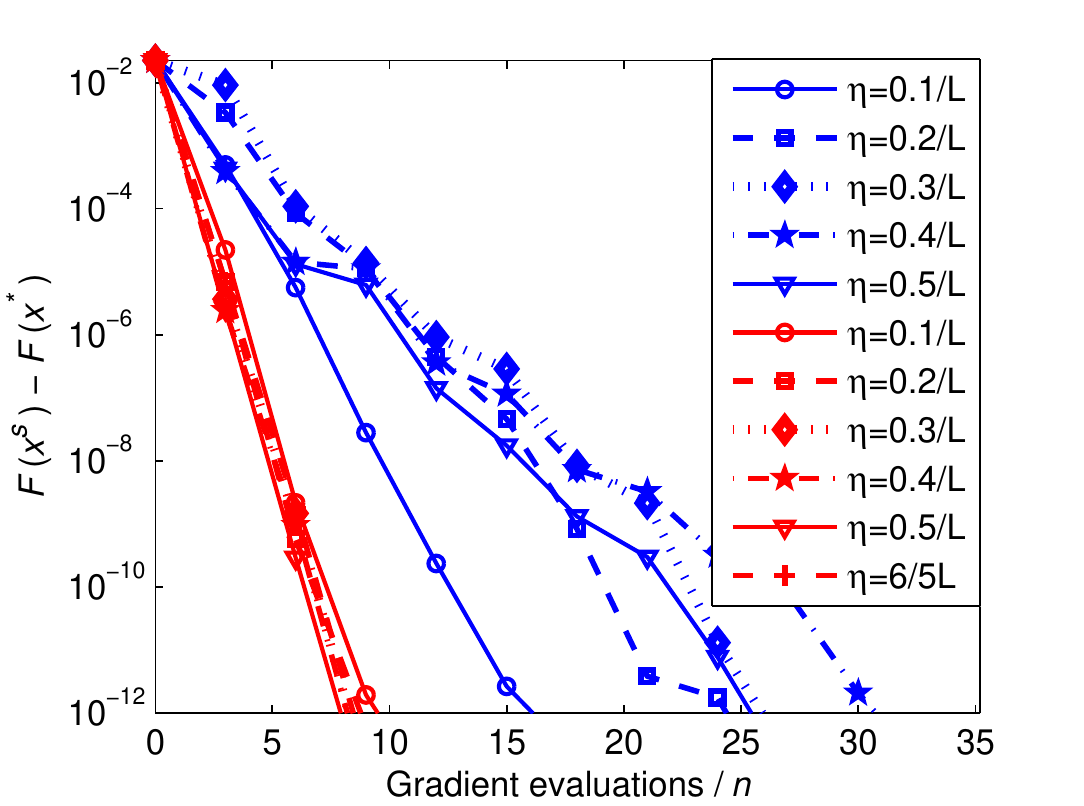}\label{figs01a}}
\vspace{0.6mm}

\subfigure[Ridge regression: $\lambda=10^{-4}$ (left) \;and\; $\lambda=10^{-5}$ (right)]{
\includegraphics[width=0.489\columnwidth]{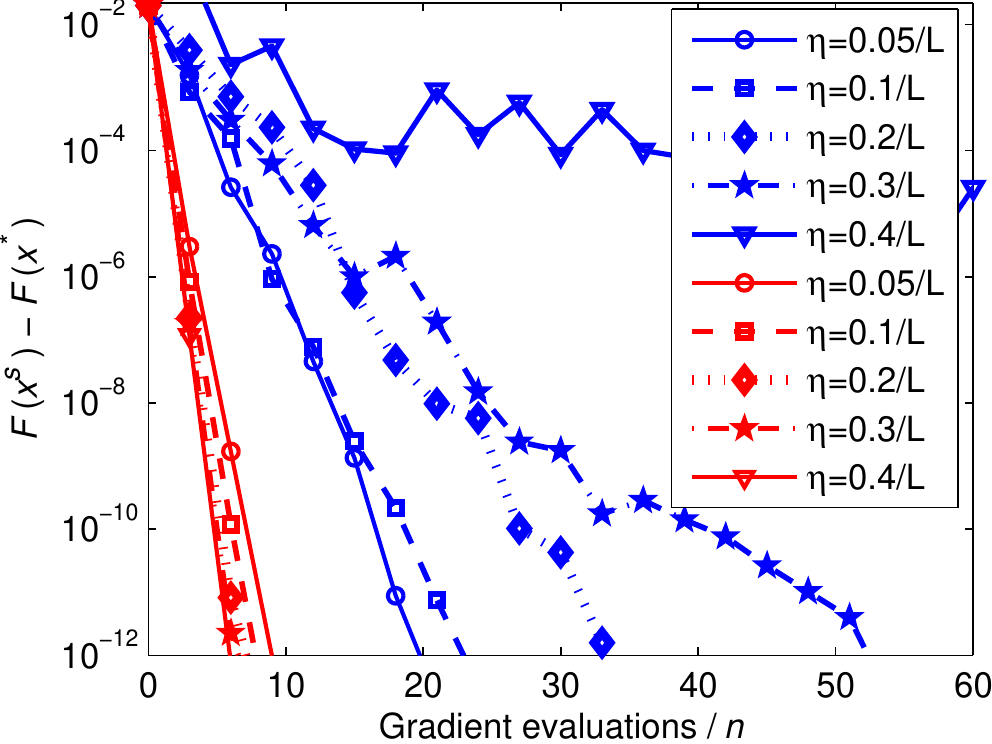}
\includegraphics[width=0.489\columnwidth]{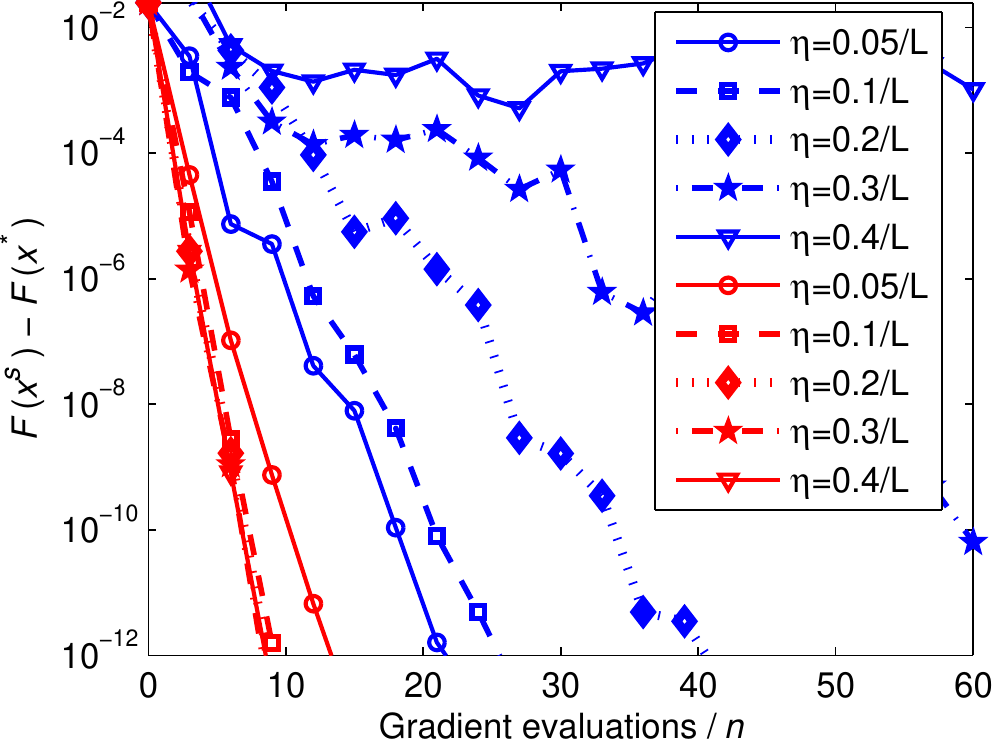}}
\caption{Comparison of SVRG~\cite{johnson:svrg} and VR-SGD with different learning rates for solving $\ell_{2}$-norm (i.e., $(\lambda/2)\|\cdot\|^{2}$) regularized logistic regression and ridge regression problems on the Covtype data set. In each plot, the vertical axis shows the objective value minus the minimum, and the horizontal axis is the number of effective passes. Note that the \emph{blue} lines stand for the results of SVRG with different learning rates, while the \emph{red} lines correspond to the results of VR-SGD with different learning rates (best viewed in colors).}
\label{figs01}
\end{figure}

\subsection{The Algorithm for Non-Strongly Convex Objectives}
Although many variance reduced stochastic methods have been proposed, most of them, including SVRG and Prox-SVRG, only have convergence guarantees for the case of Problem (\ref{equ01}) when the objective function $F(x)$ is strongly convex. However, $F(x)$ may be non-strongly convex in many machine learning applications such as Lasso and $\ell_{1}$-norm regularized logistic regression. As suggested in~\cite{zhu:Katyusha,zhu:box}, this class of problems can be transformed into strongly convex ones by adding a proximal term $(\tau/2)\|x-x^{s}_{0}\|^2$, which can be efficiently solved by Algorithm~\ref{alg2}. However, the reduction technique may degrade the performance of the involved algorithms both in theory and in practice~\cite{zhu:univr}. Thus, we present an efficient VR-SGD algorithm for directly solving the \emph{non-strongly convex} problem (\ref{equ01}), as outlined in \textbf{Algorithm} \ref{alg3}.

The main difference between Algorithm~\ref{alg2} and Algorithm~\ref{alg3} is the setting for the learning rate. Similar to Algorithm~\ref{alg2}, the learning rate $\eta_{s}$ of Algorithm~\ref{alg3} can also be fixed to a constant. Inspired by existing accelerated stochastic algorithms~\cite{zhu:Katyusha,shang:fsvrg}, the learning rate $\eta_{s}$ in Algorithm~\ref{alg3} can be gradually \emph{increased}, which in principle leads to faster convergence (see Section~\ref{sec55} for details). Different from existing stochastic methods such as Katyusha~\cite{zhu:Katyusha}, the update rule of the learning rate $\eta_{s}$ for \textbf{Algorithm} \ref{alg3} is defined as follows: $\eta_{0}\!=\!c$, where $c\!>\!0$ is an initial learning rate, and for any $s\!\geq\!1$,
\begin{equation}\label{equ23}
\eta_{s}=\eta_{0}/\max\{\alpha,\,2/(s+1)\}
\end{equation}
where $0\!<\!\alpha\!\leq\!1$ is a given constant, e.g., $\alpha\!=\!0.2$.

\subsection{Complexity Analysis}
From Algorithms~\ref{alg2} and \ref{alg3}, we can see that the per-iteration cost of VR-SGD is dominated by the computation of $\nabla\! f_{i^{s}_{k}}\!(x^{s}_{k})$, $\nabla\! f_{i^{s}_{k}}\!(\widetilde{x}^{s-1})$, and $\nabla\!g(x^{s}_{k})$ or the proximal update in~\eqref{equ22}. Thus, the complexity is $O(d)$, as low as that of SVRG~\cite{johnson:svrg} and Prox-SVRG~\cite{xiao:prox-svrg}. In fact, for some ERM problems, we can save the intermediate gradients $\nabla\! f_{i}(\widetilde{x}^{s-1})$ in the computation of $\widetilde{\mu}^{s}$, which generally requires $O(n)$ additional storage. As a result, each epoch only requires $(n+m)$ component gradient evaluations. In addition, for extremely sparse data, we can introduce the lazy update tricks in~\cite{koneeny:mini,carpenter:lazy,langford:online} to our algorithms, and perform the update steps in (\ref{equ21}) and (\ref{equ22}) only for the non-zero dimensions of each example, rather than all dimensions. In other words, the per-iteration complexity of VR-SGD can be improved from $O(d)$ to $O(d')$, where $d'\!\leq\!d$ is the sparsity of feature vectors. Moreover, VR-SGD has a much lower per-iteration complexity than existing accelerated stochastic variance reduction methods such as Katyusha~\cite{zhu:Katyusha}, which have at least two more update rules for additional variables, as shown in~\eqref{equ061}-\eqref{equ063}.

\begin{algorithm}[t]
\caption{VR-SGD for non-strongly convex objectives}
\label{alg3}
\renewcommand{\algorithmicrequire}{\textbf{Input:}}
\renewcommand{\algorithmicensure}{\textbf{Initialize:}}
\renewcommand{\algorithmicoutput}{\textbf{Output:}}
\begin{algorithmic}[1]
\REQUIRE The number of epochs $S$, and the number of iterations $m$ per epoch.\\
\ENSURE $x^{1}_{0}=\widetilde{x}^{0}$, \,$\alpha\!>\!0$\, and\, $\eta_{0}$.\\
\FOR{$s=1,2,\ldots,S$}
\STATE {$\widetilde{\mu}^{s}=\frac{1}{n}\!\sum^{n}_{i=1}\!\nabla\!f_{i}(\widetilde{x}^{s-1})$; \hfill $\%$\:\emph{Compute full gradient}}
\STATE {$\eta_{s}=\eta_{0}/\max\{\alpha, 2/(s\!+\!1)\}$; \hfill $\%$\:\emph{Compute step sizes}}
\FOR{$k=0,1,\ldots,m-1$}
\STATE {Pick $i^{s}_{k}$ uniformly at random from $[n]$;}
\STATE {$\widetilde{\nabla}\! f_{i^{s}_{k}}(x^{s}_{k})=\nabla\! f_{i^{s}_{k}}(x^{s}_{k})-\nabla\! f_{i^{s}_{k}}(\widetilde{x}^{s-1})+\widetilde{\mu}^{s}$;  \hfill $\%$\:\emph{The stochastic gradient estimator}}
\STATE {$x^{s}_{k+1}=x^{s}_{k}\!-\eta_{s}\!\left[\widetilde{\nabla}\!f_{i^{s}_{k}}(x^{s}_{k})+\nabla g(x^{s}_{k})\right]$,    \hfill $\%$\:\emph{Smooth case of $g(\cdot)$}\\
or\; $x^{s}_{k+1}=\textrm{Prox}_{\,\eta_{s},\,g}\!\left(x^{s}_{k}-\eta_{s}\widetilde{\nabla}\!f_{i^{s}_{k}}(x^{s}_{k})\right)$;     \hfill $\%$\:\emph{Non-smooth case of $g(\cdot)$}}
\ENDFOR
\STATE {Option I:\,$\widetilde{x}^{s}\!=\frac{1}{m-\!1}\!\sum^{m-\!1}_{k=1}x^{s}_{k}$; \hfill $\%$\:\emph{Iterate averaging for snapshot $\widetilde{x}$}}
\STATE {Option II:\,$\widetilde{x}^{s}\!=\!\frac{1}{m}\!\sum^{m}_{k=1}\!x^{s}_{k}$; \hfill $\%$\:\emph{Iterate averaging for snapshot $\widetilde{x}$}}
\STATE {$x^{s+1}_{0}\!=x^{s}_{m}$; \hfill $\%$\:\emph{Initiate $x^{s+1}_{0}$ for the next epoch}}
\ENDFOR
\OUTPUT {$\widehat{x}^{S}\!=\widetilde{x}^{S}$ if $F(\widetilde{x}^{S})\leq F(\frac{1}{S}\!\sum^{S}_{s=1}\!\widetilde{x}^{s})$, and $\widehat{x}^{S}\!=\!\frac{1}{S}\!\sum^{S}_{s=1}\!\widetilde{x}^{s}$ otherwise.}
\end{algorithmic}
\end{algorithm}

\subsection{Extensions of VR-SGD}
It has been shown in~\cite{koneeny:mini,nitanda:svrg} that \emph{mini-batching} can effectively decrease the variance of stochastic gradient estimates. In this part, we first extend the proposed VR-SGD method to the mini-batch setting, as well as its convergence results below. Here, we denote by $b$ the mini-batch size and $I^{s}_{k}$ the selected random index set $I_{k}\!\subset\![n]$ for each outer-iteration $s\!\in\![S]$ and inner-iteration $k\!\in\!\{0,1,\ldots,m\!-\!1\}$. The variance reduced stochastic gradient estimator in (\ref{equ051}) becomes
\begin{equation*}
\widetilde{\nabla}\! f_{I^{s}_{k}}(x^{s}_{k})=\frac{1}{b}\!\sum_{i\in I^{s}_{k}}\!\left[\nabla\!f_{i}(x^{s}_{k})\!-\!\nabla\! f_{i}(\widetilde{x}^{s-1})\right]\!+\!{\nabla}\! f(\widetilde{x}^{s-1})
\end{equation*}
where $I^{s}_{k}\!\subset\![n]$ is a mini-batch of size $b$. If some component functions are non-smooth, we can use the proximal operator oracle~\cite{zhu:box} or the Nesterov's smoothing~\cite{nesterov:smooth} and homotopy smoothing~\cite{xu:hs} techniques to smoothen them, and thereby obtain the smoothed approximations of the functions $f_{i}(x)$. In addition, we can directly extend our algorithms to the non-smooth setting as in~\cite{shang:fsvrg}, e.g., Algorithm 3 in~\cite{shang:fsvrg}.

Considering that each component function $f_{i}(x)$ maybe have different degrees of smoothness, picking the random index $i^{s}_{k}$ from a non-uniform distribution is a much better choice than the commonly used uniform random sampling~\cite{zhao:prox-smd,needell:sgd}, as well as without-replacement sampling vs.\ with-replacement sampling~\cite{shamir:sgd}. This can be done using the same techniques in~\cite{xiao:prox-svrg,zhu:Katyusha}, i.e., the sampling probabilities for all $f_{i}(x)$ are proportional to their Lipschitz constants, i.e., $p_{i}\!=\!L_{i}/\sum^{n}_{j=1}\!L_{j}$. Moreover, our VR-SGD method can also be combined with other \emph{accelerated} techniques proposed for SVRG. For instance, the epoch length of VR-SGD can be automatically determined by the techniques in~\cite{zhu:univr,konecny:s2gd} instead of a fixed epoch length. We can reduce the number of gradient calculations in the early iterations as in~\cite{shang:fsvrg,babanezhad:vrsg,zhu:univr}, which leads to faster convergence in general. Moreover, we can also introduce the Nesterov's acceleration technique as in~\cite{hu:sgd,nitanda:svrg,lan:rpdg,frostig:sgd,lin:vrsg} and momentum acceleration trick as in~\cite{zhu:Katyusha,shang:fsvrg} to further improve the performance of VR-SGD.

\section{Convergence Analysis}
\label{sec4}
In this section, we provide the convergence guarantees of VR-SGD for solving both smooth and non-smooth general convex problems. We also extend these results to the mini-batch setting. Moreover, we analyze the convergence properties of VR-SGD for solving both smooth and non-smooth strongly convex objective functions. We first introduce the following lemma which is useful in our analysis.

\begin{lemma}[3-point property, \cite{lan:sgd}]
\label{lemm1}
Let $\hat{z}$ be the optimal solution of the following problem,
\begin{displaymath}
\min_{z\in\mathbb{R}^{d}}\frac{\tau}{2}\|z-z_{0}\|^{2}+r(z)
\end{displaymath}
where $r(z)$ is a convex function (but possibly non-differentiable), and $\tau\!\geq\!0$. Then for any $z\!\in\!\mathbb{R}^{d}$, the following inequality holds
\begin{displaymath}
r(\hat{z})+\frac{\tau}{2}\|\hat{z}-z_{0}\|^{2}\leq r(z)+\frac{\tau}{2}\|z-z_{0}\|^{2}-\frac{\tau}{2}\|z-\hat{z}\|^{2}.
\end{displaymath}
\end{lemma}

\subsection{Convergence Properties for Non-strongly Convex Problems}
In this part, we analyze the convergence properties of VR-SGD for solving the more general non-strongly convex problems. Considering that the two proposed algorithms (i.e., Algorithms~\ref{alg2} and \ref{alg3}) have two different update rules for the smooth and non-smooth problems, we give the convergence guarantees of VR-SGD for the two cases as follows.

\subsubsection{Smooth Objectives}
We first give the convergence guarantee of Problem \eqref{equ01} when the objective function $F(x)$ is \emph{smooth} and non-strongly convex. In order to simplify analysis, we denote $F(x)$ by $f(x)$, that is, $f_{i}(x)\!:=\!f_{i}(x)\!+\!g(x)$ for all $i\!=\!1,2,\ldots,n$.

\begin{lemma}[Variance bound of smooth objectives]
\label{lemm2}
Let $x^{*}$ be the optimal solution of Problem \eqref{equ01}. Suppose Assumption~\ref{assum1} holds, and let $\widetilde{\nabla}\! f_{i^{s}_{k}}(x^{s}_{k})\!:=\!\nabla\! f_{i^{s}_{k}}(x^{s}_{k})-\nabla\! f_{i^{s}_{k}}(\widetilde{x}^{s-1})+\nabla\! f(\widetilde{x}^{s-1})$. Then the following inequality holds
\begin{equation*}
\mathbb{E}\!\left[\|\widetilde{\nabla}\! f_{i^{s}_{k}}(x^{s}_{k})-\nabla\! f(x^{s}_{k})\|^{2}\right]\leq4L[f(x^{s}_{k})-f(x^{*})+f(\widetilde{x}^{s-1})-f(x^{*})].
\end{equation*}
\end{lemma}

The proof of this lemma is included in APPENDIX A. Lemma~\ref{lemm2} provides the upper bound on the expected variance of the variance reduced gradient estimator in \eqref{equ051}, i.e., the SVRG estimator independently introduced in~\cite{johnson:svrg,zhang:svrg}. For Algorithm \ref{alg3} with \emph{Option I} and a fixed learning rate, we give the following key result for our analysis.

\begin{lemma}[Option I and smooth objectives]
\label{lemm4}
Let $\beta\!=\!1/(L\eta)$. If each $f_{i}(\cdot)$ is convex and $L$-smooth, then the following inequality holds for all $s=1,2,\ldots,S$,
\begin{equation*}
\begin{split}
&\left(1-\frac{2}{\beta\!-\!1}\right)\mathbb{E}\!\left[f(\widetilde{x}^{s})-f(x^{*})\right]+\frac{1}{m\!-\!1}\mathbb{E}\!\left[f(x^{s}_{m})-f(x^{*})\right]\\
\leq&\frac{2m}{(\beta\!-\!1)(m\!-\!1)}\mathbb{E}\!\left[f(\widetilde{x}^{s-1})\!-\!f(x^{*})\right]+\frac{2}{(\beta\!-\!1)(m\!-\!1)}\mathbb{E}\!\left[f(x^{s}_{0})\!-\!f(x^{*})\right]+\frac{L\beta }{2(m\!-\!1)}\mathbb{E}\!\left[\|x^{*}\!-\!x^{s}_{0}\|^2-\|x^{*}\!-\!x^{s}_{m}\|^2\right].
\end{split}
\end{equation*}
\end{lemma}

\begin{proof}
In order to simplify notation, the stochastic gradient estimator is defined as: $v^{s}_{k}\!:=\!\nabla f_{i^{s}_{k}}(x^{s}_{k})-\nabla f_{i^{s}_{k}}(\widetilde{x}^{s-\!1})+\nabla f(\widetilde{x}^{s-\!1})$. Since each component function $f_{i}(x)$ is $L$-smooth, which implies that the gradient of the average function $f(x)$ is also $L$-smooth, i.e., for all $x,y\!\in\! \mathbb{R}^{d}$,
\begin{displaymath}
\|\nabla f(x)-\nabla f(y)\|\leq L\|x-y\|,
\end{displaymath}
whose equivalent form is
\begin{displaymath}
f(y)\leq f(x)+\langle\nabla f(x),\;y-x\rangle+\frac{L}{2}\|y-x\|^{2}.
\end{displaymath}
Applying the above smoothness inequality, we have
\begin{equation}\label{equ31}
\begin{split}
f(x^{s}_{k+1})\leq\,& f(x^{s}_{k})+\left\langle\nabla f(x^{s}_{k}),\,x^{s}_{k+1}-x^{s}_{k}\right\rangle+\frac{L}{2}\!\left\|x^{s}_{k+1}-x^{s}_{k}\right\|^{2}\\
=\,& f(x^{s}_{k})+\left\langle\nabla f(x^{s}_{k}),\,x^{s}_{k+1}-x^{s}_{k}\right\rangle+\frac{L\beta }{2}\!\left\|x^{s}_{k+1}-x^{s}_{k}\right\|^{2}-\frac{L(\beta \!-\!1)}{2}\!\left\|x^{s}_{k+1}-x^{s}_{k}\right\|^{2}\\
=\,& f(x^{s}_{k})+\left\langle v^{s}_{k},\,x^{s}_{k+1}-x^{s}_{k}\right\rangle+\frac{L\beta }{2}\|x^{s}_{k+1}-x^{s}_{k}\|^2\\
&+\left\langle\nabla f(x^{s}_{k})-v^{s}_{k},\,x^{s}_{k+1}-x^{s}_{k}\right\rangle-\frac{L(\beta \!-\!1)}{2}\|x^{s}_{k+1}-x^{s}_{k}\|^{2},
\end{split}
\end{equation}
where $\beta\!=\!1/(L\eta)\!>\!3$ is a constant. Using Lemma~\ref{lemm2}, then we get
\begin{equation}\label{equ32}
\begin{split}
&\mathbb{E}\!\left[\left\langle\nabla\! f(x^{s}_{k})-v^{s}_{k},\,x^{s}_{k+1}-x^{s}_{k}\right\rangle-\frac{L(\beta \!-\!1)}{2}\|x^{s}_{k+1}-x^{s}_{k}\|^{2}\right]\\
\leq\,& \mathbb{E}\!\left[\frac{1}{2L(\beta \!-\!1)}\|\nabla\!f(x^{s}_{k})-v^{s}_{k}\|^{2}+\frac{L(\beta \!-\!1)}{2}\|x^{s}_{k+1}\!-\!x^{s}_{k}\|^{2}-\frac{L(\beta \!-\!1)}{2}\|x^{s}_{k+1}\!-\!x^{s}_{k}\|^{2}\right]\\
\leq\,& \frac{2}{\beta \!-\!1}\!\left[f(x^{s}_{k})-f(x^{*})+f(\widetilde{x}^{s-1})-f(x^{*})\right],
\end{split}
\end{equation}
where the first inequality holds due to the Young's inequality (i.e., $y^{T}z\!\leq\!{\|y\|^2}/{(2\gamma)}\!+\!{\gamma\|z\|^2}/{2}$ for all $\gamma\!>\!0$ and $y,z\!\in\! \mathbb{R}^{d}$), and the second inequality follows from Lemma~\ref{lemm2}.

Substituting the inequality in \eqref{equ32} into the inequality in \eqref{equ31}, and taking the expectation with respect to the random choice $i^{s}_{k}$, we have
\begin{equation*}
\begin{split}
&\quad\;\, \mathbb{E}[f(x^{s}_{k+1})]\\
&\leq \mathbb{E}[f(x^{s}_{k})]+\mathbb{E}\!\left[\left\langle v^{s}_{k}, \,x^{s}_{k+\!1}\!-\!x^{s}_{k}\right\rangle+\frac{L\beta }{2}\|x^{s}_{k+\!1}\!-\!x^{s}_{k}\|^2\right]+\frac{2}{\beta \!-\!1}\!\left[f(x^{s}_{k})\!-\!f(x^{*})\!+\!f(\widetilde{x}^{s-\!1})\!-\!f(x^{*})\right]\\
&\leq \mathbb{E}[f(x^{s}_{k})]+\mathbb{E}\!\left[\left\langle v^{s}_{k}, \,x^{*}\!-\!x^{s}_{k}\right\rangle+\frac{L\beta }{2}(\|x^{*}\!-\!x^{s}_{k}\|^2-\|x^{*}\!-\!x^{s}_{k+\!1}\|^2)\right]+\frac{2}{\beta \!-\!1}\!\left[f(x^{s}_{k})\!-\!f(x^{*})\!+\!f(\widetilde{x}^{s-\!1})\!-\!f(x^{*})\right]\\
&\leq \mathbb{E}[f(x^{s}_{k})]+\mathbb{E}[\left\langle \nabla\! f(x^{s}_{k}), \,x^{*}\!-\!x^{s}_{k}\right\rangle]+\mathbb{E}\!\left[\frac{L\beta }{2}(\|x^{*}\!-\!x^{s}_{k}\|^2-\|x^{*}\!-\!x^{s}_{k+\!1}\|^2)\right]+\frac{2}{\beta \!-\!1}\!\left[f(x^{s}_{k})\!-\!f(x^{*})\!+\!f(\widetilde{x}^{s-\!1})\!-\!f(x^{*})\right]\\
&\leq f(x^{*})+\frac{L\beta }{2}\mathbb{E}\!\left[\left(\|x^{*}\!-x^{s}_{k}\|^2-\|x^{*}\!-x^{s}_{k+1}\|^2\right)\right]+\frac{2}{\beta \!-\!1}\!\left[f(x^{s}_{k})-f(x^{*})+f(\widetilde{x}^{s-1})-f(x^{*})\right].
\end{split}
\end{equation*}
Here, the first inequality holds due to the inequality in \eqref{equ31} and the inequality in \eqref{equ32}; the second inequality follows from Lemma~\ref{lemm1} with $\hat{z}\!=\!x^{s}_{k+1}$, $z\!=\!x^{*}$, $z_{0}\!=\!x^{s}_{k}$, $\tau\!=\!L\beta \!=\!1/\eta$, and $r(z)\!:=\!\langle v^{s}_{k}, \,z\!-\!x^{s}_{k}\rangle$; the third inequality holds due to the fact that $\mathbb{E}[v^{s}_{k}]\!=\!\nabla\! f(x^{s}_{k})$; and the last inequality follows from the convexity of the smooth function $f(\cdot)$, i.e., $f(x^{s}_{k})\!+\!\langle\nabla\! f(x^{s}_{k}), \,x^{*}\!-\!x^{s}_{k}\rangle\!\leq\!f(x^{*})$. The above inequality can be rewritten as follows:
\begin{equation*}
\begin{split}
\mathbb{E}[f(x^{s}_{k+1})]-f(x^{*})
\leq \frac{2}{\beta \!-\!1}\!\left[f(x^{s}_{k})-f(x^{*})+f(\widetilde{x}^{s-1})-f(x^{*})\right]+ \frac{L\beta }{2}\mathbb{E}\!\left[\|x^{*}-x^{s}_{k}\|^2-\|x^{*}-x^{s}_{k+1}\|^2\right].
\end{split}
\end{equation*}
Summing the above inequality over $k=0,1,\ldots,m\!-\!1$, then
\begin{equation}\label{equ33}
\begin{split}
\sum^{m}_{k=1}\mathbb{E}\!\left[f(x^{s}_{k})-f(x^{*})\right]
\leq \sum^{m}_{k=1}\!\left\{\frac{2}{\beta \!-\!1}\!\left[f(x^{s}_{k-1})-f(x^{*})+f(\widetilde{x}^{s-1})-f(x^{*})\right]+\frac{L\beta }{2}\mathbb{E}\!\left[\|x^{*}\!-x^{s}_{k-1}\|^2-\|x^{*}\!-x^{s}_{k}\|^2\right]\right\}.
\end{split}
\end{equation}

Due to the setting of $\widetilde{x}^{s}\!=\!\frac{1}{m-1}\sum^{m-1}_{k=1}x^{s}_{k}$ in Option I, and the convexity of $f(\cdot)$, then we have
\begin{equation}\label{equ34}
f(\widetilde{x}^{s})\leq \frac{1}{m-1}\sum^{m-1}_{k=1}f(x^{s}_{k}).
\end{equation}
The left and right hand sides of the inequality in \eqref{equ33} can be rewritten as follows:
\begin{equation*}
\begin{split}
&\qquad\quad\qquad\sum^{m}_{k=1}\mathbb{E}\!\left[f(x^{s}_{k})-f(x^{*})\right]=\sum^{m-1}_{k=1}\mathbb{E}\!\left[f(x^{s}_{k})-f(x^{*})\right]+\mathbb{E}\!\left[f(x^{s}_{m})-f(x^{*})\right],\\
&\sum^{m}_{k=1}\left\{\frac{2}{\beta \!-\!1}\!\left[f(x^{s}_{k-1})-f(x^{*})+f(\widetilde{x}^{s-1})-f(x^{*})\right]+\frac{L\beta }{2}\mathbb{E}\!\left[\|x^{*}-x^{s}_{k-1}\|^2-\|x^{*}-x^{s}_{k}\|^2\right]\right\}\\
=& \frac{2}{\beta\!-\!1}\sum^{m-1}_{k=1}\!\left[f(x^{s}_{k})\!-\!f(x^{*})\right]+\frac{2}{\beta\!-\!1}\!\left\{f(x^{s}_{0})\!-\!f(x^{*})+m[f(\widetilde{x}^{s-1})\!-\!f(x^{*})]\right\}+\frac{L\beta }{2}\mathbb{E}\!\left[\|x^{*}\!-\!x^{s}_{0}\|^2\!-\!\|x^{*}\!-\!x^{s}_{m}\|^2\right]\!.
\end{split}
\end{equation*}
Subtracting $\frac{2}{\beta-1}\sum^{m-1}_{k=1}\!\left[f(x^{s}_{k})\!-\!f(x^{*})\right]$ from both sides of the inequality in \eqref{equ33}, then we obtain
\begin{equation*}
\begin{split}
&\left(1-\frac{2}{\beta\!-\!1}\right)\sum^{m-1}_{k=1}\mathbb{E}\!\left[f(x^{s}_{k})-f(x^{*})\right]+\mathbb{E}\!\left[f(x^{s}_{m})-f(x^{*})\right]\\
\leq& \frac{2}{\beta\!-\!1}\mathbb{E}\!\left\{f(x^{s}_{0})\!-\!f(x^{*})+m[f(\widetilde{x}^{s-1})\!-\!f(x^{*})]\right\}+\frac{L\beta }{2}\mathbb{E}\!\left[\|x^{*}\!-\!x^{s}_{0}\|^2-\|x^{*}\!-\!x^{s}_{m}\|^2\right].
\end{split}
\end{equation*}
Applying the inequality in \eqref{equ34}, we have
\begin{equation*}
\begin{split}
&\left(1-\frac{2}{\beta\!-\!1}\right)(m\!-\!1)\mathbb{E}\!\left[f(\widetilde{x}^{s})-f(x^{*})\right]+\mathbb{E}\!\left[f(x^{s}_{m})-f(x^{*})\right]\\
\leq&\left(1-\frac{2}{\beta\!-\!1}\right)\sum^{m-1}_{k=1}\mathbb{E}\!\left[f(x^{s}_{k})-f(x^{*})\right]+\mathbb{E}\!\left[f(x^{s}_{m})-f(x^{*})\right]\\
\leq& \frac{2}{\beta\!-\!1}\mathbb{E}\!\left\{f(x^{s}_{0})\!-\!f(x^{*})+m[f(\widetilde{x}^{s-1})\!-\!f(x^{*})]\right\}+\frac{L\beta }{2}\mathbb{E}\!\left[\|x^{*}\!-\!x^{s}_{0}\|^2-\|x^{*}\!-\!x^{s}_{m}\|^2\right].
\end{split}
\end{equation*}
Dividing both sides of the above inequality by $(m\!-\!1)$, we arrive at
\begin{equation*}
\begin{split}
&\left(1-\frac{2}{\beta\!-\!1}\right)\mathbb{E}\!\left[f(\widetilde{x}^{s})-f(x^{*})\right]+\frac{1}{m\!-\!1}\mathbb{E}\!\left[f(x^{s}_{m})-f(x^{*})\right]\\
\leq& \frac{2}{(\beta\!-\!1)(m\!-\!1)}\mathbb{E}\!\left[f(x^{s}_{0})\!-\!f(x^{*})\right]\!+\!\frac{2m}{(\beta\!-\!1)(m\!-\!1)}\mathbb{E}\!\left[f(\widetilde{x}^{s-\!1})\!-\!f(x^{*})\right]\!+\!\frac{L\beta }{2(m\!-\!1)}\mathbb{E}\!\left[\|x^{*}\!-\!x^{s}_{0}\|^2\!-\!\|x^{*}\!-\!x^{s}_{m}\|^2\right].
\end{split}
\end{equation*}
This completes the proof.
\end{proof}
\vspace{3mm}

The first main result is the following theorem, which provides the convergence guarantee of VR-SGD with Option I for solving smooth and general convex minimization problems.

\begin{theorem}[Option I and smooth objectives]
\label{the1}
Suppose Assumption~\ref{assum1} holds. Then the following inequality holds
\begin{equation*}
\begin{split}
\mathbb{E}\!\left[f(\widehat{x}^{S})\right]-f(x^{*})
\leq\frac{2(m+1)}{[(\beta\!-\!1)(m\!-\!1)\!-\!4m\!+\!2]S}[f(\widetilde{x}^{0})-f(x^{*})]+\frac{\beta(\beta-1)L}{2[(\beta\!-\!1)(m\!-\!1)\!-\!4m\!+\!2]S}\|\widetilde{x}^{0}\!-x^{*}\|^{2},
\end{split}
\end{equation*}
where $\widehat{x}^{S}=\widetilde{x}^{S}$ if $f(\widetilde{x}^{S})\leq f(\overline{x}^{S})$, and $\overline{x}^{S}=\frac{1}{S}\!\sum^{S}_{s=1}\widetilde{x}^{s}$. Otherwise,  $\widehat{x}^{S}=\overline{x}^{S}$.
\end{theorem}

\vspace{3mm}

\begin{proof}
Since $2/(\beta\!-\!1)<1$, it is easy to verify that
\begin{equation}
\frac{2}{(\beta\!-\!1)(m\!-\!1)}\left\{\mathbb{E}[f(x^{s}_{m})]-f(x^{*})\right\}\leq\frac{1}{m\!-\!1}\left\{\mathbb{E}[f(x^{s}_{m})]-f(x^{*})\right\}.
\end{equation}
Applying the above inequality and Lemma~\ref{lemm4}, we have
\begin{equation*}
\begin{split}
&\left(1-\frac{2}{\beta\!-\!1}\right)\mathbb{E}\!\left[f(\widetilde{x}^{s})-f(x^{*})\right]+\frac{2}{(\beta\!-\!1)(m\!-\!1)}\mathbb{E}[f(x^{s}_{m})-f(x^{*})]\\
\leq&\left(1-\frac{2}{\beta\!-\!1}\right)\mathbb{E}\!\left[f(\widetilde{x}^{s})-f(x^{*})\right]+\frac{1}{m\!-\!1}\mathbb{E}[f(x^{s}_{m})-f(x^{*})]\\
\leq&\frac{2}{(\beta\!-\!1)(m\!-\!1)}\mathbb{E}\!\left[f(x^{s}_{0})-f(x^{*})\right]+\frac{2m}{(\beta\!-\!1)(m\!-\!1)}\mathbb{E}\!\left[f(\widetilde{x}^{s-1})-f(x^{*})\right]+\frac{L\beta}{2(m\!-\!1)}\mathbb{E}\!\left[\|x^{*}\!-x^{s}_{0}\|^{2}-\|x^{*}\!-x^{s}_{m}\|^{2}\right].
\end{split}
\end{equation*}
Summing the above inequality over $s\!=\!1,2,\ldots,S$, taking expectation with respect to the history of random variables $i^{s}_{k}$, and using the setting of $x^{s+1}_{0}\!=\!x^{s}_{m}$, we obtain
\begin{equation*}
\begin{split}
&\sum^{S}_{s=1}\left(1-\frac{2}{\beta\!-\!1}\right)\mathbb{E}\!\left[f(\widetilde{x}^{s})-f(x^{*})\right]\\
\leq&\:\sum^{S}_{s=1}\left\{\frac{2}{(\beta\!-\!1)(m\!-\!1)}\mathbb{E}\!\left[f(x^{s}_{0})-f(x^{*})-(f(x^{s}_{m})-f(x^{*}))\right]+\frac{2m}{(\beta\!-\!1)(m\!-\!1)}\mathbb{E}[f(\widetilde{x}^{s-1})-f(x^{*})]\right\}\\
&+\frac{L\beta}{2(m\!-\!1)}\sum^{S}_{s=1}\mathbb{E}\!\left[\|x^{*}-x^{s}_{0}\|^{2}-\|x^{*}-x^{s}_{m}\|^{2}\right].
\end{split}
\end{equation*}

Subtracting $\frac{2m}{(\beta-1)(m-1)}\!\sum^{S-1}_{s=1}\left[f(\widetilde{x}^{s})\!-\!f(x^{*})\right]$ from both sides of the above inequality, we have
\begin{equation*}
\begin{split}
&\:\frac{2m}{(\beta\!-\!1)(m\!-\!1)}{\mathbb{E}[f(\widetilde{x}^{S})-f(x^{*})}]+\sum^{S}_{s=1}\left(1-\frac{4}{\beta\!-\!1}-\frac{2}{(\beta\!-\!1)(m\!-\!1)}\right)\mathbb{E}\!\left[f(\widetilde{x}^{s})-f(x^{*})\right]\\
\leq&\:\frac{2}{(\beta\!-\!1)(m\!-\!1)}\mathbb{E}\!\left[f(x^{1}_{0})-f(x^{*})-(f(x^{S}_{m})-f(x^{*}))\right]+\frac{2m}{(\beta\!-\!1)(m\!-\!1)}\mathbb{E}[f(\widetilde{x}^{0})-f(x^{*})]\\
&+\frac{L\beta}{2(m-1)}\mathbb{E}\!\left[\|x^{*}-x^{1}_{0}\|^{2}-\|x^{*}-x^{S}_{m}\|^{2}\right].
\end{split}
\end{equation*}
Dividing both sides of the above inequality by $S$, and using the setting of $\widetilde{x}^{0}\!=\!x^{1}_{0}$, we arrive at
\begin{equation*}
\begin{split}
&\:\frac{1}{S}\left(1-\frac{4}{\beta\!-\!1}-\frac{2}{(\beta\!-\!1)(m\!-\!1)}\right)\sum^{S}_{s=1}\mathbb{E}\!\left[f(\widetilde{x}^{s})-f(x^{*})\right]\\
\leq&\,\frac{2m}{(\beta\!-\!1)(m\!-\!1)S}{\mathbb{E}[f(\widetilde{x}^{S})-f(x^{*})}]+\frac{1}{S}\!\left(1-\frac{4}{\beta\!-\!1}-\frac{2}{(\beta\!-\!1)(m\!-\!1)}\right)\sum^{S}_{s=1}\mathbb{E}\!\left[f(\widetilde{x}^{s})-f(x^{*})\right]\\
\leq&\,\frac{2}{(\beta\!-\!1)(m\!-\!1)S}[f(x^{1}_{0})-f(x^{*})]+\frac{2m}{(\beta\!-\!1)(m\!-\!1)S}[f(\widetilde{x}^{0})-f(x^{*})]+\frac{L\beta}{2(m\!-\!1)S}\|x^{*}-x^{1}_{0}\|^{2}\\
=&\,\frac{2(m+1)}{(\beta\!-\!1)(m\!-\!1)S}[f(\widetilde{x}^{0})-f(x^{*})]+\frac{L\beta}{2(m\!-\!1)S}\|\widetilde{x}^{0}-x^{*}\|^{2},
\end{split}
\end{equation*}
where the first inequality holds due to the fact that $f(\widetilde{x}^{S})\!-\!f(x^{*})\!\geq\!0$; the second inequality holds due to the facts that $f(x^{S}_{m})\!-\!f(x^{*})\!\geq\!0$ and $\|x^{*}\!-\!x^{S}_{m}\|^{2}\!\geq\!0$; and the last equality follows from the setting of $\widetilde{x}^{0}\!=\!x^{1}_{0}$.

Due to the definition of $\overline{x}^{S}$ and the convexity of $f(\cdot)$, we have $f(\overline{x}^{S})\!\leq\! \frac{1}{S}\sum^{S}_{s=1}f(\widetilde{x}^{s})$, and therefore the above inequality becomes:
\begin{equation}
\begin{split}
&\:\left(1-\frac{4}{\beta\!-\!1}-\frac{2}{(\beta\!-\!1)(m\!-\!1)}\right)\mathbb{E}\!\left[f(\overline{x}^{S})-f(x^{*})\right]\\
\leq&\,\frac{2(m+1)}{(\beta\!-\!1)(m\!-\!1)S}[f(\widetilde{x}^{0})-f(x^{*})]+\frac{L\beta}{2(m\!-\!1)S}\|\widetilde{x}^{0}-x^{*}\|^{2}.
\end{split}
\end{equation}
Dividing both sides of the above inequality by $c_1\!=\!1\!-\!\frac{4}{\beta-1}\!-\!\frac{2}{(\beta-1)(m-1)}\!>\!0$, we have
\begin{equation*}
\mathbb{E}\!\left[f(\overline{x}^{S})\right]-f(x^{*})\leq\frac{2(m+1)}{c_1(\beta\!-\!1)(m\!-\!1)S}[f(\widetilde{x}^{0})-f(x^{*})]+\frac{L\beta}{2c_1(m\!-\!1)S}\|\widetilde{x}^{0}-x^{*}\|^{2}.
\end{equation*}
Due to the setting for the output of Algorithm~\ref{alg3}, $\widehat{x}^{S}=\widetilde{x}^{S}$ if $f(\widetilde{x}^{S})\leq f(\overline{x}^{S})$. Then
\begin{equation*}
\begin{split}
\mathbb{E}\!\left[f(\widehat{x}^{S})\right]-f(x^{*})\leq\mathbb{E}\!\left[f(\overline{x}^{S})\right]-f(x^{*})\leq\frac{2(m+1)}{c_1(\beta\!-\!1)(m\!-\!1)S}[f(\widetilde{x}^{0})-f(x^{*})]+\frac{L\beta}{2c_1(m\!-\!1)S}\|\widetilde{x}^{0}-x^{*}\|^{2}.
\end{split}
\end{equation*}
Alternatively, when $f(\widetilde{x}^{S})\geq f(\overline{x}^{S})$, let $\widehat{x}^{S}=\overline{x}^{S}$, and the above inequality still holds. This completes the proof.
\end{proof}
\vspace{3mm}

From Lemma~\ref{lemm4}, Theorem~\ref{the1}, and their proofs, one can see that our convergence analysis is very different from those of existing stochastic methods, such as SVRG~\cite{johnson:svrg} and Prox-SVRG~\cite{xiao:prox-svrg}. For Algorithm \ref{alg3} with \emph{Option II} and a fixed learning rate, we give the following key lemma for our convergence analysis.

\begin{lemma}[Option II and smooth objectives]
\label{lemm5}
If each $f_{i}(\cdot)$ is convex and $L$-smooth, then the following inequality holds for all $s=1,2,\ldots,S$,
\begin{equation*}
\begin{split}
&\left(1-\frac{2}{\beta\!-\!1}\right)\mathbb{E}\!\left[f(\widetilde{x}^{s})-f(x^{*})\right]+\frac{2}{(\beta\!-\!1)m}\mathbb{E}[f(x^{s}_{m})-f(x^{*})]\\
\leq&\: \frac{2}{(\beta\!-\!1)}\mathbb{E}\!\left[f(\widetilde{x}^{s-1})-f(x^{*})\right]+\frac{2}{(\beta\!-\!1)m}\mathbb{E}\!\left[f(x^{s}_{0})-f(x^{*})\right]+\frac{L\beta}{2m}\mathbb{E}\!\left[\|x^{*}-x^{s}_{0}\|^2-\|x^{*}-x^{s}_{m}\|^2\right].
\end{split}
\end{equation*}
\end{lemma}

This lemma is a slight generalization of Lemma~\ref{lemm4}, and we give the proof in APPENDIX B for completeness.

\begin{theorem}[Option II and smooth objectives]
\label{the2}
If each $f_{i}(\cdot)$ is convex and $L$-smooth, then the following inequality holds
\begin{equation*}
\mathbb{E}\!\left[f(\widehat{x}^{S})\right]-f(x^{*})\leq\frac{2(m\!+\!1)}{mS(\beta\!-\!5)}[f(\widetilde{x}^{0})-f(x^{*})]+\frac{L\beta(\beta\!-\!1)}{2mS(\beta\!-\!5)}\|\widetilde{x}^{0}-x^{*}\|^{2},
\end{equation*}
where $\widehat{x}^{S}=\widetilde{x}^{S}$ if $f(\widetilde{x}^{S})\leq f(\overline{x}^{S})$, and $\overline{x}^{S}=\frac{1}{S}\!\sum^{S}_{s=1}\widetilde{x}^{s}$. Otherwise,  $\widehat{x}^{S}=\overline{x}^{S}$.
\end{theorem}

The proof of this theorem can be found in APPENDIX C. Clearly, Theorems~\ref{the1} and \ref{the2} show that VR-SGD with Option I or Option II attains a sub-linear convergence rate for smooth general convex objective functions. This means that VR-SGD is guaranteed to have a similar convergence rate with other variance reduced stochastic methods such as SAGA~\cite{defazio:saga}, and a slower theoretical rate than accelerated methods such as Katyusha~\cite{zhu:Katyusha}. Nevertheless, VR-SGD usually converges much faster than the best known stochastic method, Katyusha~\cite{zhu:Katyusha} in practice (see Section~\ref{sec53} for details).

\subsubsection{Non-Smooth Objectives}
Next, we provide the convergence guarantee of Problem \eqref{equ01} when the objective function $F(x)$ is \emph{non-smooth} (i.e., the regularizer $g(x)$ is non-smooth) and non-strongly convex. We first give the following lemma.

\begin{lemma}[Variance bound of non-smooth objectives]
\label{lemm5}
Let $x^{*}$ be the optimal solution of Problem \eqref{equ01}. If each $f_{i}(\cdot)$ is convex and $L$-smooth, and $\widetilde{\nabla}\! f_{i^{s}_{k}}(x^{s}_{k})\!:=\!\nabla\! f_{i^{s}_{k}}(x^{s}_{k})-\nabla\! f_{i^{s}_{k}}(\widetilde{x}^{s-1})+\nabla\! f(\widetilde{x}^{s-1})$, then the following inequality holds
\begin{equation*}
\mathbb{E}\!\left[\|\widetilde{\nabla}\! f_{i^{s}_{k}}(x^{s}_{k})-\nabla\! f(x^{s}_{k})\|^{2}\right]\leq4L[F(x^{s}_{k})-F(x^{*})+F(\widetilde{x}^{s-1})-F(x^{*})].
\end{equation*}
\end{lemma}

This lemma can be viewed as a generalization of Lemma~\ref{lemm2}, and is essentially identical to Corollary 3.5 in~\cite{xiao:prox-svrg} and Lemma 1 in~\cite{shang:vrsgd}, and hence its proof is omitted. For Algorithm \ref{alg3} with \emph{Option II} and a fixed learning rate, we give the following results.

\begin{lemma}[Option II and non-smooth objectives]
\label{lemm6}
If each $f_{i}(\cdot)$ is convex and $L$-smooth, then the following inequality holds for all $s=1,2,\ldots,S$,
\begin{equation}\label{equ39}
\begin{split}
&\left(1-\frac{2}{\beta\!-\!1}\right)\mathbb{E}\!\left[F(\widetilde{x}^{s})-F(x^{*})\right]+\frac{2}{(\beta\!-\!1)m}\mathbb{E}\!\left[F(x^{s}_{m})\!-\!F(x^{*})\right]\\
\leq&\frac{2}{(\beta\!-\!1)m}\mathbb{E}\!\left[F(x^{s}_{0})-F(x^{*})\right]+\frac{2}{\beta\!-\!1}\mathbb{E}\!\left[F(\widetilde{x}^{s-1})\!-\!F(x^{*})\right]+ \frac{L\beta}{2m}\mathbb{E}\!\left[\|x^{*}\!-\!x^{s}_{0}\|^2\!-\!\|x^{*}\!-\!x^{s}_{m}\|^2\right].
\end{split}
\end{equation}
\end{lemma}

The proof of this lemma can be found in APPENDIX D. Using the above lemma, we give the following convergence result for VR-SGD.

\begin{theorem}[Option II and non-smooth objectives]
\label{the3}
Suppose Assumption~\ref{assum1} holds. Then the following inequality holds
\begin{equation*}
\begin{split}
\mathbb{E}\!\left[F(\widehat{x}^{S})\right]-F(x^{*})\leq\,\frac{2(m\!+\!1)}{(\beta\!-\!5)mS}[F(\widetilde{x}^{0})-F(x^{*})]+\frac{\beta(\beta\!-\!1)L}{2(\beta\!-\!5)mS}\|\widetilde{x}^{0}-x^{*}\|^{2}.
\end{split}
\end{equation*}
\end{theorem}

The proof of this theorem is included in APPENDIX E. Similarly, for Algorithm \ref{alg3} with \emph{Option I} and a fixed learning rate, we give the following results.

\begin{corollary}[Option I and non-smooth objectives]
\label{lemm7}
If each $f_{i}(\cdot)$ is convex and $L$-smooth, then the following inequality holds for all $s=1,2,\ldots,S$,
\begin{equation*}
\begin{split}
&\left(1-\frac{2}{\beta\!-\!1}\right)\mathbb{E}\!\left[F(\widetilde{x}^{s})-F(x^{*})\right]+\frac{1}{m\!-\!1}\mathbb{E}\!\left[F(x^{s}_{m})-F(x^{*})\right]\\
\leq&\frac{2m}{(\beta\!-\!1)(m\!-\!1)}\mathbb{E}\!\left[F(\widetilde{x}^{s-1})\!-\!F(x^{*})\right]+\frac{2}{(\beta\!-\!1)(m\!-\!1)}\mathbb{E}\!\left[F(x^{s}_{0})\!-\!F(x^{*})\right]+\frac{L\beta }{2(m\!-\!1)}\mathbb{E}\!\left[\|x^{*}\!-\!x^{s}_{0}\|^2-\|x^{*}\!-\!x^{s}_{m}\|^2\right].
\end{split}
\end{equation*}
\end{corollary}

\begin{corollary}[Option I and non-smooth objectives]
\label{the4}
Suppose Assumption~\ref{assum1} holds. Then the following inequality holds
\begin{equation*}
\begin{split}
&\:\mathbb{E}\!\left[F(\widehat{x}^{S})\right]-F(x^{*})\\
\leq&\:\frac{2(m+1)}{[(\beta\!-\!1)(m\!-\!1)\!-\!4m\!+\!2]S}[F(\widetilde{x}^{0})-F(x^{*})]+\frac{\beta(\beta-1)L}{2[(\beta\!-\!1)(m\!-\!1)\!-\!4m\!+\!2]S}\|\widetilde{x}^{0}-x^{*}\|^{2},
\end{split}
\end{equation*}
where $\widehat{x}^{S}=\widetilde{x}^{S}$ if $F(\widetilde{x}^{S})\leq F(\overline{x}^{S})$, and $\overline{x}^{S}=\frac{1}{S}\!\sum^{S}_{s=1}\widetilde{x}^{s}$. Otherwise,  $\widehat{x}^{S}=\overline{x}^{S}$.
\end{corollary}

Corollaries~\ref{lemm7} and~\ref{the4} can be viewed as the generalizations of Lemma~\ref{lemm4} and Theorem~\ref{the1}, respectively, and hence their proofs are omitted. Obviously, Theorems~\ref{the3} and Corollary~\ref{the4} show that VR-SGD with Option I or Option II attains a sub-linear convergence rate for non-smooth general convex objective functions.

\subsubsection{Mini-Batch Settings}

The upper bound on the variance of the stochastic gradient estimator $\widetilde{\nabla}\!f_{i^{s}_{k}}(x^{s}_{k})$ in Lemma~\ref{lemm5} is extended to the \emph{mini-batch setting} as follows.

\begin{lemma}[Variance bound of mini-batch]
\label{lemm8}
If each $f_{i}(\cdot)$ is convex and $L$-smooth, and $\widetilde{\nabla}\! f_{I^{s}_{k}}(x^{s}_{k})\!:=\!\frac{1}{b}\!\sum_{i\in{I}^{s}_{k}}\!\left[\nabla\! f_{i}(x^{s}_{k})-\nabla\! f_{i}(\widetilde{x}^{s-1})\right]+\nabla\! f(\widetilde{x}^{s-1})$, then the following inequality holds
\begin{equation*}
\mathbb{E}\!\left[\|\widetilde{\nabla}\! f_{I^{s}_{k}}(x^{s}_{k})-\nabla\! f(x^{s}_{k})\|^{2}\right]\leq4L\delta(b)[F(x^{s}_{k})-F(x^{*})+F(\widetilde{x}^{s-1})-F(x^{*})],
\end{equation*}
where $\delta(b)\!=\!(n\!-\!b)/[(n\!-\!1)b]$.
\end{lemma}

It is not hard to verify that $0\leq\delta(b)\leq1$. This lemma is essentially identical to Theorem 4 in~\cite{koneeny:mini}, and hence its proof is omitted. Based on the variance upper bound in Lemma~\ref{lemm8}, we further analyze the convergence properties of VR-SGD for the mini-batch setting. Lemma~\ref{lemm6} is first extended to the mini-batch setting as follows.

\begin{lemma}[Mini-batch]
\label{lemm9}
Using the same notation as in Lemma~\ref{lemm8}, we have
\begin{equation*}
\begin{split}
&\:\frac{2\delta(b)}{(\beta\!-\!1)m}\left\{\mathbb{E}[F(x^{s}_{m})]-F(x^{*})\right\}+\left(1-\frac{2\delta(b)}{\beta\!-\!1}\right)\left\{\mathbb{E}[F(\widetilde{x}^{s})]-F(x^{*})\right\}\\
\leq&\: \frac{2\delta(b)}{(\beta\!-\!1)}\left[F(\widetilde{x}^{s-1})-F(x^{*})\right]+\frac{2\delta(b)}{(\beta\!-\!1)m}\left[F(x^{s}_{0})-F(x^{*})\right]+\frac{L\beta}{2m}\mathbb{E}\!\left[\|x^{*}-x^{s}_{0}\|^2-\|x^{*}-x^{s}_{m}\|^2\right].
\end{split}
\end{equation*}
\end{lemma}

\begin{proof}
In order to simplify notation, the stochastic gradient estimator of mini-batch is defined as:
\begin{equation*}
v^{s}_{k}:=\frac{1}{b}\!\sum_{i\in{I}^{s}_{k}}\!\left[\nabla\! f_{i}(x^{s}_{k})-\nabla\! f_{i}(\widetilde{x}^{s-1})\right]+\nabla\! f(\widetilde{x}^{s-1}).
\end{equation*}
\begin{equation}\label{equ61}
\begin{split}
F(x^{s}_{k+1})\leq\,&\, g(x^{s}_{k+1})+f(x^{s}_{k})+\left\langle\nabla f(x^{s}_{k}),\,x^{s}_{k+1}-x^{s}_{k}\right\rangle+\frac{L\beta}{2}\!\left\|x^{s}_{k+1}-x^{s}_{k}\right\|^{2}-\frac{L(\beta\!-\!1)}{2}\!\left\|x^{s}_{k+1}-x^{s}_{k}\right\|^{2}\\
=&\, g(x^{s}_{k+1})+f(x^{s}_{k})+\left\langle v^{s}_{k},\,x^{s}_{k+1}-x^{s}_{k}\right\rangle+\frac{L\beta}{2}\|x^{s}_{k+1}-x^{s}_{k}\|^2\\
&\,+\left\langle\nabla f(x^{s}_{k})-v^{s}_{k},\,x^{s}_{k+1}-x^{s}_{k}\right\rangle-\frac{L(\beta\!-\!1)}{2}\|x^{s}_{k+1}-x^{s}_{k}\|^{2}.
\end{split}
\end{equation}
Using Lemma~\ref{lemm8}, then we obtain
\begin{equation}\label{equ62}
\begin{split}
&\mathbb{E}\!\left[\left\langle\nabla\! f(x^{s}_{k})-v^{s}_{k},\,x^{s}_{k+1}-x^{s}_{k}\right\rangle-\frac{L(\beta\!-\!1)}{2}\|x^{s}_{k+1}-x^{s}_{k}\|^{2}\right]\\
\leq&\, \mathbb{E}\!\left[\frac{1}{2L(\beta\!-\!1)}\|\nabla\!f(x^{s}_{k})-v^{s}_{k}\|^{2}+\frac{L(\beta\!-\!1)}{2}\|x^{s}_{k+1}\!-\!x^{s}_{k}\|^{2}-\frac{L(\beta\!-\!1)}{2}\|x^{s}_{k+1}\!-\!x^{s}_{k}\|^{2}\right]\\
\leq&\, \frac{2\delta(b)}{\beta\!-\!1}\!\left[F(x^{s}_{k})-F(x^{*})+F(\widetilde{x}^{s-1})-F(x^{*})\right],
\end{split}
\end{equation}
where the first inequality holds due to the Young's inequality, and the second inequality follows from Lemma~\ref{lemm8}. Substituting the inequality \eqref{equ62} into the inequality \eqref{equ61}, and taking the expectation over the random mini-batch set $I^{s}_{k}$, we have
\begin{equation*}
\begin{split}
&\,\mathbb{E}[F(x^{s}_{k+1})]\\
\leq&\,\mathbb{E}[g(x^{s}_{k+\!1})]+\mathbb{E}[f(x^{s}_{k})]+\mathbb{E}\!\left[\left\langle v^{s}_{k}, \,x^{s}_{k+\!1}\!-\!x^{s}_{k}\right\rangle+\!\frac{L\beta}{2}\|x^{s}_{k+\!1}\!-\!x^{s}_{k}\|^2\right]+\frac{2\delta(b)}{\beta\!-\!1}\!\left[F(x^{s}_{k})\!-\!F(x^{*})\!+\!F(\widetilde{x}^{s-\!1})\!-\!F(x^{*})\right]\\
\leq&\, g(x^{*})+\mathbb{E}[f(x^{s}_{k})]+\mathbb{E}\!\left[\left\langle v^{s}_{k}, \,x^{*}\!-\!x^{s}_{k}\right\rangle+\frac{L\beta}{2}(\|x^{*}\!-\!x^{s}_{k}\|^2\!-\!\|x^{*}\!-\!x^{s}_{k+1}\|^2)\right]+\frac{2\delta(b)}{\beta\!-\!1}\!\left[F(x^{s}_{k})\!-\!F(x^{*})\!+\!F(\widetilde{x}^{s-\!1})\!-\!F(x^{*})\right]\\
\leq&\, g(x^{*})+ f(x^{*})+\mathbb{E}\!\left[\frac{L\beta}{2}(\|x^{*}-x^{s}_{k}\|^2-\|x^{*}-x^{s}_{k+1}\|^2)\right]+\frac{2\delta(b)}{\beta\!-\!1}\!\left[F(x^{s}_{k})\!-\!F(x^{*})\!+\!F(\widetilde{x}^{s-\!1})\!-\!F(x^{*})\right]\\
=&\, F(x^{*})+\frac{L\beta}{2}\mathbb{E}\!\left[\left(\|x^{*}-x^{s}_{k}\|^2-\|x^{*}-x^{s}_{k+1}\|^2\right)\right]+\frac{2\delta(b)}{\beta-1}\left[F(x^{s}_{k})\!-\!F(x^{*})\!+\!F(\widetilde{x}^{s-\!1})\!-\!F(x^{*})\right],\\
\end{split}
\end{equation*}
where the second inequality holds from Lemma~\ref{lemm1}. Then the above inequality is rewritten as follows:
\begin{equation}\label{equ63}
\begin{split}
\mathbb{E}\!\left[F(x^{s}_{k+1})\right]-F(x^{*})\leq \frac{2\delta(b)}{\beta\!-\!1}\!\left[F(x^{s}_{k})-F(x^{*})+F(\widetilde{x}^{s-1})-F(x^{*})\right]+ \frac{L\beta}{2}\mathbb{E}\!\left[\|x^{*}-x^{s}_{k}\|^2-\|x^{*}-x^{s}_{k+1}\|^2\right].
\end{split}
\end{equation}
Summing the above inequality over $k=0,1,\cdots,(m-1)$, then
\begin{equation*}
\begin{split}
&\,\sum^{m-1}_{k=0}\left\{\mathbb{E}[F(x^{s}_{k+1})]-F(x^{*})\right\}\\
\leq&\, \sum^{m-1}_{k=0}\left\{\frac{2\delta(b)}{\beta\!-\!1}\!\left[F(x^{s}_{k})-F(x^{*})+F(\widetilde{x}^{s-1})-F(x^{*})\right]+ \frac{L\beta}{2}\mathbb{E}\!\left[\|x^{*}-x^{s}_{k}\|^2-\|x^{*}-x^{s}_{k+1}\|^2\right]\right\}.
\end{split}
\end{equation*}

Since $\widetilde{x}^{s}=\frac{1}{m}\sum^{m}_{k=1}x^{s}_{k}$, we have $F(\widetilde{x}^{s})\leq \frac{1}{m}\sum^{m}_{k=1}F(x^{s}_{k})$, and
\begin{equation*}
\begin{split}
&\:\frac{2\delta(b)}{(\beta\!-\!1)m}\left\{\mathbb{E}[F(x^{s}_{m})]-F(x^{*})\right\}+\left(1-\frac{2\delta(b)}{\beta\!-\!1}\right)\left\{\mathbb{E}[F(\widetilde{x}^{s})]-F(x^{*})\right\}\\
\leq&\: \frac{2\delta(b)}{(\beta\!-\!1)}\left[F(\widetilde{x}^{s-1})-F(x^{*})\right]+\frac{2\delta(b)}{(\beta\!-\!1)m}\left[F(x^{s}_{0})-F(x^{*})\right]+\frac{L\beta}{2m}\mathbb{E}\!\left[\|x^{*}-x^{s}_{0}\|^2-\|x^{*}-x^{s}_{m}\|^2\right].
\end{split}
\end{equation*}
This completes the proof.
\end{proof}
\vspace{3mm}

Similar to Lemma~\ref{lemm9}, we can also extend Lemma~\ref{lemm4} and Corollary~\ref{lemm7} to the mini-batch setting.

\begin{theorem}[Mini-batch]
\label{the5}
If each $f_{i}(\cdot)$ is convex and $L$-smooth, then the following inequality holds
\begin{equation}\label{equ55}
\begin{split}
\mathbb{E}\!\left[F(\widehat{x}^{S})\right]-F(x^{*})\leq\,\frac{2\delta(b)(m+1)}{(\beta\!-\!1\!-\!4\delta(b))mS}\mathbb{E}\!\left[F(\widetilde{x}^{0})\!-\!F(x^{*})\right]+\!\frac{L\beta(\beta\!-\!1)}{2(\beta\!-\!1\!-\!4\delta(b))mS}\mathbb{E}\!\left[\|x^{*}\!-\!\widetilde{x}^{0}\|^2\right].
\end{split}
\end{equation}
\end{theorem}

\begin{proof}
Using Lemma~\ref{lemm9}, we have
\begin{equation*}
\begin{split}
&\,\left(1-\frac{2\delta(b)}{\beta\!-\!1}\right)\mathbb{E}\!\left[F(\widetilde{x}^{s})-F(x^{*})\right]\\
\leq&\,\frac{2\delta(b)}{(\beta\!-\!1)}\left[F(\widetilde{x}^{s-1})-F(x^{*})\right]+\frac{2\delta(b)}{(\beta\!-\!1)m}\left[F(x^{s}_{0})-F(x^{*})\right]+\frac{L\beta}{2m}\mathbb{E}\!\left[\|x^{*}\!-x^{s}_{0}\|^2-\|x^{*}\!-x^{s}_{m}\|^2\right].
\end{split}
\end{equation*}
Summing the above inequality over $s=1,2,\ldots,S$, taking expectation over whole history of $I^{s}_{k}$, and using $x^{s+1}_{0}=x^{s}_{m}$, we obtain
\begin{equation*}
\begin{split}
&\,\sum^{S}_{s=1}\left(1-\frac{2\delta(b)}{\beta\!-\!1}\right)\mathbb{E}\left[F(\widetilde{x}^{s})-F(x^{*})\right]\\
\leq&\, \sum^{S}_{s=1}\frac{2\delta(b)}{(\beta\!-\!1)}\!\left[F(\widetilde{x}^{s-1})-F(x^{*})\right]+\sum^{S}_{s=1}\frac{2\delta(b)}{(\beta\!-\!1)m}\mathbb{E}\!\left[F(x^{s}_{0})-F(x^{s}_{m})\right]+\frac{L\beta}{2m}\sum^{S}_{s=1}\mathbb{E}\!\left[\|x^{*}\!-x^{s}_{0}\|^2-\|x^{*}\!-x^{s}_{m}\|^2\right].
\end{split}
\end{equation*}

Subtracting $\frac{2\delta(b)}{\beta-1}\!\sum^{S-1}_{s=1}\mathbb{E}\!\left[F(\widetilde{x}^{s})\!-\!F(x^{*})\right]$ from both sides of the above inequality, we have
\begin{equation*}
\begin{split}
&\:\frac{2\delta(b)}{\beta\!-\!1}\mathbb{E}\!\left[F(\widetilde{x}^{S})-F(x^{*})\right]+\sum^{S}_{s=1}\left(1-\frac{4\delta(b)}{\beta\!-\!1}\right)\mathbb{E}\!\left[F(\widetilde{x}^{s})-F(x^{*})\right]\\
\leq&\:\frac{2\delta(b)}{(\beta\!-\!1)}\left[F(\widetilde{x}^{0})-F(x^{*})\right]+\frac{2\delta(b)}{(\beta\!-\!1)m}\mathbb{E}\!\left[F(x^{1}_{0})-F(x^{S}_{m})\right]+\frac{L\beta}{2m}\mathbb{E}\!\left[\|x^{*}\!-x^{1}_{0}\|^2-\|x^{*}\!-x^{S}_{m}\|^2\right]\\
\end{split}
\end{equation*}
Dividing both sides of the above inequality by $S$ and using $\mathbb{E}[F(\overline{x})]\leq \frac{1}{S}\sum^{S}_{s=1}F(\widetilde{x}^{s})$, we arrive at
\begin{equation*}
\begin{split}
&\,\frac{2\delta(b)}{(\beta\!-\!1)S}\mathbb{E}[F(\widetilde{x}^{S})-F(x^{*})]+\left(1-\frac{4\delta(b)}{\beta\!-\!1}\right)\mathbb{E}\!\left[F(\overline{x})-F(x^{*})\right]\\
\leq&\, \frac{2\delta(b)}{(\beta\!-\!1)S}\!\left[F(\widetilde{x}^{0})-F(x^{*})\right]+\frac{2\delta(b)}{(\beta\!-\!1)mS}\mathbb{E}\!\left[F(x^{1}_{0})-F(x^{S}_{m})\right]+\frac{L\beta}{2mS}\mathbb{E}\!\left[\|x^{*}\!-x^{1}_{0}\|^2-\|x^{*}\!-x^{S}_{m}\|^2\right].
\end{split}
\end{equation*}
Subtracting $\frac{2\delta(b)}{(\beta\!-\!1)S}\mathbb{E}[F(\widetilde{x}^{S})\!-\!F(x^{*})]$ from both sides of the above inequality, we have
\begin{equation*}
\begin{split}
&\,\left(1-\frac{4\delta(b)}{\beta\!-\!1}\right)\mathbb{E}\!\left[F(\overline{x})-F(x^{*})\right]\\
\leq&\, \frac{2\delta(b)}{(\beta\!-\!1)S}\mathbb{E}\!\left[F(\widetilde{x}^{0})-F(\widetilde{x}^{S})\right]+\frac{2\delta(b)}{(\beta\!-\!1)mS}\mathbb{E}\!\left[F(x^{1}_{0})-F(x^{S}_{m})\right]+\frac{L\beta}{2mS}\mathbb{E}\!\left[\|x^{*}-x^{1}_{0}\|^2-\|x^{*}-x^{S}_{m}\|^2\right].
\end{split}
\end{equation*}

Dividing both sides of the above inequality by $(1\!-\!\frac{4\delta(b)}{\beta-1})\!>\!0$, we arrive at
\begin{equation*}
\begin{split}
&\,\mathbb{E}\!\left[F(\overline{x})\right]-F(x^{*})\\
\leq&\, \frac{2\delta(b)}{(\beta\!-\!1\!-\!4\delta(b))S}\mathbb{E}\!\left[F(\widetilde{x}^{0})\!-\!F(\widetilde{x}^{S})\right]+\!\frac{2\delta(b)}{(\beta\!-\!1\!-\!4\delta(b))mS}\mathbb{E}\!\left[F(\widetilde{x}^{0})\!-\!F(x^{S}_{m})\right]+\!\frac{L\beta(\beta\!-\!1)}{2(\beta\!-\!1\!-\!4\delta(b))mS}\mathbb{E}\!\left[\|x^{*}\!-\!\widetilde{x}^{0}\|^2\right]\\
\leq&\, \frac{2\delta(b)}{(\beta\!-\!1\!-\!4\delta(b))S}\mathbb{E}\!\left[F(\widetilde{x}^{0})\!-\!F(x^{*})\right]+\!\frac{2\delta(b)}{(\beta\!-\!1\!-\!4\delta(b))mS}\mathbb{E}\!\left[F(\widetilde{x}^{0})\!-\!F(x^{*})\right]+\!\frac{L\beta(\beta\!-\!1)}{2(\beta\!-\!1\!-\!4\delta(b))mS}\mathbb{E}\!\left[\|x^{*}\!-\!\widetilde{x}^{0}\|^2\right]\\
=&\, \frac{2\delta(b)(m+1)}{(\beta\!-\!1\!-\!4\delta(b))mS}\mathbb{E}\!\left[F(\widetilde{x}^{0})\!-\!F(x^{*})\right]+\!\frac{L\beta(\beta\!-\!1)}{2(\beta\!-\!1\!-\!4\delta(b))mS}\mathbb{E}\!\left[\|x^{*}\!-\!\widetilde{x}^{0}\|^2\right].
\end{split}
\end{equation*}

When $F(\widetilde{x}^{S})\!\leq\!F(\overline{x}^{S})$, then $\widehat{x}^{S}\!=\!\widetilde{x}^{S}$, and
\begin{equation*}
\begin{split}
\mathbb{E}\!\left[F(\widehat{x}^{S})\right]-F(x^{*})\leq\,\frac{2\delta(b)(m+1)}{(\beta\!-\!1\!-\!4\delta(b))mS}\mathbb{E}\!\left[F(\widetilde{x}^{0})\!-\!F(x^{*})\right]+\!\frac{L\beta(\beta\!-\!1)}{2(\beta\!-\!1\!-\!4\delta(b))mS}\mathbb{E}\!\left[\|x^{*}\!-\!\widetilde{x}^{0}\|^2\right].
\end{split}
\end{equation*}
Alternatively, if $F(\widetilde{x}^{S})\!\geq\!F(\overline{x}^{S})$, then $\widehat{x}^{S}\!=\!\overline{x}^{S}$, and the above inequality still holds. This completes the proof.
\end{proof}
\vspace{3mm}

From Theorem~\ref{the5}, one can see that when $b\!=\!n$ (i.e., the batch setting), we have $\delta(n)\!=\!0$, and the first term on the right-hand side of \eqref{equ55} diminishes. In other words, our VR-SGD method degenerates to the deterministic method with the convergence rate of $\mathcal{O}(1/T)$. Furthermore, when $b\!=\!1$, we have $\delta(1)\!=\!1$, and then Theorem~\ref{the5} degenerates to Theorem~\ref{the3}.

\subsection{Convergence Properties for Strongly Convex Problems}
In this part, we analyze the convergence properties of VR-SGD (i.e., Algorithm~\ref{alg2}) for solving the \emph{strongly convex} objective function \eqref{equ01}. According to the above analysis, one can see that if $F(\cdot)$ is convex, and each convex component function $f_{i}(\cdot)$ is $L$-smooth, both Algorithms~\ref{alg2} and~\ref{alg3}  converge to the optimal solution. In the following, we provide stronger convergence rate guarantees for VR-SGD under the strongly convex condition. We first give the following assumption.

\begin{assumption}\label{assum3}
For all $s=1,2,\ldots,S$, the following inequality holds
\begin{equation*}
\mathbb{E}\!\left[F(x^{s}_{0})-F(x^{*})\right]\leq C\:\!\mathbb{E}\!\left[F(\widetilde{x}^{s-\!1})-F(x^{*})\right]
\end{equation*}
where $C>0$ is a constant{\footnote{This assumption shows the relationship of the gaps between the function values at the starting and snapshot points of each epoch and the optimal value of the objective function. In fact, both $\widetilde{x}^{s}$ and $x^{s}_{m}$ (i.e., $x^{s+1}_{0}$) converge to $x^{*}$, and thus $C$ is far less than $m$, i.e., $C\ll m$.}}.
\end{assumption}

Similar to Algorithm~\ref{alg3}, Algorithm~\ref{alg2} also has two different update rules for the two cases of smooth and non-smooth objective functions. We first give the following convergence result for Algorithm~\ref{alg2} with \emph{Option I}.

\begin{theorem}[Option I]
\label{the6}
Suppose Assumptions~\ref{assum1}, \ref{assum2}, and \ref{assum3} hold, and $m$ is sufficiently large so that
\begin{equation*}
\rho_{I}:=\,\frac{2L\eta(m\!+\!C)}{(m\!-\!1)(1\!-\!3L\eta)}+\frac{C(1\!-\!L\eta)}{\mu\eta(m\!-\!1)(1\!-\!3L\eta)}< 1.
\end{equation*}
Then Algorithm~\ref{alg2} with Option I has the following geometric convergence in expectation:
\begin{equation*}
\mathbb{E}\left[F(\widetilde{x}^{s})-F(x^{*})\right]\leq\,\rho^{s}_{I}\left[F(\widetilde{x}^{0})-F(x^{*})\right].
\end{equation*}
\end{theorem}

\begin{proof}
Since each $f_{i}(\cdot)$ is convex and $L$-smooth, then Corollary~\ref{lemm7} holds, which then implies
\begin{equation}\label{equ56}
\begin{split}
&\left(1-\frac{2}{\beta\!-\!1}\right)\mathbb{E}\!\left[F(\widetilde{x}^{s})-F(x^{*})\right]+\frac{1}{m\!-\!1}\mathbb{E}\!\left[F(x^{s+1}_{0})-F(x^{*})\right]\!+\!\frac{L\beta }{2(m\!-\!1)}\mathbb{E}\!\left[\|x^{s+1}_{0}\!-\!x^{*}\|^2\right]\\
\leq& \frac{2m}{(m\!-\!1)(\beta\!-\!1)}\mathbb{E}\!\left[F(\widetilde{x}^{s-\!1})\!-\!F(x^{*})\right]\!+\!\frac{2}{(m\!-\!1)(\beta\!-\!1)}\mathbb{E}\!\left[F(x^{s}_{0})\!-\!F(x^{*})\right]\!+\!\frac{L\beta }{2(m\!-\!1)}\mathbb{E}\!\left[\|x^{s}_{0}\!-\!x^{*}\|^2\right].
\end{split}
\end{equation}
Due to the strong convexity of $F(\cdot)$, we have $\|x^{s}_{0}-x^{*}\|^2\leq({2}/{\mu})[F(x^{s}_{0})-F(x^{*})]$. Then the inequality in (\ref{equ56}) can be rewritten as follows:
\begin{equation*}
\begin{split}
&\:\left(1-\frac{2}{\beta\!-\!1}\right)\mathbb{E}\!\left[F(\widetilde{x}^{s})-F(x^{*})\right]\\
\leq&\: \frac{2m}{(m\!-\!1)(\beta\!-\!1)}\mathbb{E}\!\left[F(\widetilde{x}^{s-\!1})\!-\!F(x^{*})\right]\!+\!\left(\frac{2}{(m\!-\!1)(\beta\!-\!1)}+\frac{L\beta }{\mu(m\!-\!1)}\right)\mathbb{E}\!\left[F(x^{s}_{0})\!-\!F(x^{*})\right]\\
\leq&\: \frac{2m}{(m\!-\!1)(\beta\!-\!1)}\mathbb{E}\!\left[F(\widetilde{x}^{s-\!1})\!-\!F(x^{*})\right]\!+\!\left(\frac{2C}{(m\!-\!1)(\beta\!-\!1)}+\frac{CL\beta }{\mu(m\!-\!1)}\right)\mathbb{E}\!\left[F(\widetilde{x}^{s-\!1})\!-\!F(x^{*})\right]\\
=&\:\left(\frac{2(m\!+\!C)}{(m\!-\!1)(\beta\!-\!1)}+\frac{CL\beta }{\mu(m\!-\!1)}\right)\mathbb{E}\!\left[F(\widetilde{x}^{s-\!1})\!-\!F(x^{*})\right]
\end{split}
\end{equation*}
where the first inequality holds due to the fact that $\|x^{s}_{0}-x^{*}\|^2\leq({2}/{\mu})[F(x^{s}_{0})-F(x^{*})]$, and the second inequality follows from Assumption~\ref{assum3}. Dividing both sides of the above inequality by $[1\!-\!{2}/({\beta\!-\!1})]\!>\!0$ (that is, $\beta$ is required to be larger than 3) and using the definition of $\beta\!=\!1/(L\eta)$, we arrive at
\begin{equation*}
\begin{split}
\mathbb{E}\!\left[F(\widetilde{x}^{s})-F(x^{*})\right]\leq&\:\left(\frac{2(m\!+\!C)}{(m\!-\!1)(\beta\!-\!3)}+\frac{CL\beta(\beta\!-\!1)}{\mu(m\!-\!1)(\beta\!-\!3)}\right)\mathbb{E}\!\left[F(\widetilde{x}^{s-\!1})\!-\!F(x^{*})\right]\\
=&\:\left(\frac{2(m\!+\!C)L\eta}{(m\!-\!1)(1\!-\!3L\eta)}+\frac{C(1\!-\!L\eta)}{\mu\eta(m\!-\!1)(1\!-\!3L\eta)}\right)\mathbb{E}\!\left[F(\widetilde{x}^{s-\!1})\!-\!F(x^{*})\right].
\end{split}
\end{equation*}
This completes the proof.
\end{proof}

Although the learning rate $\eta$ in Theorem~\ref{the6} needs to be less than $1/(3L)$, we can use much larger learning rates in practice, e.g., $\eta=3/(7L)$. Similar to Theorem~\ref{the6}, we give the following convergence result for Algorithm~\ref{alg2} with \emph{Option II}.

\begin{theorem}[Option II]
\label{the7}
Suppose Assumptions~\ref{assum1}, \ref{assum2}, and \ref{assum3} hold, and $m$ is sufficiently large so that
\begin{equation*}
\rho_{I\!I}:=\,\frac{2L\eta(m\!+\!C)}{m(1\!-\!3L\eta)}+\frac{C(1\!-\!L\eta)}{m\mu\eta(1\!-\!3L\eta)}< 1.
\end{equation*}
Then Algorithm~\ref{alg2} with Option II has the following geometric convergence in expectation:
\begin{equation*}
\mathbb{E}\left[F(\widetilde{x}^{s})-F(x^{*})\right]\leq\,\rho^{s}_{I\!I}\left[F(\widetilde{x}^{0})-F(x^{*})\right].
\end{equation*}
\end{theorem}

The proof of Theorem~\ref{the7} can be found in APPENDIX F. In addition, we can also provide the linear convergence guarantees of Algorithm~\ref{alg2} with Option I or Option II for solving smooth strongly convex objective functions. From all the results, one can see that VR-SGD attains a \emph{linear} convergence rate for both \emph{smooth} and \emph{non-smooth} strongly convex minimization problems.

\section{Experimental Results}
In this section, we evaluate the performance of our VR-SGD method for solving various ERM problems, such as logistic regression, Lasso, and ridge regression, and compare its performance with several related stochastic variance reduced methods, including SVRG~\cite{johnson:svrg}, Prox-SVRG~\cite{xiao:prox-svrg}, and Katyusha~\cite{zhu:Katyusha}. All the codes of VR-SGD and related methods can be downloaded from the author's website{\footnote{\url{https://sites.google.com/site/fanhua217/publications}}}.

\subsection{Experimental Setup}
We used four publicly available data sets in the experiments: Adult (also called a9a), Covtype, Protein, and Sido0, as listed in Table~\ref{tab1}. \emph{It should be noted that each example of these date sets was normalized so that they have unit length as in~\cite{xiao:prox-svrg,shang:fsvrg}, which leads to the same upper bound on the Lipschitz constants $L_{i}$, i.e., $L\!=\!L_{i}$ for all $i\!=\!1,\ldots,n$}. As suggested in~\cite{johnson:svrg,xiao:prox-svrg,zhu:Katyusha}, the epoch length is set to $m\!=\!2n$ for the stochastic variance reduced methods, SVRG~\cite{johnson:svrg}, Prox-SVRG~\cite{xiao:prox-svrg}, and Katyusha~\cite{zhu:Katyusha}, as well as our VR-SGD method. Then the only parameter we have to tune by hand is the step size (or learning rate), $\eta$. Since Katyusha has a much higher per-iteration complexity than SVRG, Prox-SVRG, and VR-SGD, we compare their performance in terms of both the number of effective passes and running time (seconds), where computing a single full gradient or evaluating $n$ component gradients is considered as one effective pass over the data. For fair comparison, we implemented SVRG, Prox-SVRG, Katyusha, and VR-SGD in C++ with a Matlab interface, and performed all the experiments on a PC with an Intel i5-2400 CPU and 16GB RAM. In addition, we do not compare with other stochastic algorithms such as SAGA~\cite{defazio:saga} and Catalyst~\cite{lin:vrsg}, as they have been shown to be comparable or inferior to Katyusha~\cite{zhu:Katyusha}.

\begin{table}[!th]
\centering
\caption{Summary of data sets used for our experiments.}
\label{tab1}
\setlength{\tabcolsep}{8.9pt}
\renewcommand\arraystretch{1.19}
\begin{tabular}{lccc}
\hline
\ Data sets   & Sizes $n$    & Dimensions $d$  & Sparsity\\
\hline
\ Adult       & 32,562         & 123            & 11.28\% \\
\ Protein     & 145,751        & 74             & 99.21\% \\
\ Covtype     & 581,012        & 54             & 22.12\% \\
\ Sido0       & 12,678         & 4,932          & 9.84\% \\
\hline
\end{tabular}
\end{table}

\begin{figure}[t]
\centering
\subfigure[Adult: $\lambda\!=\!10^{-5}$ (left) \;and\; $\lambda\!=\!10^{-6}$ (right)]{\includegraphics[width=0.246\columnwidth]{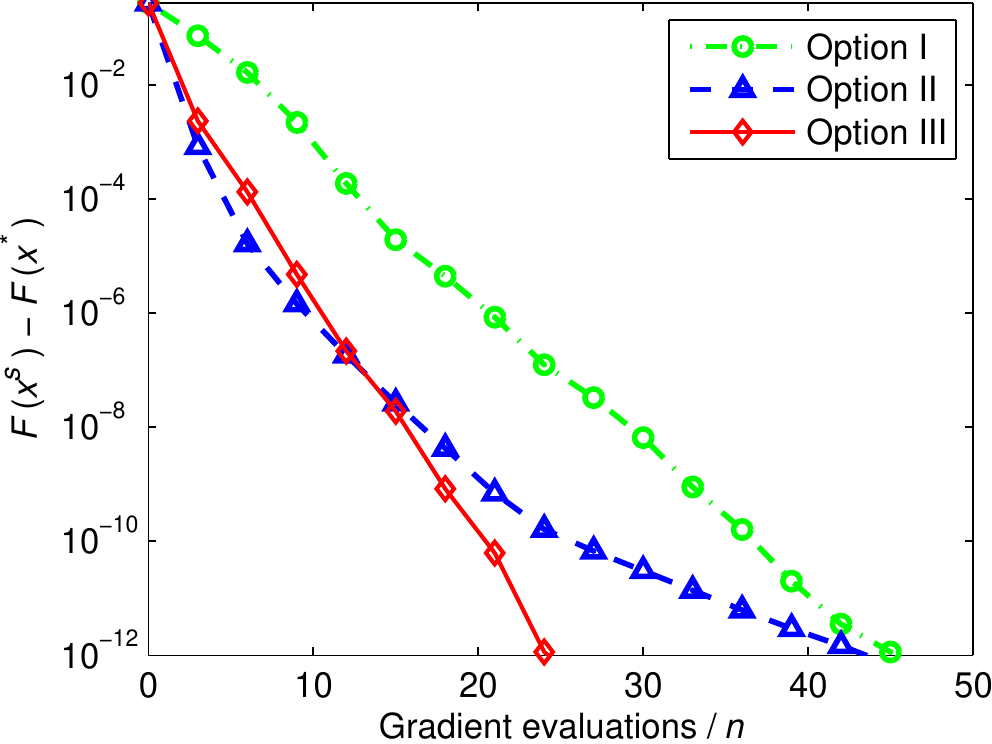}\:\includegraphics[width=0.246\columnwidth]{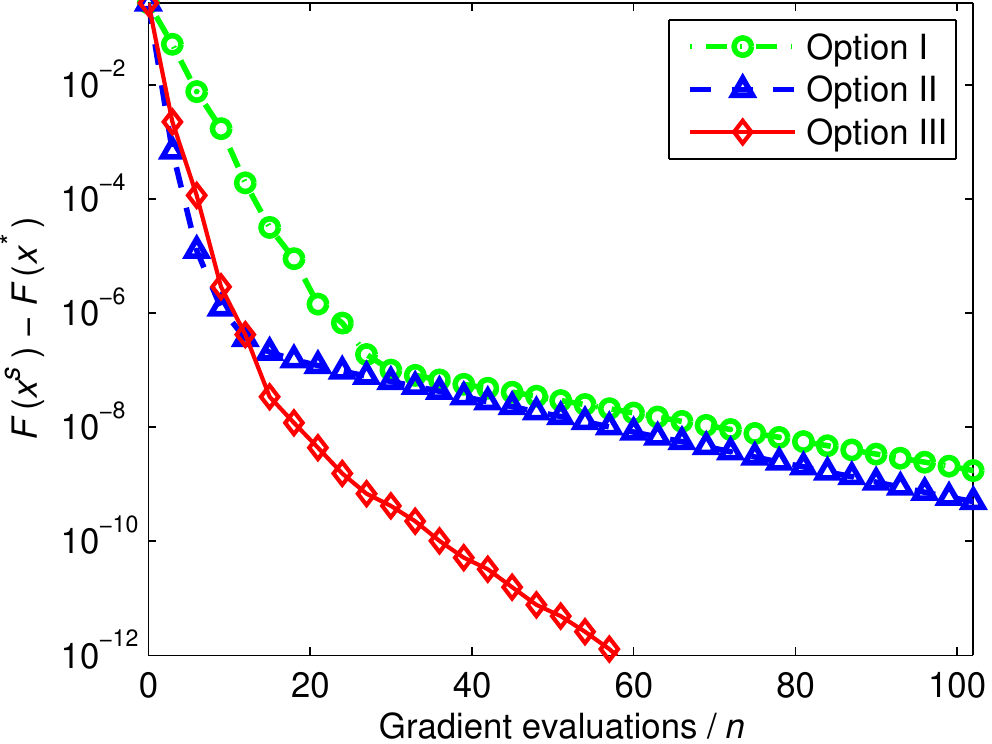}}
\subfigure[Protein: $\lambda\!=\!10^{-5}$ (left) \;and\; $\lambda\!=\!10^{-6}$ (right)]{\includegraphics[width=0.246\columnwidth]{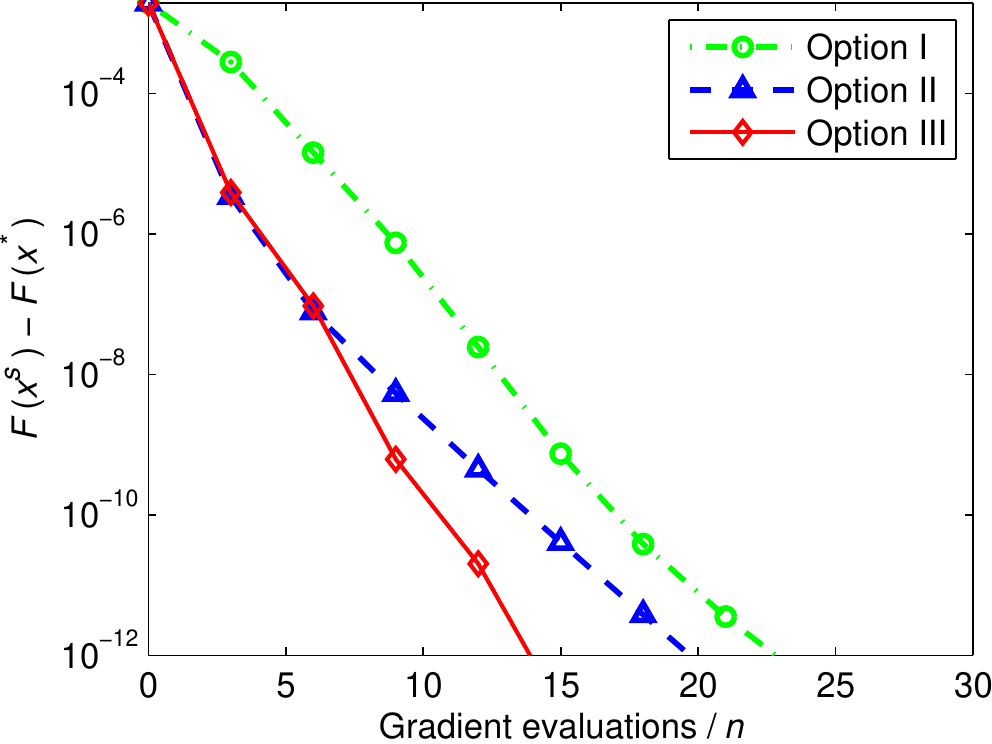}\:\includegraphics[width=0.246\columnwidth]{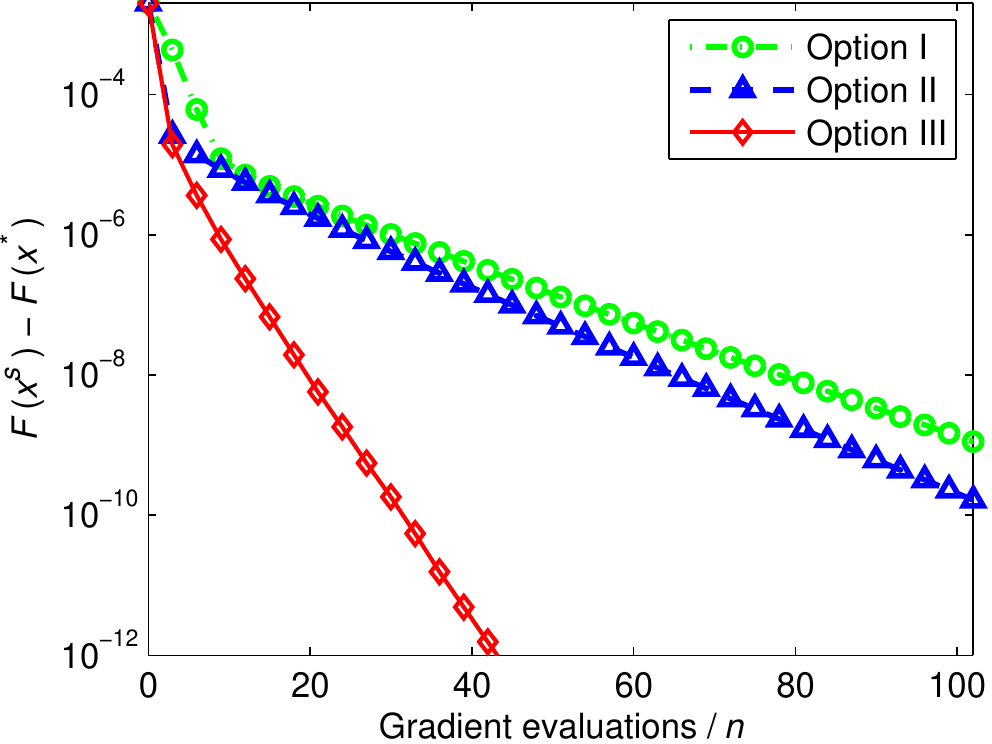}}
\subfigure[Covtype: $\lambda\!=\!10^{-5}$ (left) \;and\; $\lambda\!=\!10^{-6}$ (right)]{\includegraphics[width=0.246\columnwidth]{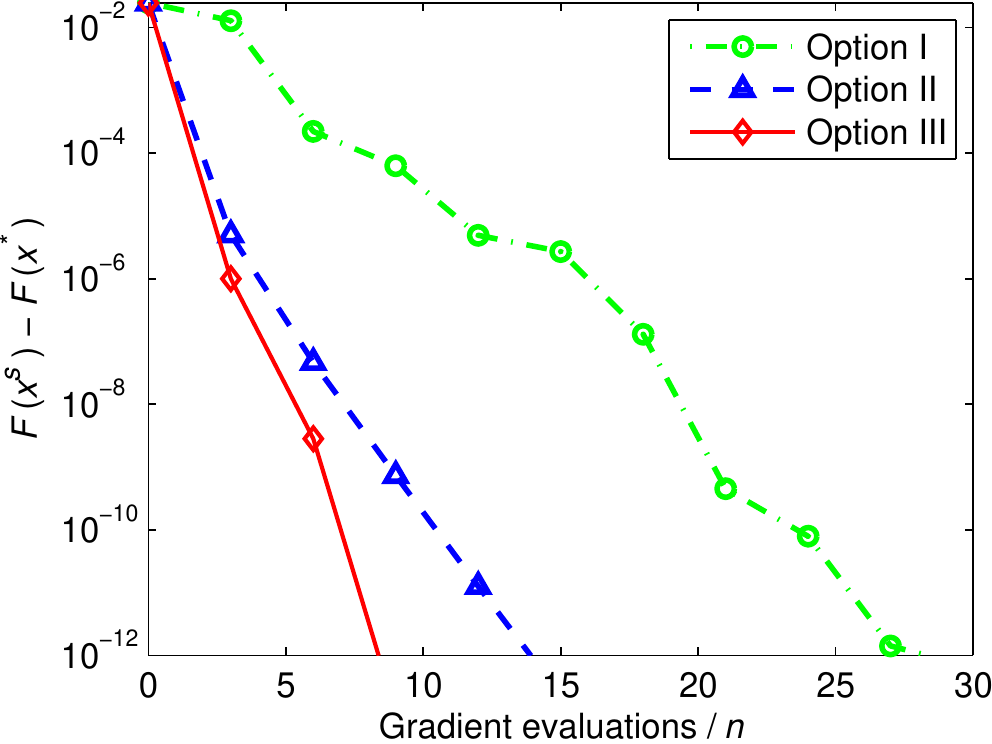}\:\includegraphics[width=0.246\columnwidth]{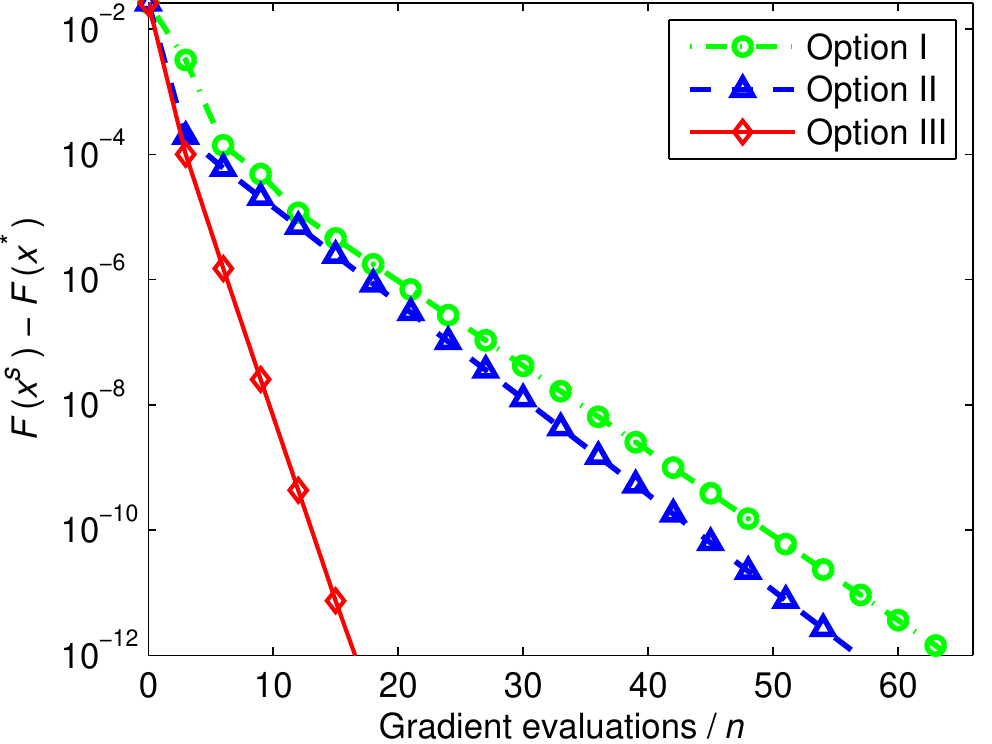}}
\subfigure[Sido0: $\lambda\!=\!10^{-4}$ (left) \;and\; $\lambda\!=\!10^{-5}$ (right)]{\includegraphics[width=0.246\columnwidth]{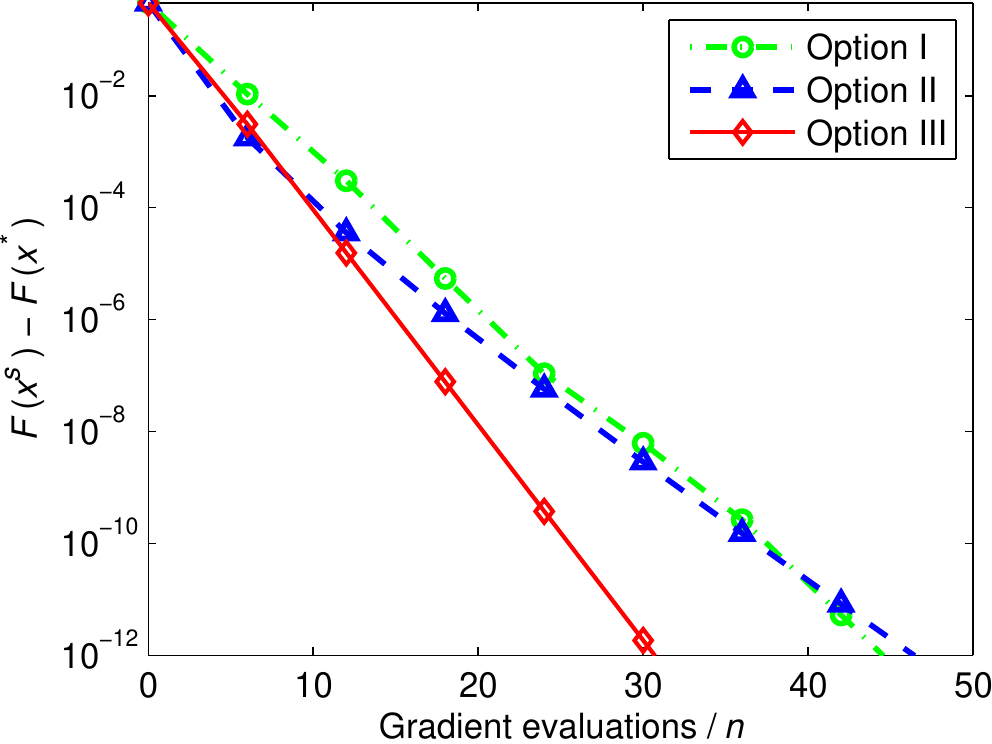}\:\includegraphics[width=0.246\columnwidth]{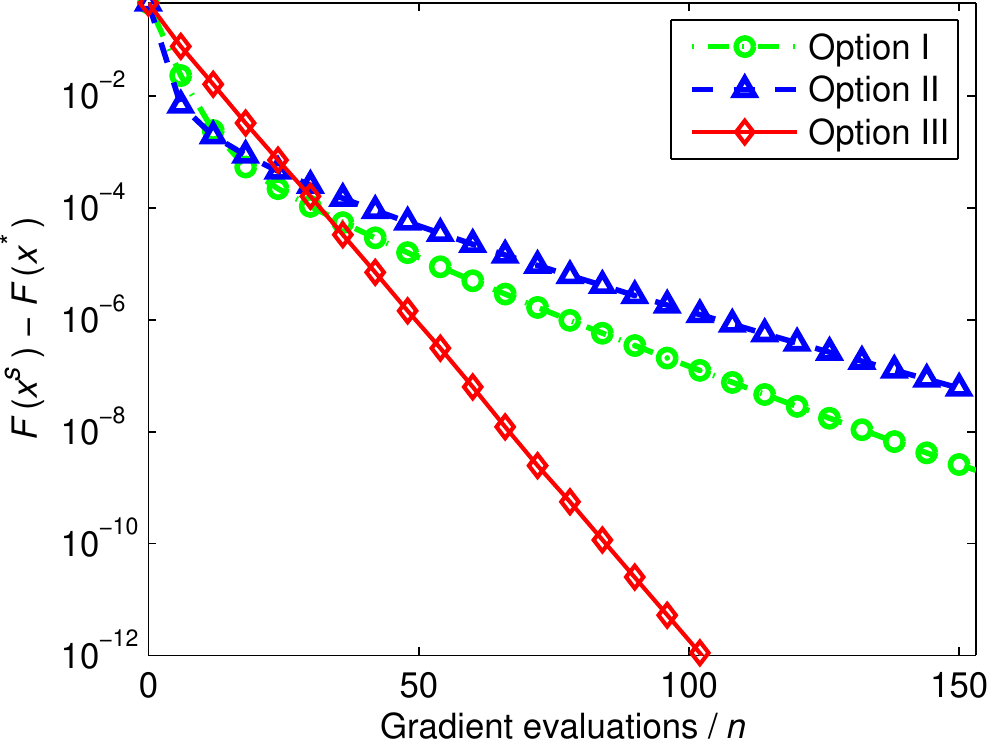}}
\vspace{-2.6mm}
\caption{Comparison of Options I, II, and III for solving ridge regression problems with the regularizer $(\lambda/2)\|\cdot\|^{2}$. In each plot, the vertical axis shows the objective value minus the minimum, and the horizontal axis is the number of effective passes.}
\label{figs02}
\end{figure}

\begin{figure}[t]
\centering
\subfigure[Adult: $\lambda\!=\!10^{-4}$ (left) \;and\; $\lambda\!=\!10^{-5}$ (right)]{\includegraphics[width=0.246\columnwidth]{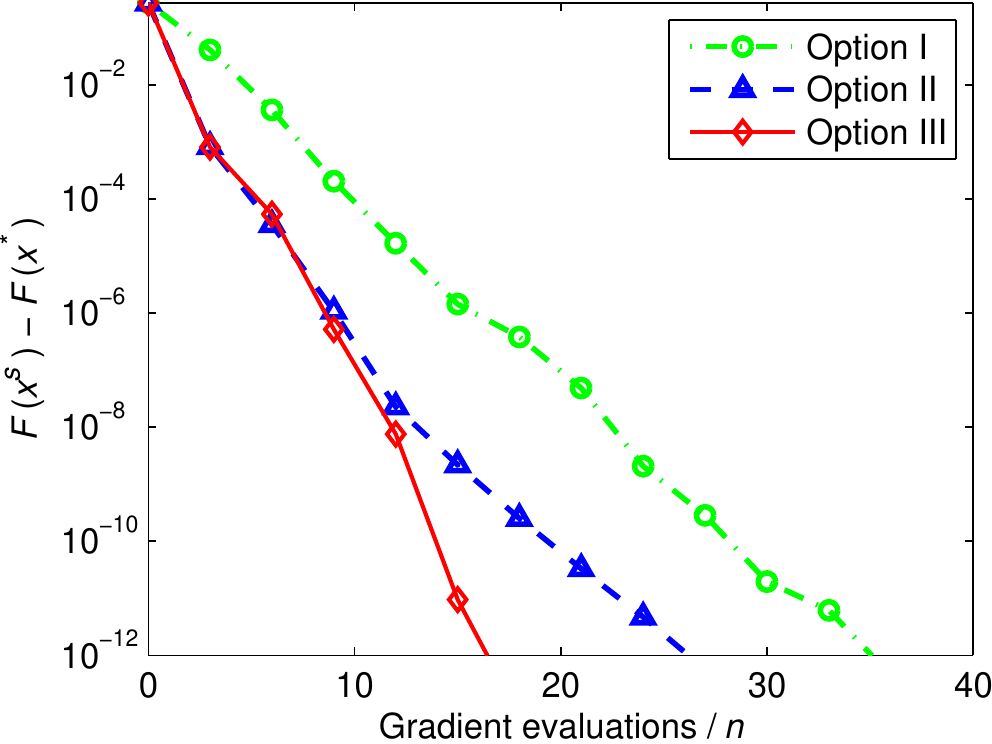}\:\includegraphics[width=0.246\columnwidth]{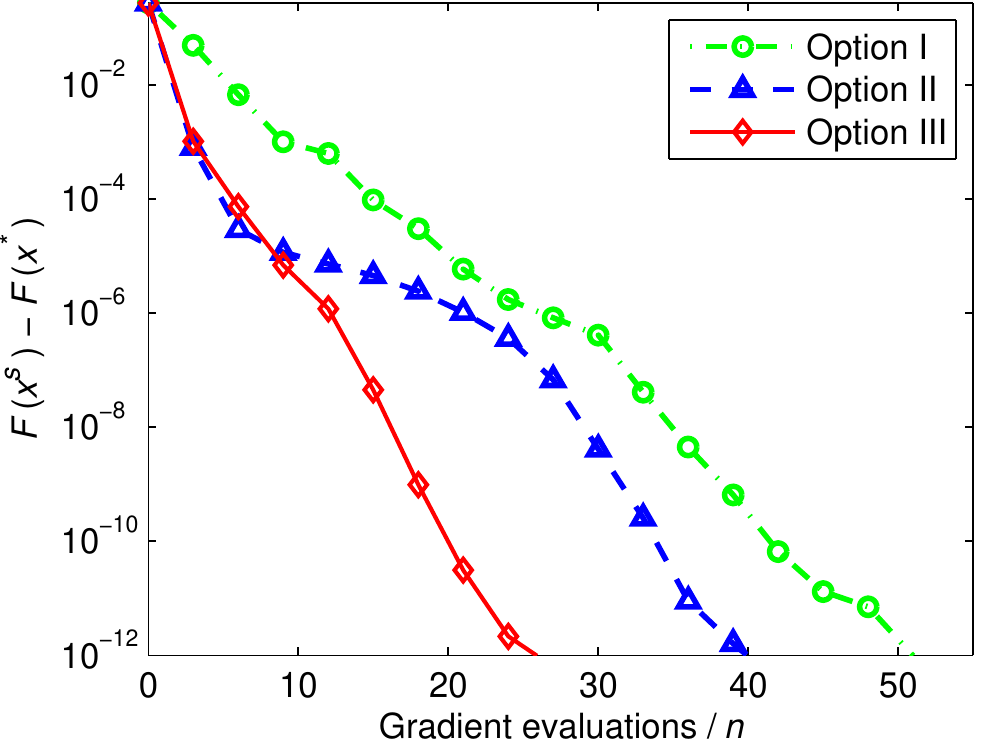}}
\subfigure[Protein: $\lambda\!=\!10^{-4}$ (left) \;and\; $\lambda\!=\!10^{-5}$ (right)]{\includegraphics[width=0.246\columnwidth]{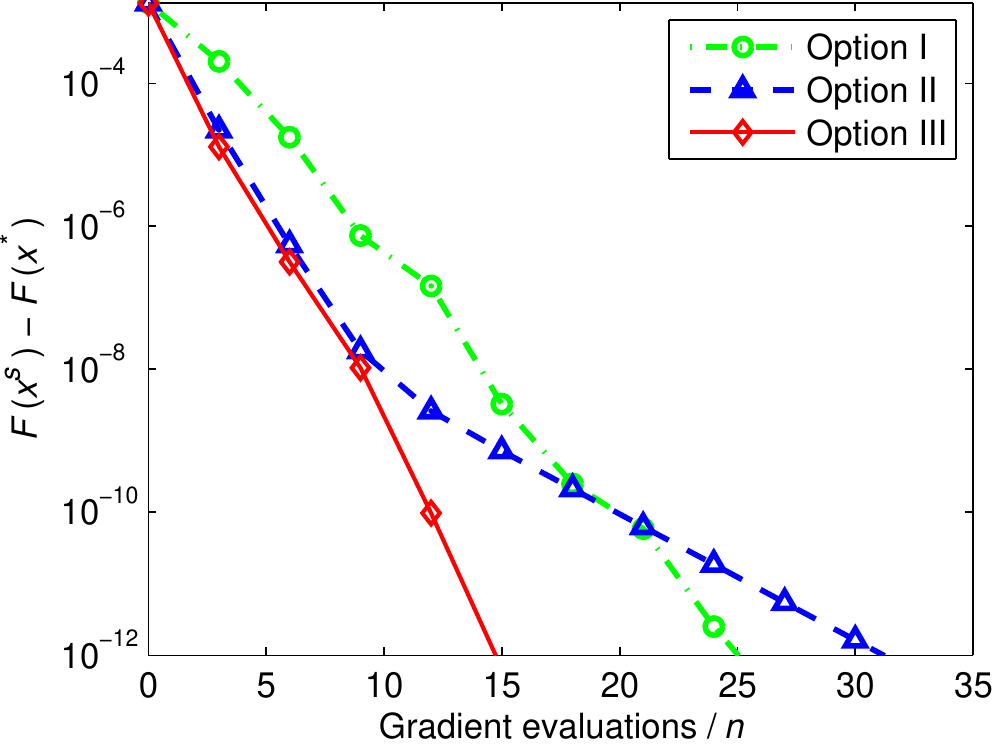}\:\includegraphics[width=0.246\columnwidth]{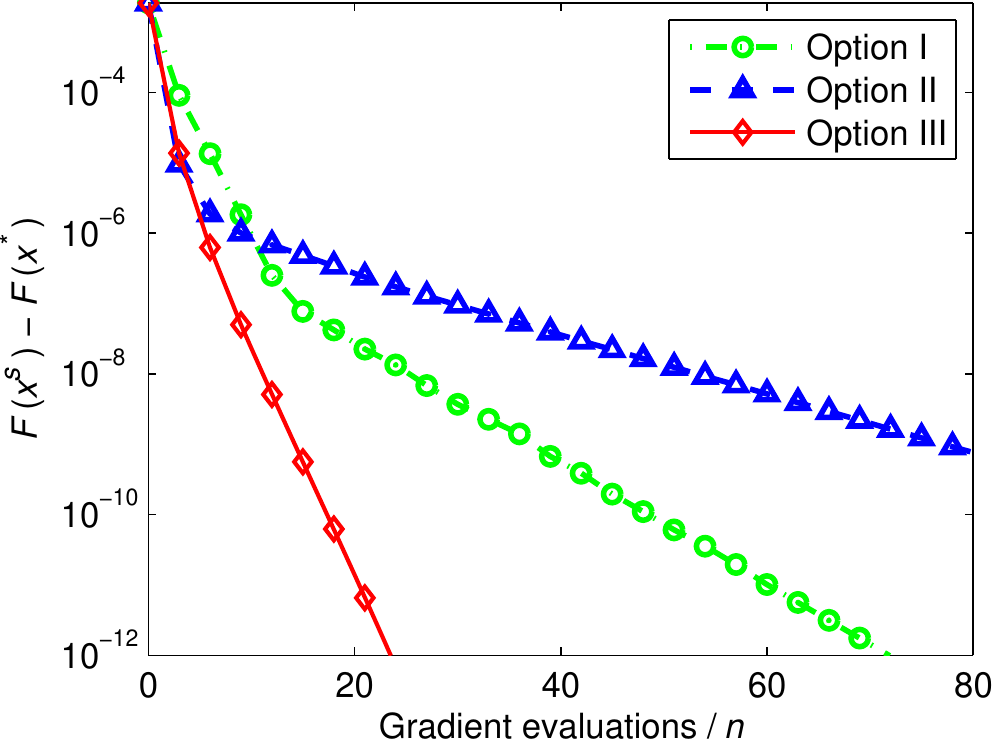}}
\subfigure[Covtype: $\lambda\!=\!10^{-4}$ (left) \;and\; $\lambda\!=\!10^{-5}$ (right)]{\includegraphics[width=0.246\columnwidth]{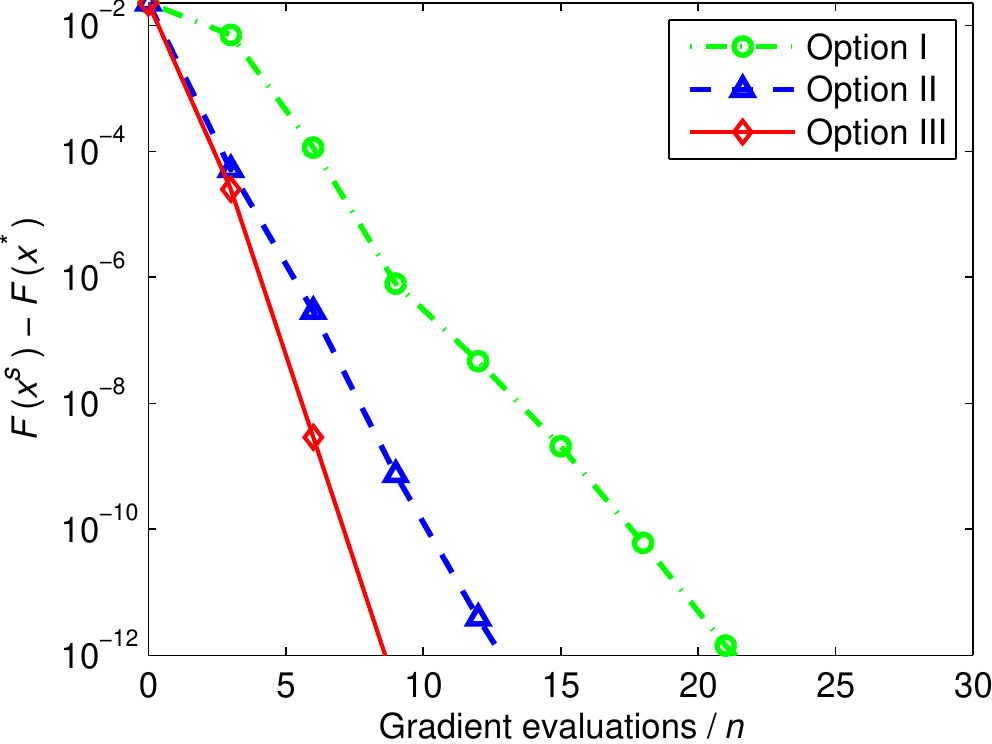}\:\includegraphics[width=0.246\columnwidth]{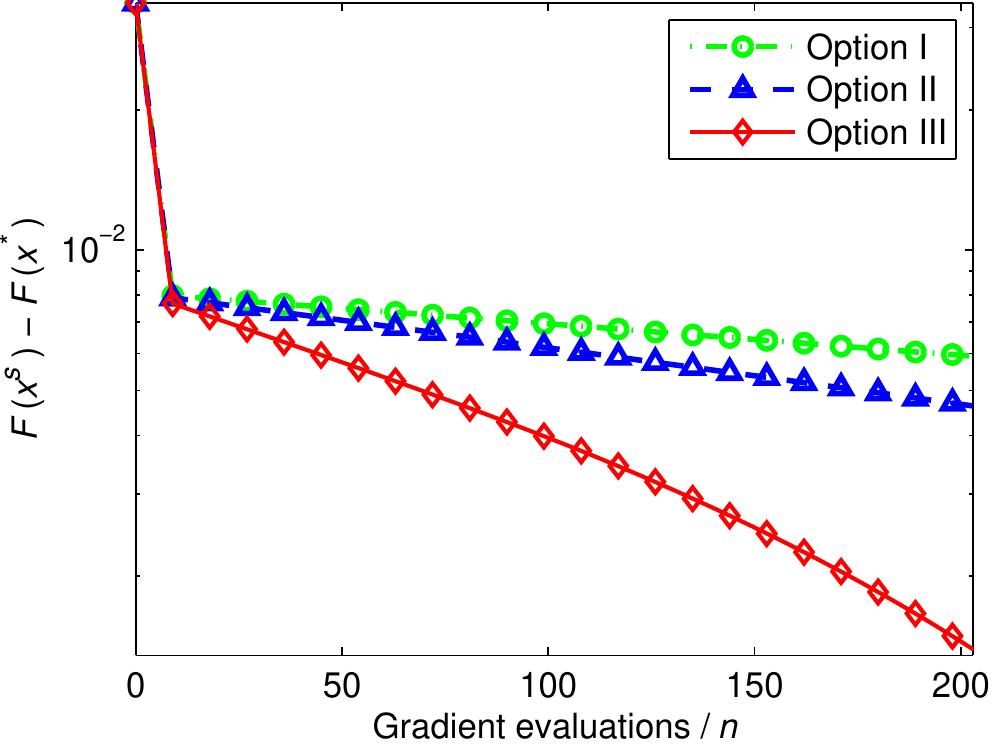}}
\subfigure[Sido0: $\lambda\!=\!10^{-4}$ (left) \;and\; $\lambda\!=\!10^{-5}$ (right)]{\includegraphics[width=0.246\columnwidth]{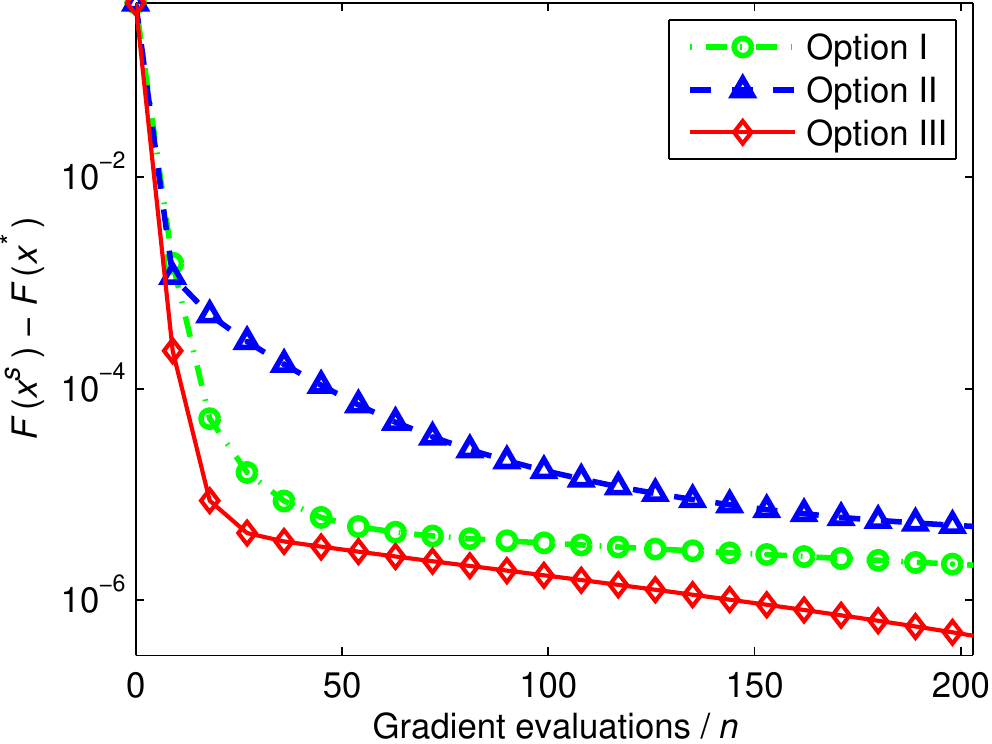}\:\includegraphics[width=0.246\columnwidth]{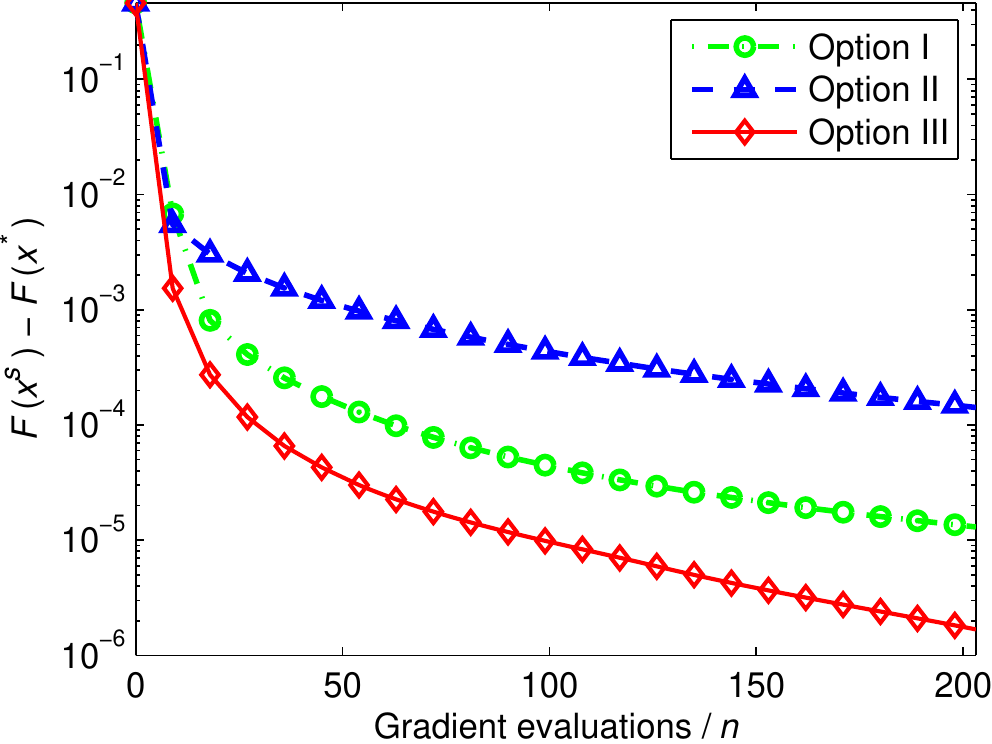}}
\vspace{-2.6mm}
\caption{Comparison of Options I, II, and III for solving Lasso problems with the regularizer $\lambda\|x\|_{1}$. In each plot, the vertical axis shows the objective value minus the minimum, and the horizontal axis is the number of effective passes.}
\label{figs03}
\end{figure}

\begin{table}[!th]
\centering
\caption{The three choices of snapshot and starting points for stochastic optimization.}
\label{tab2}
\setlength{\tabcolsep}{10.6pt}
\renewcommand\arraystretch{1.56}
\begin{tabular}{c|c|c}
\hline
Option I    & Option II  & Option III\\
\hline
\ $\widetilde{x}^{s+\!1}\!=x^{s}_{m}$ \;and\; $x^{s+\!1}_{0}\!=x^{s}_{m}$        & $\widetilde{x}^{s+\!1}\!=\frac{1}{m}\!\sum^{m}_{k=1}\!x^{s}_{k}$ \;and\; $x^{s+\!1}_{0}\!=\frac{1}{m}\!\sum^{m}_{k=1}\!x^{s}_{k}$            & $\widetilde{x}^{s+\!1}\!=\frac{1}{m-1}\!\sum^{m-1}_{k=1}\!x^{s}_{k}$ \;and\; $x^{s+\!1}_{0}\!=x^{s}_{m}$\\
\hline
\end{tabular}
\end{table}

\subsection{Different Choices for Snapshot and Starting Points}
\label{sec52}
In the practical implementation of SVRG~\cite{johnson:svrg}, both the snapshot point $\widetilde{x}^{s}$ and the starting point $x^{s+\!1}_{0}$ of each epoch are set to the last iterate $x^{s}_{m}$ of the previous epoch (i.e., Option I in Algorithm~\ref{alg1}), while the two vectors in~\cite{xiao:prox-svrg} (also suggested in~\cite{johnson:svrg}) are set to the average point of the previous epoch, $\frac{1}{m}\!\sum^{m}_{k=1}\!x^{s}_{k}$ (i.e., Option II in Algorithm~\ref{alg1}). In contrast, $\widetilde{x}^{s}$ and $x^{s+\!1}_{0}$ in this paper are set to $\frac{1}{m-\!1}\!\sum^{m-\!1}_{k=1}\!x^{s}_{k}$ and $x^{s}_{m}$ (denoted by Option III, i.e., Option I{\footnote{As Options I and II in Algorithms~\ref{alg2} and~\ref{alg3} achieve very similar performance, we only report the results of Algorithms~\ref{alg2} and~\ref{alg3} with Option I.}} in Algorithms~\ref{alg2} and~\ref{alg3}), respectively. Note that Johnson and Zhang~\cite{johnson:svrg} first presented the choices of Options I and II, while the setting of Option III is suggested in this paper. In the following, we compare the performance of the three choices (i.e., Options I, II and III in Table~\ref{tab2}) for snapshot and starting points for solving ridge regression and Lasso problems, as shown in Figs.\ \ref{figs02} and~\ref{figs03}. Except for the three different settings for snapshot and starting points, we use the update rules in (\ref{equ21}) and (\ref{equ22}) for ridge regression and Lasso problems, respectively.

From all the results shown in Figs.\ \ref{figs02} and~\ref{figs03}, we can see that our algorithms with Option III (i.e., Algorithms~\ref{alg2} and \ref{alg3} with their Option I) consistently converge much faster than SVRG with the choices of Options I and II for both strongly convex and non-strongly convex cases. This indicates that the setting of Option III suggested in this paper is a better choice than Options I and II for stochastic optimization.

\subsection{Logistic Regression}
\label{sec53}
In this part, we focus on the following generalized logistic regression problems for binary classification,
\begin{equation}\label{equ52}
\min_{x\in\mathbb{R}^{d}}\frac{1}{n}\sum^{n}_{i=1}\log(1+\exp(-b_{i}a^{T}_{i}x))+\frac{\lambda_{1}}{2}\|x\|^{2}+\lambda_{2}\|x\|_{1},
\end{equation}
where $\{(a_{i},b_{i})\}$ is a set of training examples, and $\lambda_{1},\lambda_{2}\!\geq\!0$ are the regularization parameters. Note that when $\lambda_{2}\!>\!0$, $f_{i}(x)\!=\!\log(1\!+\!\exp(-b_{i}a^{T}_{i}x))\!+\!(\lambda_{1}/{2})\|x\|^{2}$. The formulation (\ref{equ52}) includes the $\ell_{2}$-norm (i.e., $\lambda_{2}\!=\!0$),  $\ell_{1}$-norm (i.e., $\lambda_{1}\!=\!0$), and elastic net (i.e., $\lambda_{1}\!\neq\!0$ and $\lambda_{2}\!\neq\!0$) regularized logistic regression problems. Figs.\ \ref{figs1}, \ref{figs2} and~\ref{figs3} show how the objective gap, i.e., $F(x^{s})\!-\!F(x^{*})$, decreases for the $\ell_{2}$-norm, $\ell_{1}$-norm, and elastic net regularized logistic regression problems, respectively. From all the results, we make the following observations.
\begin{itemize}
  \item When the regularization parameters $\lambda_{1}$ and $\lambda_{2}$ are relatively large, e.g., $\lambda_{1}\!=\!10^{-4}$ or $\lambda_{2}\!=\!10^{-4}$, Prox-SVRG usually converges faster than SVRG for both strongly convex (e.g., $\ell_{2}$-norm regularized logistic regression) and non-strongly convex (e.g., $\ell_{1}$-norm regularized logistic regression) cases, as shown in Figs.\ \ref{figs1a}--\ref{figs1c} and Figs.\ \ref{figs2a}--\ref{figs2b}. On the contrary, SVRG often outperforms Prox-SVRG, when the regularization parameters are relatively small, e.g., $\lambda_{1}\!=\!10^{-6}$ or $\lambda_{2}\!=\!10^{-6}$, as observed in~\cite{shang:vrsgd}. The main reason is that they have different initialization settings, i.e., $\widetilde{x}^{s}\!=\!x^{s}_{m}$ and $x^{s+\!1}_{0}\!=\!x^{s}_{m}$ for SVRG vs.\ $\widetilde{x}^{s}\!=\!\frac{1}{m}\!\sum^{m}_{k=1}\!x^{s}_{k}$ and $x^{s+\!1}_{0}\!=\!\frac{1}{m}\!\sum^{m}_{k=1}\!x^{s}_{k}$ for Prox-SVRG.
  \item Katyusha converges much faster than SVRG and Prox-SVRG for the cases when the regularization parameters are relatively small, e.g., $\lambda_{1}\!=\!10^{-6}$, whereas it often achieves similar or inferior performance when the regularization parameters are relatively large, e.g., $\lambda_{1}\!=\!10^{-4}$, as shown in Figs.\ \ref{figs1a}--\ref{figs1d} and Figs.\ \ref{figs2a}--\ref{figs2b}. Note that we implemented the original algorithms with Option I in~\cite{zhu:Katyusha} for Katyusha. In other words, Katyusha is an accelerated proximal stochastic gradient method. Obviously, the above observation matches the convergence properties of Katyusha provided in~\cite{zhu:Katyusha}, that is, only if $m\mu/L\!\leq\!3/4$, Katyusha attains the best known overall complexities of $\mathcal{O}((n\!+\!\!\sqrt{n{L}/{\mu}})\log({1}/{\epsilon}))$ for strongly convex problems.
  \item Our VR-SGD method consistently converges much faster than SVRG and Prox-SVRG, especially when the regularization parameters are relatively small, e.g., $\lambda_{1}\!=\!10^{-6}$ or $\lambda_{2}\!=\!10^{-6}$, as shown in Figs.\ \ref{figs1i}--\ref{figs1l} and Figs.\ \ref{figs2e}--\ref{figs2f}. The main reason is that VR-SGD can use much larger learning rates than SVRG (e.g., $6/(5L)$ for VR-SGD vs.\ $1/(10L)$ for SVRG), which leads to faster convergence. This further verifies that the setting of both snapshot and starting points in our algorithms (i.e., Algorithms~\ref{alg2} and \ref{alg3}) is a better choice than Options I and II in Algorithm~\ref{alg1}.
  \item In particular, VR-SGD generally outperforms the best-known stochastic method, Katyusha, in terms of the number of passes through the data, especially when the regularization parameters are relatively large, e.g., $10^{-4}$ and $10^{-5}$, as shown in Figs.\ \ref{figs1a}--\ref{figs1h} and Figs.\ \ref{figs2a}--\ref{figs2d}. Since VR-SGD has a much lower per-iteration complexity than Katyusha, VR-SGD has more obvious advantage over Katyusha in terms of running time. From the algorithms of Katyusha proposed in~\cite{zhu:Katyusha}, one can see that the learning rate of Katyusha is at least set to $1/(3L)$. Similarly, the learning rate used in VR-SGD is comparable to Katyusha, which may be the main reason why the performance of VR-SGD is much better than that of Katyusha. This also implies that the algorithm that enjoys larger learning rates can yield better performance.
\end{itemize}

\subsection{Common Stochastic Gradient and Prox-SG Updates}
\label{sec54}
In this part, we compare the original Katyusha algorithm, i.e., Algorithm 1 in~\cite{zhu:Katyusha}, with the slightly modified Katyusha algorithm (denoted by Katyusha-I). In Katyusha-I, only the following two update rules for smooth objective functions are used to replace the original proximal stochastic gradient update rules in~\eqref{equ062} and \eqref{equ063}.
\begin{equation}\label{equ53}
\begin{split}
&y^{s}_{k+1}=y^{s}_{k}-\eta [\widetilde{\nabla}\! f_{i_{k}}\!(x^{s}_{k+1})+\nabla\!g(x^{s}_{k+1})],\\
&z^{s}_{k+1}=x^{s}_{k+1}-[\widetilde{\nabla}\! f_{i_{k}}\!(x^{s}_{k+1})+\nabla\!g(x^{s}_{k+1})]/(3L).
\end{split}
\end{equation}
Similarly, we also implement the proximal versions{\footnote{Here, the proximal variant of SVRG is different from Prox-SVRG~\cite{xiao:prox-svrg}, and their main difference is the choices of both the snapshot point and starting point. That is, the two vectors of the former are set to the last iterate $x^{s}_{m}$, while those of Prox-SVRG are set to the average point of the previous epoch, i.e., $\frac{1}{m}\!\sum^{m}_{k=1}\!x^{s}_{k}$.} for the original SVRG (also called SVRG-I) and the proposed VR-SGD (denoted by VR-SGD-I) methods, and denote their proximal variants by SVRG-II and VR-SGD-II, respectively. Here, the original Katyusha method is denoted by Katyusha-II.

Figs.\ \ref{figs4} and \ref{figs5} show the performance of Katyusha-I and Katyusha-II for solving ridge regression problems on the two popular data sets: Adult and Covtype. We also report the results of SVRG, VR-SGD, and their proximal variants. It is clear that Katyusha-I usually performs better than Katyusha-II (i.e., the original proximal stochastic method, Katyusha~\cite{zhu:Katyusha}), and converges significantly faster for the case when the regularization parameter is $10^{-4}$. This seems to be the main reason why Katyusha has inferior performance when the regularization parameter is relatively large, as shown in Section~\ref{sec53}. In contrast, VR-SGD and its proximal variant have similar performance, and the former slightly outperforms the latter in most cases, as well as SVRG vs.\ its proximal variant. All this suggests that stochastic gradient update rules as in (\ref{equ21}) and (\ref{equ53}) are better choices than proximal stochastic gradient update rules as in (\ref{equ14}), \eqref{equ062} and \eqref{equ063} for smooth objective functions. We also believe that our new insight can help us to design accelerated stochastic optimization methods. Both Katyusha-I and Katyusha-II usually outperform SVRG and its proximal variant, especially when the regularization parameter is relatively small, e.g., $\lambda\!=\!10^{-6}$, as shown in Figs.\ \ref{figs4d} and \ref{figs5d}. Unfortunately, both Katyusha-I and Katyusha-II cannot solve the convex objectives without any regularization term (i.e., $\lambda\!=\!0$, as shown in Figs.\ \ref{figs4f} and \ref{figs5f}) or with too large and small regularization parameters, e.g., $10^{-3}$ and $10^{-7}$, as shown in Figs.\ \ref{figs5a} and \ref{figs4e}.

Moreover, it can be seen that both VR-SGD and its proximal variant achieve much better performance than the other methods in most cases, and are also comparable to Katyusha-I and Katyusha-II in the remaining cases. This further verifies that VR-SGD is suitable for various large-scale machine learning.

\subsection{Fixed and Varied Learning Rates}
\label{sec55}
Finally, we compare the performance of Algorithm~\ref{alg3} with fixed and varied learning rates for solving $\ell_{1}$-norm regularized logistic regression and Lasso problems, as shown in Fig.\ \ref{figs6}. Note that the learning rates in Algorithm~\ref{alg3} are varied according to the update formula in (\ref{equ23}). We can observe that Algorithm~\ref{alg3} with varied step-sizes performs very similar to Algorithm~\ref{alg3} with fixed step-sizes in most cases. In the remaining cases, the former slightly outperforms the latter.

\begin{figure}[!th]
\centering
\includegraphics[width=0.246\columnwidth]{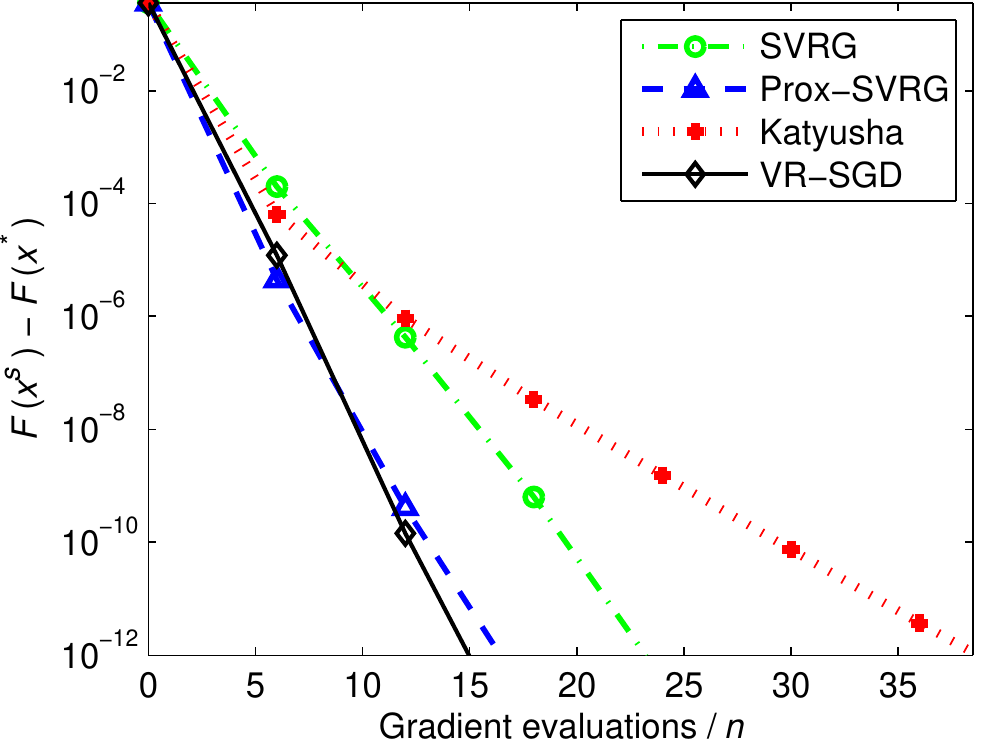}
\includegraphics[width=0.246\columnwidth]{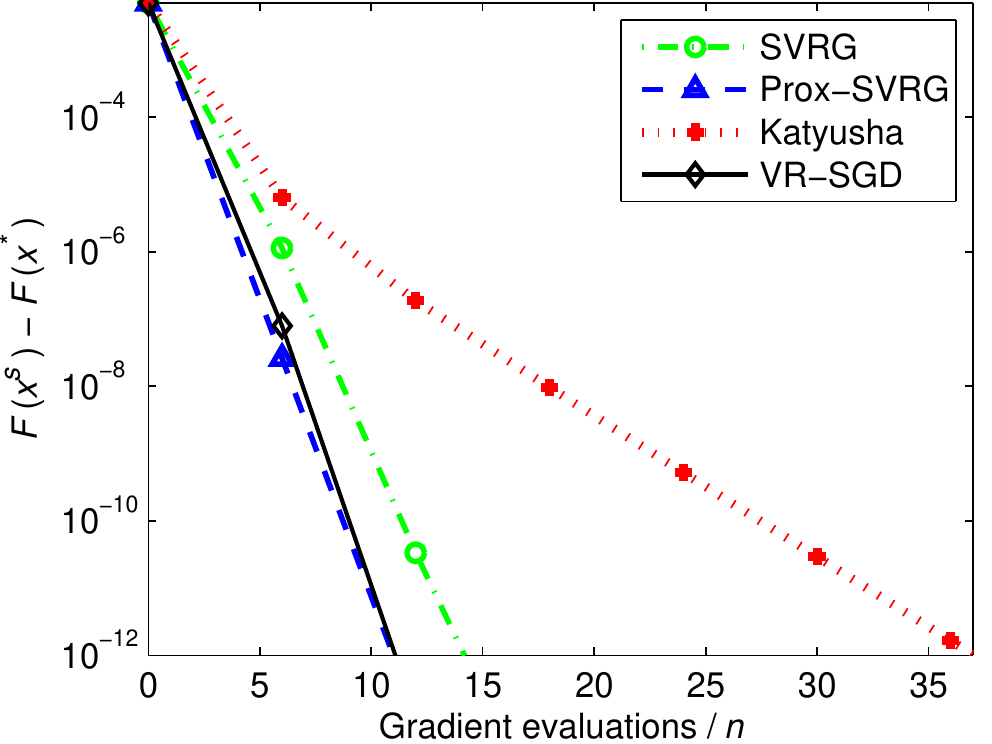}
\includegraphics[width=0.246\columnwidth]{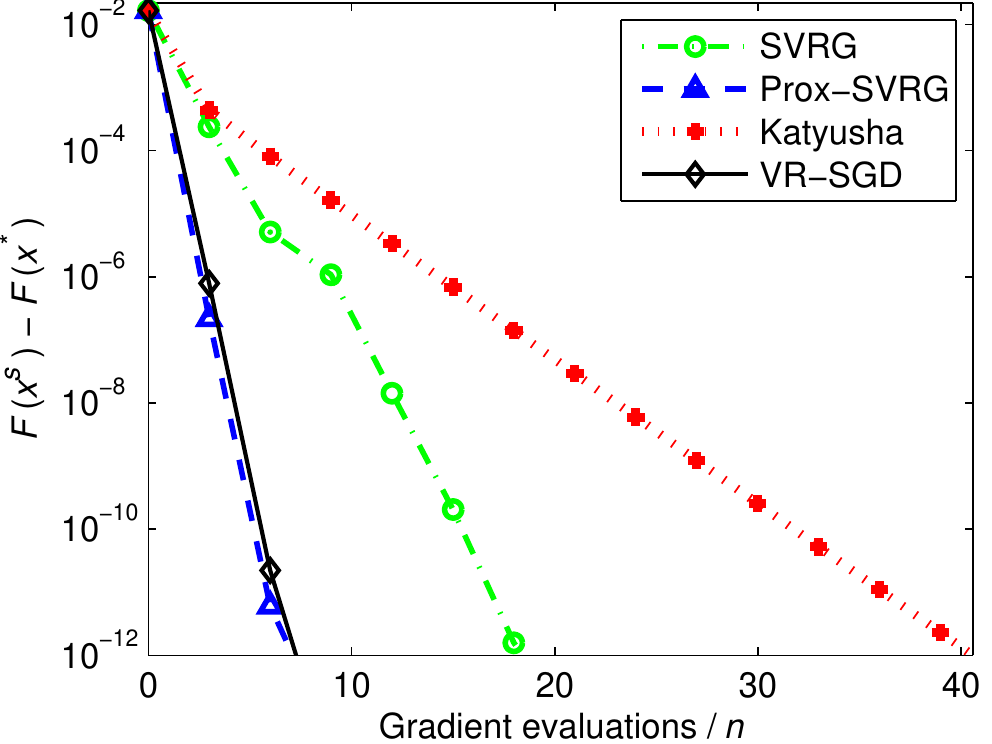}
\includegraphics[width=0.246\columnwidth]{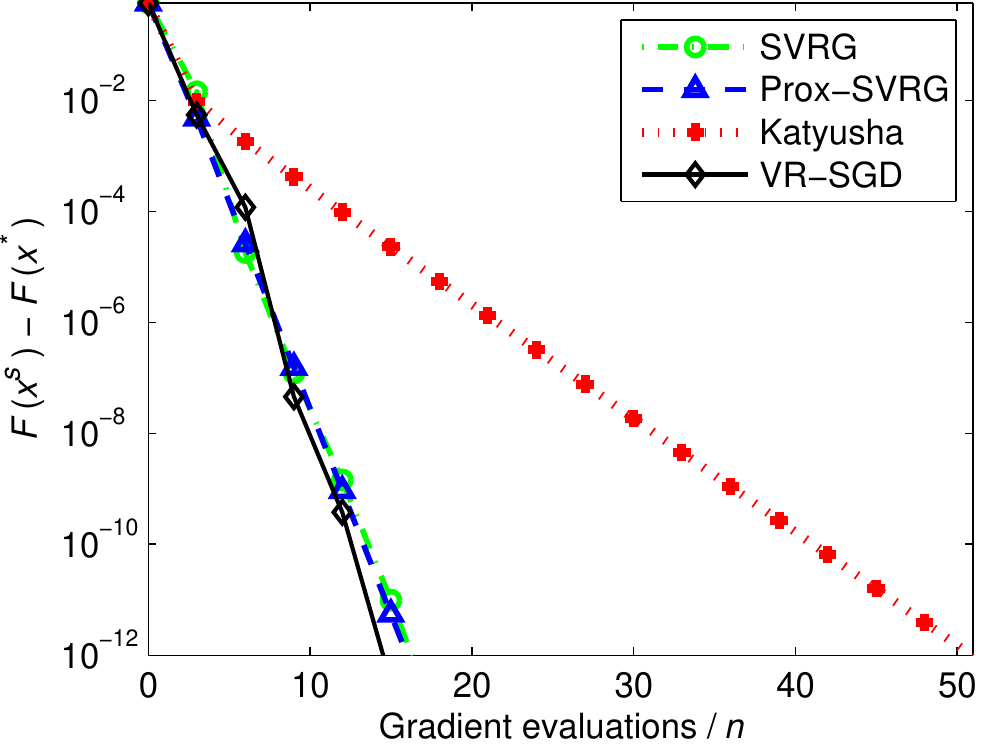}

\subfigure[Adult: $\lambda_{1}\!=\!10^{-4}$]{\includegraphics[width=0.246\columnwidth]{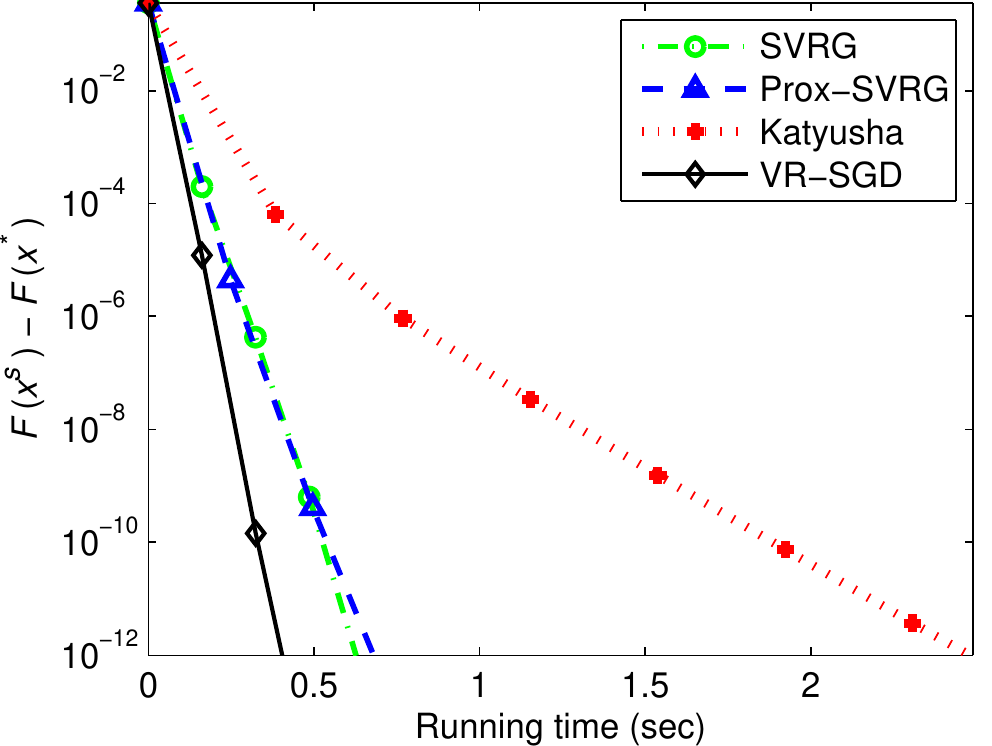}\label{figs1a}}
\subfigure[Protein: $\lambda_{1}\!=\!10^{-4}$]{\includegraphics[width=0.246\columnwidth]{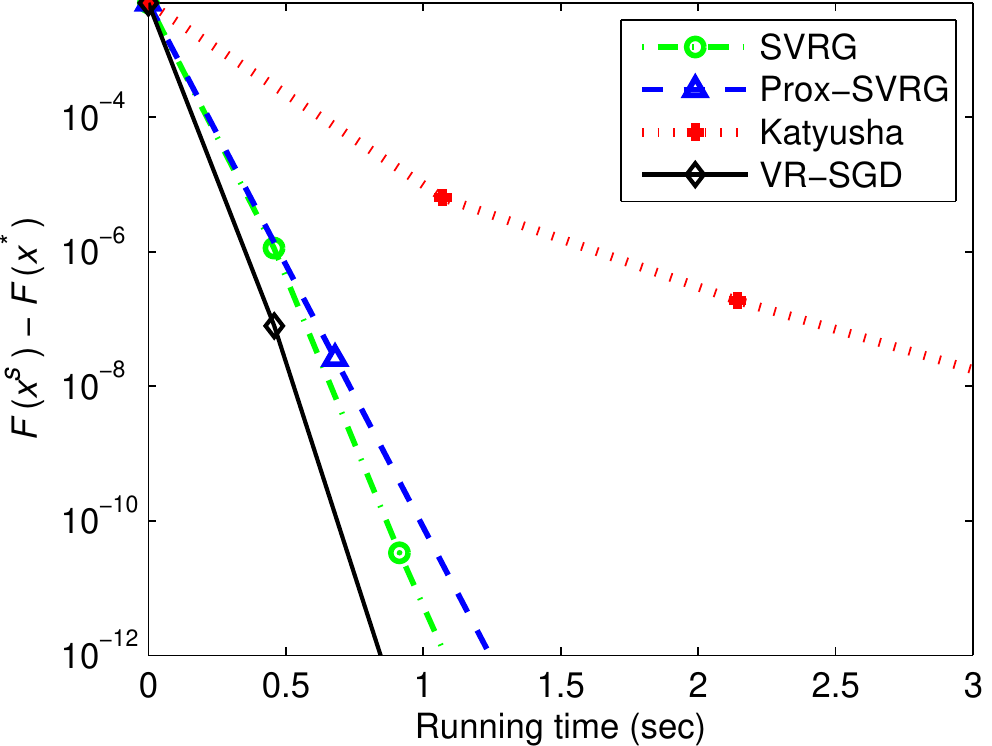}\label{figs1b}}
\subfigure[Covtype: $\lambda_{1}\!=\!10^{-4}$]{\includegraphics[width=0.246\columnwidth]{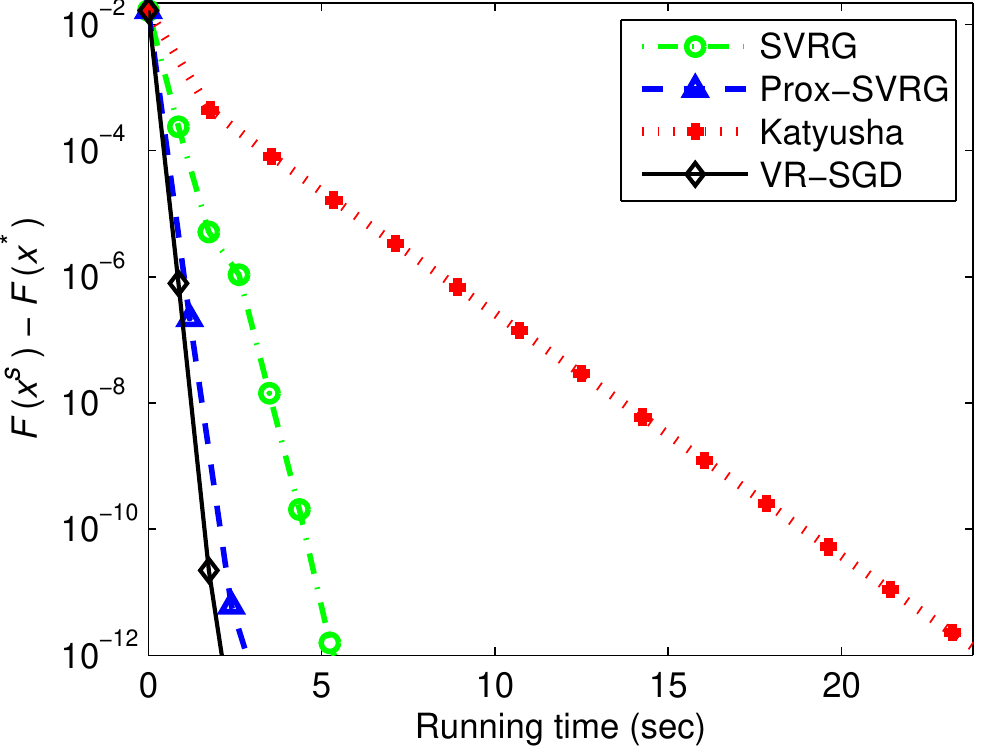}\label{figs1c}}
\subfigure[Sido0: $\lambda_{1}\!=\!5\!*\!10^{-3}$]{\includegraphics[width=0.246\columnwidth]{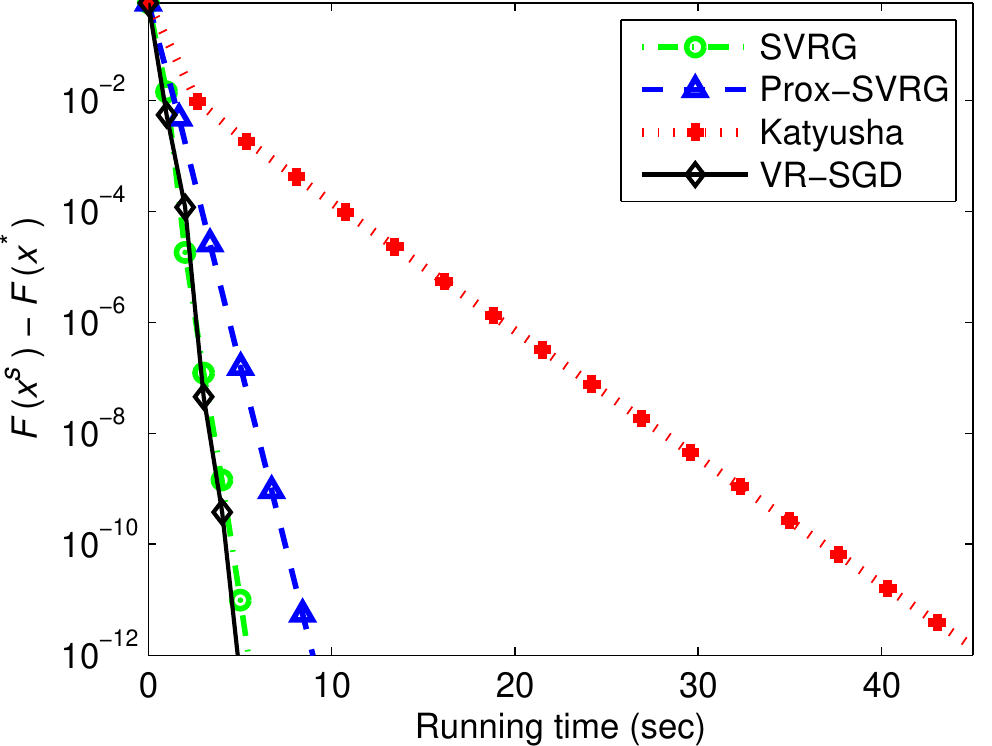}\label{figs1d}}
\vspace{1.6mm}

\includegraphics[width=0.246\columnwidth]{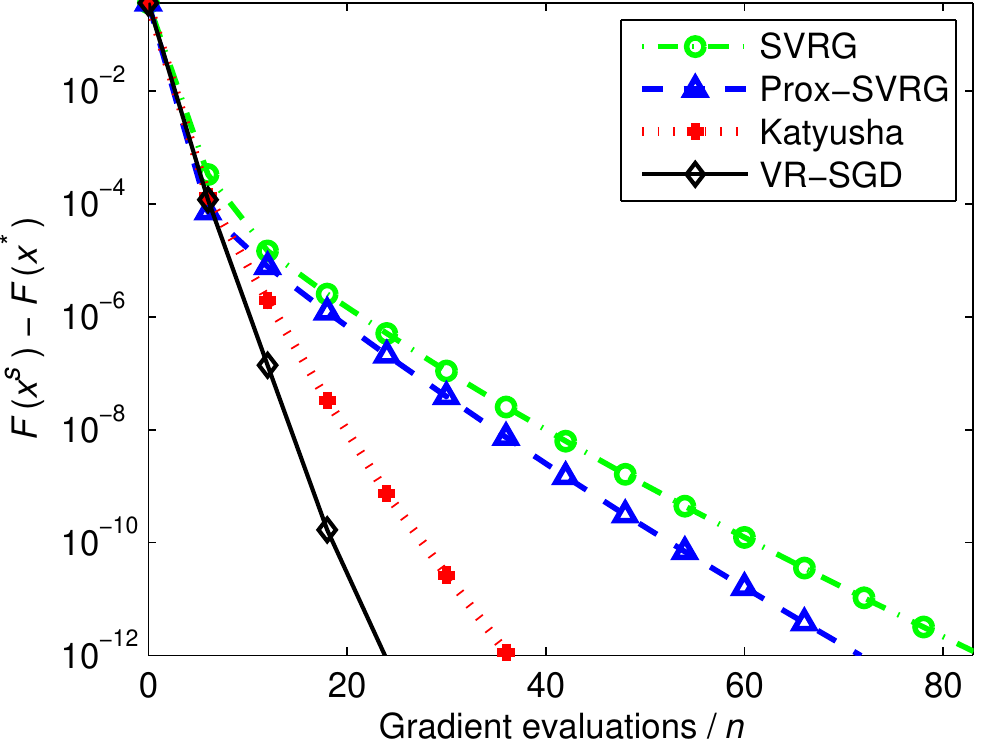}
\includegraphics[width=0.246\columnwidth]{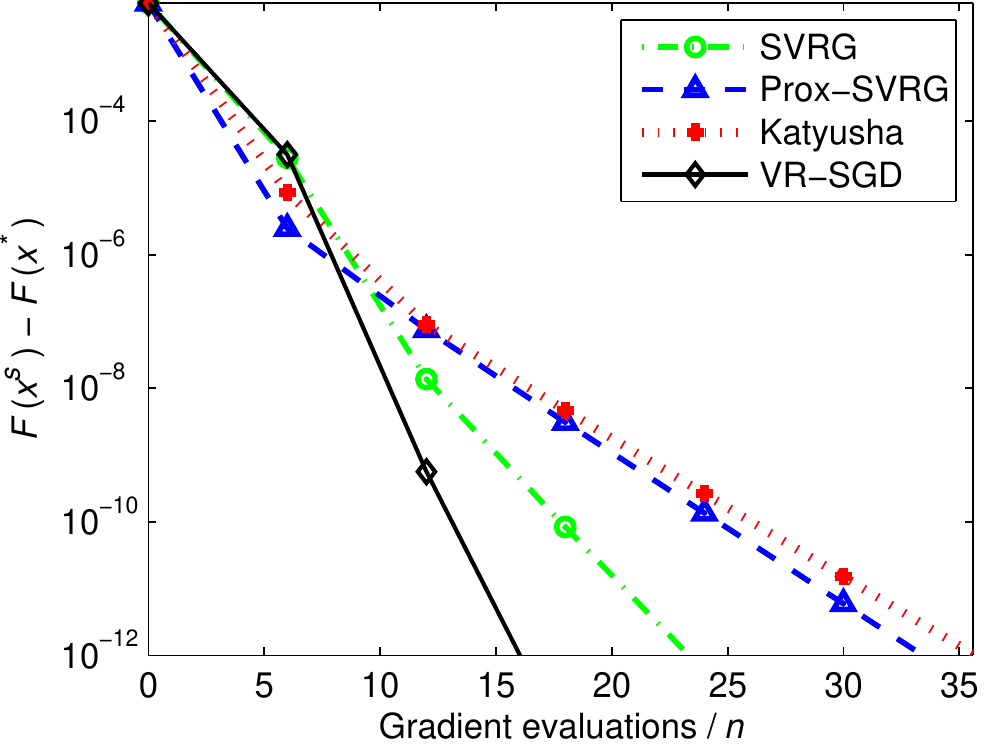}
\includegraphics[width=0.246\columnwidth]{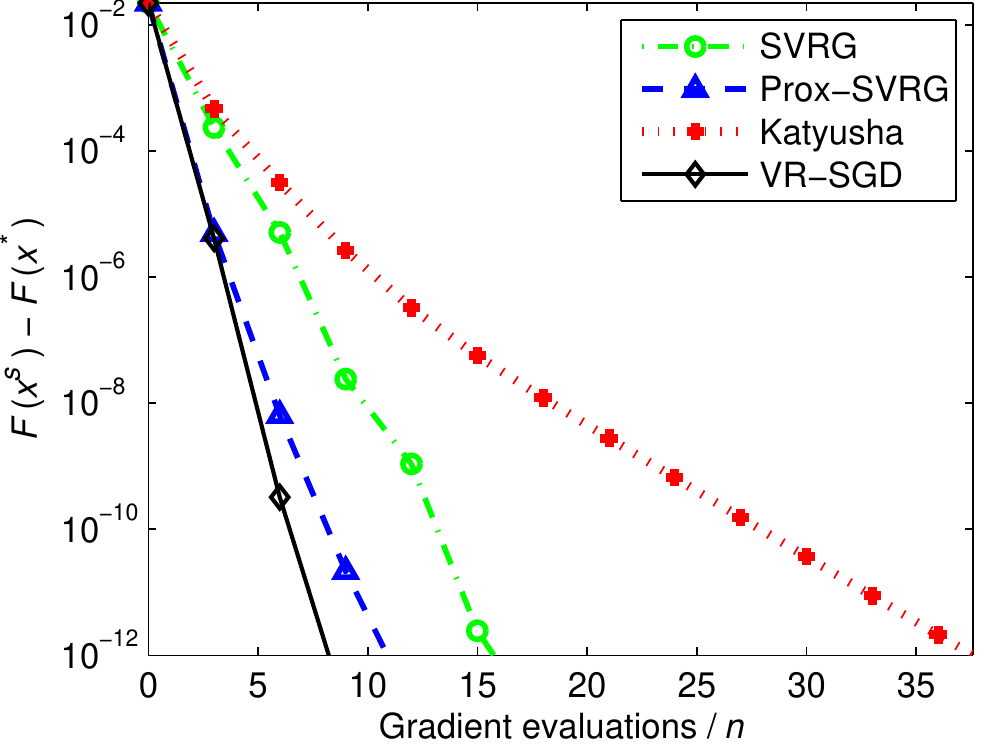}
\includegraphics[width=0.246\columnwidth]{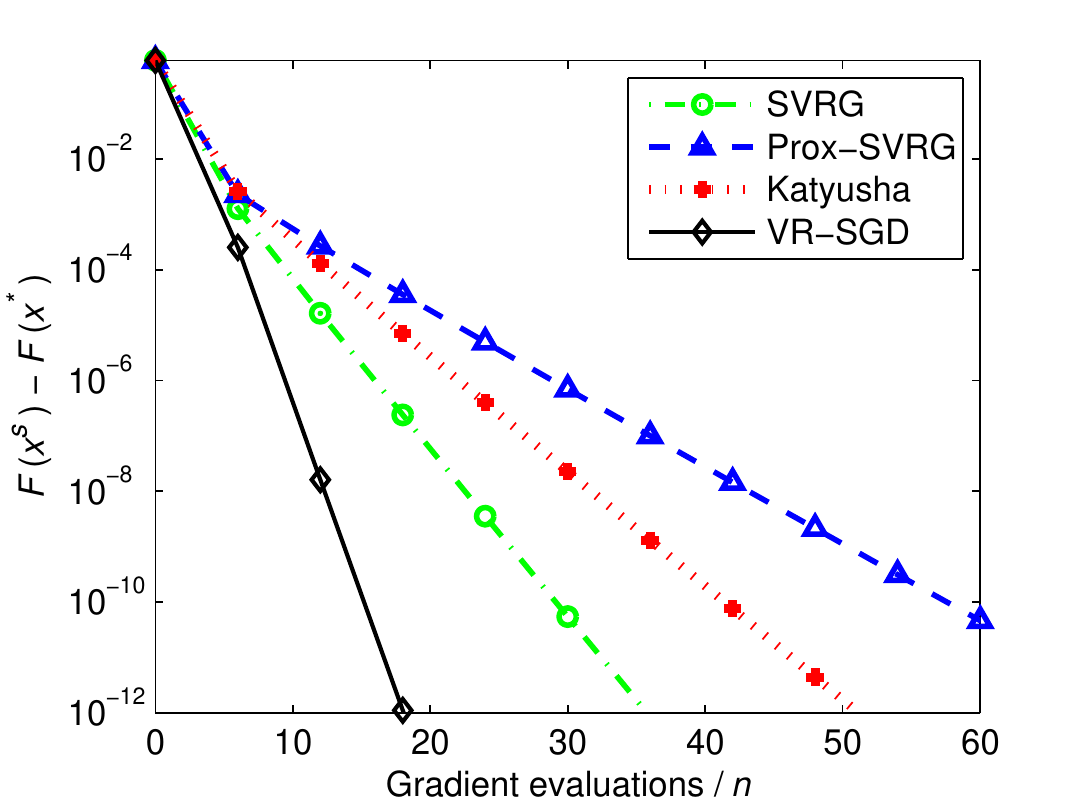}

\subfigure[Adult: $\lambda_{1}\!=\!10^{-5}$]{\includegraphics[width=0.246\columnwidth]{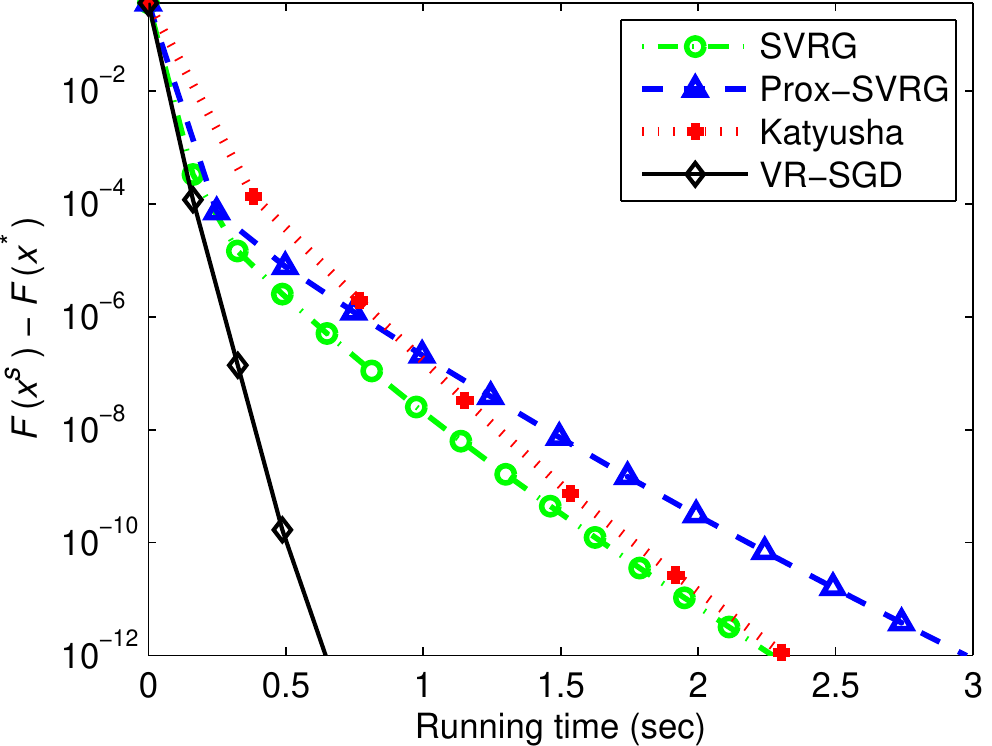}}
\subfigure[Protein: $\lambda_{1}\!=\!10^{-5}$]{\includegraphics[width=0.246\columnwidth]{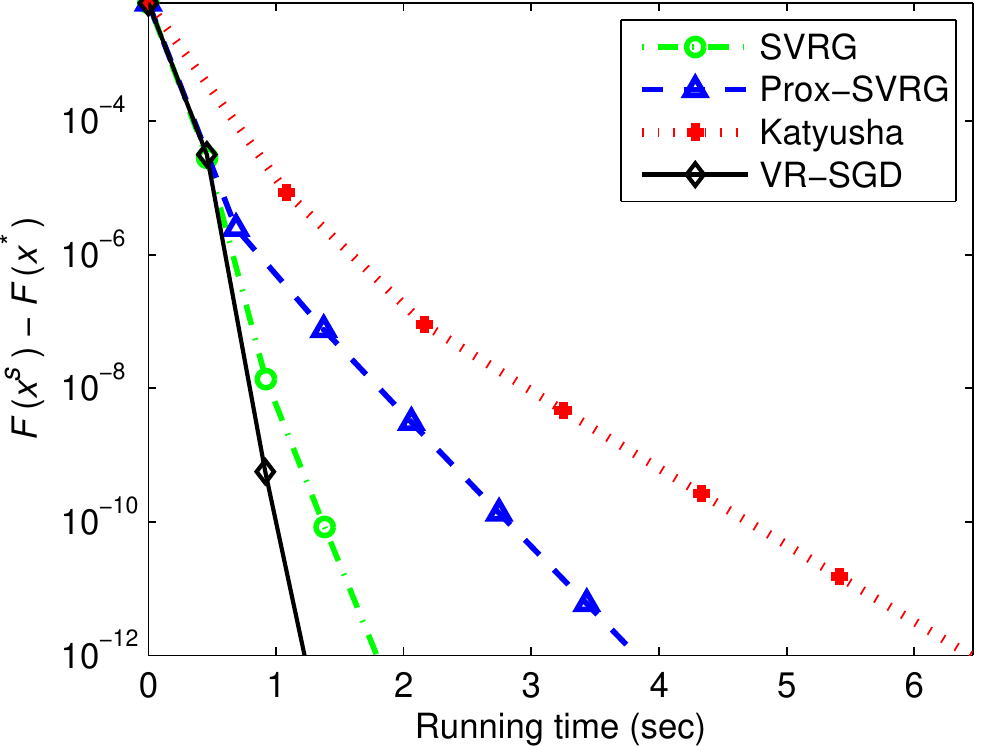}}
\subfigure[Covtype: $\lambda_{1}\!=\!10^{-5}$]{\includegraphics[width=0.246\columnwidth]{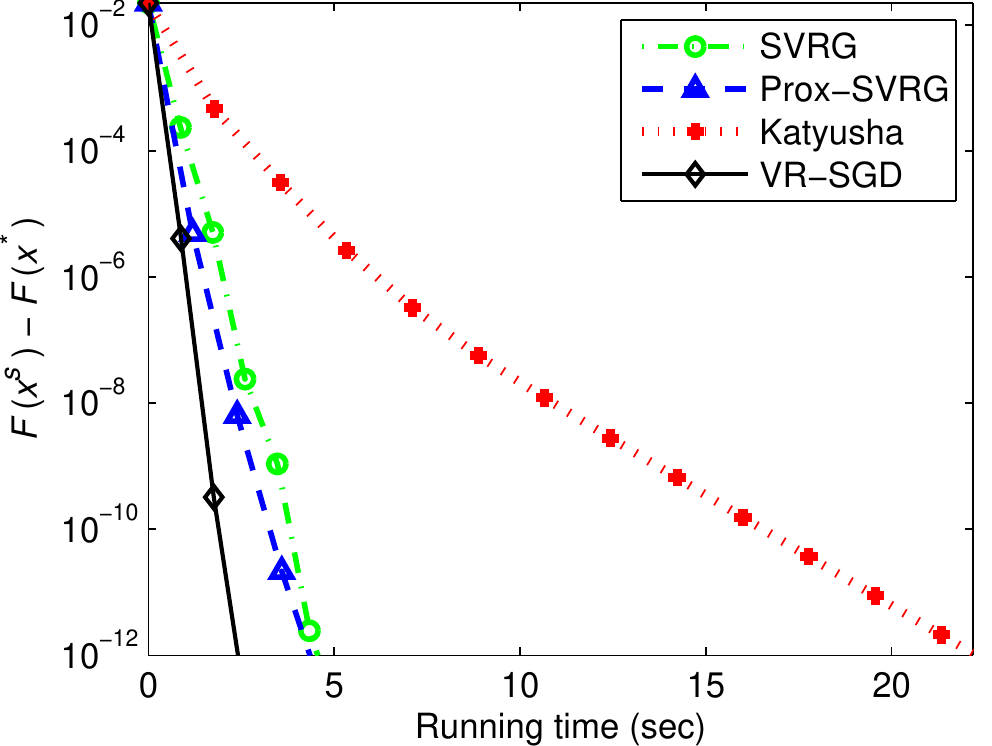}}
\subfigure[Sido0: $\lambda_{1}\!=\!10^{-4}$]{\includegraphics[width=0.246\columnwidth]{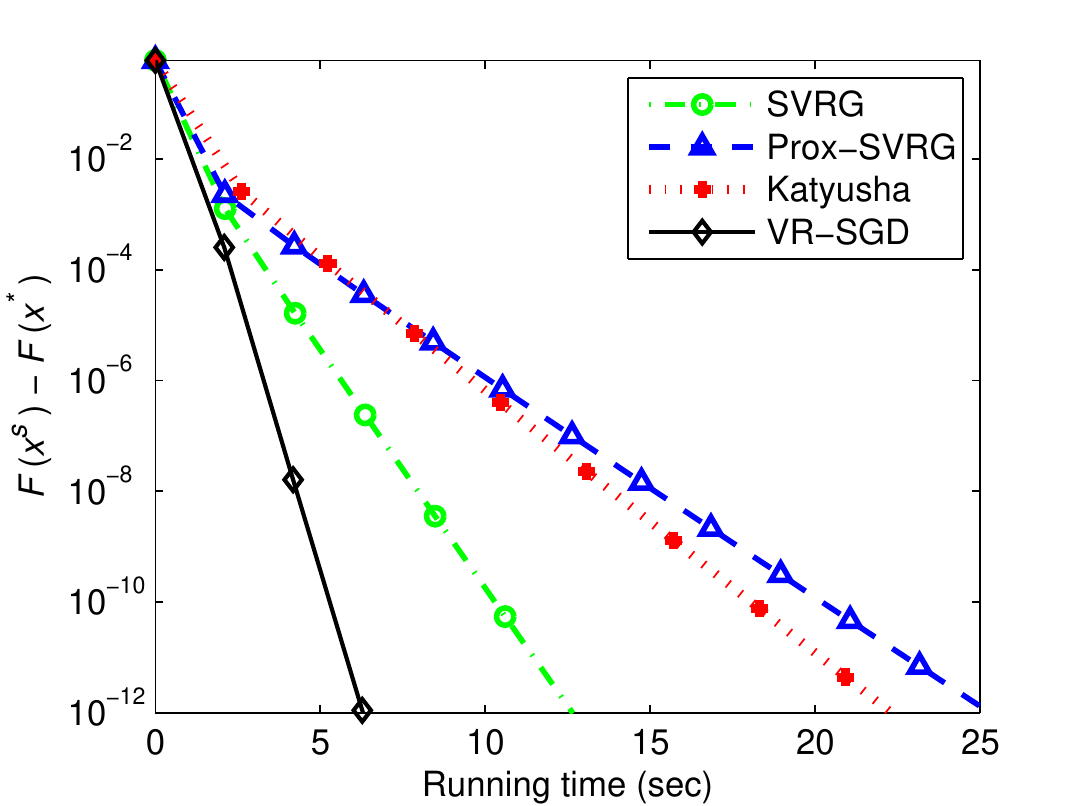}\label{figs1h}}
\vspace{1.6mm}

\includegraphics[width=0.246\columnwidth]{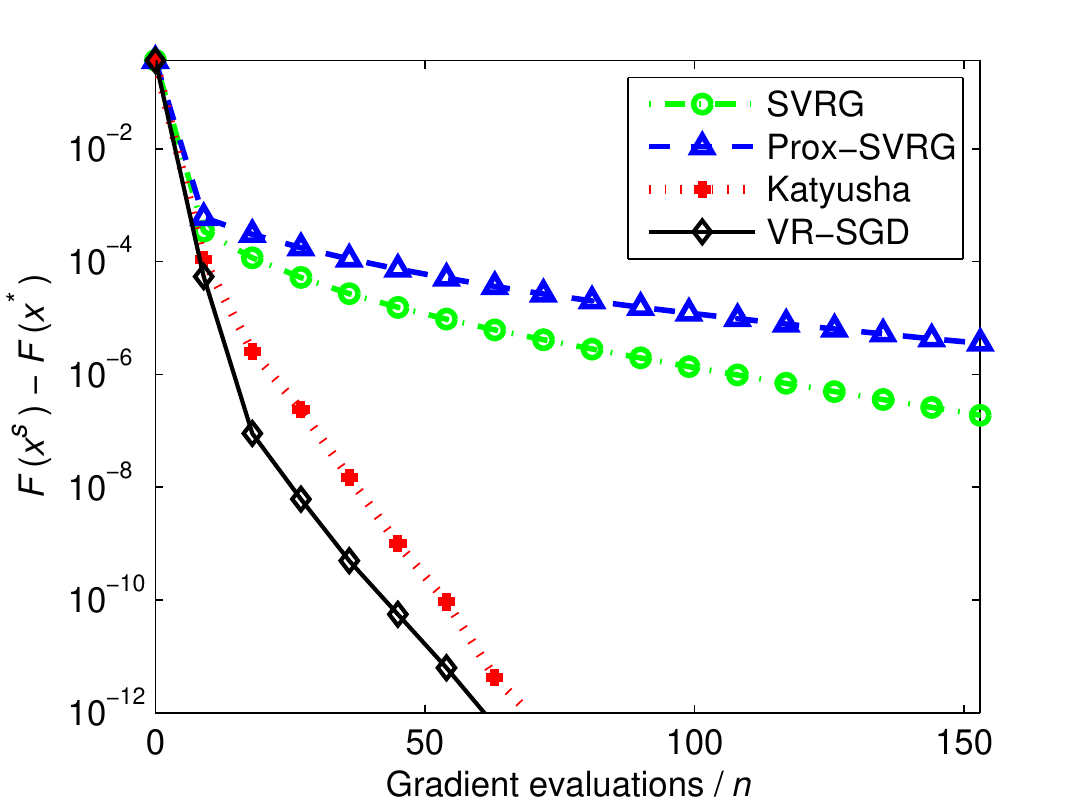}
\includegraphics[width=0.246\columnwidth]{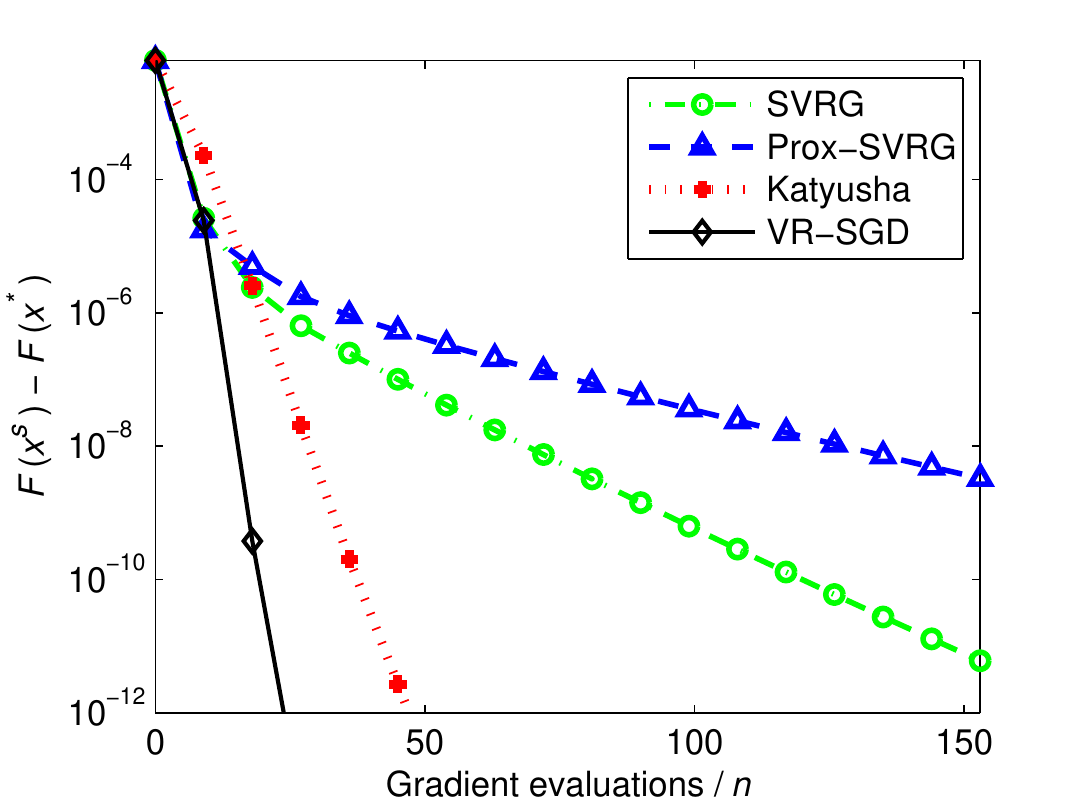}
\includegraphics[width=0.246\columnwidth]{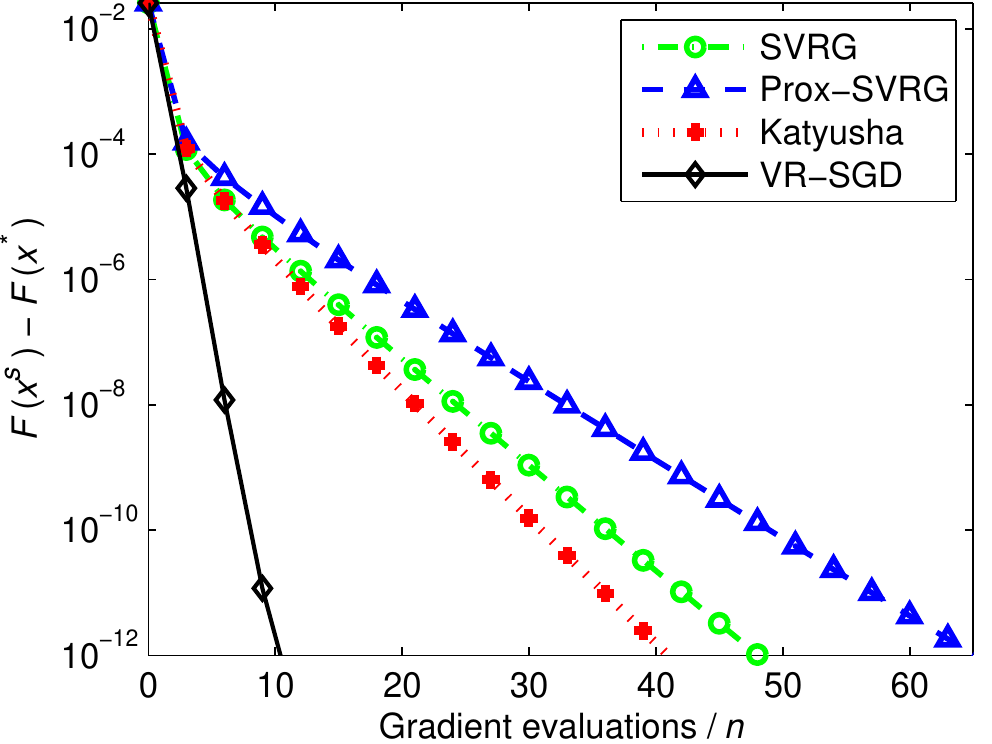}
\includegraphics[width=0.246\columnwidth]{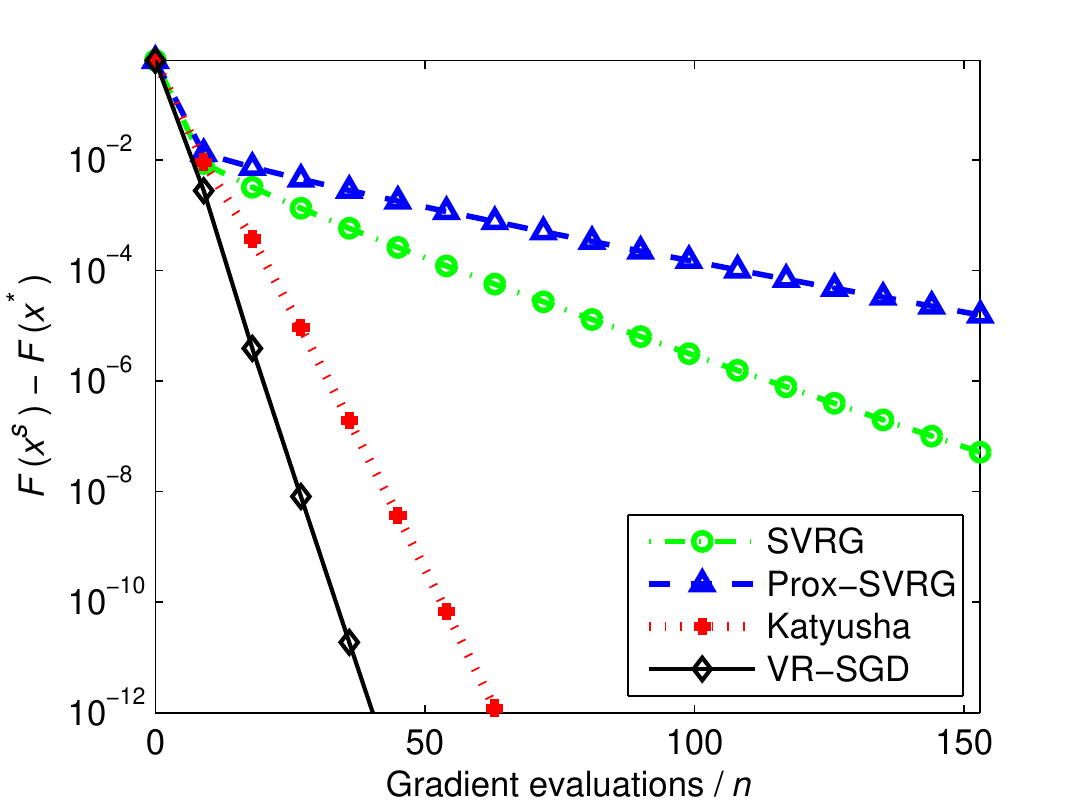}

\subfigure[Adult: $\lambda_{1}\!=\!10^{-6}$]{\includegraphics[width=0.246\columnwidth]{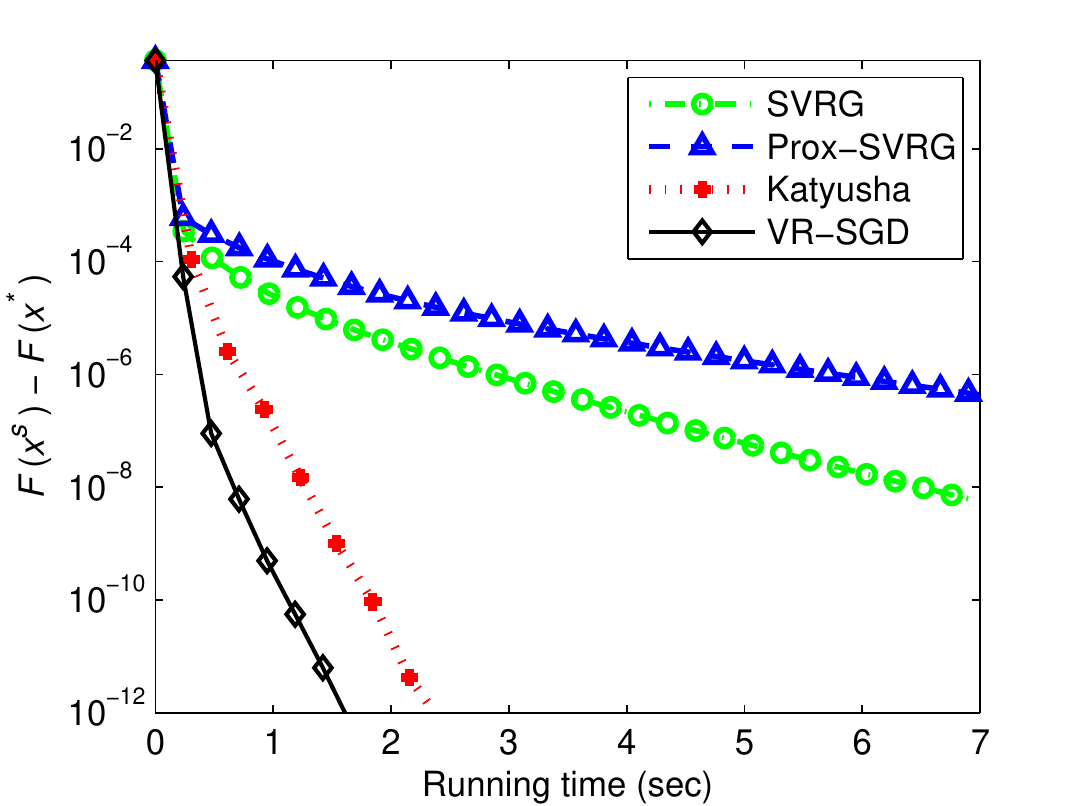}\label{figs1i}}
\subfigure[Protein: $\lambda_{1}\!=\!10^{-6}$]{\includegraphics[width=0.246\columnwidth]{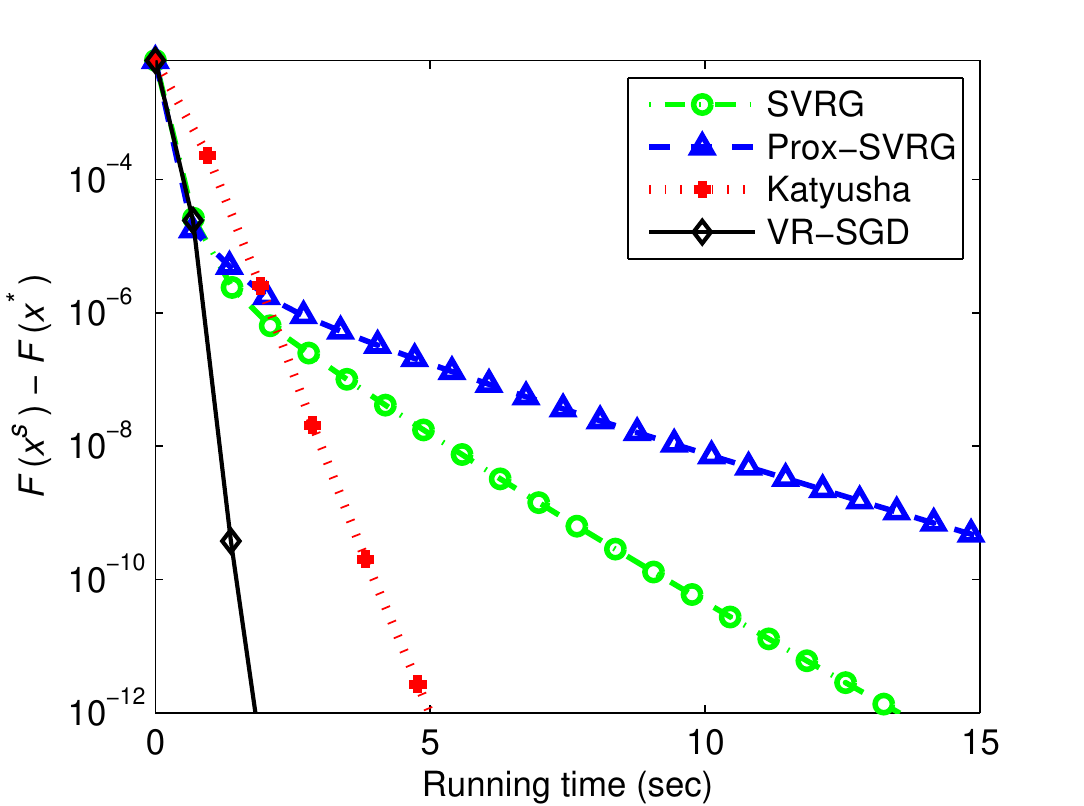}}
\subfigure[Covtype: $\lambda_{1}\!=\!10^{-6}$]{\includegraphics[width=0.246\columnwidth]{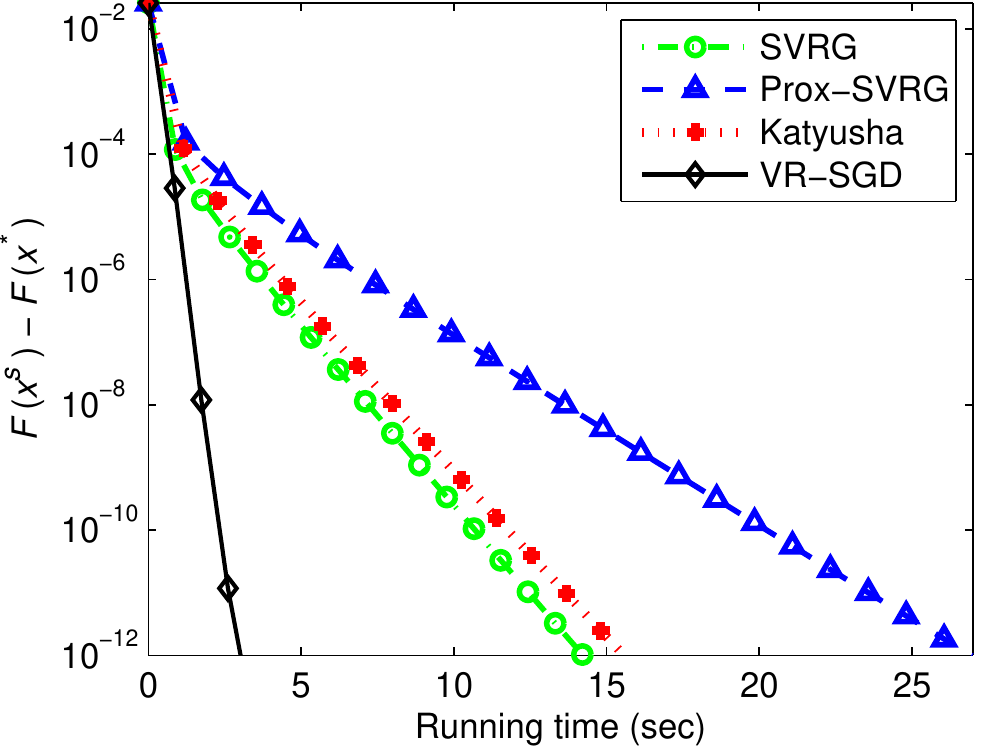}}
\subfigure[Sido0: $\lambda_{1}\!=\!10^{-5}$]{\includegraphics[width=0.246\columnwidth]{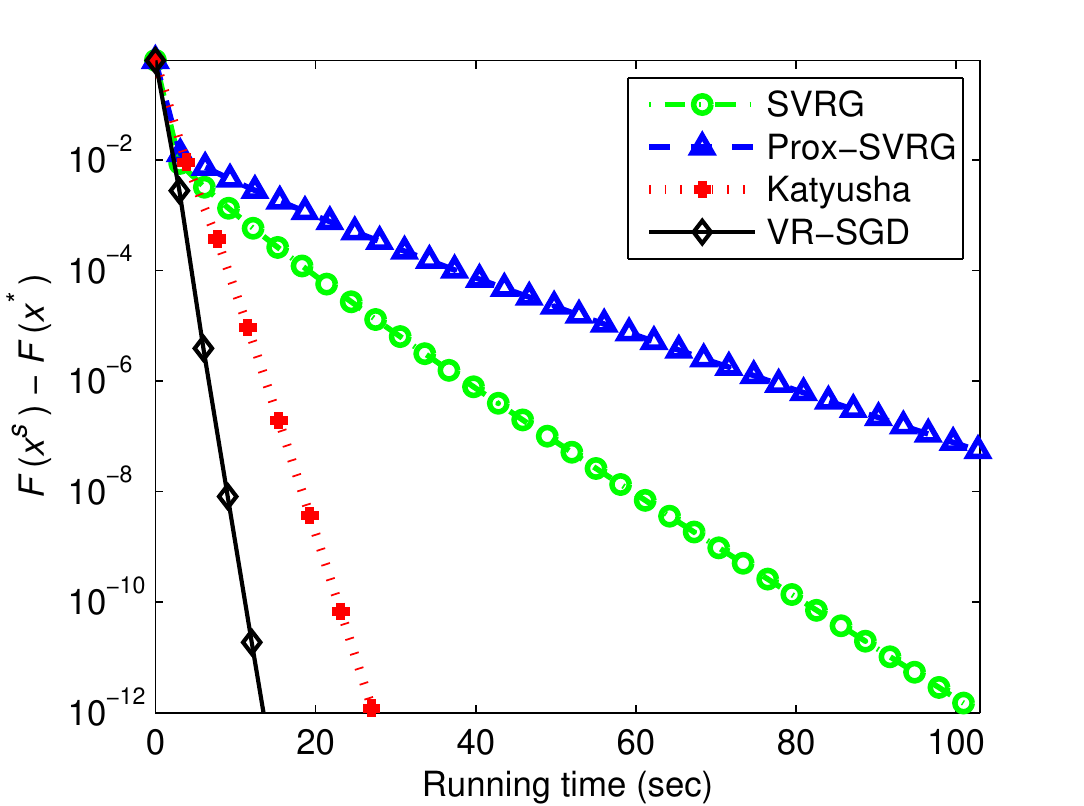}\label{figs1l}}
\vspace{-2.6mm}
\caption{Comparison of SVRG~\cite{johnson:svrg}, Prox-SVRG~\cite{xiao:prox-svrg}, Katyusha~\cite{zhu:Katyusha}, and VR-SGD for solving $\ell_{2}$-norm regularized logistic regression problems (i.e., $\lambda_{2}=0$). In each plot, the vertical axis shows the objective value minus the minimum, and the horizontal axis is the number of effective passes (top) or running time (bottom).}
\label{figs1}
\end{figure}

\begin{figure}[!th]
\centering
\includegraphics[width=0.246\columnwidth]{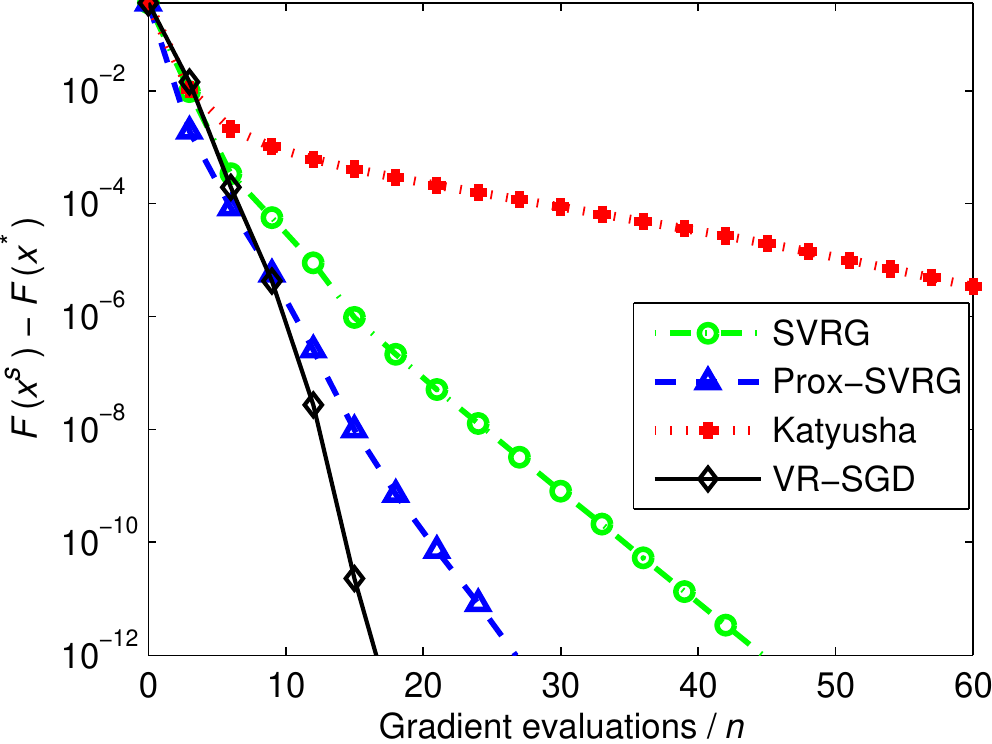}
\includegraphics[width=0.246\columnwidth]{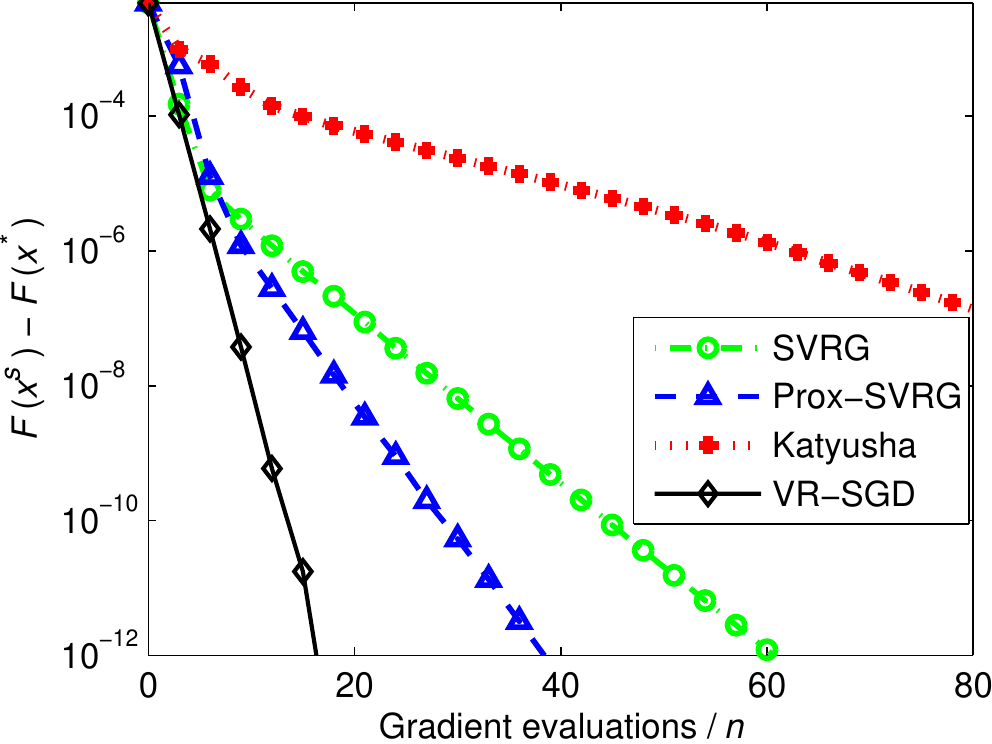}
\includegraphics[width=0.246\columnwidth]{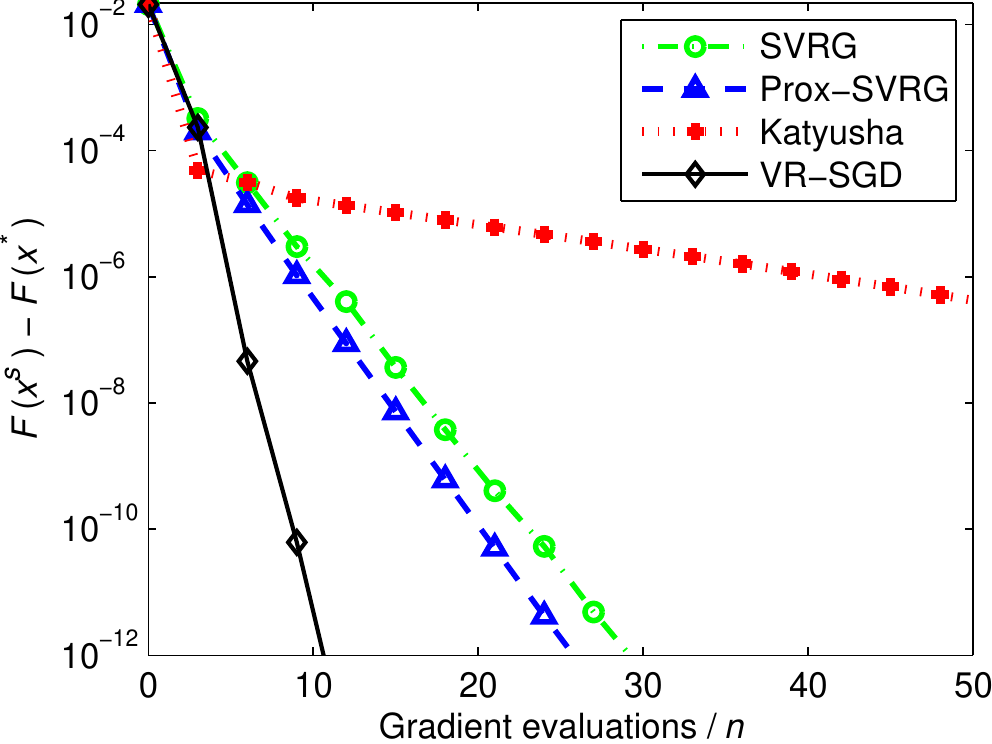}
\includegraphics[width=0.246\columnwidth]{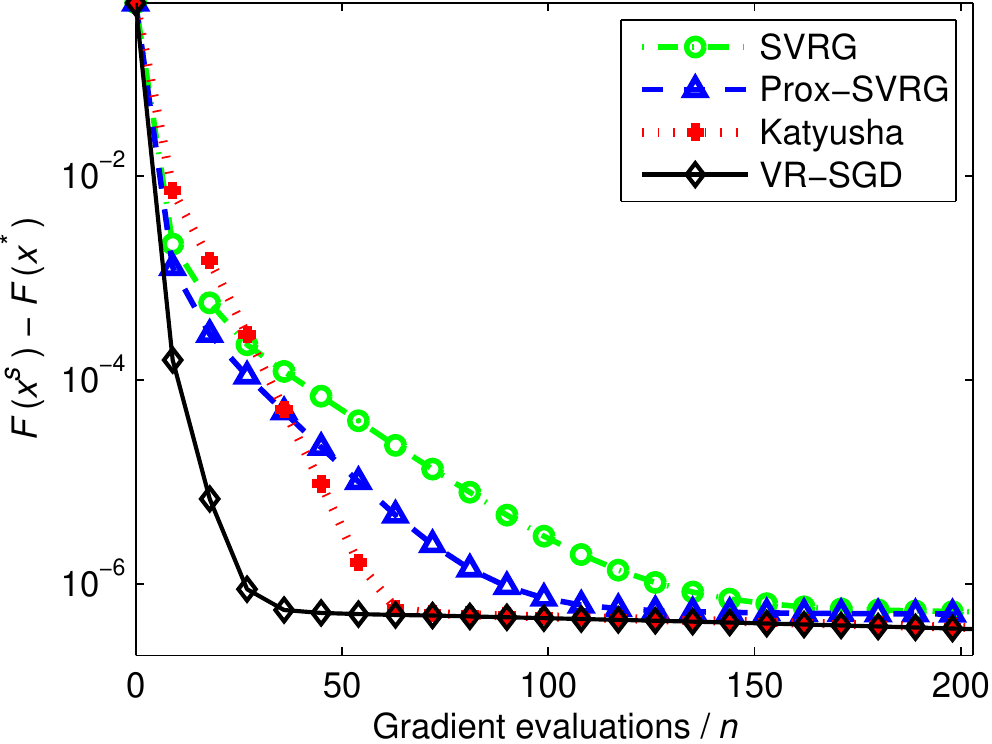}

\subfigure[$\lambda_{2}=10^{-4}$]{\includegraphics[width=0.246\columnwidth]{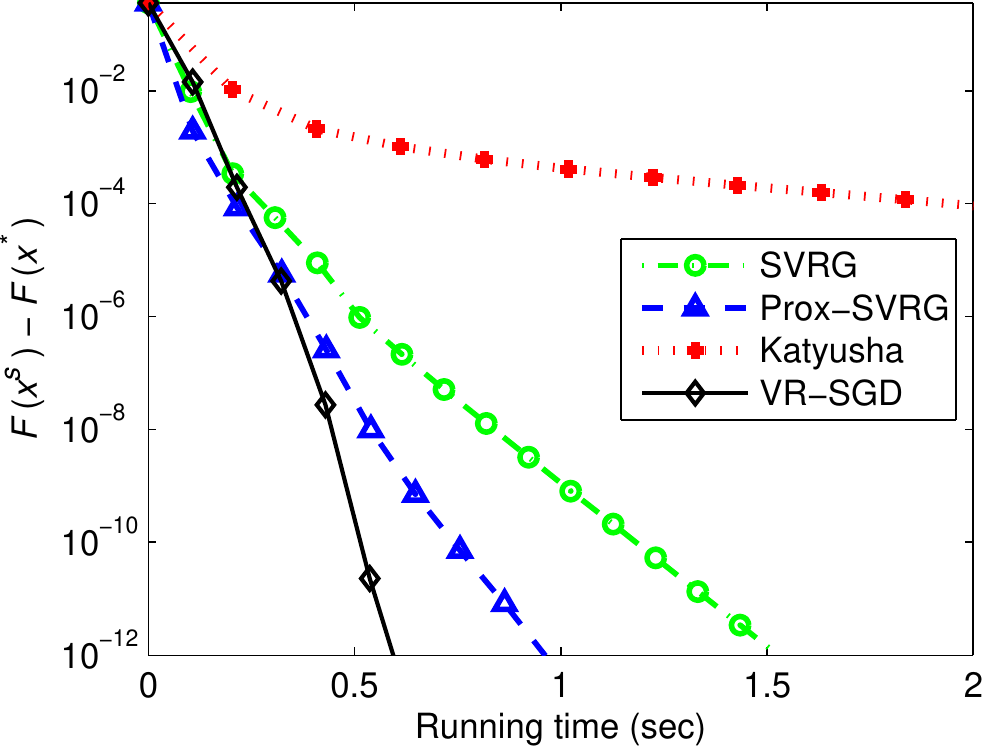}\:\includegraphics[width=0.246\columnwidth]{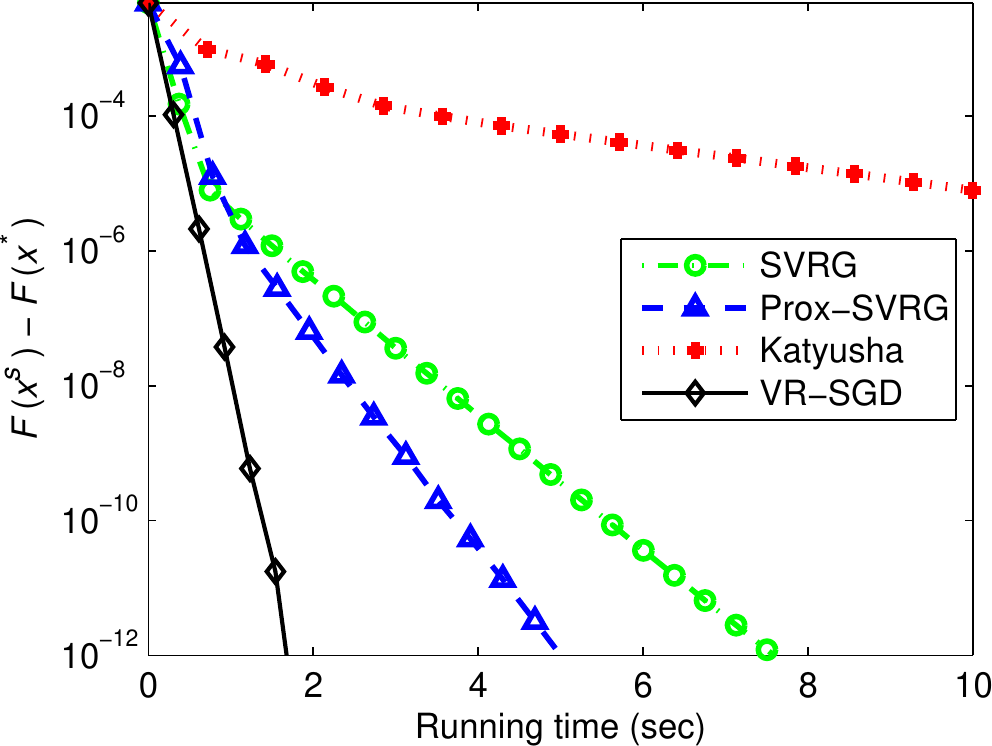}\:\includegraphics[width=0.246\columnwidth]{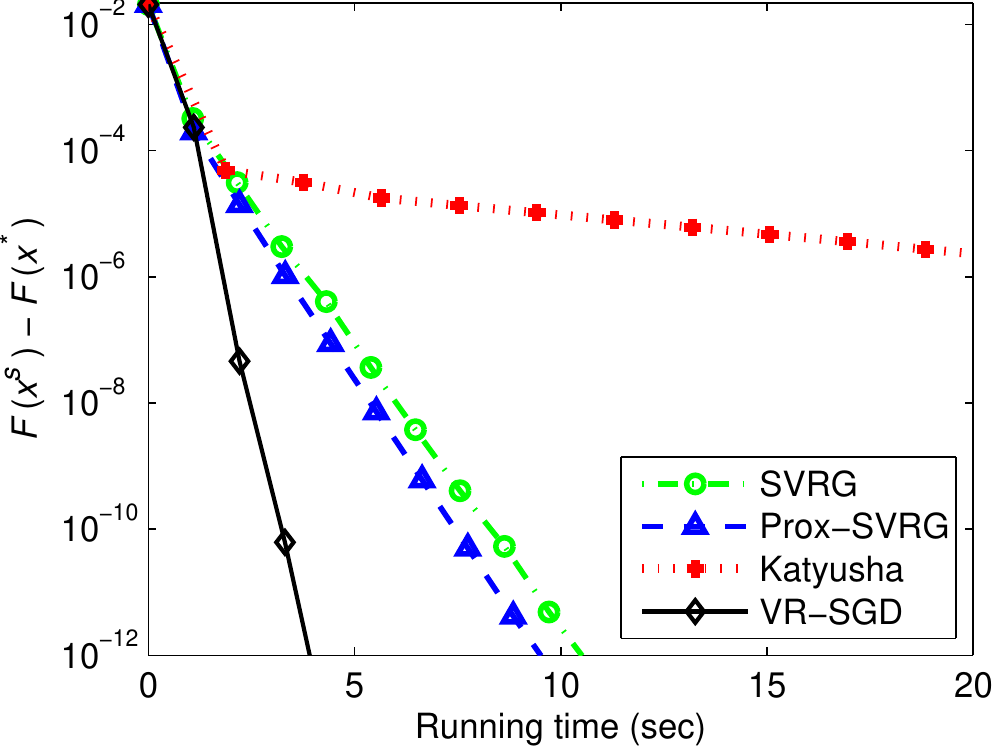}\label{figs2a}}
\subfigure[$\lambda_{2}=10^{-3}$]{\includegraphics[width=0.246\columnwidth]{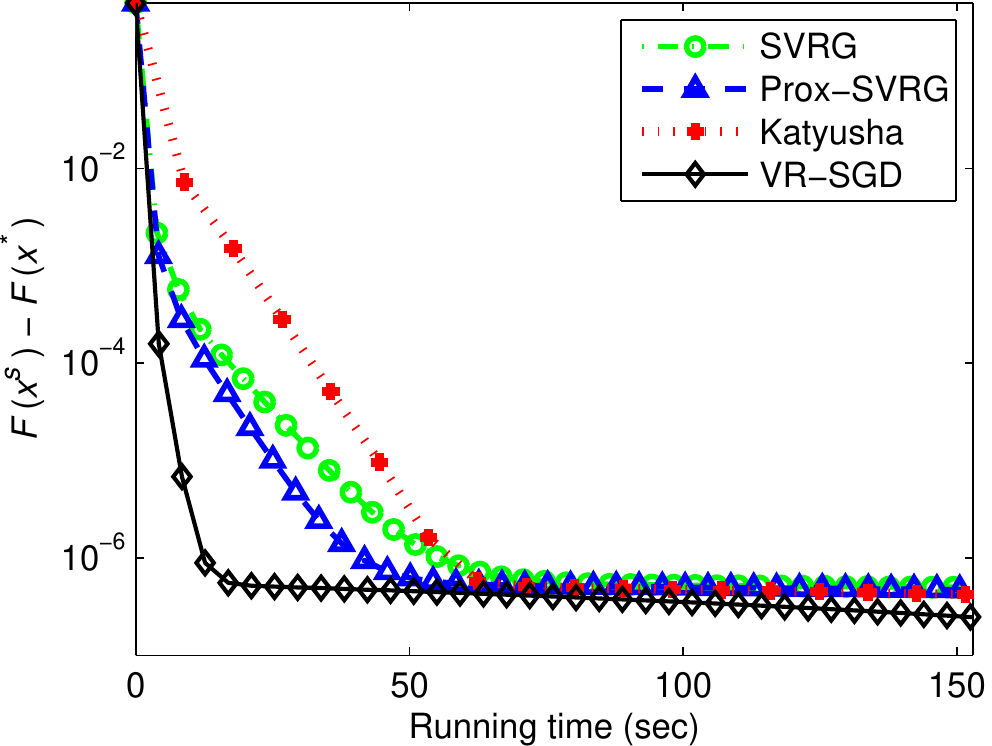}\label{figs2b}}
\vspace{1.6mm}

\includegraphics[width=0.246\columnwidth]{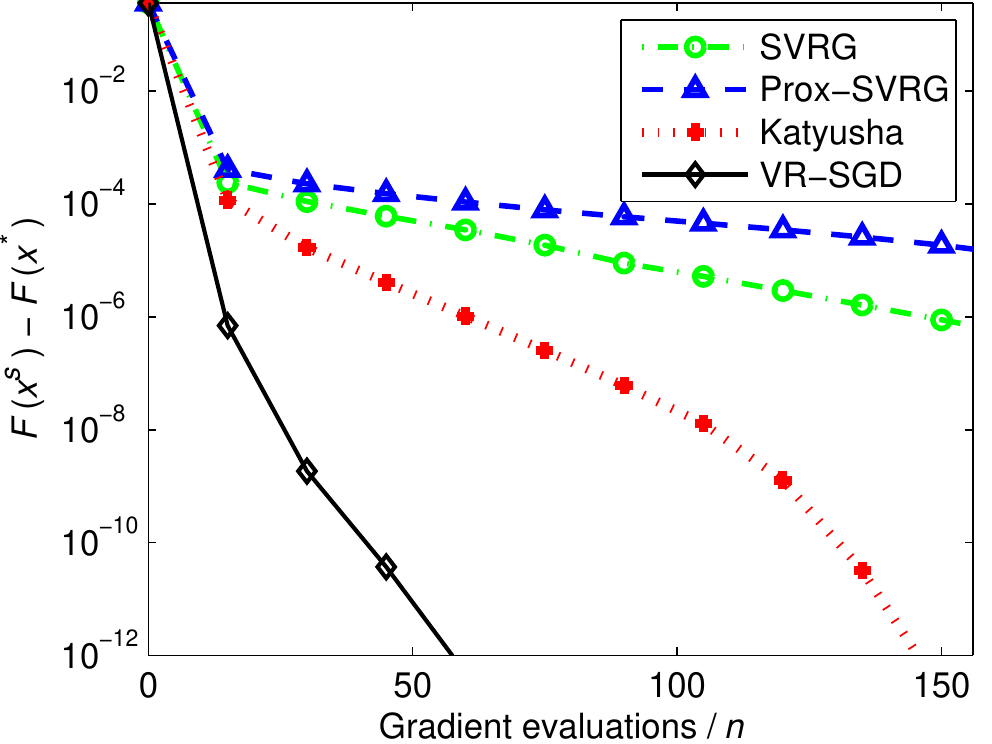}
\includegraphics[width=0.246\columnwidth]{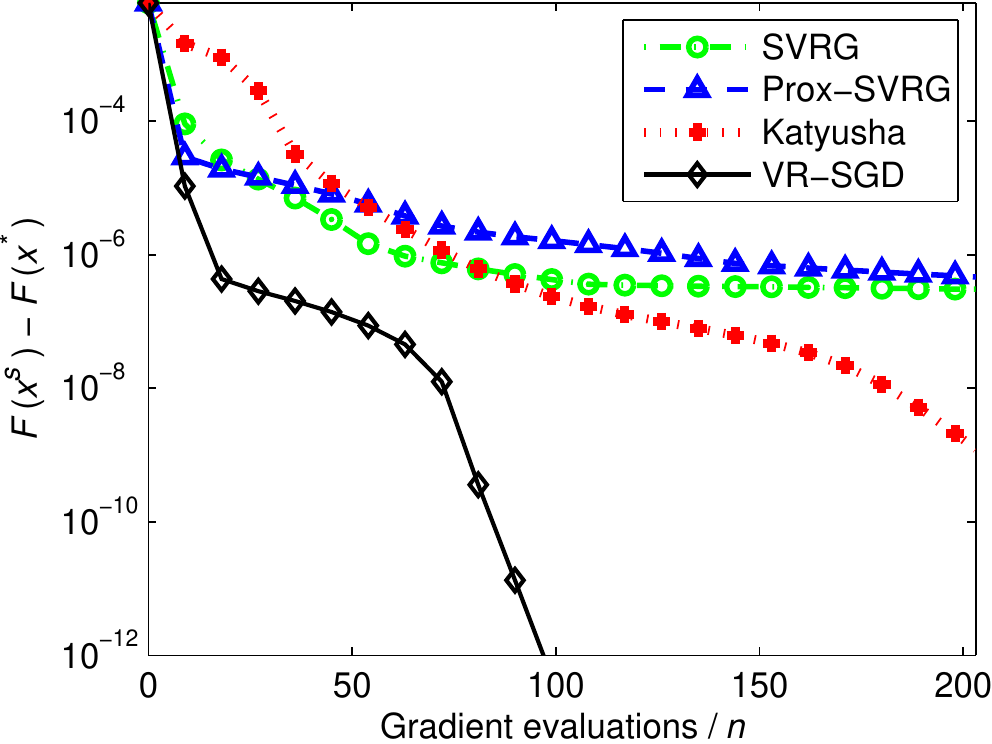}
\includegraphics[width=0.246\columnwidth]{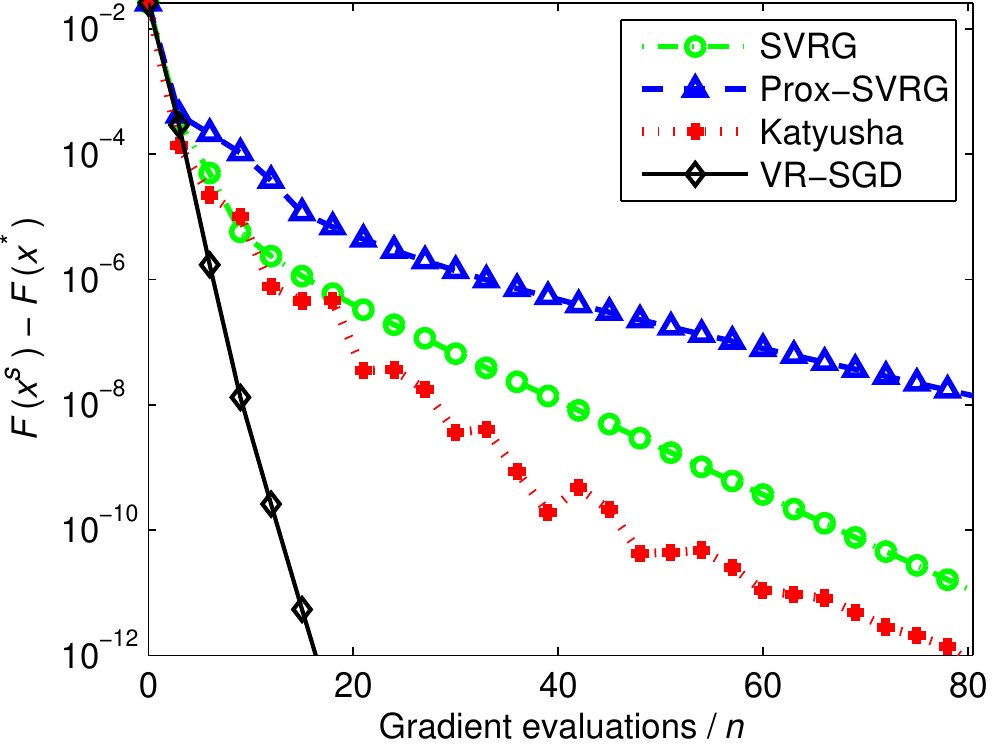}
\includegraphics[width=0.246\columnwidth]{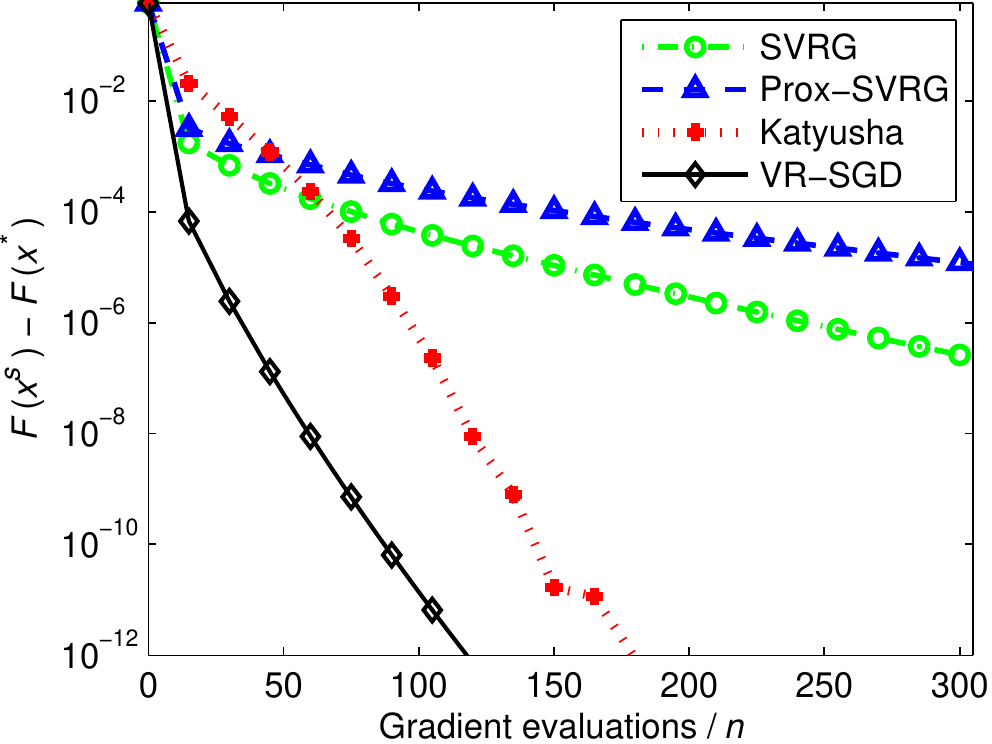}

\subfigure[$\lambda_{2}=10^{-5}$]{\includegraphics[width=0.246\columnwidth]{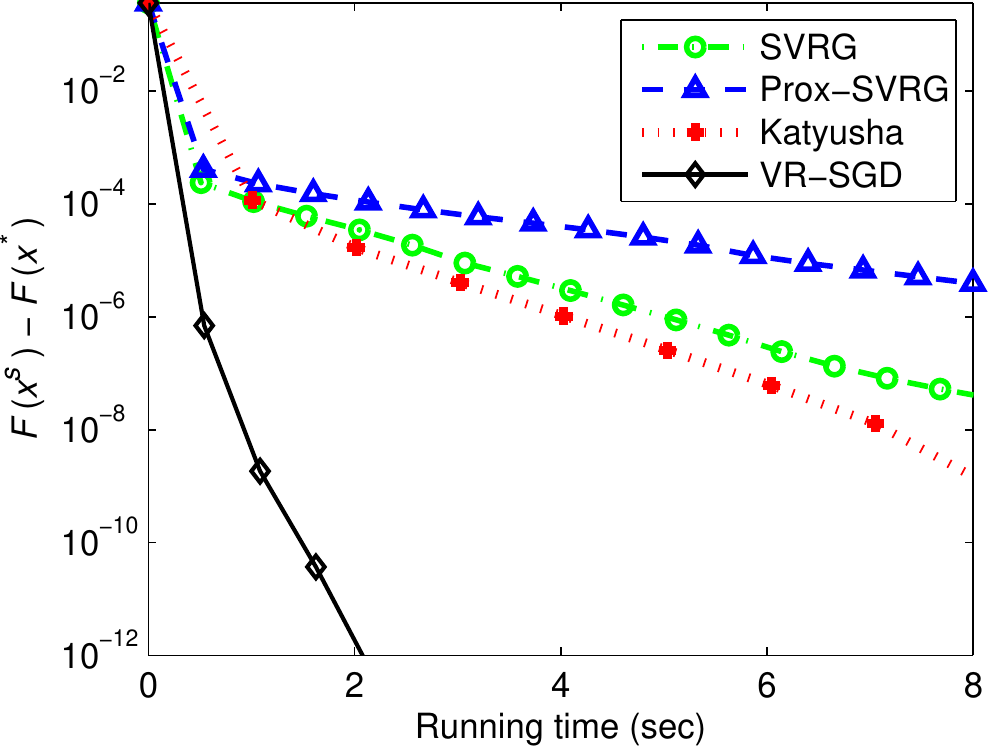}\:\includegraphics[width=0.246\columnwidth]{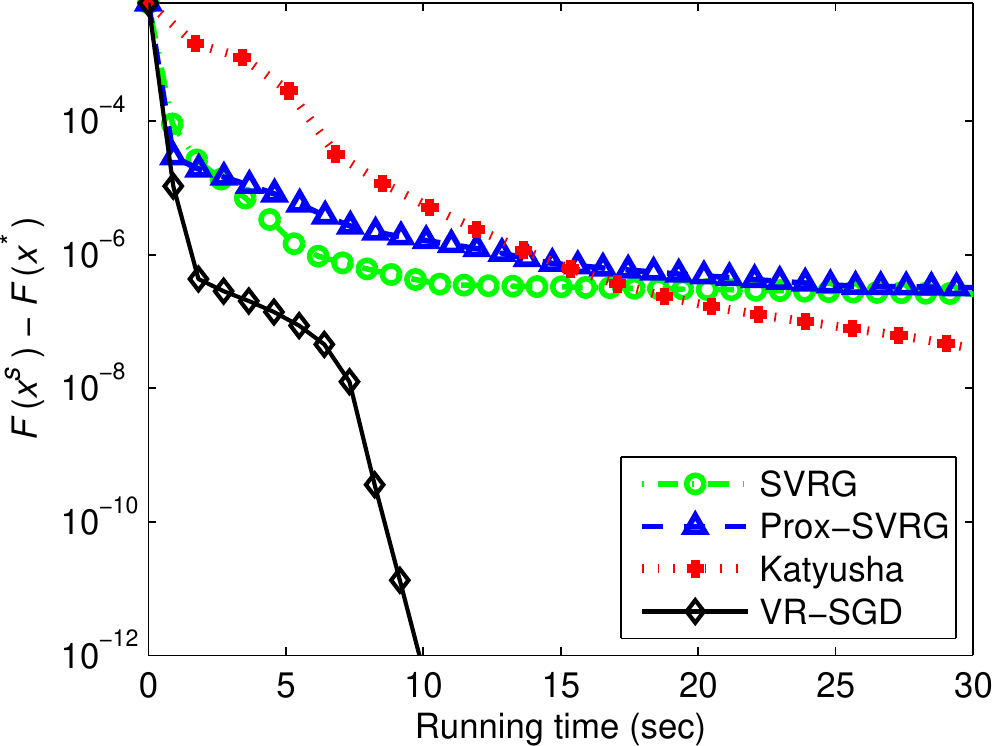}\:\includegraphics[width=0.246\columnwidth]{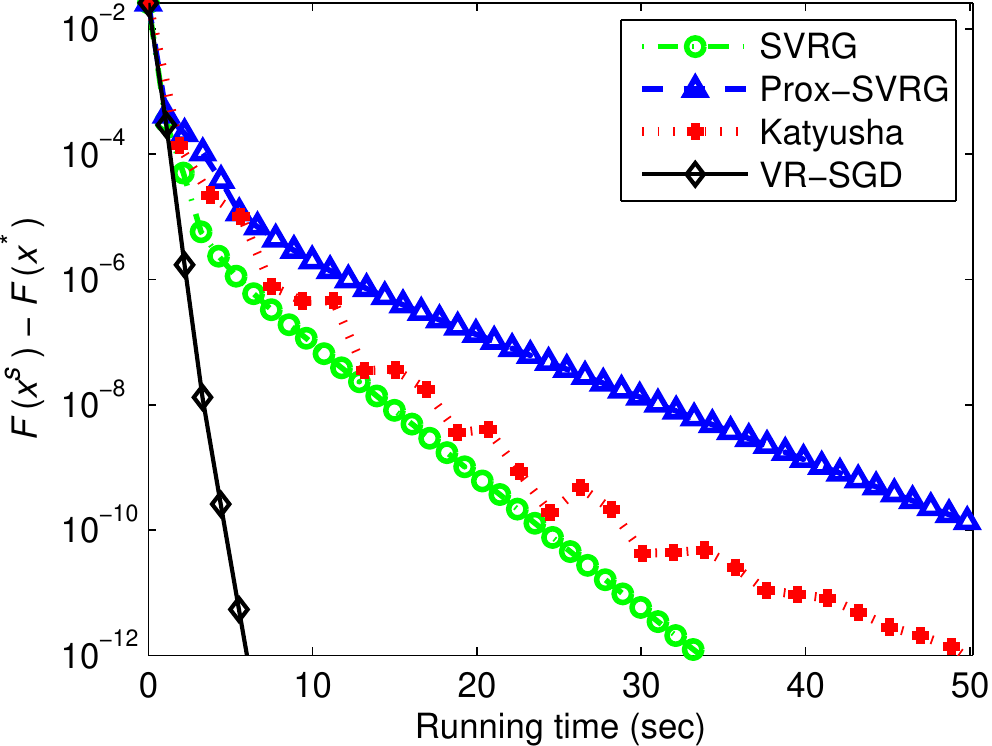}}
\subfigure[$\lambda_{2}=10^{-4}$]{\includegraphics[width=0.246\columnwidth]{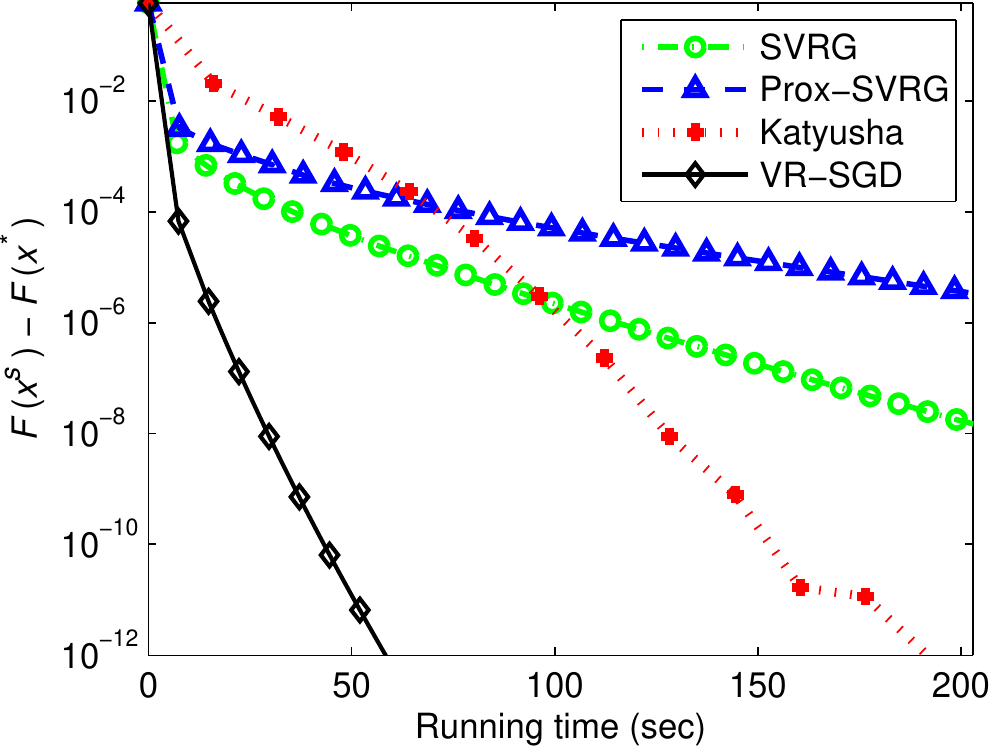}\label{figs2d}}
\vspace{1.6mm}

\includegraphics[width=0.246\columnwidth]{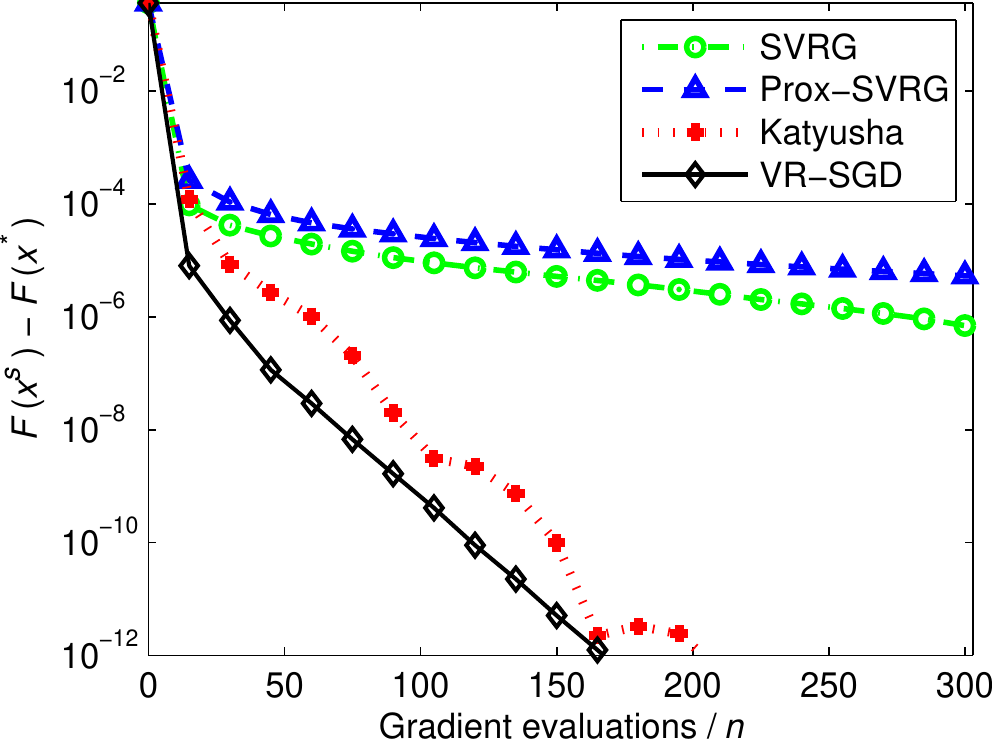}
\includegraphics[width=0.246\columnwidth]{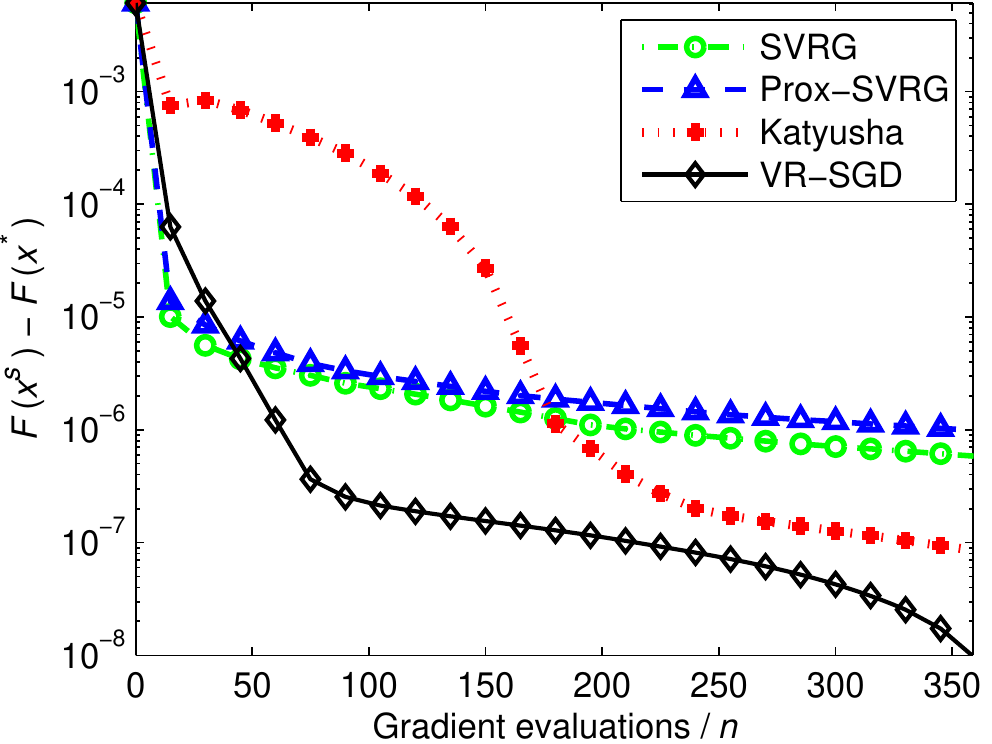}
\includegraphics[width=0.246\columnwidth]{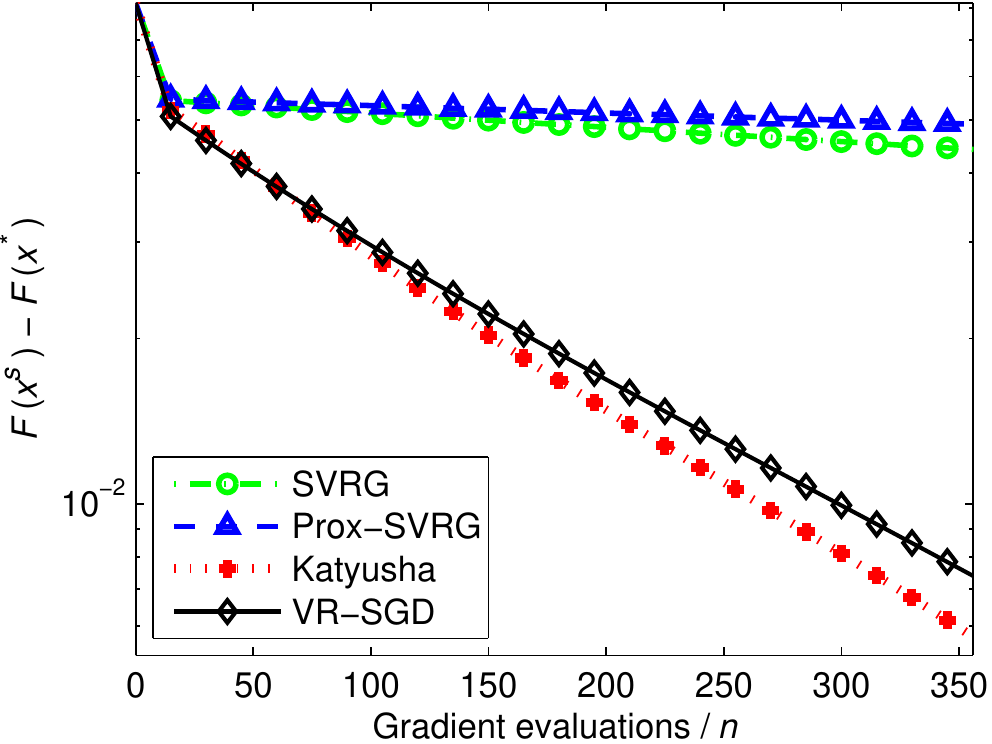}
\includegraphics[width=0.246\columnwidth]{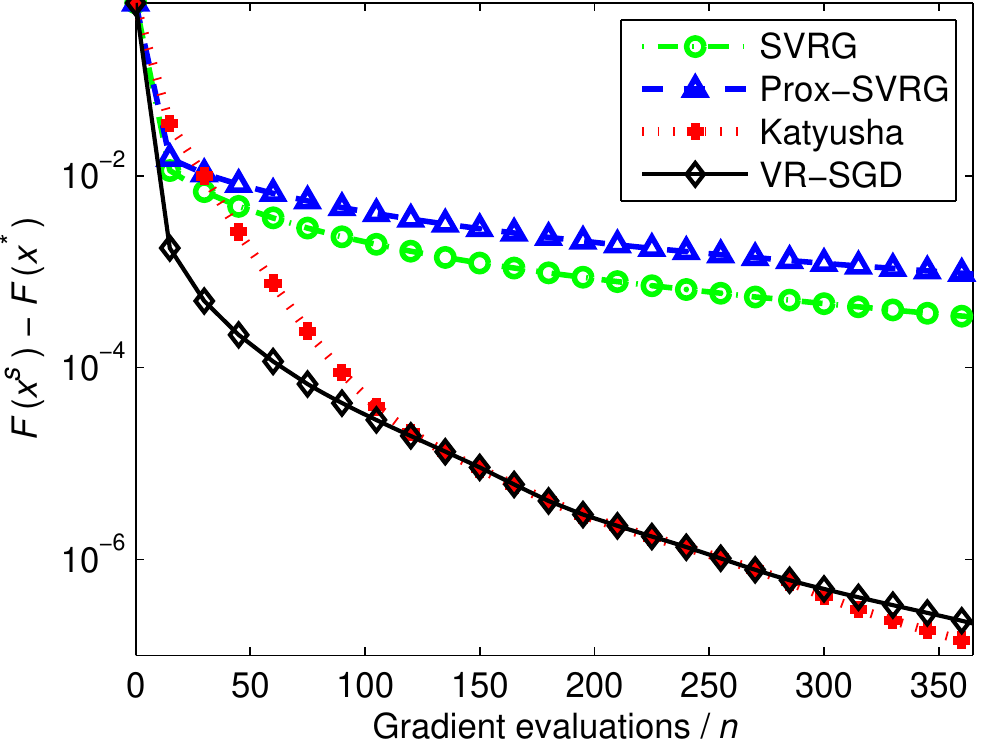}

\subfigure[$\lambda_{2}=10^{-6}$]{\includegraphics[width=0.246\columnwidth]{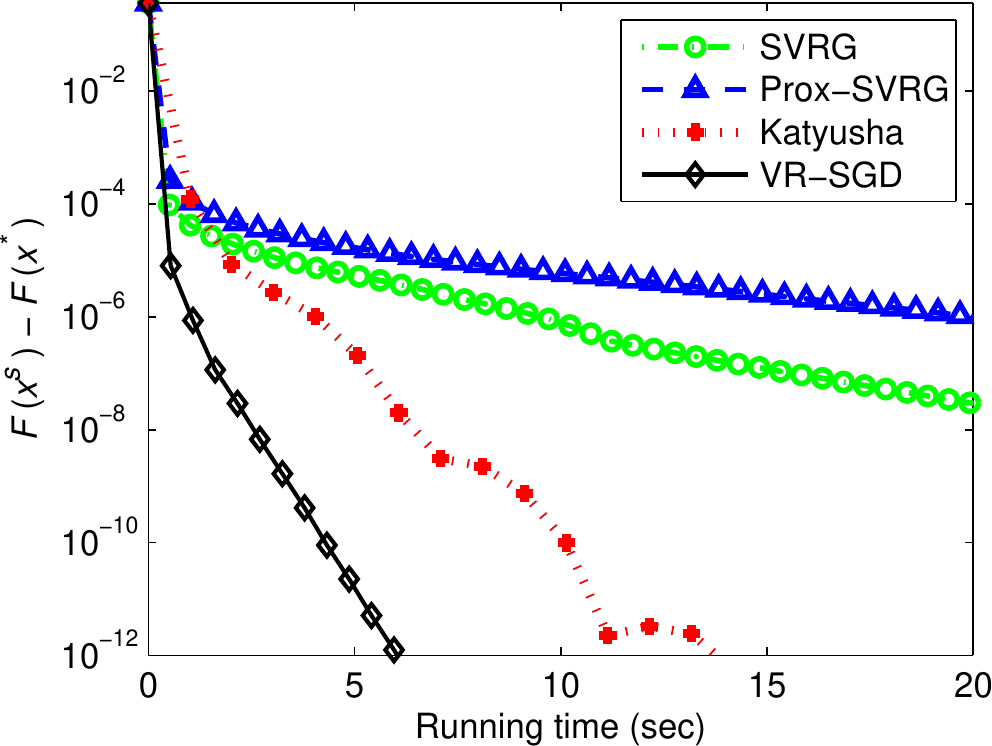}\,\includegraphics[width=0.246\columnwidth]{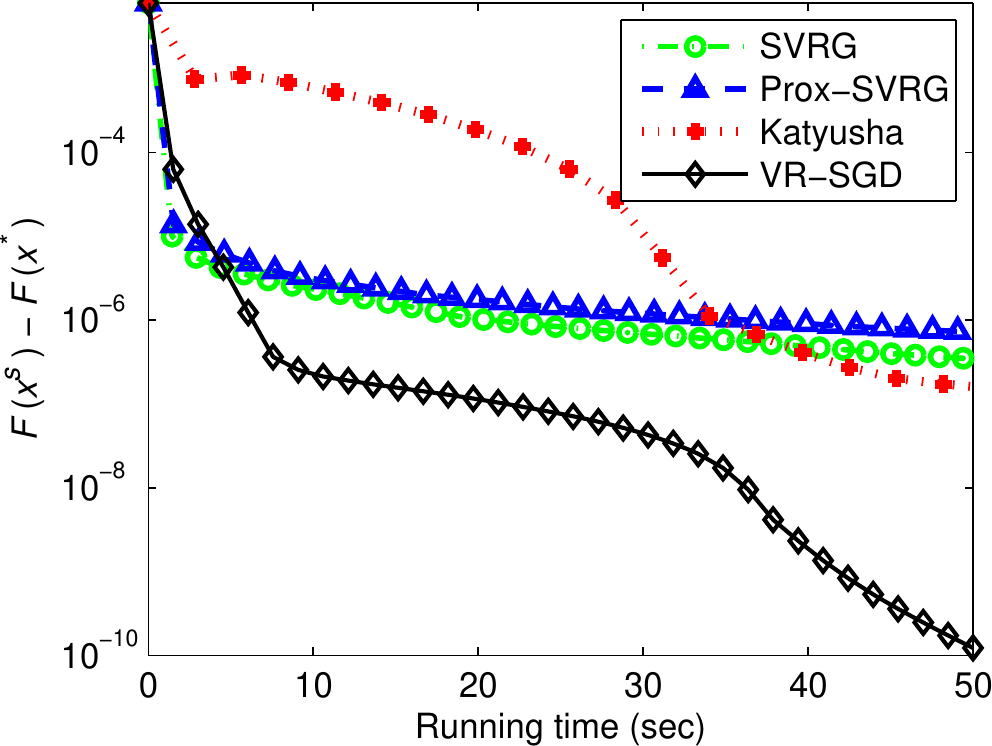}\,\includegraphics[width=0.246\columnwidth]{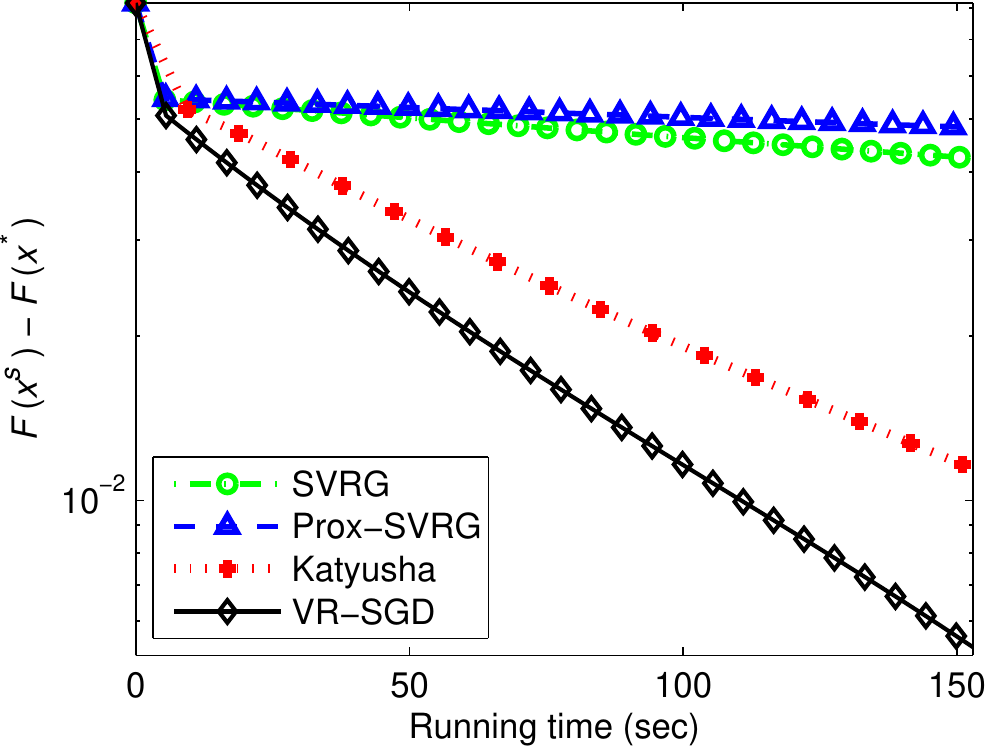}\label{figs2e}}
\subfigure[$\lambda_{2}=10^{-5}$]{\includegraphics[width=0.246\columnwidth]{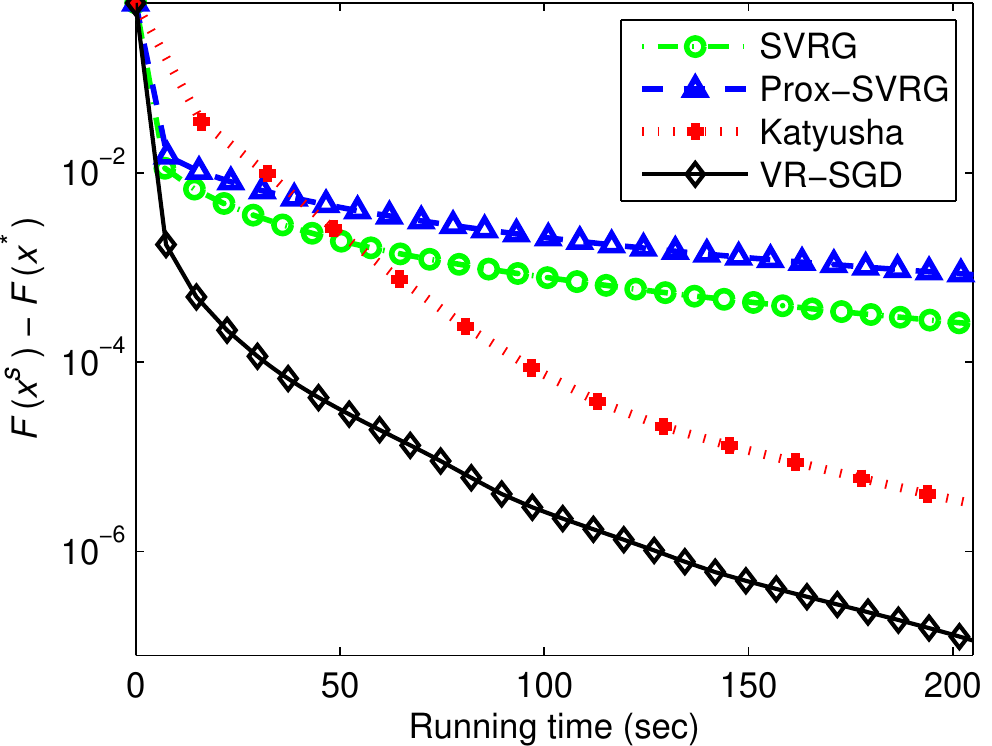}\label{figs2f}}
\caption{Comparison of SVRG, Prox-SVRG~\cite{xiao:prox-svrg}, Katyusha~\cite{zhu:Katyusha}, and VR-SGD for $\ell_{1}$-norm regularized logistic regression problems (i.e., $\lambda_{1}=0$) on the four data sets: Adult (the first column), Protein (the sconced column), Covtype (the third column), and Sido0 (the last column). In each plot, the vertical axis shows the objective value minus the minimum, and the horizontal axis is the number of effective passes (top) or running time (bottom).}
\label{figs2}
\end{figure}

\begin{figure}[!th]
\centering
\includegraphics[width=0.246\columnwidth]{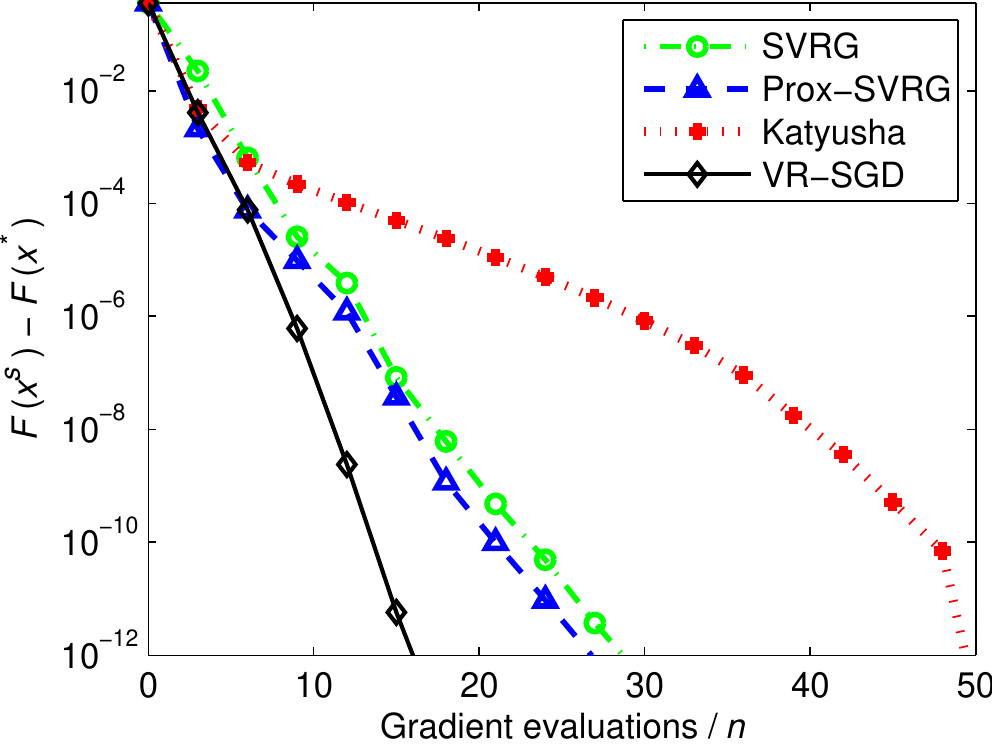}
\includegraphics[width=0.246\columnwidth]{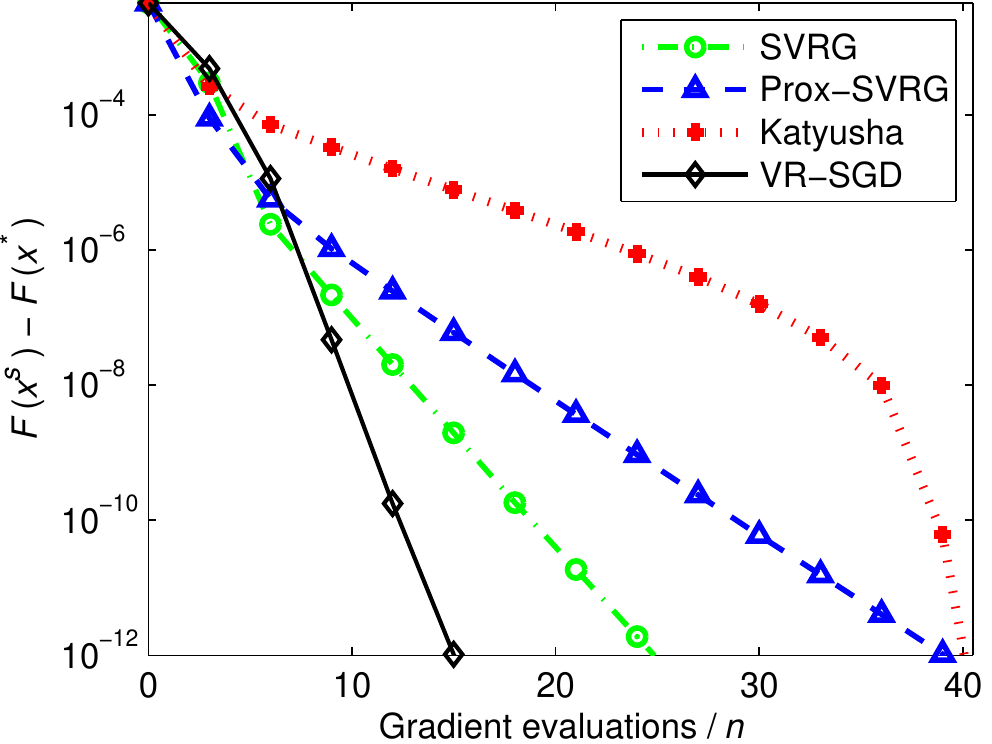}
\includegraphics[width=0.246\columnwidth]{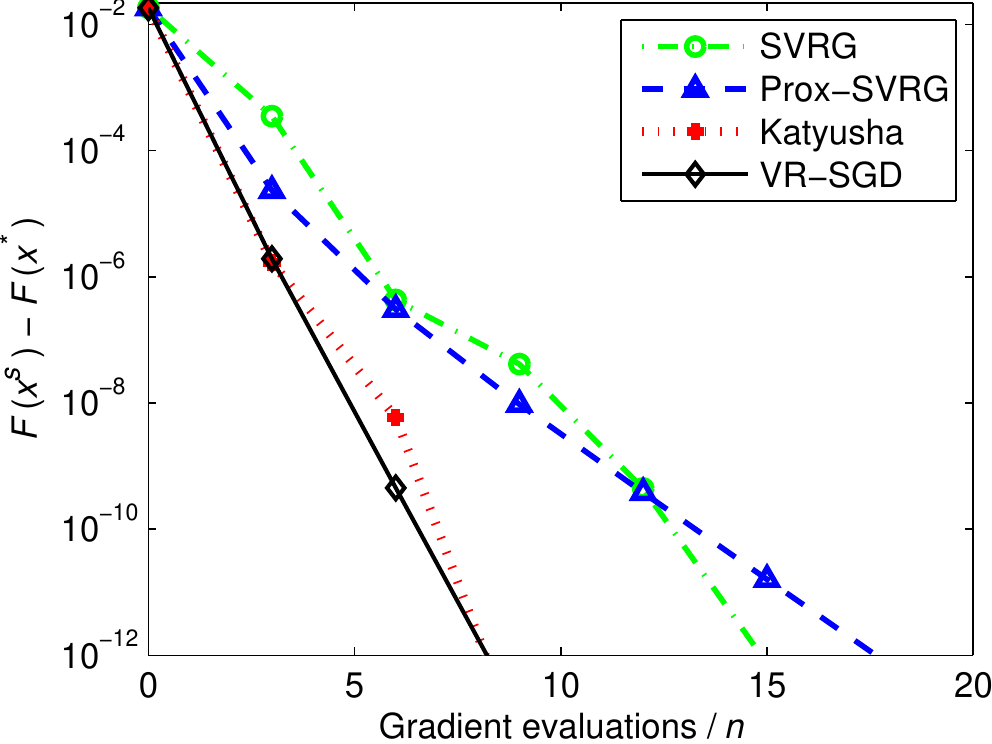}
\includegraphics[width=0.246\columnwidth]{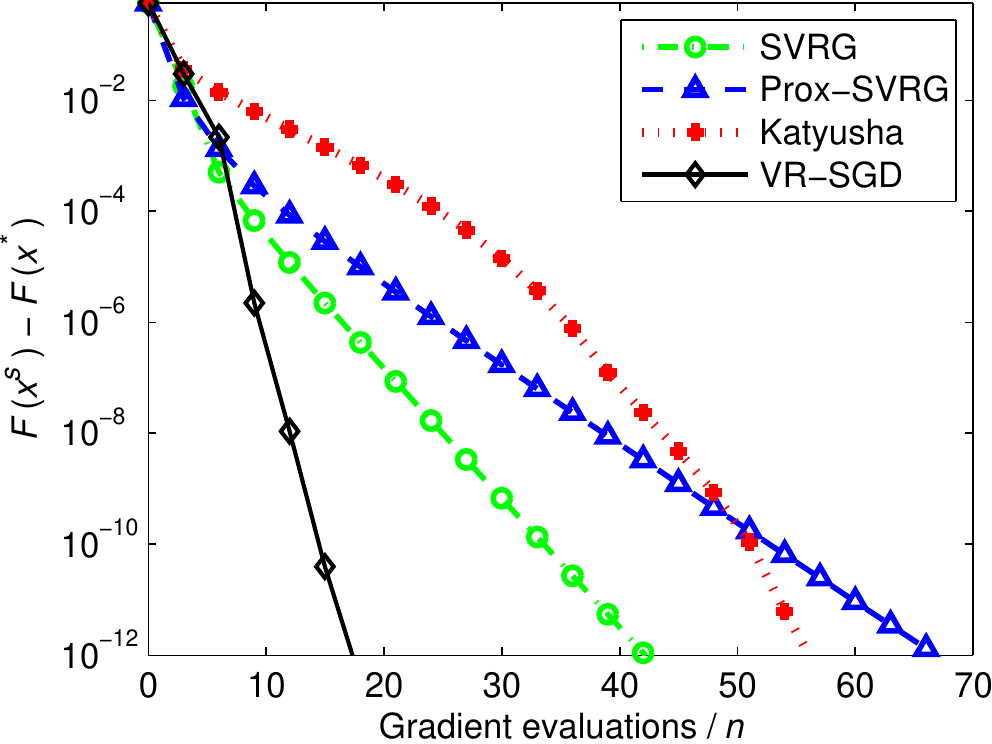}

\subfigure[$\lambda_{1}=10^{-5}$ \;and\; $\lambda_{2}=10^{-4}$]{\includegraphics[width=0.246\columnwidth]{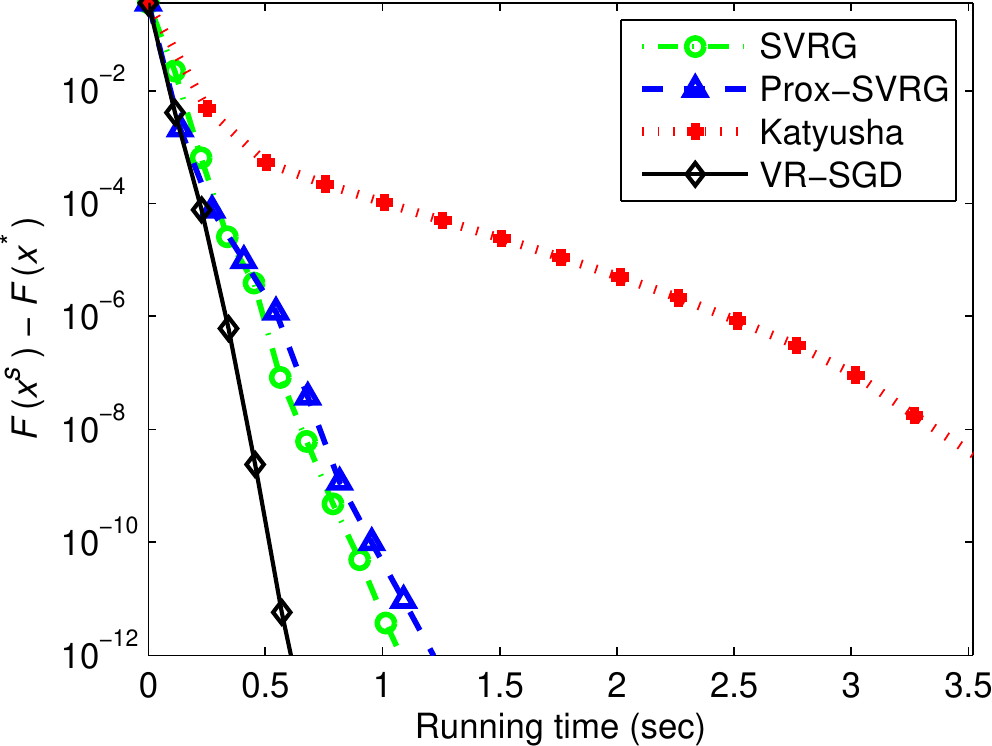}\:\includegraphics[width=0.246\columnwidth]{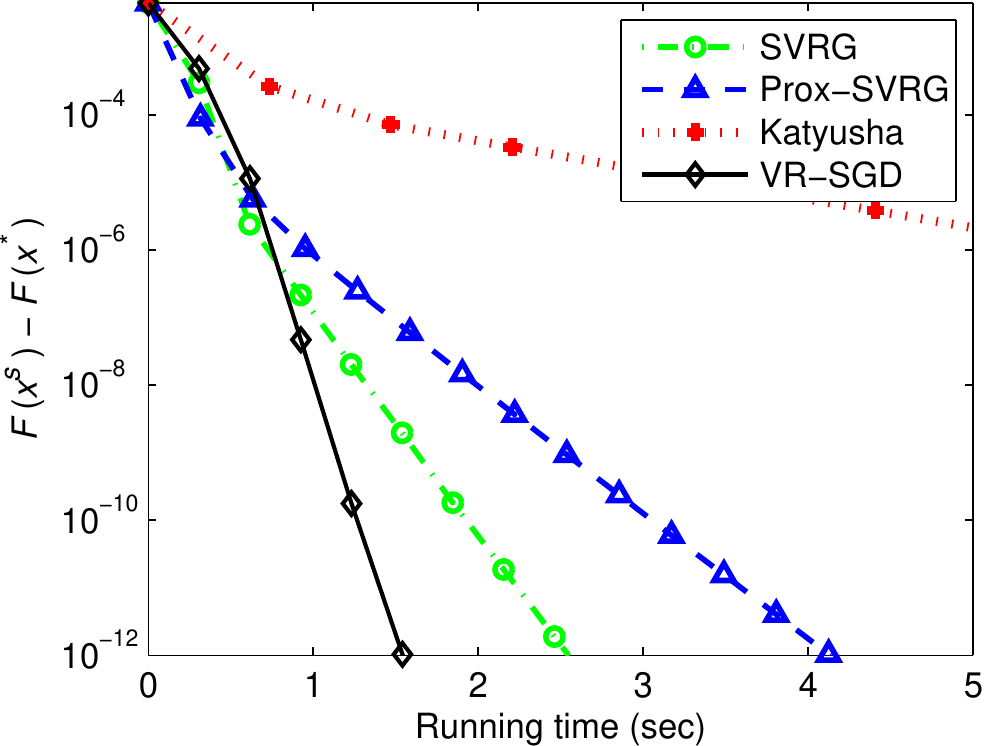}\:\includegraphics[width=0.246\columnwidth]{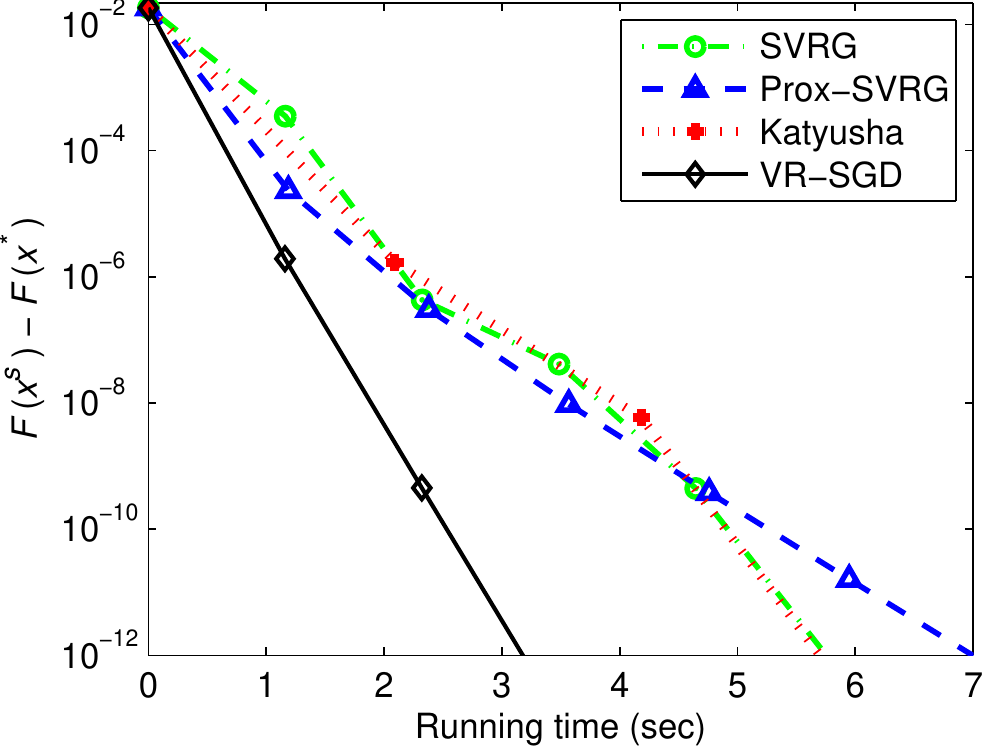}}\:\subfigure[$\lambda_{1}\!=\!10^{-4}$ and $\lambda_{2}\!=\!10^{-5}$]{\includegraphics[width=0.246\columnwidth]{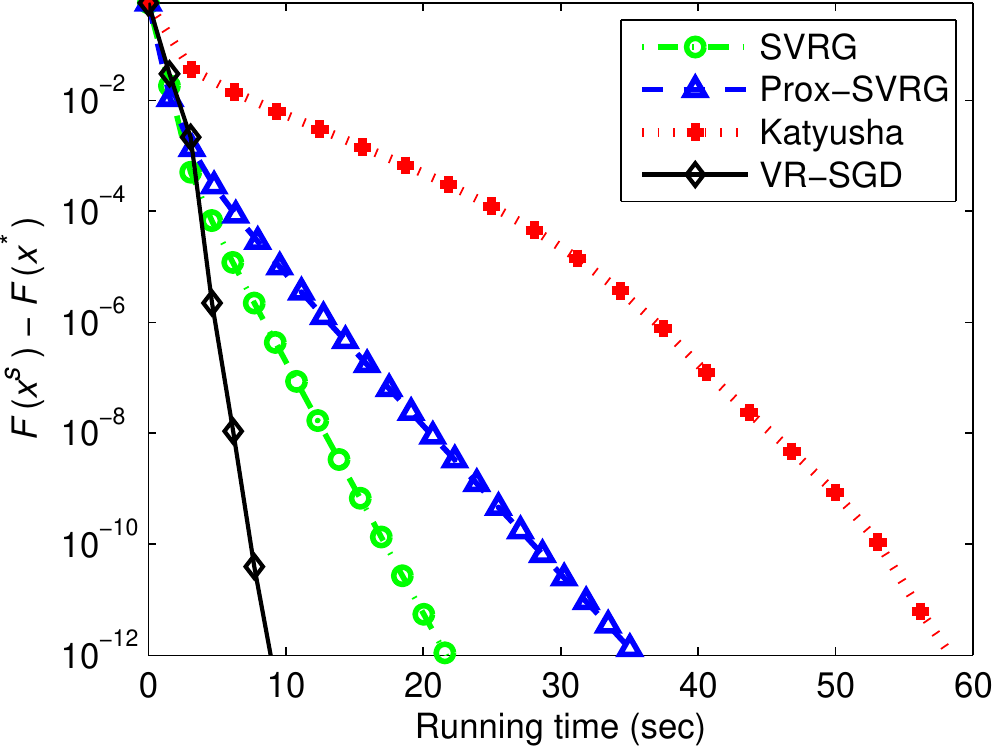}}
\vspace{1.6mm}

\includegraphics[width=0.246\columnwidth]{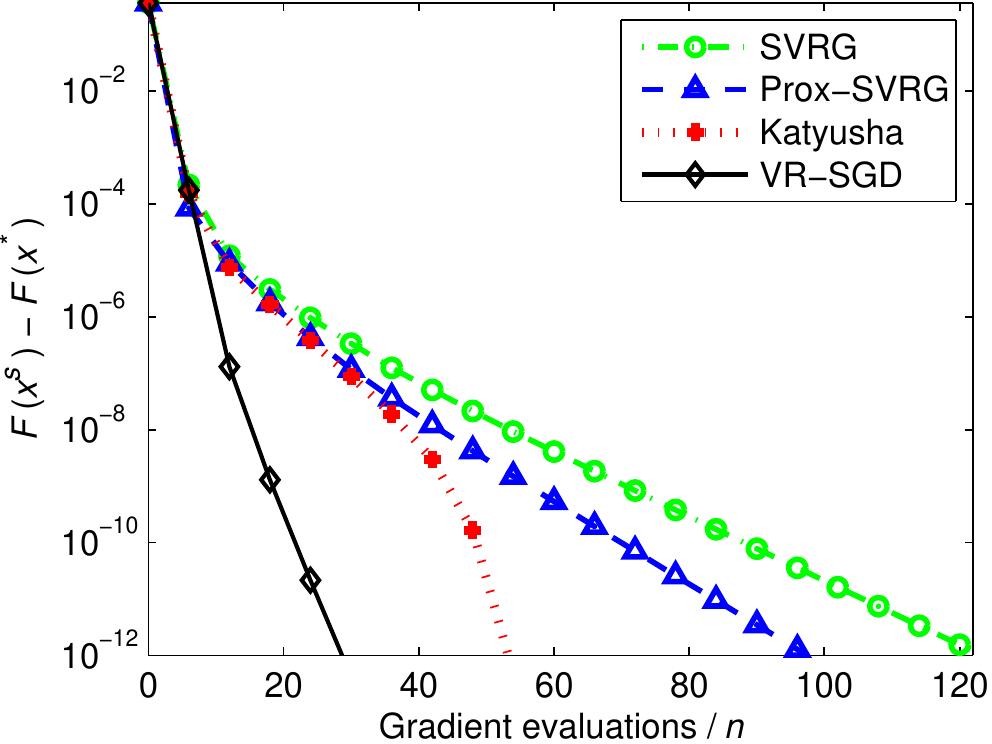}
\includegraphics[width=0.246\columnwidth]{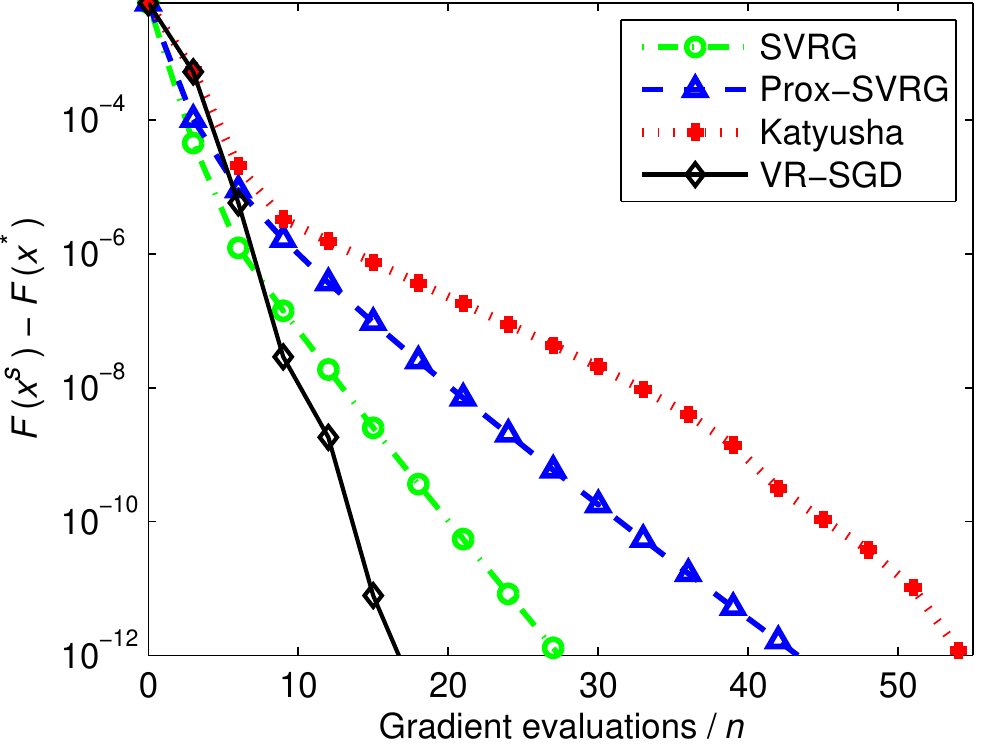}
\includegraphics[width=0.246\columnwidth]{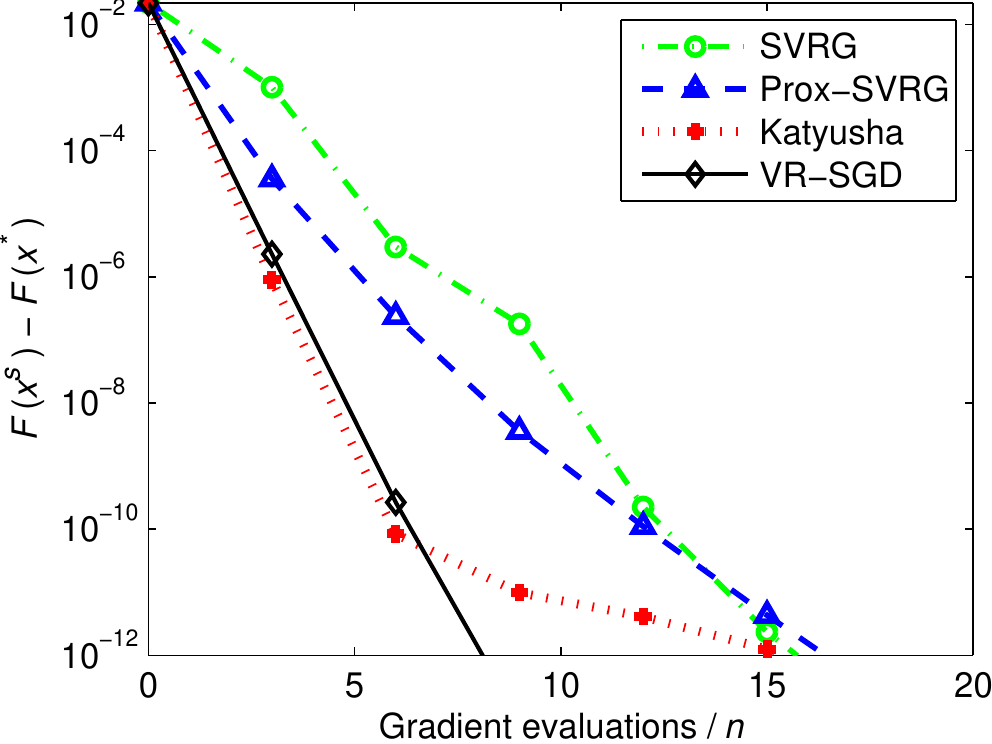}
\includegraphics[width=0.246\columnwidth]{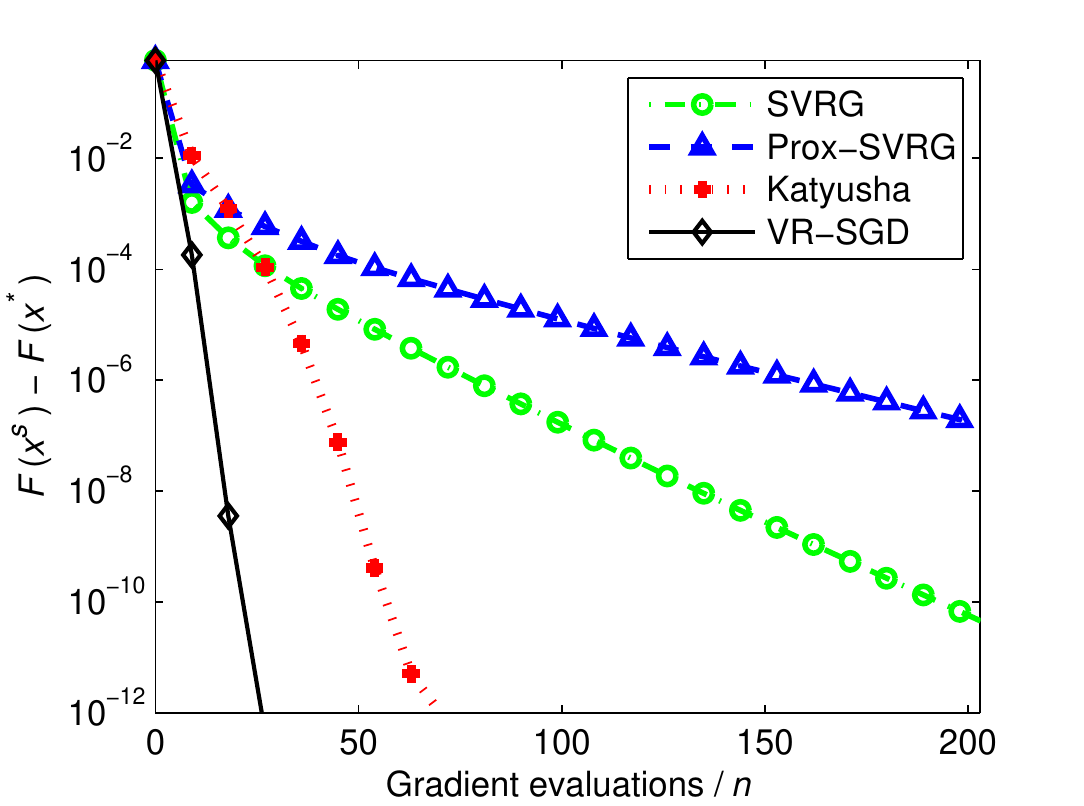}

\subfigure[$\lambda_{1}=10^{-5}$ \;and\; $\lambda_{2}=10^{-5}$]{\includegraphics[width=0.246\columnwidth]{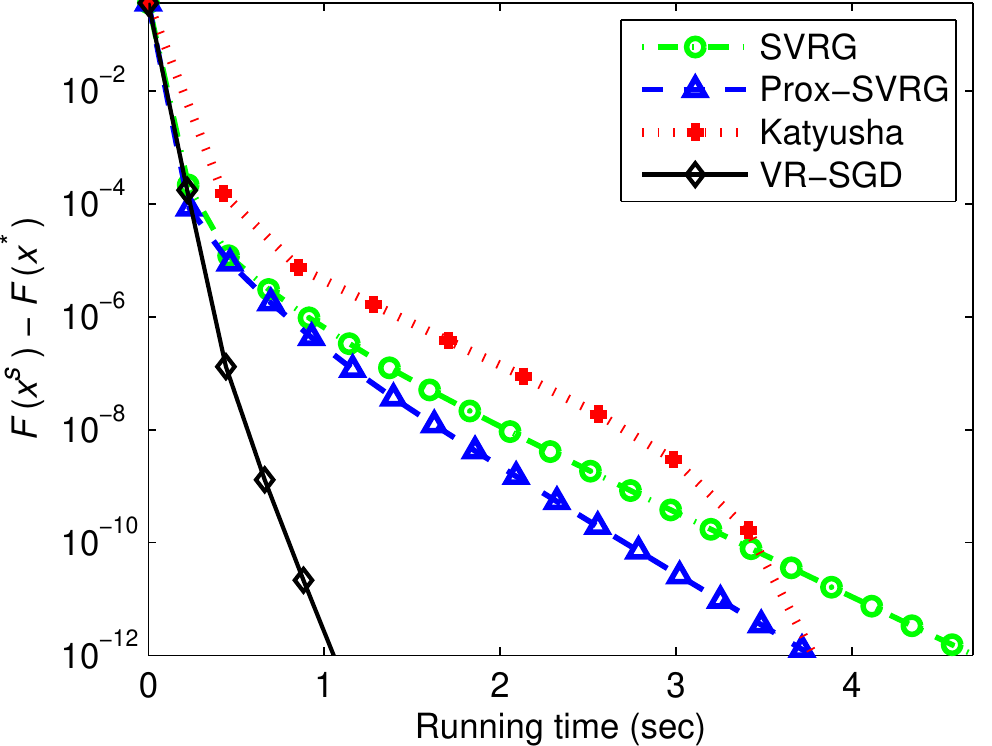}\:\includegraphics[width=0.246\columnwidth]{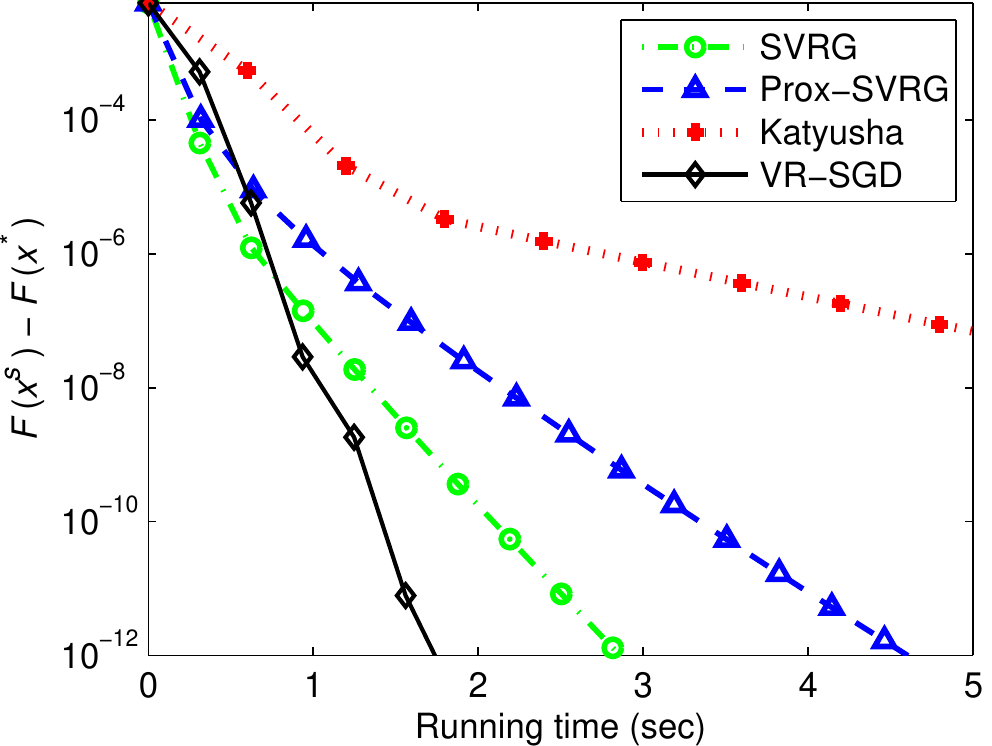}\:\includegraphics[width=0.246\columnwidth]{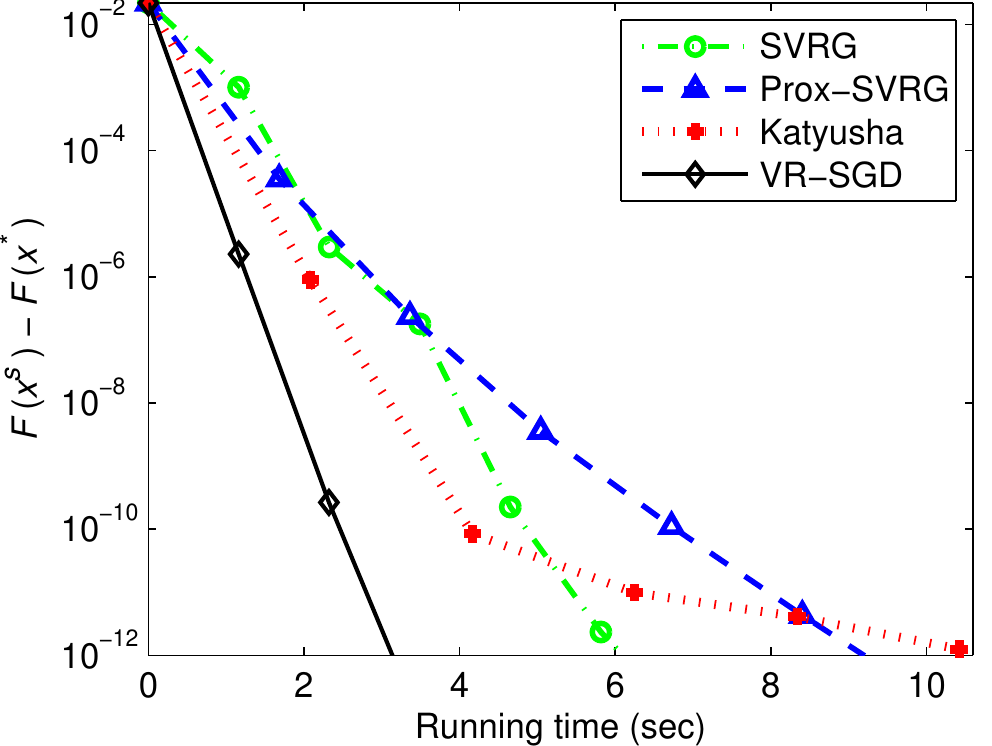}}\:\subfigure[$\lambda_{1}\!=\!10^{-5}$ and $\lambda_{2}\!=\!10^{-4}$]{\includegraphics[width=0.246\columnwidth]{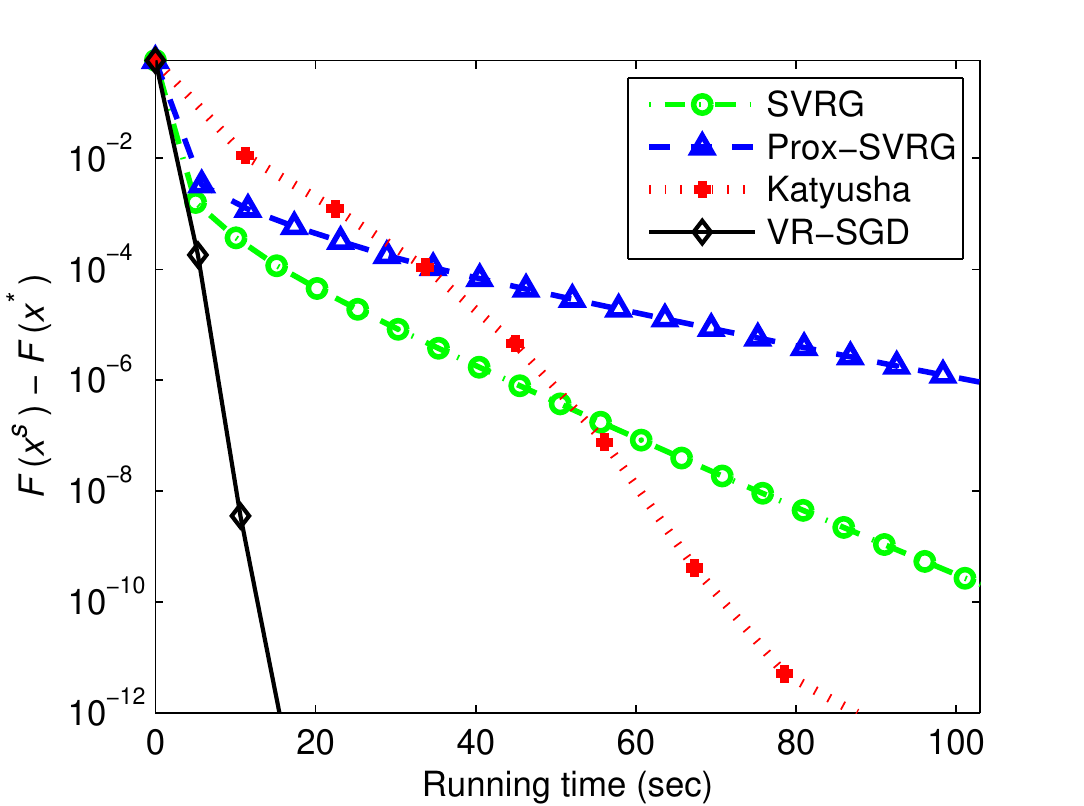}}
\vspace{1.6mm}

\includegraphics[width=0.246\columnwidth]{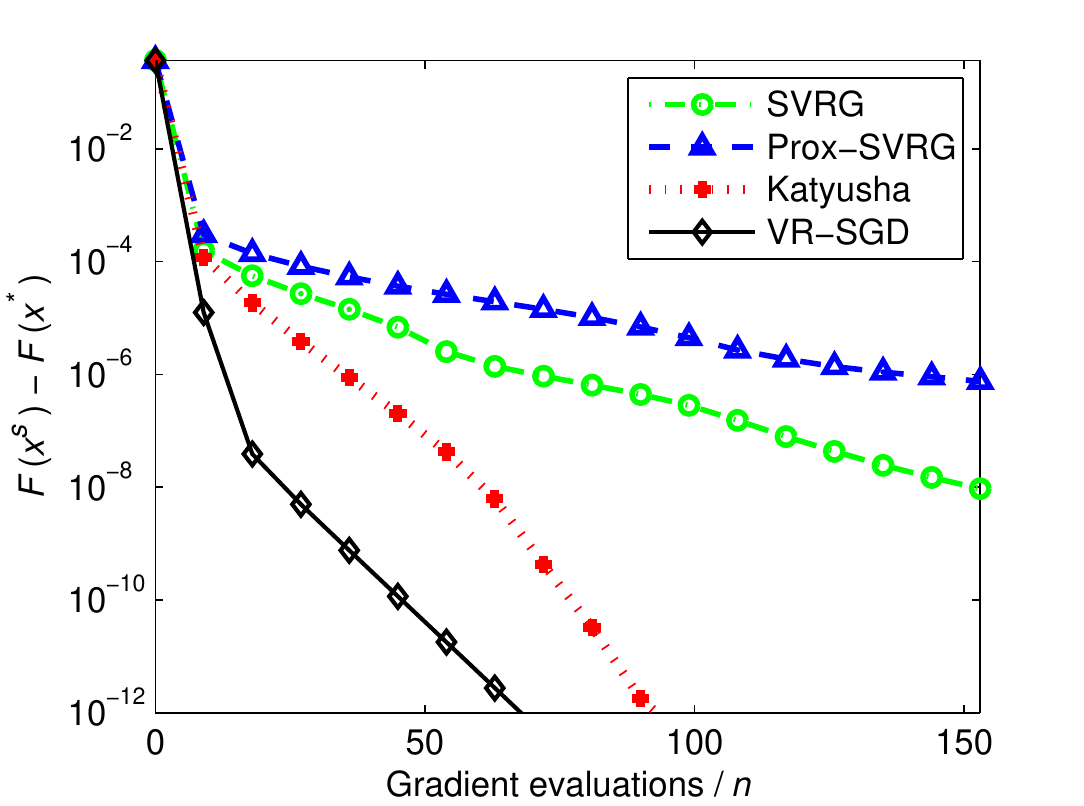}
\includegraphics[width=0.246\columnwidth]{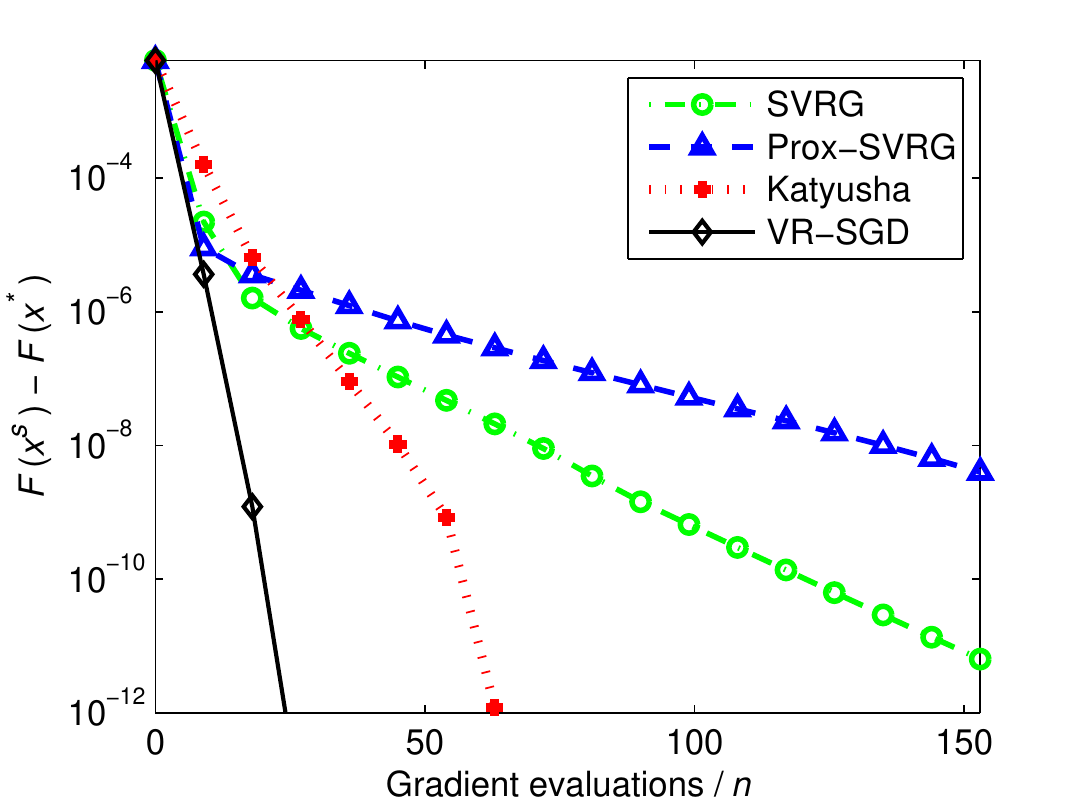}
\includegraphics[width=0.246\columnwidth]{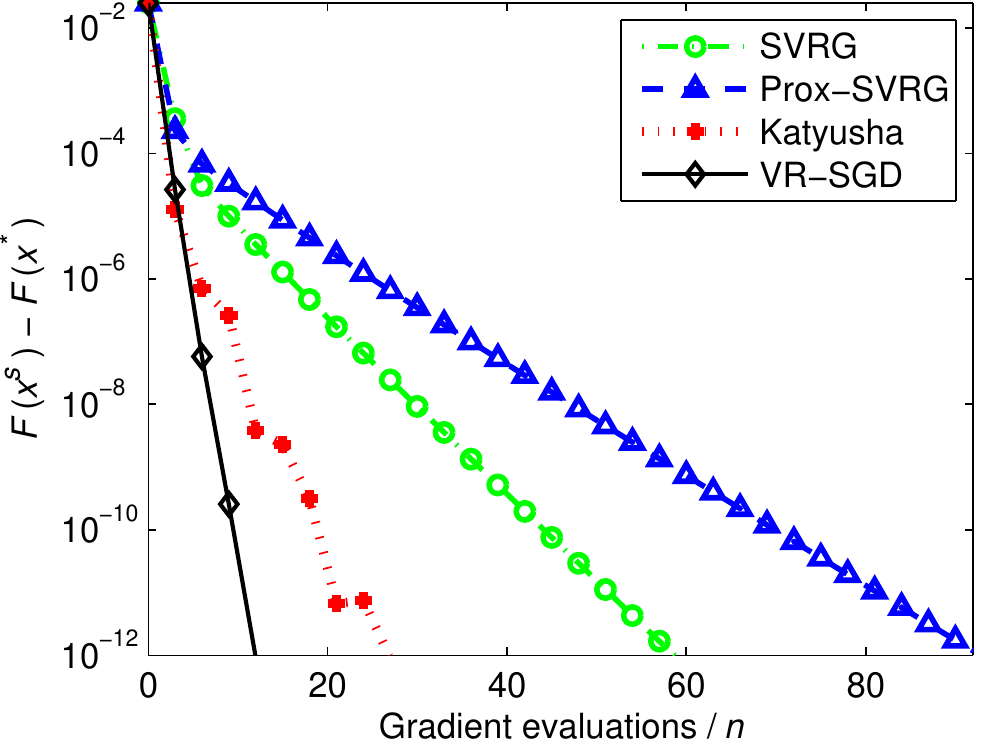}
\includegraphics[width=0.246\columnwidth]{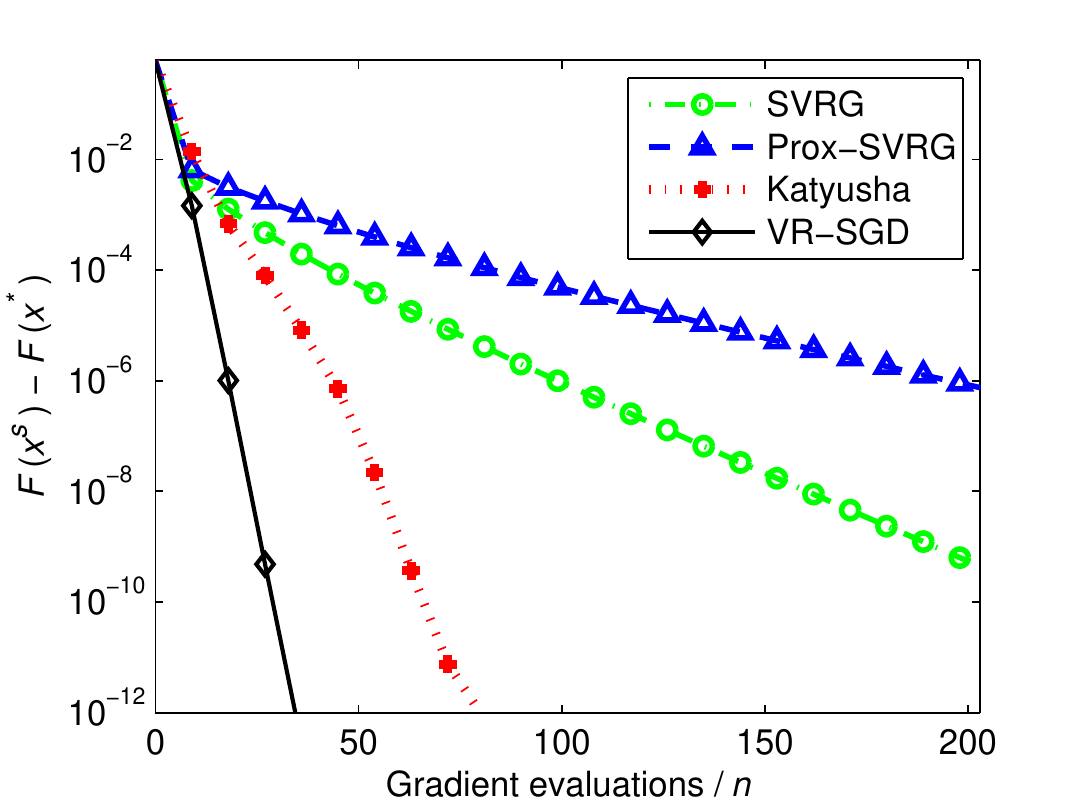}

\subfigure[$\lambda_{1}=10^{-6}$ \;and\; $\lambda_{2}=10^{-5}$]{\includegraphics[width=0.246\columnwidth]{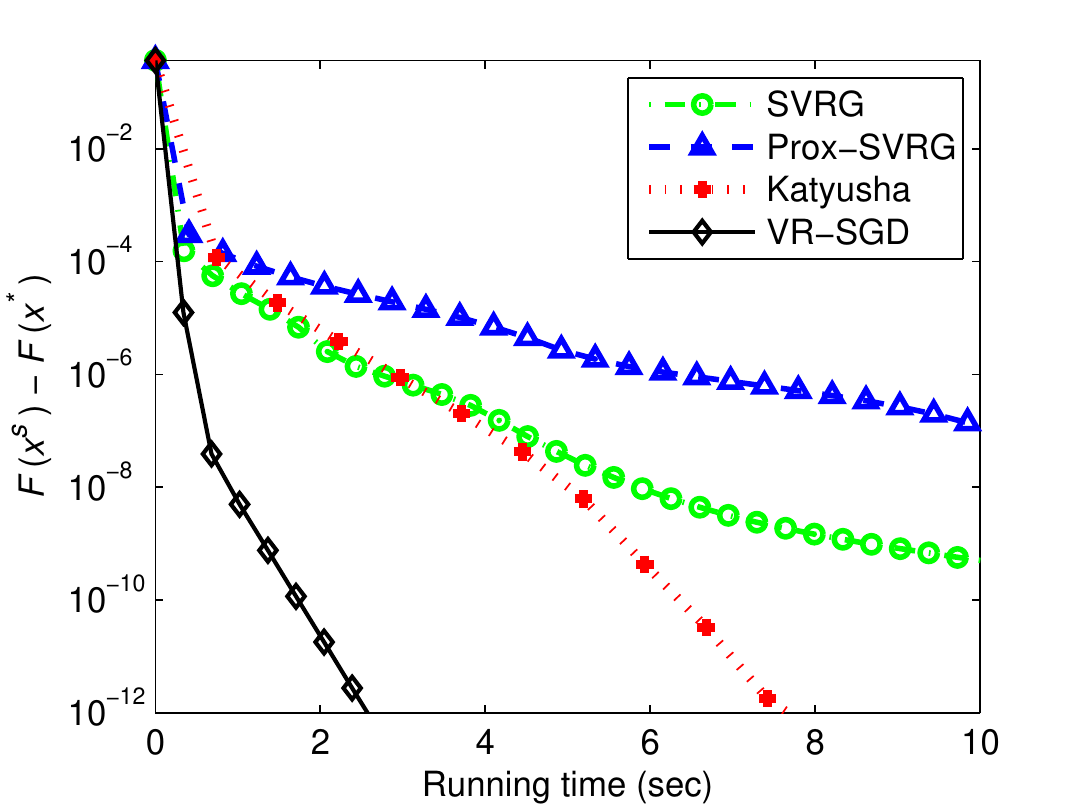}\:\includegraphics[width=0.246\columnwidth]{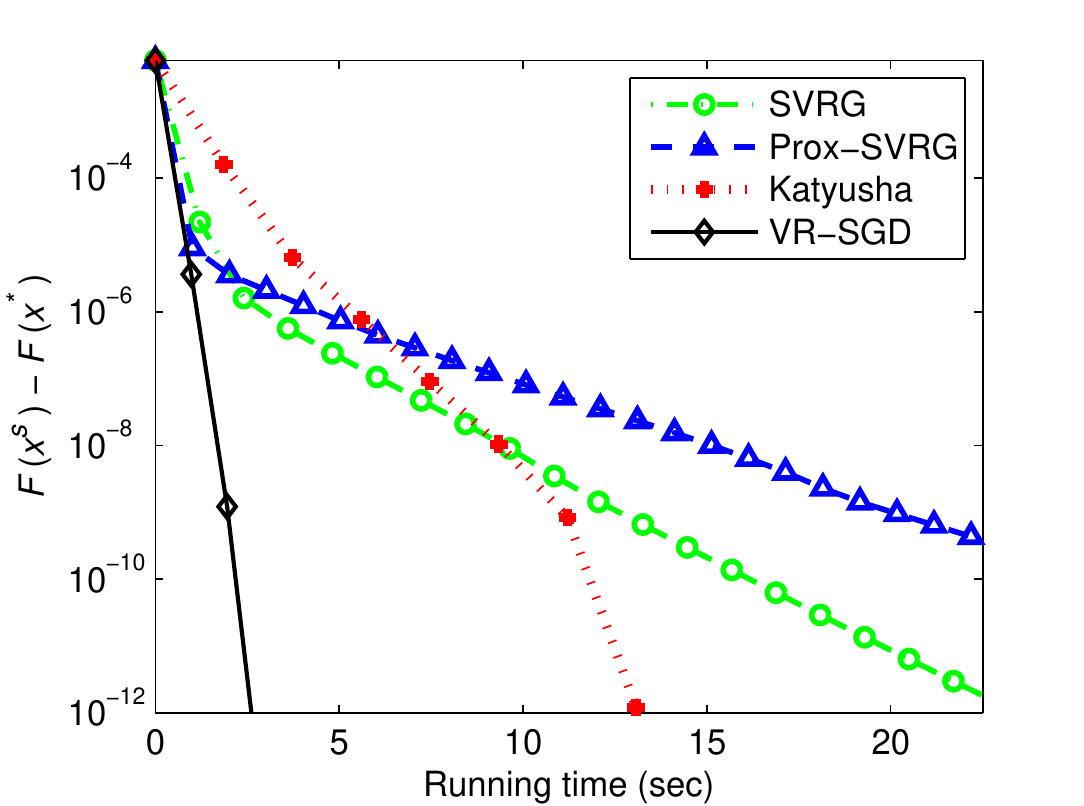}\:\includegraphics[width=0.246\columnwidth]{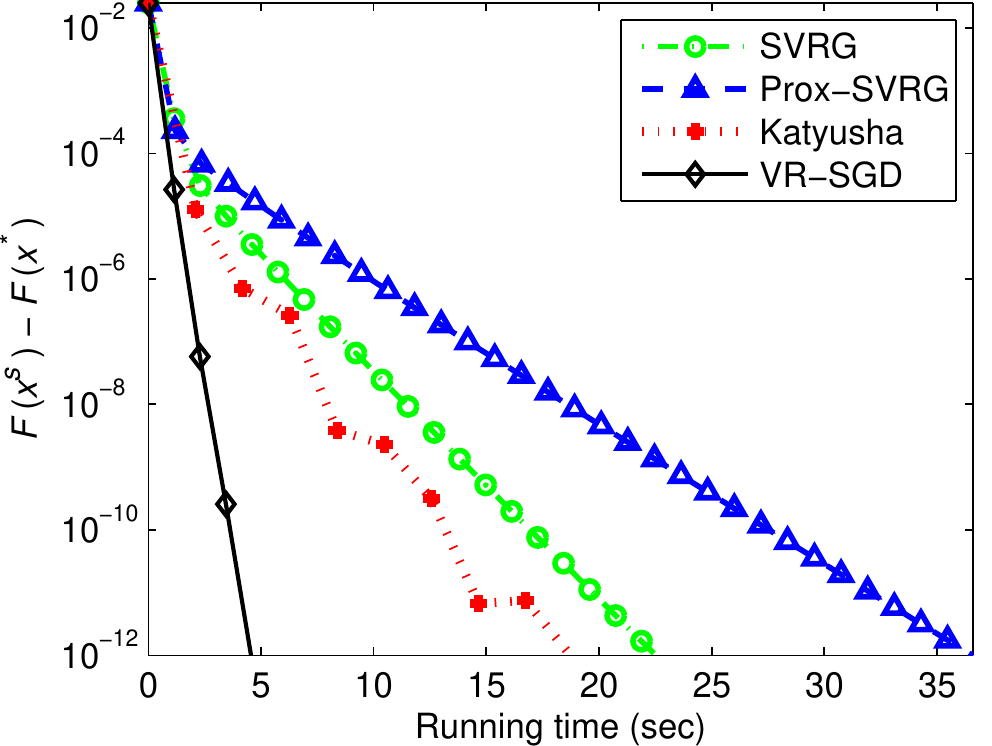}}\:\subfigure[$\lambda_{1}\!=\!10^{-5}$ and $\lambda_{2}\!=\!10^{-5}$]{\includegraphics[width=0.246\columnwidth]{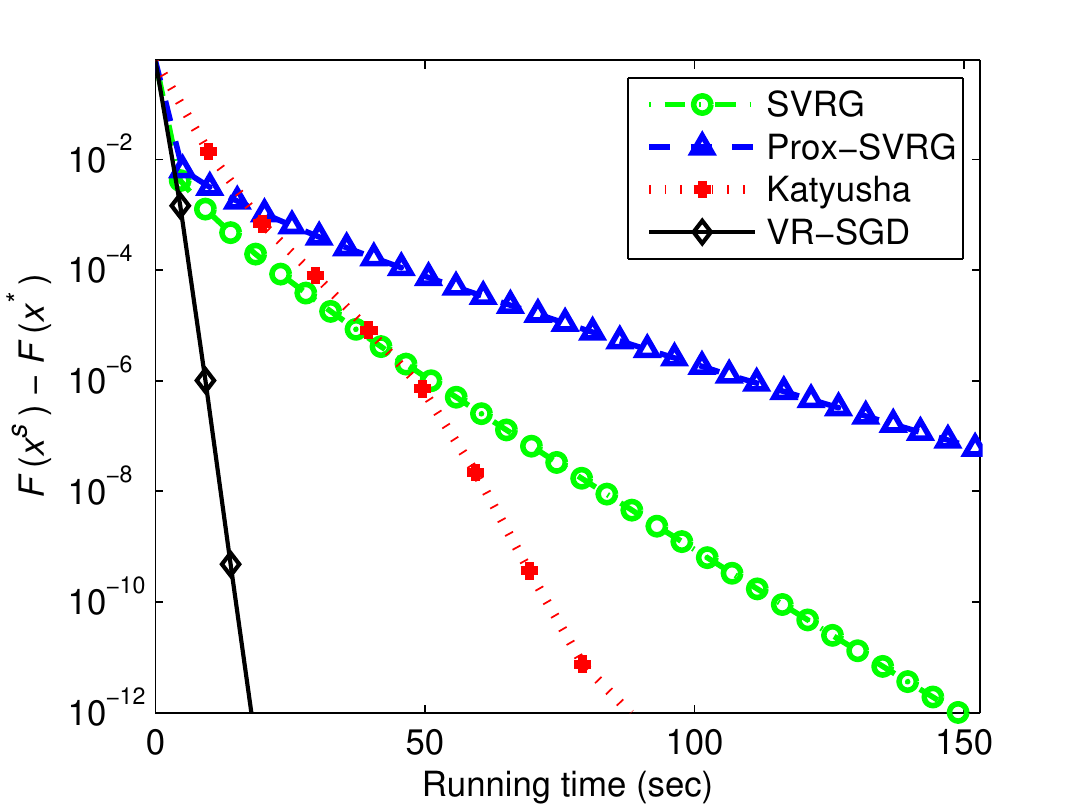}}
\vspace{-2.6mm}
\caption{Comparison of SVRG, Prox-SVRG~\cite{xiao:prox-svrg}, Katyusha~\cite{zhu:Katyusha}, and VR-SGD for solving elastic net regularized logistic regression problems on the four data sets: Adult (the first column), Protein (the sconced column), Covtype (the third column), and Sido0 (the last column). In each plot, the vertical axis shows the objective value minus the minimum, and the horizontal axis is the number of effective passes (top) or running time (bottom).}
\label{figs3}
\end{figure}

\begin{figure}[!th]
\centering
\includegraphics[width=0.326\columnwidth]{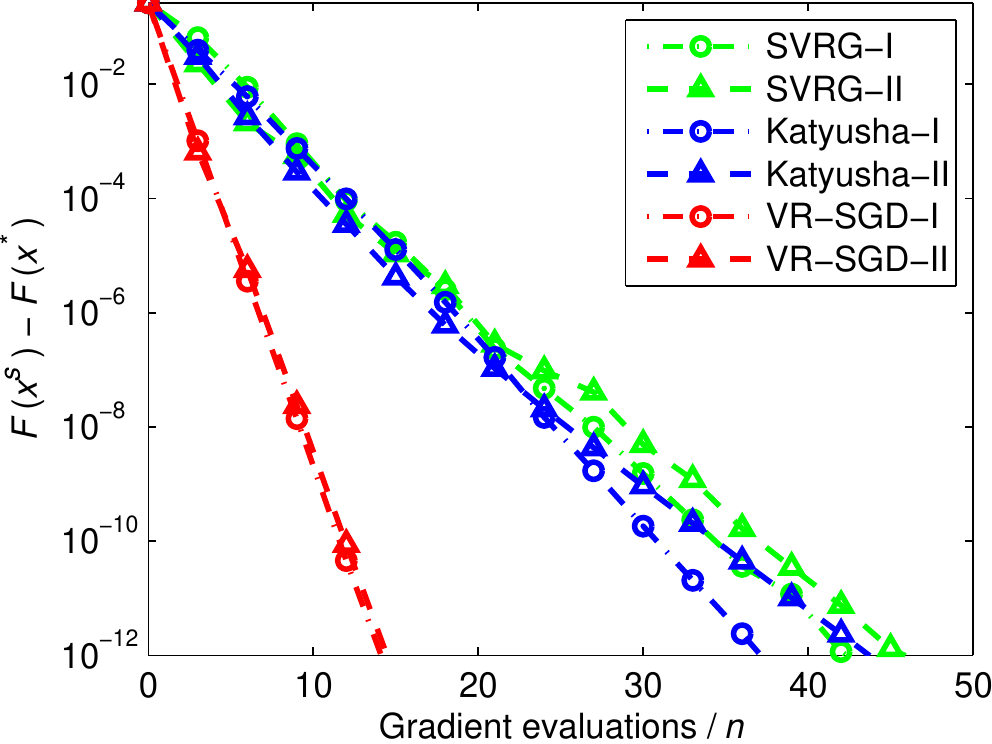}\,
\includegraphics[width=0.326\columnwidth]{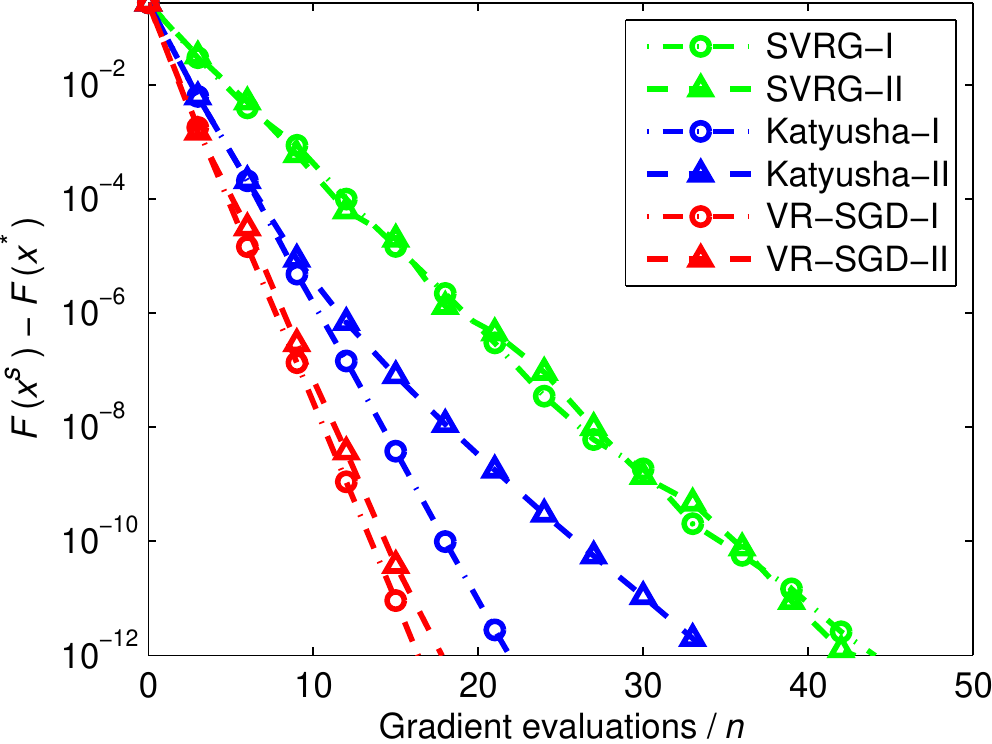}\,
\includegraphics[width=0.326\columnwidth]{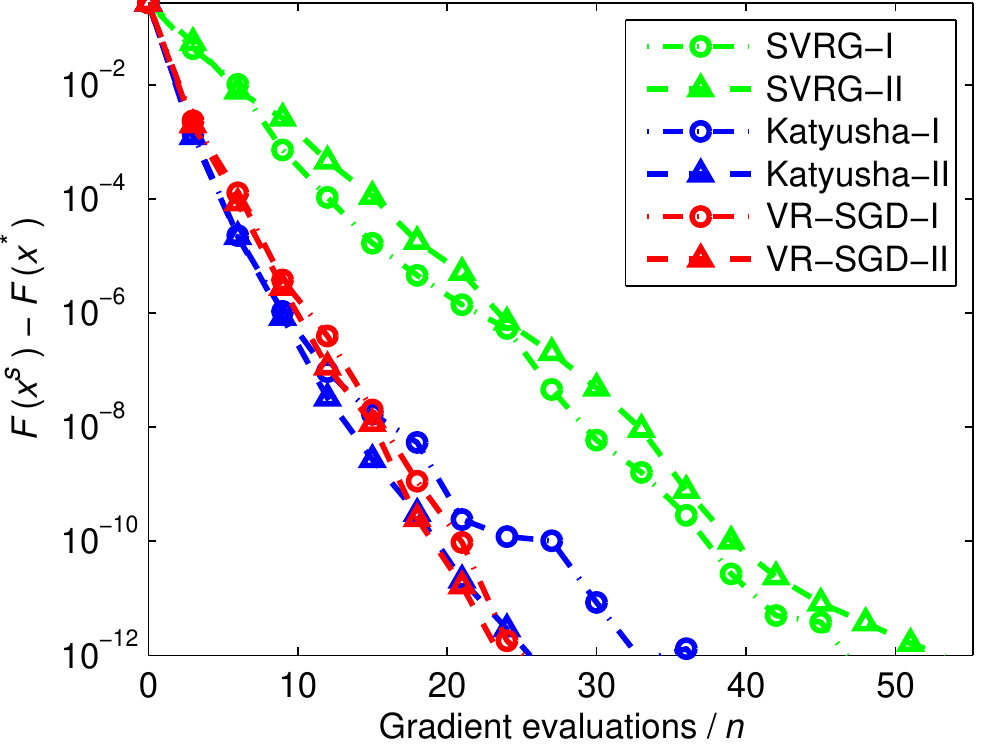}

\subfigure[$\lambda=10^{-3}$]{\includegraphics[width=0.326\columnwidth]{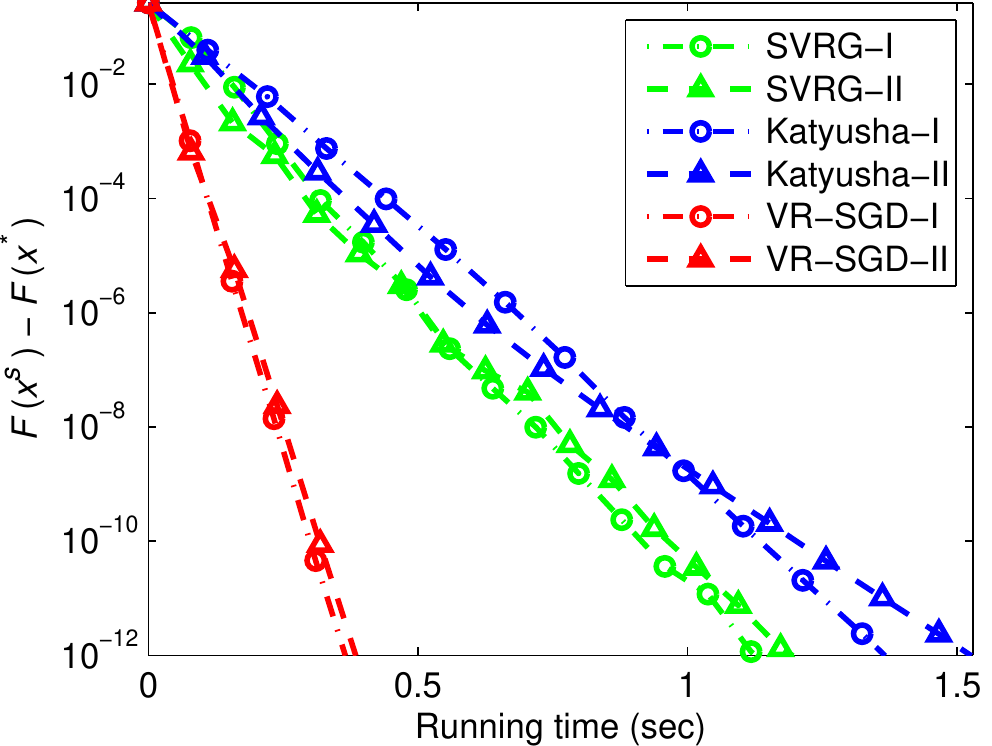}}\,
\subfigure[$\lambda=10^{-4}$]{\includegraphics[width=0.326\columnwidth]{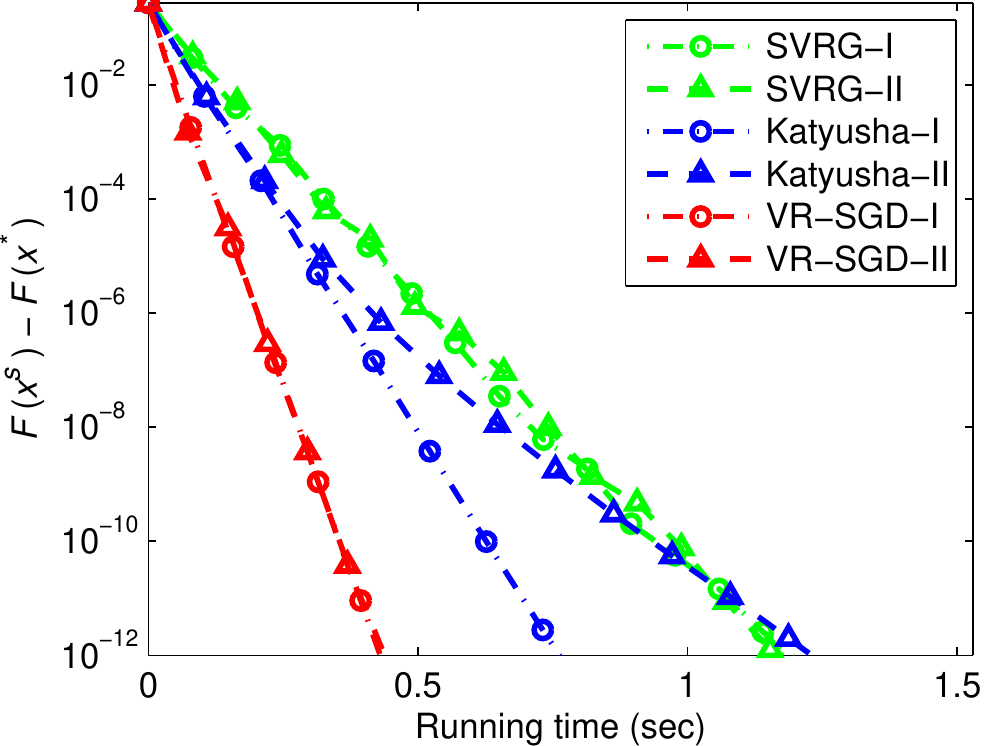}}\,
\subfigure[$\lambda=10^{-5}$]{\includegraphics[width=0.326\columnwidth]{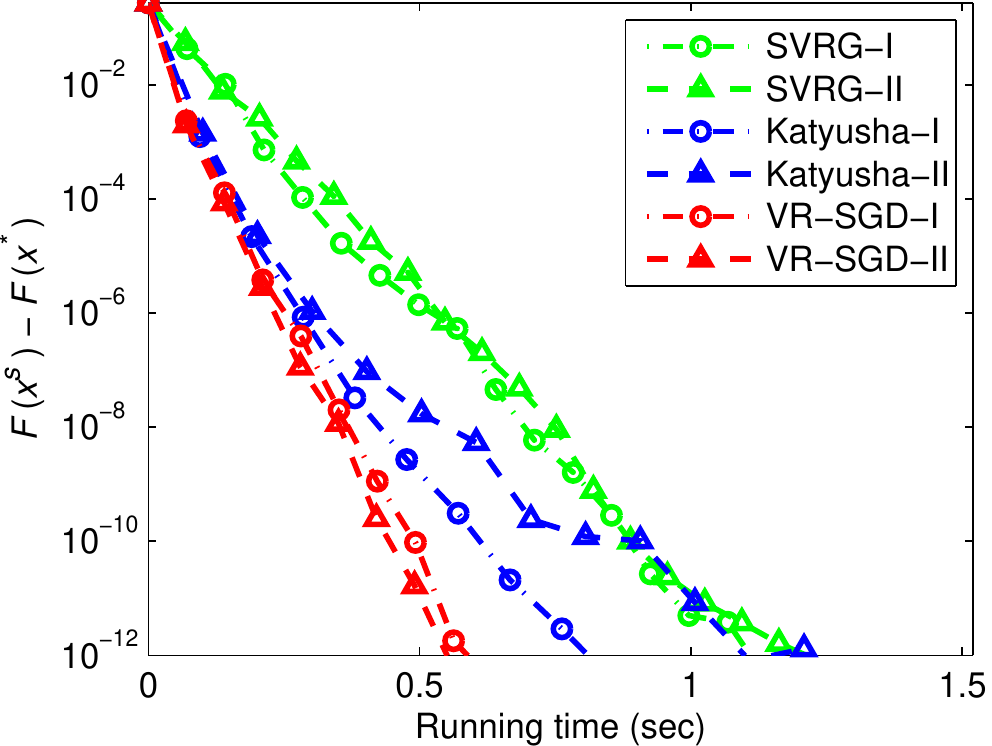}}
\vspace{1.6mm}

\includegraphics[width=0.326\columnwidth]{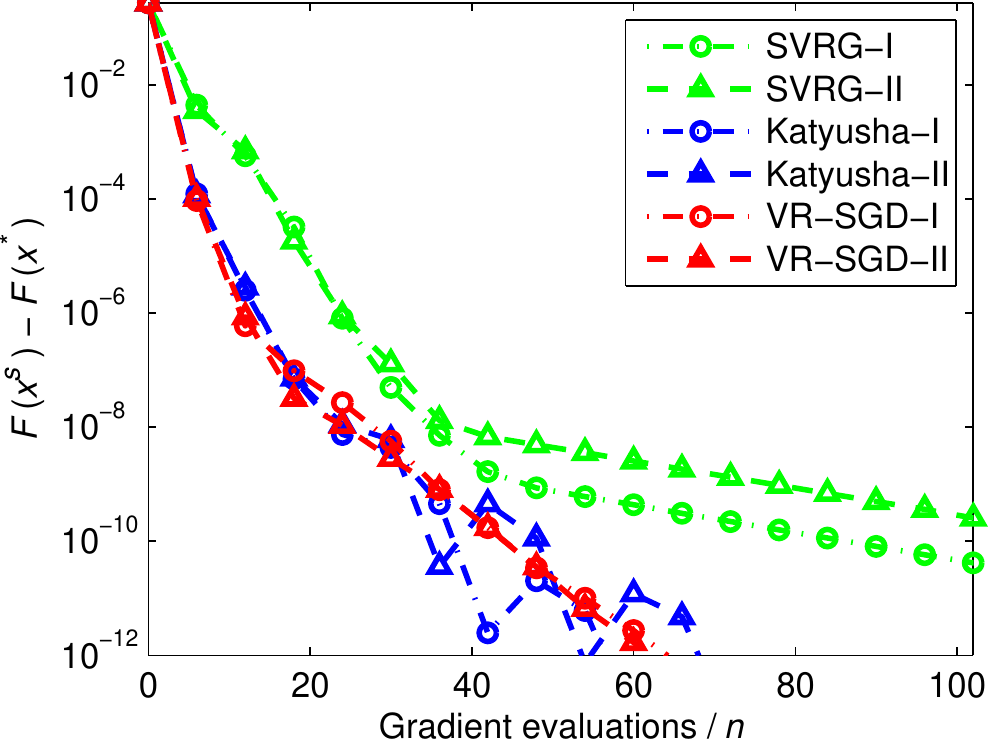}\,
\includegraphics[width=0.326\columnwidth]{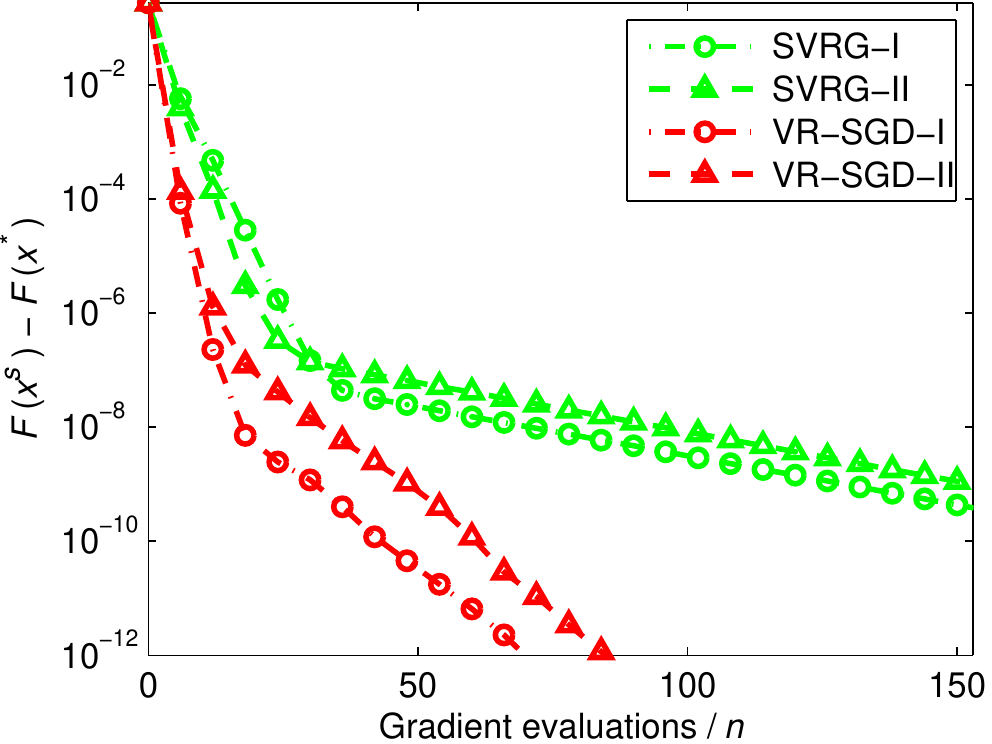}\,
\includegraphics[width=0.326\columnwidth]{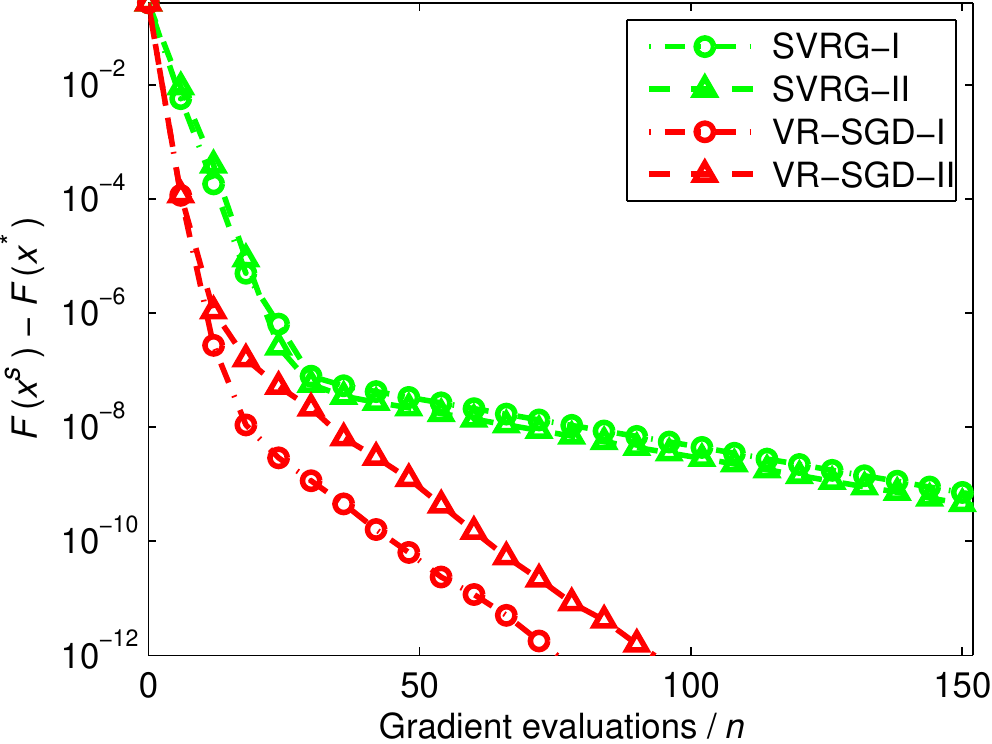}

\subfigure[$\lambda=10^{-6}$]{\includegraphics[width=0.326\columnwidth]{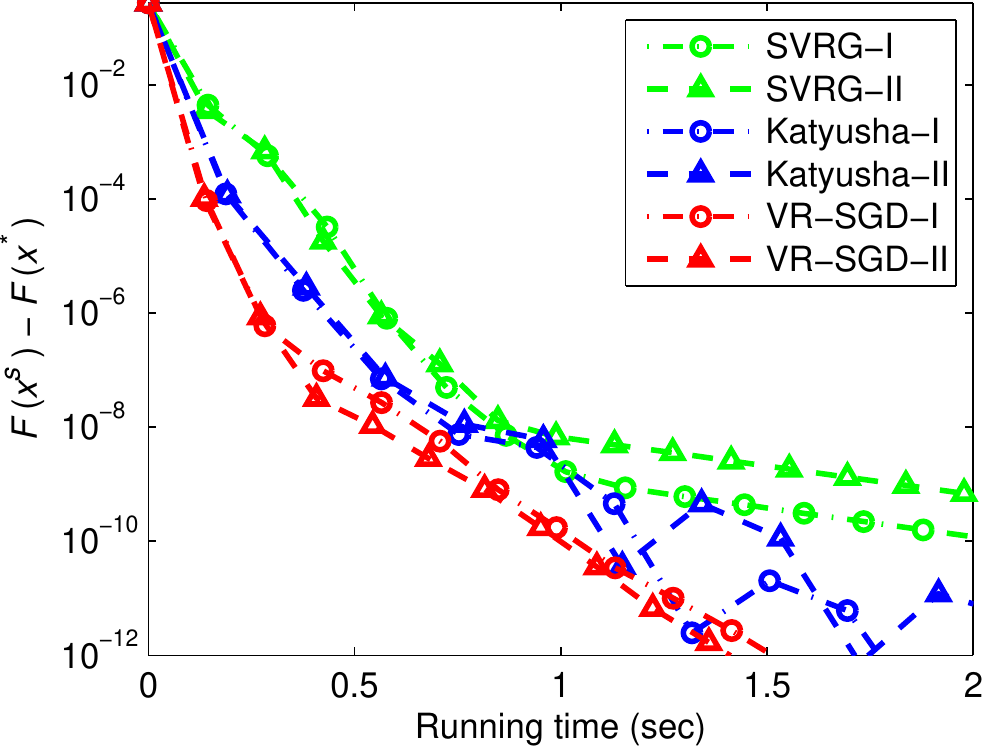}\label{figs4d}}\,
\subfigure[$\lambda=10^{-7}$]{\includegraphics[width=0.326\columnwidth]{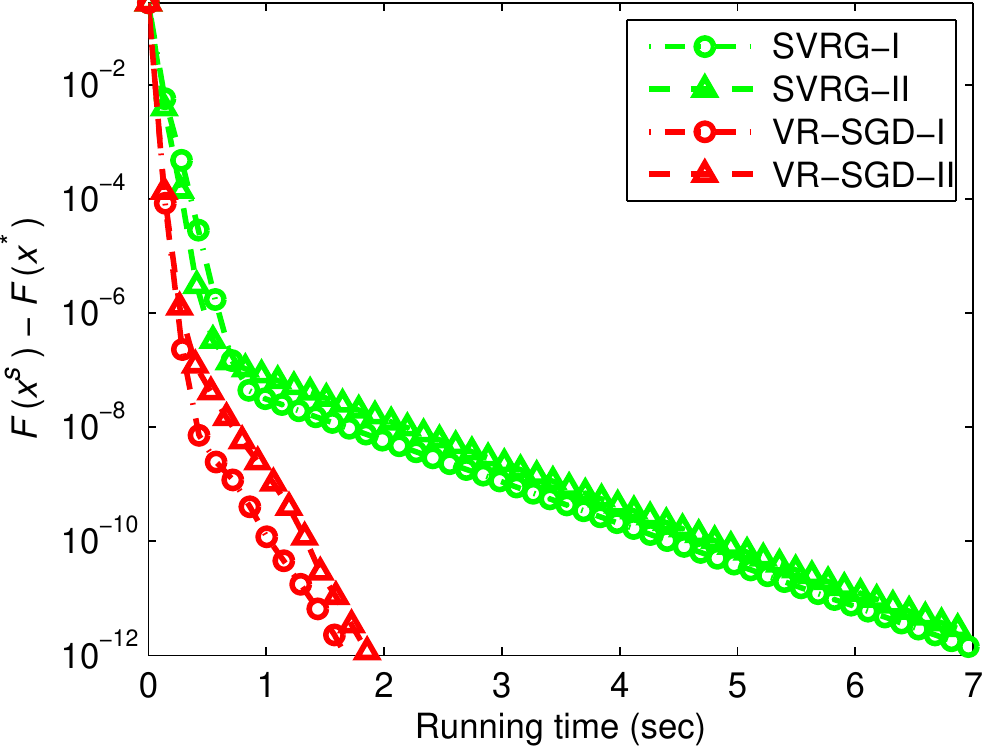}\label{figs4e}}\,
\subfigure[$\lambda=0$]{\includegraphics[width=0.326\columnwidth]{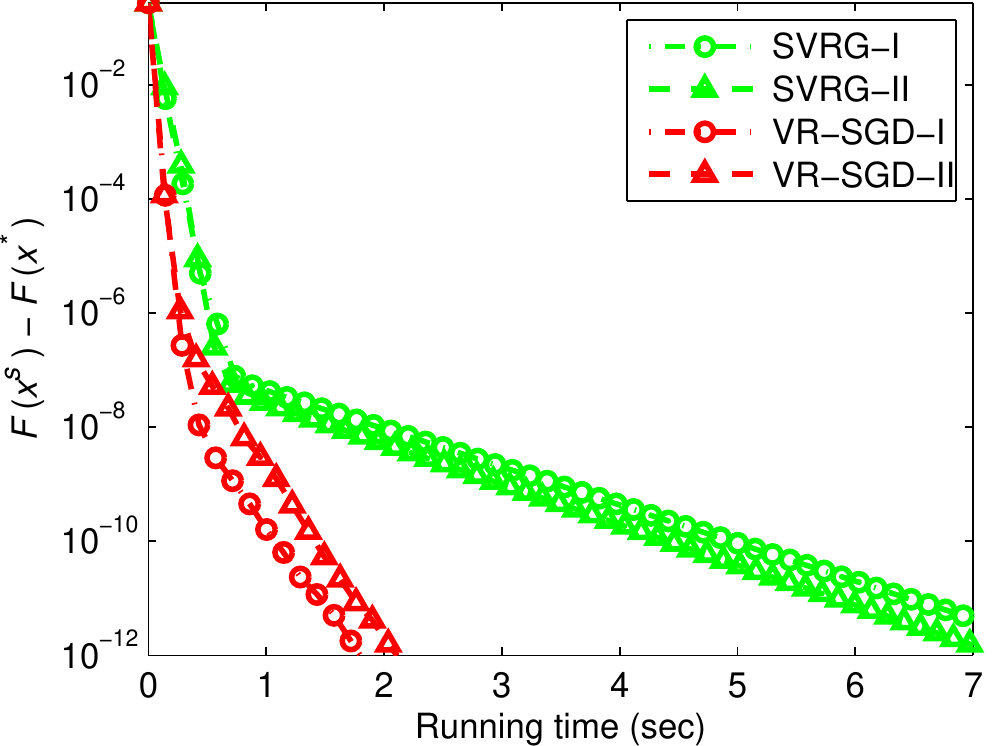}\label{figs4f}}
\caption{Comparison of SVRG~\cite{johnson:svrg}, Katyusha~\cite{zhu:Katyusha}, VR-SGD and their proximal versions for solving ridge regression problems with different regularization parameters on the Adult data set. In each plot, the vertical axis shows the objective value minus the minimum, and the horizontal axis is the number of effective passes (top) or running time (bottom).}
\label{figs4}
\end{figure}

\begin{figure}[!th]
\centering
\includegraphics[width=0.326\columnwidth]{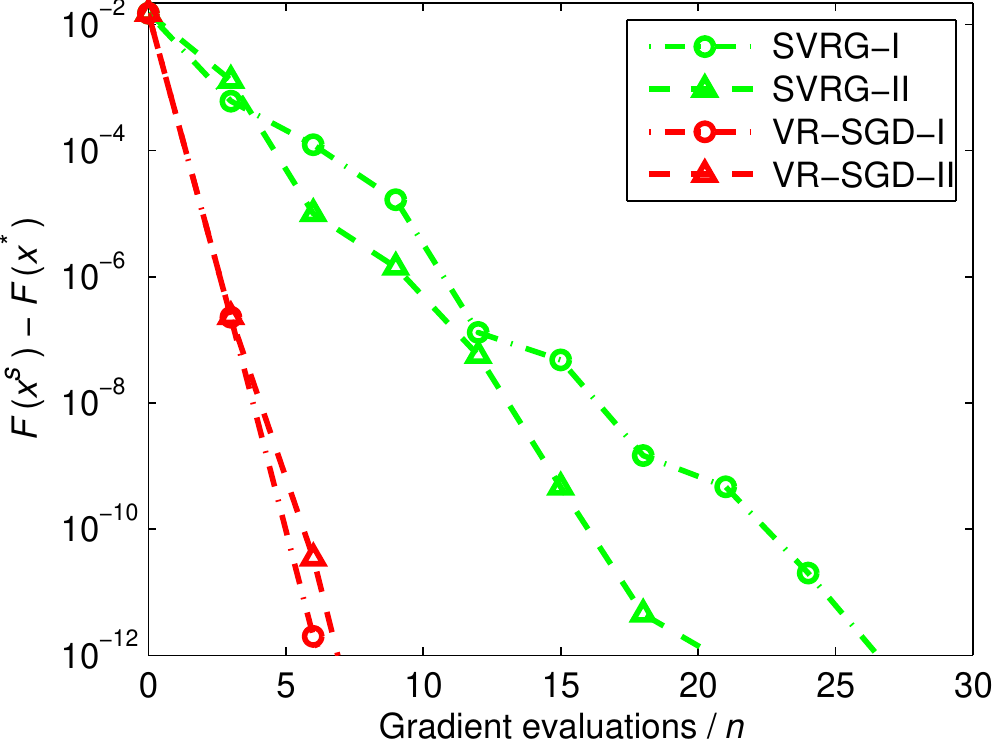}\,
\includegraphics[width=0.326\columnwidth]{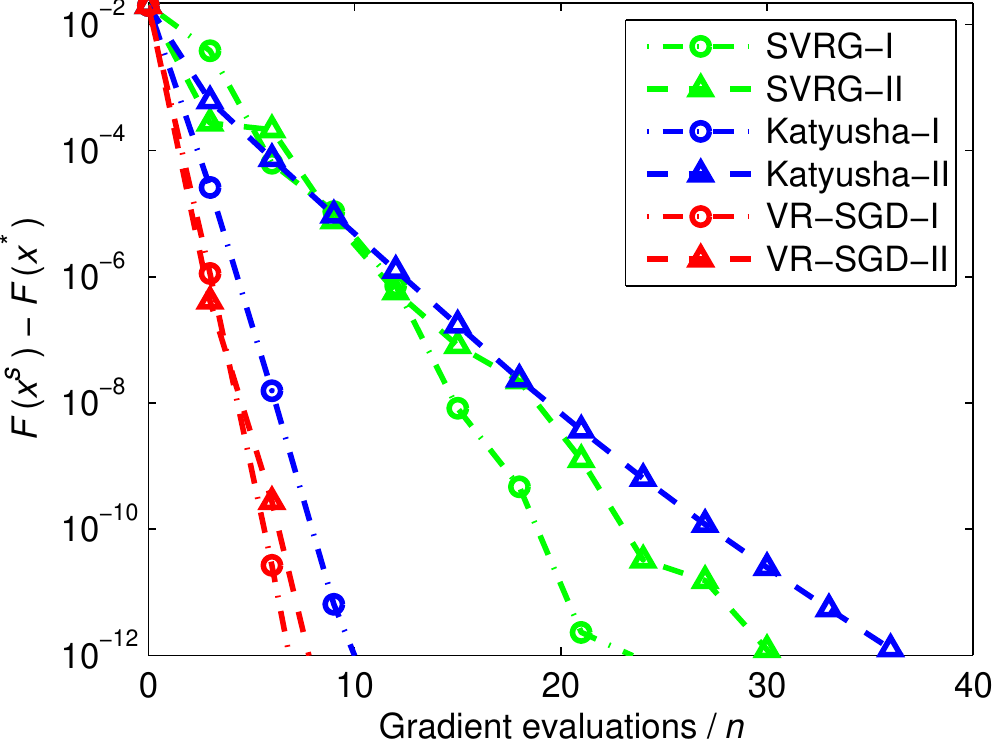}\,
\includegraphics[width=0.326\columnwidth]{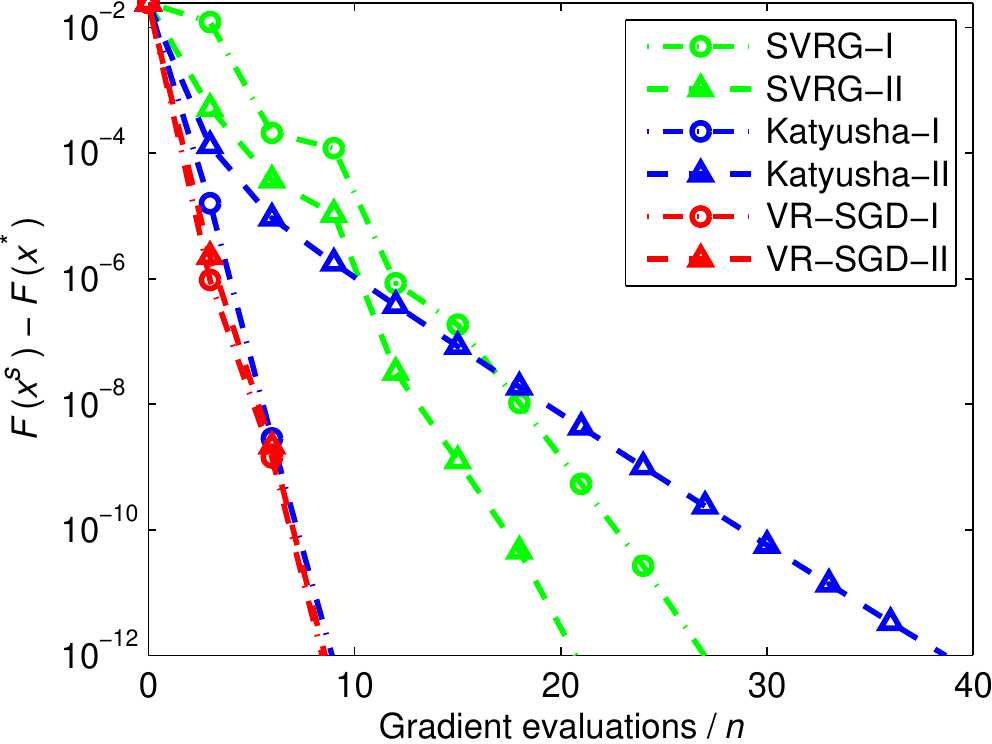}

\subfigure[$\lambda=10^{-3}$]{\includegraphics[width=0.326\columnwidth]{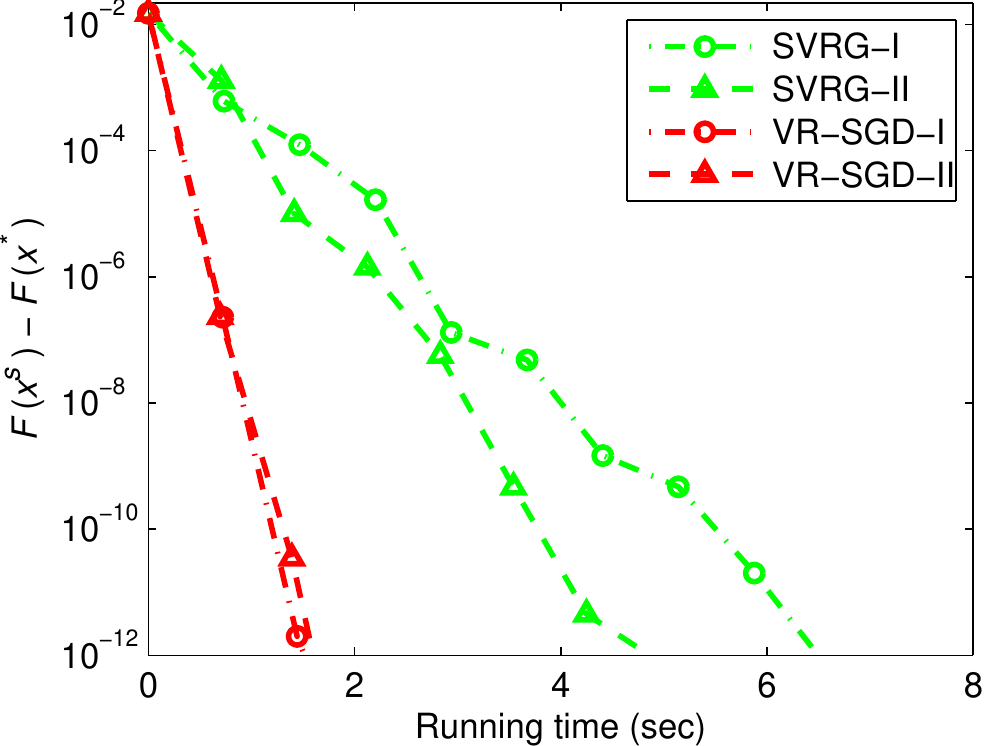}\label{figs5a}}\,
\subfigure[$\lambda=10^{-4}$]{\includegraphics[width=0.326\columnwidth]{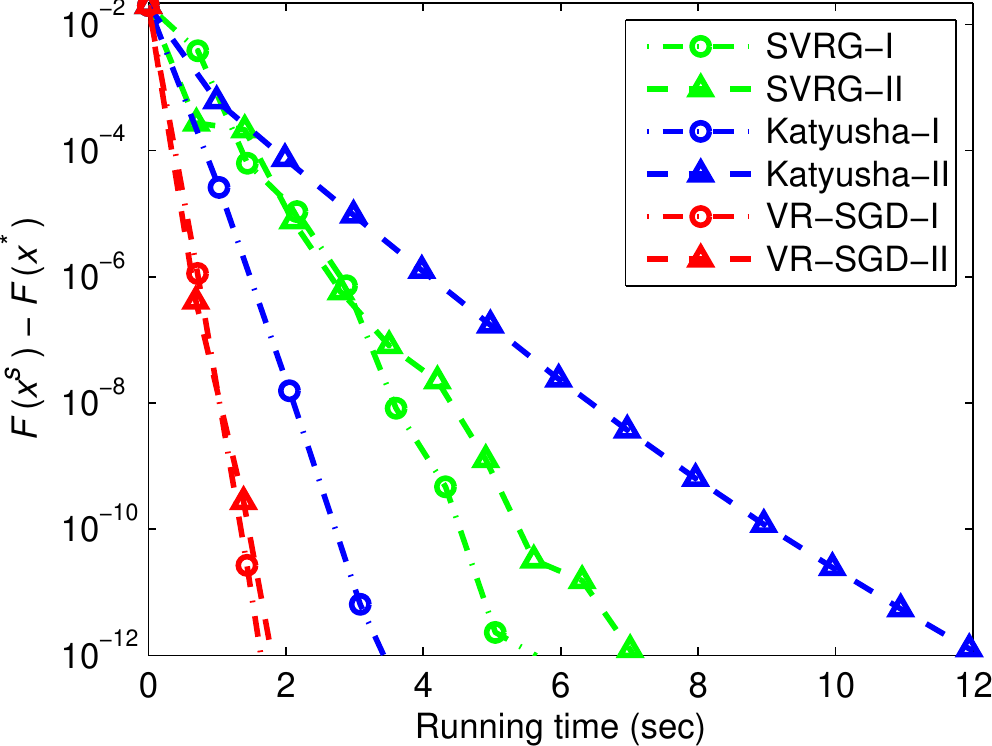}}\,
\subfigure[$\lambda=10^{-5}$]{\includegraphics[width=0.326\columnwidth]{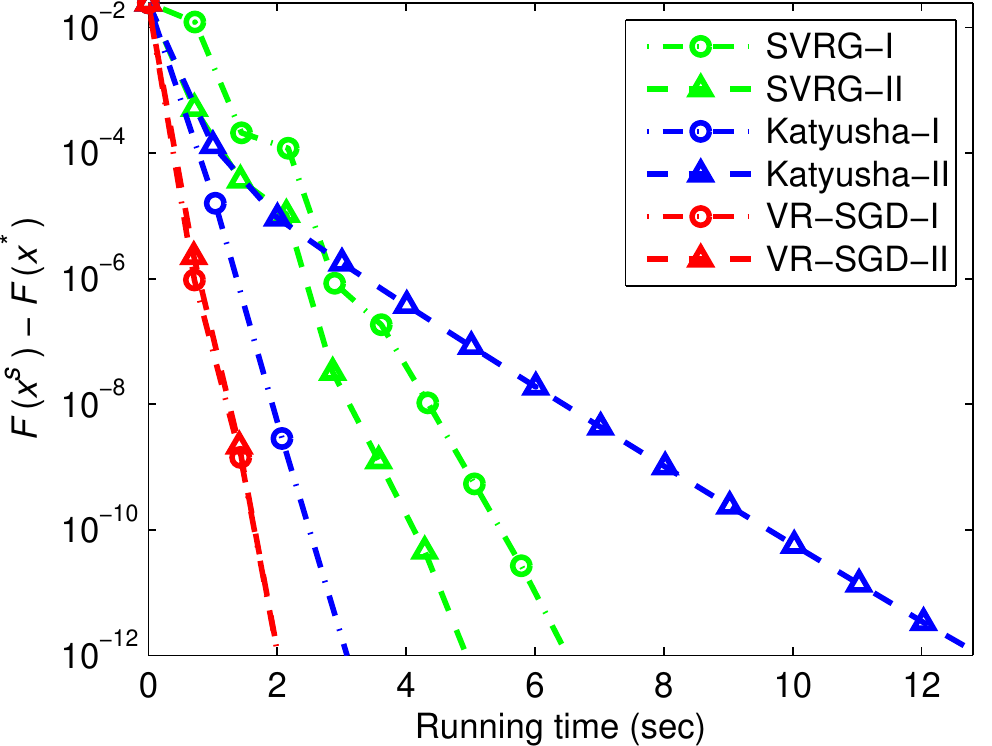}}
\vspace{1.6mm}

\includegraphics[width=0.326\columnwidth]{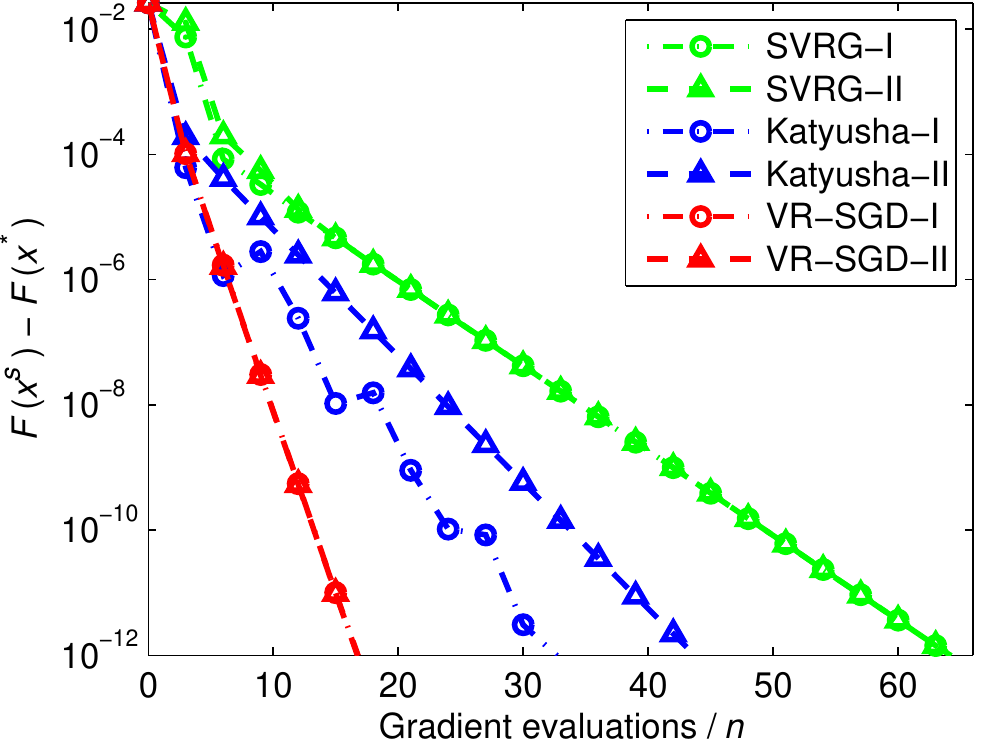}\,
\includegraphics[width=0.326\columnwidth]{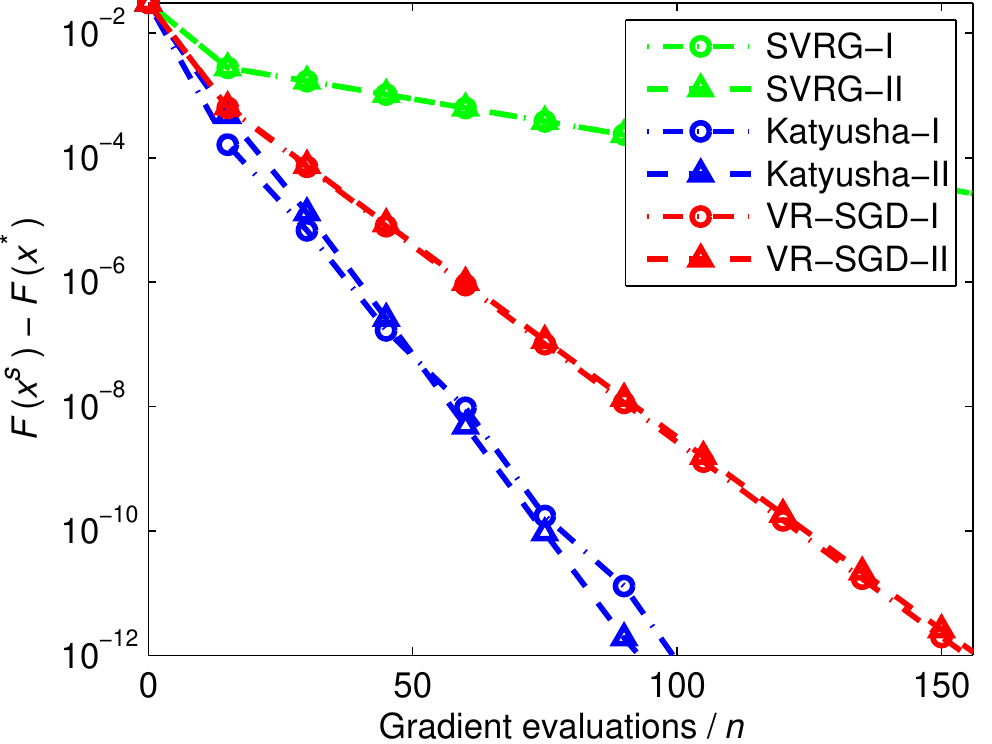}\,
\includegraphics[width=0.326\columnwidth]{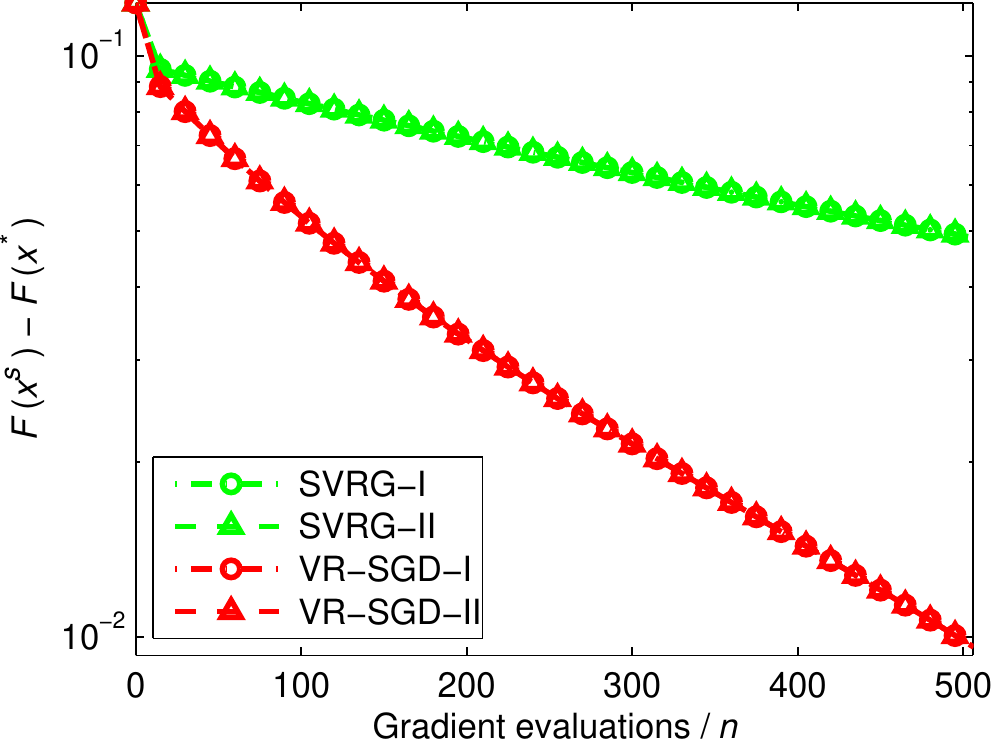}

\subfigure[$\lambda=10^{-6}$]{\includegraphics[width=0.326\columnwidth]{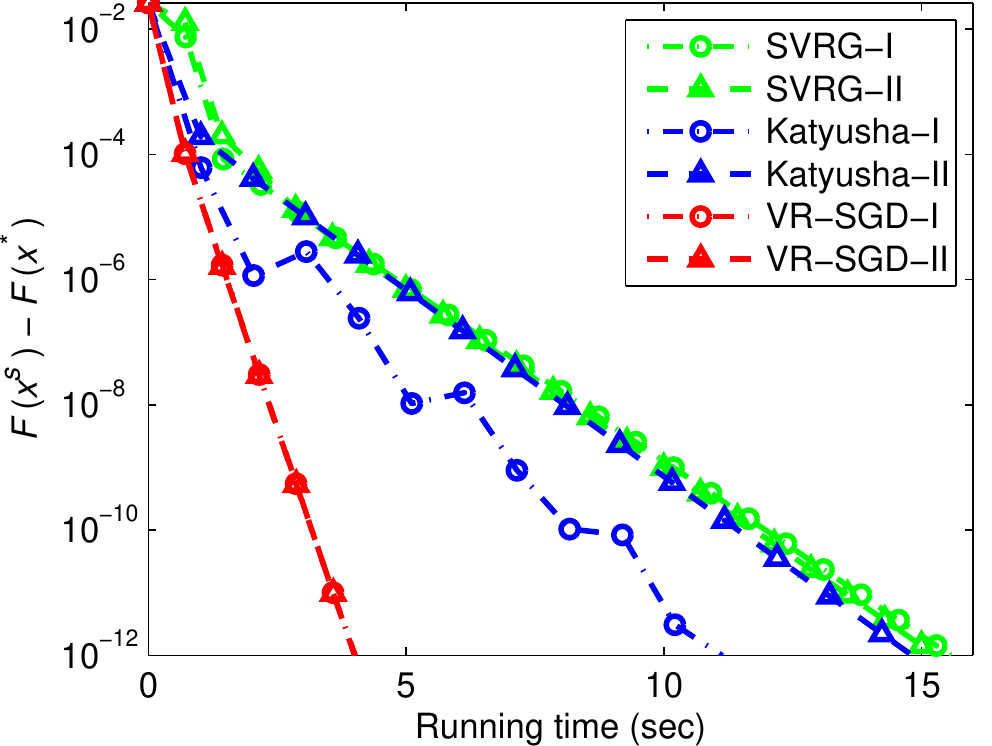}\label{figs5d}}\,
\subfigure[$\lambda=10^{-7}$]{\includegraphics[width=0.326\columnwidth]{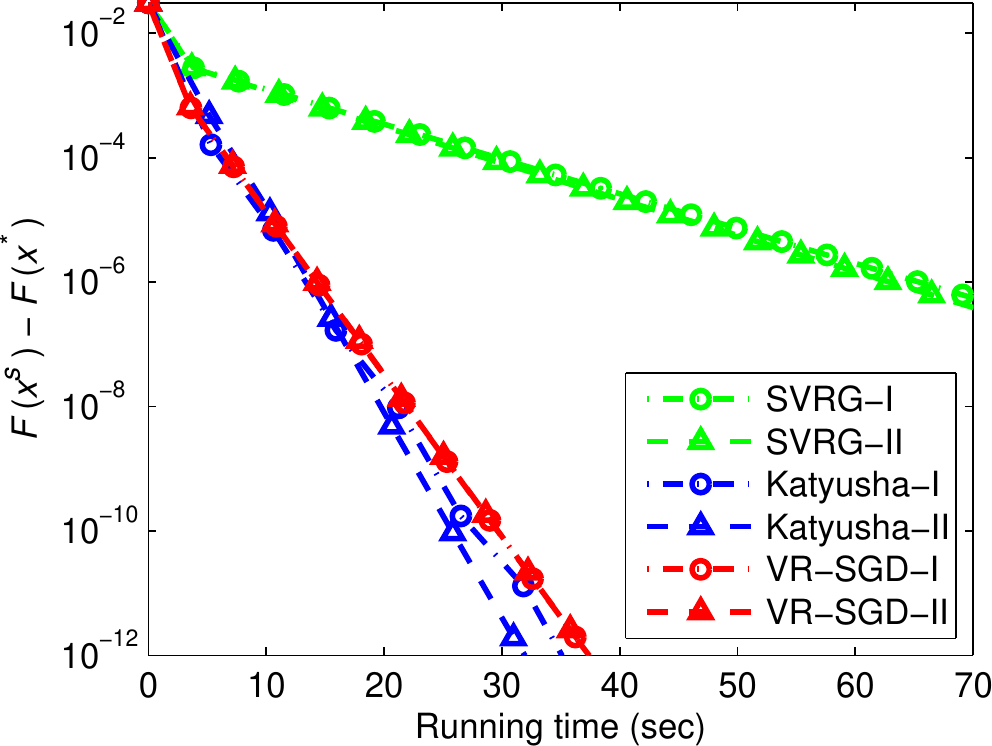}}\,
\subfigure[$\lambda=0$]{\includegraphics[width=0.326\columnwidth]{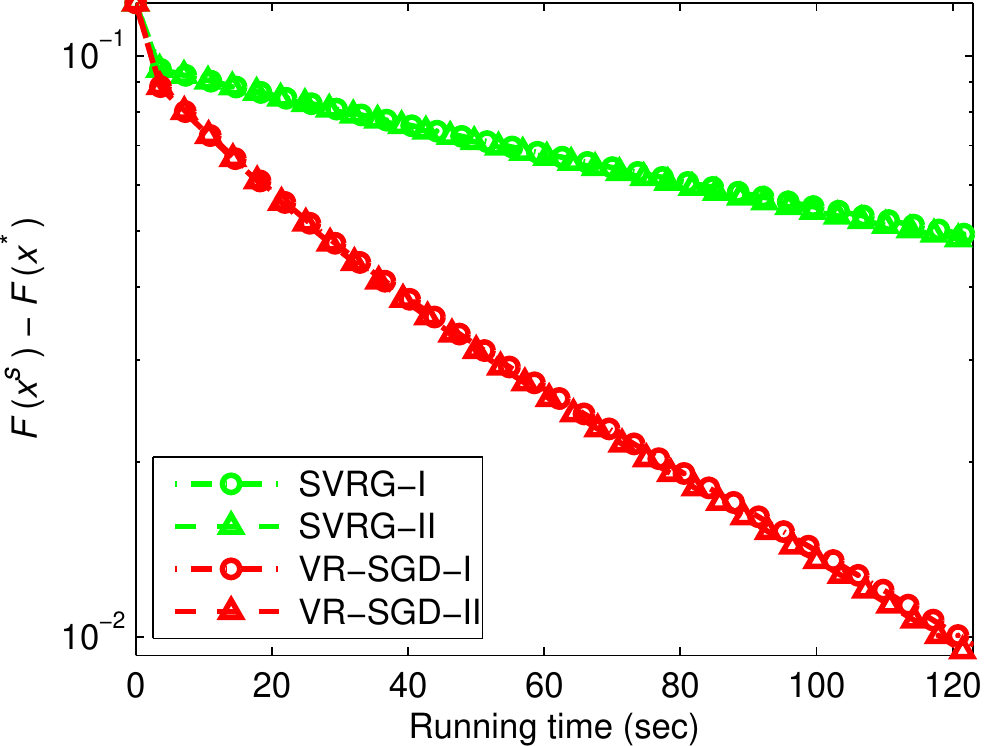}\label{figs5f}}
\caption{Comparison of SVRG~\cite{johnson:svrg}, Katyusha~\cite{zhu:Katyusha}, VR-SGD and their proximal versions for solving ridge regression problems with different regularization parameters on the Covtype data set. In each plot, the vertical axis shows the objective value minus the minimum, and the horizontal axis is the number of effective passes (top) or running time (bottom).}
\label{figs5}
\end{figure}

\begin{figure}[!th]
\centering
\subfigure[$\ell_{1}$-norm regularized logistic regression: Adult (left) \;and\; Covtype (right)]{
\includegraphics[width=0.486\columnwidth]{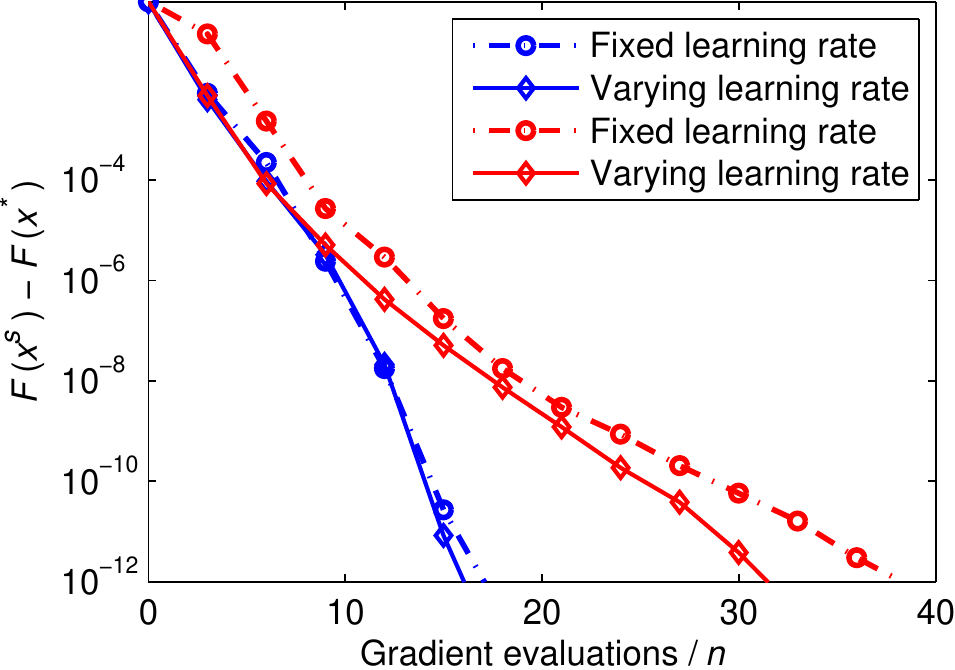}\;
\includegraphics[width=0.486\columnwidth]{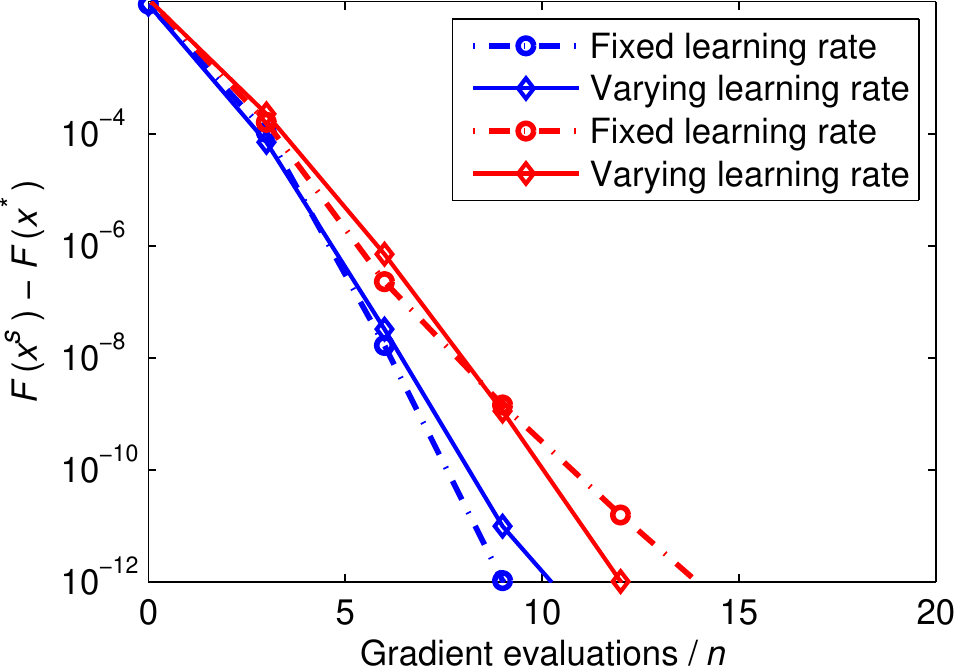}}
\vspace{0.6mm}

\subfigure[Lasso: Adult (left) \;and\; Covtype (right)]{
\includegraphics[width=0.486\columnwidth]{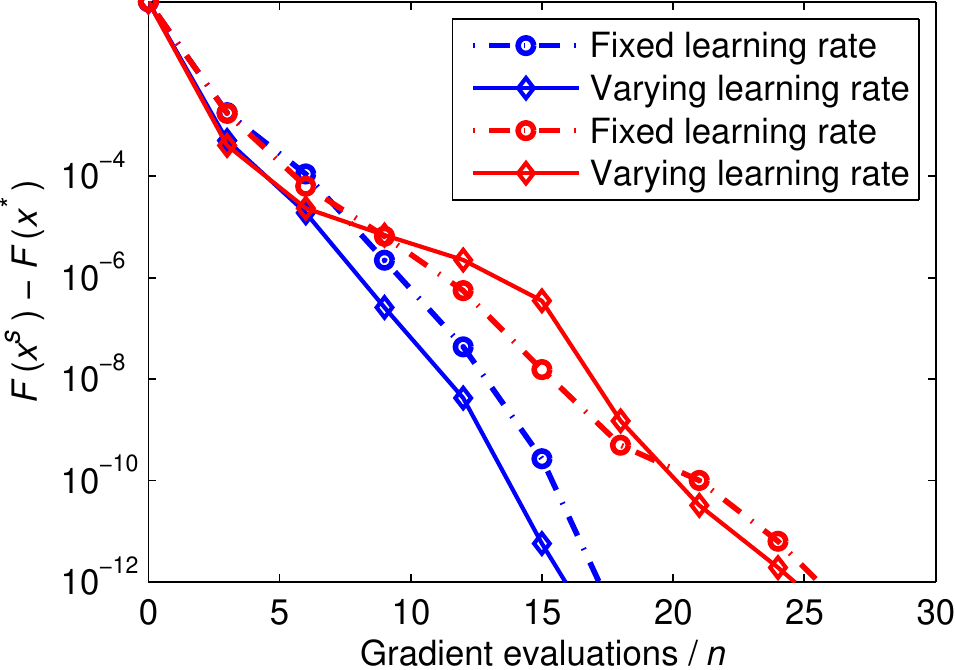}\;
\includegraphics[width=0.486\columnwidth]{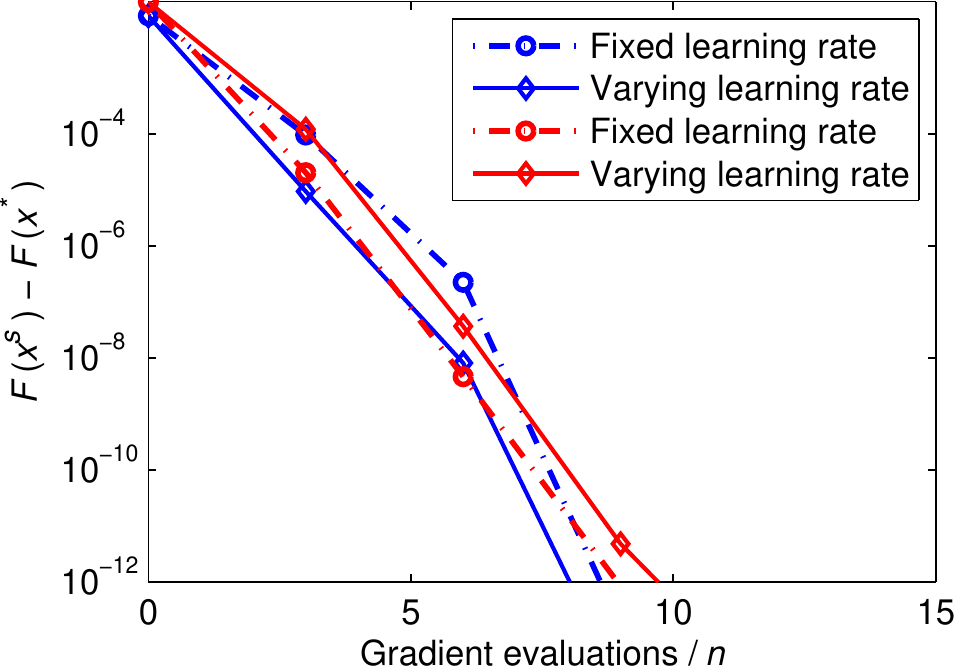}}
\caption{Comparison of Algorithm~\ref{alg3} with fixed and varying learning rates for solving $\ell_{1}$-norm (i.e., $\lambda\|x\|_{1}$) regularized logistic regression and Lasso problems with $\lambda=10^{-4}$ (blue lines) and $\lambda=10^{-5}$ (red lines, best viewed in colors). Note that the regularization parameter is set to $10^{-3}$ and $10^{-4}$ for solving Lasso problems on the Covtype data set.}
\label{figs6}
\end{figure}

\section{Conclusions}
In this paper, we proposed a simple variant of the original SVRG~\cite{johnson:svrg}, called variance reduced stochastic gradient descent (VR-SGD). Unlike the choices of the snapshot point and starting point in SVRG and its proximal variant, Prox-SVRG~\cite{xiao:prox-svrg}, the two points of each epoch in VR-SGD are set to the average and last iterate of the previous epoch, respectively. This setting allows us to use much larger learning rates than SVRG, e.g., $3/(7L)$ for VR-SGD vs.\ $1/(10L)$ for SVRG, and also makes VR-SGD more robust in terms of learning rate selection. Different from existing proximal stochastic methods such as Prox-SVRG~\cite{xiao:prox-svrg} and Katyusha~\cite{zhu:Katyusha}, we designed two different update rules for smooth and non-smooth objective functions, respectively, which makes our VR-SGD method suitable for non-smooth and/or non-strongly convex problems without using any reduction techniques as in~\cite{zhu:Katyusha,zhu:box}. In contrast, SVRG and Prox-SVRG cannot directly solve non-strongly convex objectives~\cite{zhu:univr}. Furthermore, our empirical results showed that for smooth problems stochastic gradient update rules as in (\ref{equ21}) are better choices than proximal stochastic gradient update formulas as in \eqref{equ14}.

On the practical side, the choices of the snapshot and starting points make VR-SGD significantly faster than its counterparts, SVRG and Prox-SVRG. On the theoretical side, the setting also makes our convergence analysis more challenging. We analyzed the convergence properties of VR-SGD for strongly convex objective functions, which show that VR-SGD attains a linear convergence rate. Moreover, we also provided the convergence guarantees of VR-SGD for non-strongly convex problems, which show that VR-SGD achieves a sub-linear convergence rate. All these results imply that VR-SGD is guaranteed to have a similar convergence rate with other variance reduced stochastic methods such as SVRG and Prox-SVRG, and a slower theoretical rate than accelerated methods such as Katyusha~\cite{zhu:Katyusha}. Nevertheless, various experimental results show that VR-SGD significantly outperforms than SVRG and Prox-SVRG, and is still much better than the best known stochastic method, Katyusha~\cite{zhu:Katyusha}.

\section*{Appendix A: Proof of Lemma~\ref{lemm2}}
Before proving Lemma~\ref{lemm2}, we first give and prove the following lemma.

\begin{lemma}\label{lemm11}
Suppose each $f_{i}(\cdot)$ is $L$-smooth, and let $x^{*}$ be the optimal solution of Problem \eqref{equ01} when $g(x)\!\equiv\!0$, then we have
\begin{displaymath}
\mathbb{E}\!\left[\left\|\nabla\! f_{i}(x)-\nabla\! f_{i}(x^{*})\right\|^{2}\right]\leq 2L\left[f(x)-f(x^{*})\right].
\end{displaymath}
\end{lemma}

\begin{proof}
Following Theorem 2.1.5 in~\cite{nesterov:co} and Lemma 3.4~\cite{xiao:prox-svrg}, we have
\begin{displaymath}
\begin{split}
\left\|\nabla\! f_{i}(x)-\nabla\! f_{i}(x^{*})\right\|^{2}
\leq2L\left[f_{i}(x)-f_{i}(x^{*})-\left\langle \nabla\! f_{i}(x^{*}),\;x-x^{*}\right\rangle\right].
\end{split}
\end{displaymath}
Summing the above inequality over $i=1,\ldots,n$, we obtain
\begin{displaymath}
\begin{split}
\mathbb{E}\!\left[\left\|\nabla\! f_{i}(x)\!-\!\nabla\! f_{i}(x^{*})\right\|^{2}\right]\!=\!\frac{1}{n}\!\sum^{n}_{i=1}\!\left\|\nabla\! f_{i}(x)\!-\!\nabla\! f_{i}(x^{*})\right\|^{2}
\leq\, 2L\left[f(x)-f(x^{*})-\left\langle \nabla\! f(x^{*}),\;x-x^{*}\right\rangle\right].
\end{split}
\end{displaymath}
By the optimality of $x^{*}$, i.e., $x^{*}=\mathop{\arg\min}_{x} f(x)$, we have $\nabla\! f(x^{*})=0$. Then
\begin{displaymath}
\begin{split}
\mathbb{E}\!\left[\left\|\nabla\! f_{i}(x)-\nabla\! f_{i}(x^{*})\right\|^{2}\right]
\leq&\, 2L\left[f(x)-f(x^{*})-\left\langle \nabla\! f(x^{*}),x-x^{*}\right\rangle\right]\\
=&\,2L\left[f(x)-f(x^{*})\right].
\end{split}
\end{displaymath}
\end{proof}

\textbf{Proof of Lemma~\ref{lemm2}:}
\begin{proof}
\begin{equation*}
\begin{split}
\mathbb{E}\!\left[\|\widetilde{\nabla}\! f_{i^{s}_{k}}(x^{s}_{k})-\nabla\! f(x^{s}_{k})\|^{2}\right]=\,&\mathbb{E}\!\left[\left\|\nabla\! f_{i^{s}_{k}}(x^{s}_{k})-\nabla\! f_{i^{s}_{k}}(\widetilde{x}^{s-\!1})+\nabla\! f(\widetilde{x}^{s-1})-\nabla\! f(x^{s}_{k})\right\|^{2}\right]\\
=\,&\mathbb{E}\!\left[\left\|\nabla\! f_{i^{s}_{k}}(x^{s}_{k})-\nabla\! f_{i^{s}_{k}}(\widetilde{x}^{s-1})\right\|^{2}\right]-\left\|\nabla\! f(x^{s}_{k})-\nabla\! f(\widetilde{x}^{s-1})\right\|^{2}\\
\leq\,&\mathbb{E}\!\left[\left\|\nabla\! f_{i^{s}_{k}}(x^{s}_{k})-\nabla\! f_{i^{s}_{k}}(\widetilde{x}^{s-1})\right\|^{2}\right]\\
\leq\,&2\mathbb{E}\!\left[\left\|\nabla\! f_{i^{s}_{k}}(x^{s}_{k})-\nabla\! f_{i^{s}_{k}}(x^{*})\right\|^{2}\right]+2\mathbb{E}\!\left[\left\|\nabla\! f_{i^{s}_{k}}(\widetilde{x}^{s-\!1})-\nabla\! f_{i^{s}_{k}}(x^{*})\right\|^{2}\right]\\
\leq\, & 4L\!\left[f(x^{s}_{k})-f(x^{*})+f(\widetilde{x}^{s-1})-f(x^{*})\right]\!,
\end{split}
\end{equation*}
where the second equality holds due to the fact that $\mathbb{E}[\|x\!-\!\mathbb{E}[x]\|^{2}]\!=\!\mathbb{E}[\|x\|^{2}]\!-\!\|\mathbb{E}[x]\|^{2}$; the second inequality holds due to the fact that $\|a-b\|^{2}\leq2(\|a\|^{2}+\|b\|^{2})$; and the last inequality follows from Lemma~\ref{lemm11}.
\end{proof}

\section*{Appendix B: Proof of Lemma~\ref{lemm5}}
\begin{proof}
For convenience, the stochastic gradient estimator is defined as: $v^{s}_{k}\!:=\!\nabla f_{i^{s}_{k}}(x^{s}_{k})\!-\!\nabla f_{i^{s}_{k}}(\widetilde{x}^{s-1})\!+\!\nabla f(\widetilde{x}^{s-1})$. Since each component function $f_{i}(x)$ is $L$-smooth, which implies that the gradient of the average function $f(x)$ is also $L$-smooth, i.e., for all $x,y\!\in\! \mathbb{R}^{d}$,
\begin{displaymath}
\|\nabla f(x)-\nabla f(y)\|\leq L\|x-y\|,
\end{displaymath}
whose equivalent form is
\begin{displaymath}
f(y)\leq f(x)+\langle\nabla f(x),\;y-x\rangle+\frac{L}{2}\|y-x\|^{2}.
\end{displaymath}
Using the above smoothness inequality, we have
\begin{equation}\label{equ35}
\begin{split}
f(x^{s}_{k+1})\leq\,& f(x^{s}_{k})+\left\langle\nabla f(x^{s}_{k}),\,x^{s}_{k+1}-x^{s}_{k}\right\rangle+\frac{L}{2}\!\left\|x^{s}_{k+1}-x^{s}_{k}\right\|^{2}\\
=\,& f(x^{s}_{k})+\left\langle\nabla f(x^{s}_{k}),\,x^{s}_{k+1}-x^{s}_{k}\right\rangle+\frac{L\beta }{2}\!\left\|x^{s}_{k+1}-x^{s}_{k}\right\|^{2}-\frac{L(\beta \!-\!1)}{2}\!\left\|x^{s}_{k+1}-x^{s}_{k}\right\|^{2}\\
=\,& f(x^{s}_{k})+\left\langle v^{s}_{k},\,x^{s}_{k+1}-x^{s}_{k}\right\rangle+\frac{L\beta }{2}\|x^{s}_{k+1}-x^{s}_{k}\|^2\\
&+\left\langle\nabla f(x^{s}_{k})-v^{s}_{k},\,x^{s}_{k+1}-x^{s}_{k}\right\rangle-\frac{L(\beta \!-\!1)}{2}\|x^{s}_{k+1}-x^{s}_{k}\|^{2}.
\end{split}
\end{equation}
Using Lemma~\ref{lemm2}, then we get
\begin{equation}\label{equ36}
\begin{split}
&\mathbb{E}\!\left[\left\langle\nabla\! f(x^{s}_{k})-v^{s}_{k},\,x^{s}_{k+1}-x^{s}_{k}\right\rangle-\frac{L(\beta \!-\!1)}{2}\|x^{s}_{k+1}-x^{s}_{k}\|^{2}\right]\\
\leq\,& \mathbb{E}\!\left[\frac{1}{2L(\beta \!-\!1)}\|\nabla\!f(x^{s}_{k})-v^{s}_{k}\|^{2}+\frac{L(\beta \!-\!1)}{2}\|x^{s}_{k+1}\!-\!x^{s}_{k}\|^{2}-\frac{L(\beta \!-\!1)}{2}\|x^{s}_{k+1}\!-\!x^{s}_{k}\|^{2}\right]\\
\leq\,& \frac{2}{\beta \!-\!1}\!\left[f(x^{s}_{k})-f(x^{*})+f(\widetilde{x}^{s-1})-f(x^{*})\right],
\end{split}
\end{equation}
where the first inequality holds due to the Young's inequality (i.e., $y^{T}z\!\leq\!{\|y\|^2}/{(2\theta)}\!+\!{\theta\|z\|^2}/{2}$ for all $\theta\!>\!0$ and $y,z\!\in\! \mathbb{R}^{d}$), and the second inequality follows from Lemma~\ref{lemm2}.

Taking the expectation over the random choice of $i^{s}_{k}$ and substituting the inequality in \eqref{equ36} into the inequality in \eqref{equ35}, we have
\begin{equation*}
\begin{split}
&\quad\;\, \mathbb{E}[f(x^{s}_{k+1})]\\
&\leq \mathbb{E}[f(x^{s}_{k})]+\mathbb{E}\!\left[\left\langle v^{s}_{k}, \,x^{s}_{k+1}\!-x^{s}_{k}\right\rangle+\frac{L\beta }{2}\|x^{s}_{k+1}\!-x^{s}_{k}\|^2\right]+\frac{2}{\beta \!-\!1}\!\left[f(x^{s}_{k})-f(x^{*})+f(\widetilde{x}^{s-1})-f(x^{*})\right]\\
&\leq \mathbb{E}[f(x^{s}_{k})]+\mathbb{E}\!\left[\left\langle v^{s}_{k}, \,x^{*}\!-x^{s}_{k}\right\rangle+\frac{L\beta }{2}(\|x^{*}-x^{s}_{k}\|^2-\|x^{*}-x^{s}_{k+1}\|^2)\right]+\frac{2}{\beta \!-\!1}\!\left[f(x^{s}_{k})-f(x^{*})+f(\widetilde{x}^{s-1})-f(x^{*})\right]\\
&\leq\mathbb{E}[f(x^{s}_{k})]+\left\langle \nabla\! f(x^{s}_{k}), \,x^{*}\!-x^{s}_{k}\right\rangle+\mathbb{E}\!\left[\frac{L\beta }{2}(\|x^{*}\!-x^{s}_{k}\|^2-\|x^{*}\!-x^{s}_{k+1}\|^2)\right]+\frac{2}{\beta \!-\!1}\!\left[f(x^{s}_{k})-f(x^{*})+f(\widetilde{x}^{s-1})-f(x^{*})\right]\\
&\leq f(x^{*})+\mathbb{E}\!\left[\frac{L\beta }{2}(\|x^{*}-x^{s}_{k}\|^2-\|x^{*}-x^{s}_{k+1}\|^2)\right]+\frac{2}{\beta \!-\!1}\!\left[f(x^{s}_{k})-f(x^{*})+f(\widetilde{x}^{s-1})-f(x^{*})\right]\\
&= f(x^{*})+\frac{L\beta }{2}\mathbb{E}\!\left[\left(\|x^{*}-x^{s}_{k}\|^2-\|x^{*}-x^{s}_{k+1}\|^2\right)\right]+\frac{2}{\beta \!-\!1}\left[f(x^{s}_{k})-f(x^{*})+f(\widetilde{x}^{s-1})-f(x^{*})\right],
\end{split}
\end{equation*}
where the first inequality holds due to the inequality in \eqref{equ35} and the inequality in \eqref{equ36}; the second inequality follows from Lemma~\ref{lemm1} with $\hat{z}\!=\!x^{s}_{k+1}$, $z\!=\!x^{*}$, $z_{0}\!=\!x^{s}_{k}$, $\tau\!=\!L\beta \!=\!1/\eta$, and $r(z)\!:=\!\langle v^{s}_{k}, \,z\!-\!x^{s}_{k}\rangle$; the third inequality holds due to the fact that $\mathbb{E}[v^{s}_{k}]\!=\!\nabla\! f(x^{s}_{k})$; and the last inequality follows from the convexity of $f(\cdot)$, i.e., $f(x^{s}_{k})\!+\!\langle\nabla\! f(x^{s}_{k}), \,x^{*}\!-\!x^{s}_{k}\rangle\!\leq\!f(x^{*})$. The above inequality can be rewritten as follows:
\begin{equation*}
\begin{split}
\mathbb{E}[f(x^{s}_{k+1})]-f(x^{*})
\leq\frac{2}{\beta \!-\!1}\!\left[f(x^{s}_{k})-f(x^{*})+f(\widetilde{x}^{s-1})-f(x^{*})\right]+ \frac{L\beta}{2}\mathbb{E}\!\left[\|x^{*}\!-x^{s}_{k}\|^2-\|x^{*}\!-x^{s}_{k+1}\|^2\right].
\end{split}
\end{equation*}
Summing the above inequality over $k=0,1,\ldots,m\!-\!1$, we obtain
\begin{equation*}
\begin{split}
&\sum^{m-1}_{k=0}\left\{\mathbb{E}[f(x^{s}_{k+1})]-f(x^{*})\right\}\\
\leq& \sum^{m-1}_{k=0}\left\{\frac{2}{\beta \!-\!1}\!\left[f(x^{s}_{k})-f(x^{*})+f(\widetilde{x}^{s-1})-f(x^{*})\right]+ \frac{L\beta }{2}\mathbb{E}\!\left[\|x^{*}-x^{s}_{k}\|^2-\|x^{*}-x^{s}_{k+1}\|^2\right]\right\}.
\end{split}
\end{equation*}
Then
\begin{equation}\label{equ37}
\begin{split}
&\left(1-\frac{2}{\beta\!-\!1}\right)\sum^{m}_{k=1}\left\{\mathbb{E}[f(x^{s}_{k})]-f(x^{*})\right\}+\frac{2}{\beta\!-\!1}\sum^{m}_{k=1}\left\{\mathbb{E}[f(x^{s}_{k})]-f(x^{*})\right\}\\
\leq& \sum^{m}_{k=1}\left\{\frac{2}{\beta \!-\!1}\!\left[f(x^{s}_{k-1})-f(x^{*})+f(\widetilde{x}^{s-1})-f(x^{*})\right]\right\}+ \frac{L\beta }{2}\mathbb{E}\!\left[\|x^{*}-x^{s}_{0}\|^2-\|x^{*}-x^{s}_{m}\|^2\right].
\end{split}
\end{equation}

Due to the setting of $\widetilde{x}^{s}\!=\!\frac{1}{m}\sum^{m}_{k=1}x^{s}_{k}$ in Option II, and the convexity of $f(\cdot)$, then
\begin{equation*}
f(\widetilde{x}^{s})\leq \frac{1}{m}\sum^{m}_{k=1}f(x^{s}_{k}).
\end{equation*}
Using the above inequality, the inequality in \eqref{equ37} becomes
\begin{equation*}
\begin{split}
&m\left(1-\frac{2}{\beta\!-\!1}\right)\mathbb{E}\!\left[f(\widetilde{x}^{s})-f(x^{*})\right]+\frac{2}{\beta\!-\!1}\sum^{m}_{k=1}\left\{\mathbb{E}[f(x^{s}_{k})]-f(x^{*})\right\}\\
\leq& \sum^{m}_{k=1}\left\{\frac{2}{\beta \!-\!1}\!\left[f(x^{s}_{k-1})-f(x^{*})+f(\widetilde{x}^{s-1})-f(x^{*})\right]\right\}+ \frac{L\beta }{2}\mathbb{E}\!\left[\|x^{*}-x^{s}_{0}\|^2-\|x^{*}-x^{s}_{m}\|^2\right].
\end{split}
\end{equation*}
Dividing both sides of the above inequality by $m$ and subtracting $\frac{2}{(\beta-1)m}\sum^{m-1}_{k=1}\!\left[f(x^{s}_{k})\!-\!f(x^{*})\right]$ from both sides, we arrive at
\begin{equation*}
\begin{split}
&\left(1-\frac{2}{\beta\!-\!1}\right)\mathbb{E}\!\left[f(\widetilde{x}^{s})-f(x^{*})\right]+\frac{2}{(\beta\!-\!1)m}\mathbb{E}[f(x^{s}_{m})-f(x^{*})]\\
\leq &\, \frac{2}{(\beta\!-\!1)}\mathbb{E}\!\left[f(\widetilde{x}^{s-1})-f(x^{*})\right]+\frac{2}{(\beta\!-\!1)m}\mathbb{E}\!\left[f(x^{s}_{0})-f(x^{*})\right]+\frac{L\beta}{2m}\mathbb{E}\!\left[\|x^{*}-x^{s}_{0}\|^2-\|x^{*}-x^{s}_{m}\|^2\right].
\end{split}
\end{equation*}

This completes the proof.
\end{proof}

\section*{Appendix C: Proof of Theorem~\ref{the2}}
\begin{proof}
Using Lemma~\ref{lemm5}, we have
\begin{equation*}
\begin{split}
&(1-\frac{2}{\beta\!-\!1})\mathbb{E}\!\left[f(\widetilde{x}^{s})-f(x^{*})\right]+\frac{2}{(\beta\!-\!1)m}\mathbb{E}[f(x^{s}_{m})-f(x^{*})]\\
\leq&\,\frac{2}{(\beta\!-\!1)m}\mathbb{E}\!\left[f(x^{s}_{0})-f(x^{*})\right]+\frac{2}{\beta\!-\!1}\mathbb{E}\!\left[f(\widetilde{x}^{s-1})-f(x^{*})\right]+\frac{L\beta}{2m}\mathbb{E}[\|x^{*}-x^{s}_{0}\|^{2}-\|x^{*}-x^{s}_{m}\|^{2}].
\end{split}
\end{equation*}
Summing the above inequality over $s\!=\!1,2,\ldots,S$, taking expectation with respect to the history of random variables $i^{s}_{k}$, and using the setting of $x^{s+1}_{0}=x^{s}_{m}$ in Algorithm~\ref{alg3}, we arrive at
\begin{equation*}
\begin{split}
&\sum^{S}_{s=1}(1\!-\!\frac{2}{\beta\!-\!1})\mathbb{E}\!\left[f(\widetilde{x}^{s})-f(x^{*})\right]\\
\leq&\,\sum^{S}_{s=1}\left\{\frac{2}{(\beta\!-\!1)m}\mathbb{E}\!\left\{f(x^{s}_{0})-f(x^{*})-\left[f(x^{s}_{m})-f(x^{*})\right]\right\}+\frac{2}{\beta\!-\!1}\mathbb{E}\!\left[f(\widetilde{x}^{s-1})-f(x^{*})\right]\right\}\\
&\,+\frac{L\beta}{2m}\sum^{S}_{s=1}\mathbb{E}\!\left[\|x^{*}-x^{s}_{0}\|^{2}-\|x^{*}-x^{s}_{m}\|^{2}\right].
\end{split}
\end{equation*}
Subtracting $\frac{2}{\beta-1}\sum^{S}_{s=1}\!\left[f(\widetilde{x}^{s})-f(x^{*})\right]$ from both sides of the above inequality, we obtain
\begin{equation*}
\begin{split}
&\sum^{S}_{s=1}\left(1-\frac{4}{\beta\!-\!1}\right)\mathbb{E}\!\left[f(\widetilde{x}^{s})-f(x^{*})\right]\\
\leq&\,\frac{2}{(\beta\!-\!1)m}\mathbb{E}\!\left\{f(x^{1}_{0})-f(x^{*})-[f(x^{S}_{m})-f(x^{*})]\right\}+\frac{2}{\beta\!-\!1}\mathbb{E}[f(\widetilde{x}^{0})-f(\widetilde{x}^{S})]\\
&\,+\frac{L\beta}{2m}\mathbb{E}\!\left[\|x^{*}-x^{1}_{0}\|^{2}-\|x^{*}-x^{S}_{m}\|^{2}\right].
\end{split}
\end{equation*}

It is not hard to verify that $\mathbb{E}[f(\widetilde{x}^{0})\!-\!f(\widetilde{x}^{S})]\!\leq\!f(\widetilde{x}^{0})\!-\!f(x^{*})$. Dividing both sides of the above inequality by $S$, and using the choice $\widetilde{x}^{0}\!=\!x^{1}_{0}$, we have
\begin{equation}\label{equ38}
\begin{split}
&\frac{1}{S}\left(1-\frac{4}{\beta\!-\!1}\right)\sum^{S}_{s=1}\mathbb{E}\!\left[f(\widetilde{x}^{s})-f(x^{*})\right]\\
\leq&\,\frac{2}{(\beta\!-\!1)mS}\mathbb{E}\!\left\{f(x^{1}_{0})-f(x^{*})-[f(x^{S}_{m})-f(x^{*})]\right\}+\frac{2}{(\beta\!-\!1)S}[f(\widetilde{x}^{0})-f(x^{*})]+\frac{L\beta}{2mS}\|x^{*}-x^{1}_{0}\|^{2}\\
\leq&\,\frac{2}{(\beta\!-\!1)mS}\!\left[f(\widetilde{x}^{0})-f(x^{*})\right]+\frac{2}{(\beta\!-\!1)S}[f(\widetilde{x}^{0})-f(x^{*})]+\frac{L\beta}{2mS}\|x^{*}-\widetilde{x}^{0}\|^{2}\\
=&\,\frac{2(m\!+\!1)}{(\beta\!-\!1)mS}[f(\widetilde{x}^{0})-f(x^{*})]+\frac{L\beta}{2mS}\|\widetilde{x}^{0}-x^{*}\|^{2}
\end{split}
\end{equation}
where the first inequality holds due to the facts that $\mathbb{E}[f(\widetilde{x}^{0})\!-\!f(\widetilde{x}^{S})]\!\leq\!f(\widetilde{x}^{0})\!-\!f(x^{*})$ and $\mathbb{E}\!\left[\|x^{*}\!-\!x^{S}_{m}\|^{2}\right]\!\geq\!0$, and the last inequality uses the facts that $\mathbb{E}\!\left[f(x^{S}_{m})\!-\!f(x^{*})\right]\!\geq\!0$ and $\widetilde{x}^{0}\!=\!x^{1}_{0}$.

Since $\overline{x}^{S}\!=\!\frac{1}{S}\!\sum^{S}_{s=1}\widetilde{x}^{s}$, and using the convexity of $f(\cdot)$, we have $f(\overline{x}^{S})\!\leq\! \frac{1}{S}\!\sum^{S}_{s=1}\!f(\widetilde{x}^{s})$, and therefore the inequality in \eqref{equ38} becomes
\begin{equation*}
\begin{split}
\left(1-\frac{4}{\beta\!-\!1}\right)\mathbb{E}\!\left[f(\overline{x}^{S})-f(x^{*})\right]\leq&\: \frac{1}{S}\left(1-\frac{4}{\beta\!-\!1}\right)\sum^{S}_{s=1}\mathbb{E}\!\left[f(\widetilde{x}^{s})-f(x^{*})\right]\\
\leq&\,\frac{2(m\!+\!1)}{(\beta\!-\!1)mS}[f(\widetilde{x}^{0})-f(x^{*})]+\frac{L\beta}{2mS}\|\widetilde{x}^{0}-x^{*}\|^{2}.
\end{split}
\end{equation*}
Dividing both sides of the above inequality by $(1\!-\!\frac{4}{\beta-1})\!>\!0$ (i.e., $\eta\!<\!1/(5L)$), we arrive at
\begin{equation*}
\mathbb{E}\!\left[f(\overline{x}^{S})-f(x^{*})\right]\leq\frac{2(m\!+\!1)}{mS(\beta\!-\!5)}[f(\widetilde{x}^{0})-f(x^{*})]+\frac{L\beta(\beta\!-\!1)}{2mS(\beta\!-\!5)}\|\widetilde{x}^{0}-x^{*}\|^{2}.
\end{equation*}

If $f(\widetilde{x}^{S})\!\leq\!f(\overline{x}^{S})$, then $\widehat{x}^{S}\!=\!\widetilde{x}^{S}$, and
\begin{equation*}
\begin{split}
\mathbb{E}\!\left[f(\widehat{x}^{S})-f(x^{*})\right]\leq&\,\mathbb{E}\!\left[f(\overline{x}^{S})-f(x^{*})\right]\\
\leq&\,\frac{2(m\!+\!1)}{mS(\beta\!-\!5)}[f(\widetilde{x}^{0})-f(x^{*})]+\frac{L\beta(\beta\!-\!1)}{2mS(\beta\!-\!5)}\|\widetilde{x}^{0}-x^{*}\|^{2}.
\end{split}
\end{equation*}
Alternatively, if $f(\widetilde{x}^{S})\!\geq\!f(\overline{x}^{S})$, then $\widehat{x}^{S}\!=\!\overline{x}^{S}$, and the above inequality still holds.

This completes the proof.
\end{proof}

\section*{Appendix D: Proof of Lemma~\ref{lemm6}}
\begin{proof}
Since the average function $f(x)$ is $L$-smooth, then for all $x,y\!\in\! \mathbb{R}^{d}$,
\begin{displaymath}
f(y)\leq f(x)+\langle\nabla f(x),\;y-x\rangle+\frac{L}{2}\|y-x\|^{2},
\end{displaymath}
which then implies
\begin{displaymath}
f(x^{s}_{k+1})\leq f(x^{s}_{k})+\left\langle\nabla f(x^{s}_{k}),\,x^{s}_{k+1}-x^{s}_{k}\right\rangle+\frac{L}{2}\!\left\|x^{s}_{k+1}\!-x^{s}_{k}\right\|^{2}.
\end{displaymath}
Using the above inequality, we have
\begin{equation}\label{equ41}
\begin{split}
F(x^{s}_{k+1})=\,&f(x^{s}_{k+1})+g(x^{s}_{k+1})\\
\leq\,& f(x^{s}_{k})+g(x^{s}_{k+1})+\left\langle\nabla f(x^{s}_{k}),\,x^{s}_{k+1}-x^{s}_{k}\right\rangle+\frac{L\beta}{2}\!\left\|x^{s}_{k+1}-x^{s}_{k}\right\|^{2}-\frac{L(\beta\!-\!1)}{2}\!\left\|x^{s}_{k+1}-x^{s}_{k}\right\|^{2}\\
=\,& f(x^{s}_{k})+g(x^{s}_{k+1})+\left\langle v^{s}_{k},\,x^{s}_{k+1}-x^{s}_{k}\right\rangle+\frac{L\beta}{2}\|x^{s}_{k+1}-x^{s}_{k}\|^2\\
&+\left\langle\nabla f(x^{s}_{k})-v^{s}_{k},\,x^{s}_{k+1}-x^{s}_{k}\right\rangle-\frac{L(\beta\!-\!1)}{2}\|x^{s}_{k+1}-x^{s}_{k}\|^{2}.
\end{split}
\end{equation}
According to Lemma~\ref{lemm5}, then we obtain
\begin{equation}\label{equ42}
\begin{split}
&\mathbb{E}\!\left[\left\langle\nabla\! f(x^{s}_{k})-v^{s}_{k},\,x^{s}_{k+1}-x^{s}_{k}\right\rangle-\frac{L(\beta\!-\!1)}{2}\|x^{s}_{k+1}-x^{s}_{k}\|^{2}\right]\\
\leq\,& \mathbb{E}\!\left[\frac{1}{2L(\beta\!-\!1)}\|\nabla\!f(x^{s}_{k})-v^{s}_{k}\|^{2}+\frac{L(\beta\!-\!1)}{2}\|x^{s}_{k+1}\!-\!x^{s}_{k}\|^{2}-\frac{L(\beta\!-\!1)}{2}\|x^{s}_{k+1}\!-\!x^{s}_{k}\|^{2}\right]\\
\leq\,& \frac{2}{\beta\!-\!1}\!\left[F(x^{s}_{k})-F(x^{*})+F(\widetilde{x}^{s-1})-F(x^{*})\right],
\end{split}
\end{equation}
where the first inequality holds due to the Young's inequality, and the second inequality follows from Lemma~\ref{lemm5}. Substituting the inequality \eqref{equ42} into the inequality \eqref{equ41}, and taking the expectation over the random choice $i^{s}_{k}$, we arrive at
\begin{equation*}
\begin{split}
\mathbb{E}\!\left[F(x^{s}_{k+1})\right]
\leq&\, \mathbb{E}\!\left[f(x^{s}_{k})\right]+\mathbb{E}\!\left[g(x^{s}_{k+1})\right]+\mathbb{E}\!\left[\left\langle v^{s}_{k}, \,x^{s}_{k+1}-x^{s}_{k}\right\rangle+\frac{L\beta}{2}\|x^{s}_{k+1}-x^{s}_{k}\|^2\right]\\
&\,+\frac{2}{\beta\!-\!1}\!\left[F(x^{s}_{k})-F(x^{*})+F(\widetilde{x}^{s-1})-F(x^{*})\right]\\
\leq&\, \mathbb{E}\!\left[f(x^{s}_{k})\right]+g(x^{*})+\mathbb{E}\!\left[\left\langle v^{s}_{k}, \,x^{*}-x^{s}_{k}\right\rangle+\frac{L\beta}{2}(\|x^{*}-x^{s}_{k}\|^2-\|x^{*}-x^{s}_{k+1}\|^2)\right]\\
&\,+\frac{2}{\beta\!-\!1}\!\left[F(x^{s}_{k})-F(x^{*})+F(\widetilde{x}^{s-1})-F(x^{*})\right]\\
\leq&\, f(x^{*})+g(x^{*})+\mathbb{E}\!\left[\frac{L\beta}{2}(\|x^{*}-x^{s}_{k}\|^2-\|x^{*}-x^{s}_{k+1}\|^2)\right]\\
&\,+\frac{2}{\beta\!-\!1}\!\left[F(x^{s}_{k})-F(x^{*})+F(\widetilde{x}^{s-1})-F(x^{*})\right]\\
=&\,F(x^{*})+\frac{L\beta}{2}\mathbb{E}\!\left[\|x^{*}\!-x^{s}_{k}\|^2-\|x^{*}\!-x^{s}_{k+1}\|^2\right]+\frac{2}{\beta\!-\!1}\left[F(x^{s}_{k})-F(x^{*})+F(\widetilde{x}^{s-1})-F(x^{*})\right],\\
\end{split}
\end{equation*}
where the first inequality holds due to the inequality \eqref{equ41} and the inequality \eqref{equ42}; the second inequality follows from Lemma~\ref{lemm1} with $\hat{z}\!=\!x^{s}_{k+1}$, $z\!=\!x^{*}$, $z_{0}\!=\!x^{s}_{k}$, $\tau\!=\!L\beta \!=\!1/\eta$, and $r(z)\!:=\!\langle v^{s}_{k}, \,z\!-\!x^{s}_{k}\rangle\!+\!g(z)$; and the third inequality holds due to the fact that $\mathbb{E}[v^{s}_{k}]\!=\!\nabla\! f(x^{s}_{k})$ and the convexity of $f(\cdot)$, i.e., $f(x^{s}_{k})\!+\!\langle\nabla\! f(x^{s}_{k}), \,x^{*}\!-\!x^{s}_{k}\rangle\!\leq\!f(x^{*})$. Then the above inequality is rewritten as follows:
\begin{equation}
\begin{split}
&\:\mathbb{E}[F(x^{s}_{k+1})]-F(x^{*})\\
\leq&\:\frac{2}{\beta\!-\!1}\!\left[F(x^{s}_{k})-F(x^{*})+F(\widetilde{x}^{s-1})-F(x^{*})\right]+\frac{L\beta}{2}\mathbb{E}\!\left[\|x^{*}\!-x^{s}_{k}\|^2-\|x^{*}\!-x^{s}_{k+1}\|^2\right].
\end{split}
\end{equation}

Summing the above inequality over $k=0,1,\ldots,(m\!-\!1)$ and taking expectation over whole history, we have
\begin{equation*}
\begin{split}
&\sum^{m-1}_{k=0}\left\{\mathbb{E}[F(x^{s}_{k+1})]-F(x^{*})\right\}\\
\leq& \sum^{m-1}_{k=0}\left\{\frac{2}{\beta\!-\!1}\!\left[F(x^{s}_{k})-F(x^{*})+F(\widetilde{x}^{s-1})-F(x^{*})\right]+ \frac{L\beta}{2}\mathbb{E}\!\left[\|x^{*}-x^{s}_{k}\|^2-\|x^{*}-x^{s}_{k+1}\|^2\right]\right\}.
\end{split}
\end{equation*}
Subtracting $\frac{2}{\beta-1}\!\sum^{m-2}_{k=0}\mathbb{E}\!\left[F(x^{s}_{k+1})\!-\!F(x^{*})\right]$ from both sides of the above inequality, we obtain
\begin{equation*}
\begin{split}
&\sum^{m-1}_{k=0}\mathbb{E}\!\left[F(x^{s}_{k+1})-F(x^{*})\right]-\frac{2}{\beta\!-\!1}\!\sum^{m-1}_{k=0}\mathbb{E}\!\left[F(x^{s}_{k+1})\!-\!F(x^{*})\right]+\frac{2}{\beta\!-\!1}\mathbb{E}\!\left[F(x^{s}_{m})\!-\!F(x^{*})\right]\\
\leq& \sum^{m-1}_{k=0}\!\left\{\frac{2}{\beta\!-\!1}\!\left[F(x^{s}_{k})\!-\!F(x^{*})\!+\!F(\widetilde{x}^{s-1})\!-\!F(x^{*})\right]\right\}\!-\!\frac{2}{\beta\!-\!1}\!\sum^{m-2}_{k=0}\!\mathbb{E}\!\left[F(x^{s}_{k+1})\!-\!F(x^{*})\right]\!+\! \frac{L\beta}{2}\mathbb{E}\!\left[\|x^{*}\!-\!x^{s}_{0}\|^2\!-\!\|x^{*}\!-\!x^{s}_{m}\|^2\right].
\end{split}
\end{equation*}
Then
\begin{equation*}
\begin{split}
&\left(1-\frac{2}{\beta\!-\!1}\right)\sum^{m}_{k=1}\mathbb{E}\!\left[F(x^{s}_{k})-F(x^{*})\right]+\frac{2}{\beta\!-\!1}\mathbb{E}\!\left[F(x^{s}_{m})\!-\!F(x^{*})\right]\\
\leq& \frac{2}{\beta\!-\!1}\mathbb{E}\!\left[F(x^{s}_{0})-F(x^{*})\right]+\frac{2m}{\beta\!-\!1}\mathbb{E}\!\left[F(\widetilde{x}^{s-1})\!-\!F(x^{*})\right]+ \frac{L\beta}{2}\mathbb{E}\!\left[\|x^{*}\!-\!x^{s}_{0}\|^2\!-\!\|x^{*}\!-\!x^{s}_{m}\|^2\right].
\end{split}
\end{equation*}
Due to the settings of $\widetilde{x}^{s}\!=\!\frac{1}{m}\sum^{m}_{k=1}x^{s}_{k}$ and $x^{s+1}_{0}\!=\!x^{s}_{m}$, and the convexity of the objective function $F(\cdot)$, we have $F(\widetilde{x}^{s})\leq \frac{1}{m}\sum^{m}_{k=1}F(x^{s}_{k})$, and
\begin{equation*}
\begin{split}
&m\left(1-\frac{2}{\beta\!-\!1}\right)\mathbb{E}\!\left[F(\widetilde{x}^{s})-F(x^{*})\right]+\frac{2}{\beta\!-\!1}\mathbb{E}\!\left[F(x^{s}_{m})\!-\!F(x^{*})\right]\\
\leq&\frac{2}{\beta\!-\!1}\mathbb{E}\!\left[F(x^{s}_{0})-F(x^{*})\right]+\frac{2m}{\beta\!-\!1}\mathbb{E}\!\left[F(\widetilde{x}^{s-1})\!-\!F(x^{*})\right]+ \frac{L\beta}{2}\mathbb{E}\!\left[\|x^{*}\!-\!x^{s}_{0}\|^2\!-\!\|x^{*}\!-\!x^{s}_{m}\|^2\right].
\end{split}
\end{equation*}
Dividing both sides of the above inequality by $m$, we arrive at
\begin{equation*}
\begin{split}
&\left(1-\frac{2}{\beta\!-\!1}\right)\mathbb{E}\!\left[F(\widetilde{x}^{s})-F(x^{*})\right]+\frac{2}{(\beta\!-\!1)m}\mathbb{E}\!\left[F(x^{s}_{m})\!-\!F(x^{*})\right]\\
\leq&\frac{2}{(\beta\!-\!1)m}\mathbb{E}\!\left[F(x^{s}_{0})-F(x^{*})\right]+\frac{2}{\beta\!-\!1}\mathbb{E}\!\left[F(\widetilde{x}^{s-1})\!-\!F(x^{*})\right]+ \frac{L\beta}{2m}\mathbb{E}\!\left[\|x^{*}\!-\!x^{s}_{0}\|^2\!-\!\|x^{*}\!-\!x^{s}_{m}\|^2\right].
\end{split}
\end{equation*}

This completes the proof.
\end{proof}

\section*{Appendix E: Proof of Theorem~\ref{the3}}
\begin{proof}
Summing the inequality in (\ref{equ39}) over $s\!=\!1,2,\ldots,S$, and taking expectation with respect to the history of $i^{s}_{k}$, we have
\begin{equation*}
\begin{split}
&\sum^{S}_{s=1}\left(1\!-\!\frac{2}{\beta\!-\!1}\right)\!\mathbb{E}\!\left[F(\widetilde{x}^{s})-F(x^{*})\right]+\sum^{S}_{s=1}\frac{2}{(\beta\!-\!1)m}\mathbb{E}\!\left[F(x^{s}_{m})-F(x^{*})\right]\\
\leq&\,\sum^{S}_{s=1}\left\{\frac{2}{(\beta\!-\!1)m}\mathbb{E}\!\left[F(x^{s}_{0})-F(x^{*})\right]+\frac{2}{\beta\!-\!1}\mathbb{E}\!\left[F(\widetilde{x}^{s-1})-F(x^{*})\right]\right\}\\
&\,+\frac{L\beta}{2m}\sum^{S}_{s=1}\mathbb{E}\!\left[\|x^{*}-x^{s}_{0}\|^{2}-\|x^{*}-x^{s}_{m}\|^{2}\right].
\end{split}
\end{equation*}
Subtracting $\sum^{S}_{s=1}\!\frac{2}{(\beta-1)m}\mathbb{E}\!\left[F(x^{s}_{m})\!-\!F(x^{*})\right]\!+\!\frac{2}{\beta-1}\!\sum^{S}_{s=1}\!\left[F(\widetilde{x}^{s})\!-\!F(x^{*})\right]$ from both sides of the above inequality, and using the setting of $x^{s+1}_{0}\!=\!x^{s}_{m}$, we arrive at
\begin{equation*}
\begin{split}
&\sum^{S}_{s=1}\left(1-\frac{4}{\beta\!-\!1}\right)\mathbb{E}\!\left[F(\widetilde{x}^{s})-F(x^{*})\right]\\
\leq&\,\frac{2}{(\beta\!-\!1)m}\mathbb{E}\!\left[F(x^{1}_{0})\!-\!F(x^{S}_{m})\right]+\frac{2}{\beta\!-\!1}\mathbb{E}\!\left[F(\widetilde{x}^{0})\!-\!F(\widetilde{x}^{S})\right]+\frac{L\beta}{2m}\mathbb{E}\!\left[\|x^{*}\!-\!x^{1}_{0}\|^{2}-\|x^{*}\!-\!x^{S}_{m}\|^{2}\right].
\end{split}
\end{equation*}
It is easy to verify that $\mathbb{E}[F(\widetilde{x}^{0})\!-\!F(\widetilde{x}^{S})]\!\leq\!F(\widetilde{x}^{0})\!-\!F(x^{*})$. Dividing both sides of the above inequality by $S$, and using the choice $\widetilde{x}^{0}\!=\!x^{1}_{0}$, we obtain
\begin{equation}\label{equ43}
\begin{split}
&\:\left(1-\frac{4}{\beta\!-\!1}\right)\frac{1}{S}\sum^{S}_{s=1}\mathbb{E}\!\left[F(\widetilde{x}^{s})-F(x^{*})\right]\\
\leq&\,\frac{2}{(\beta\!-\!1)mS}\mathbb{E}\!\left[F(x^{1}_{0})\!-\!F(x^{S}_{m})\right]+\frac{2}{(\beta\!-\!1)S}[F(\widetilde{x}^{0})-F(\widetilde{x}^{S})]+\frac{L\beta}{2mS}\|x^{*}-x^{1}_{0}\|^{2}\\
\leq&\,\frac{2}{(\beta\!-\!1)mS}\!\left[F(\widetilde{x}^{0})-F(x^{*})\right]+\frac{2}{(\beta\!-\!1)S}[F(\widetilde{x}^{0})-F(x^{*})]+\frac{L\beta}{2mS}\|x^{*}-\widetilde{x}^{0}\|^{2}\\
=&\,\frac{2(m\!+\!1)}{(\beta\!-\!1)mS}[F(\widetilde{x}^{0})-F(x^{*})]+\frac{L\beta}{2mS}\|\widetilde{x}^{0}-x^{*}\|^{2}
\end{split}
\end{equation}
where the first inequality uses the fact that $\|x^{*}-x^{S}_{m}\|^{2}\geq0$; and the last inequality holds due to the facts that $\mathbb{E}\!\left[F(x^{1}_{0})\!-\!F(x^{S}_{m})\right]$ $\leq\! F(x^{1}_{0})\!-\!F(x^{*})$, $\mathbb{E}[F(\widetilde{x}^{0})\!-\!F(\widetilde{x}^{S})]\leq F(\widetilde{x}^{0})\!-\!F(x^{*})$, and $\widetilde{x}^{0}=x^{1}_{0}$.

Using the definition of $\overline{x}^{S}\!=\!\frac{1}{S}\!\sum^{S}_{s=1}\widetilde{x}^{s}$ and the convexity of the objective function $F(\cdot)$, we have $F(\overline{x}^{S})\!\leq\! \frac{1}{S}\!\sum^{S}_{s=1}\!F(\widetilde{x}^{s})$, and therefore we can rewrite the above inequality in \eqref{equ43} as
\begin{equation*}
\begin{split}
\left(1-\frac{4}{\beta\!-\!1}\right)\mathbb{E}\!\left[F(\overline{x}^{S})-F(x^{*})\right]\leq&\: \left(1-\frac{4}{\beta\!-\!1}\right)\frac{1}{S}\sum^{S}_{s=1}\mathbb{E}\!\left[F(\widetilde{x}^{s})-F(x^{*})\right]\\
\leq&\,\frac{2(m\!+\!1)}{(\beta\!-\!1)mS}[F(\widetilde{x}^{0})-F(x^{*})]+\frac{L\beta}{2mS}\|\widetilde{x}^{0}-x^{*}\|^{2}.
\end{split}
\end{equation*}
Dividing both sides of the above inequality by $(1\!-\!\frac{4}{\beta-1})\!>\!0$, we have
\begin{equation*}
\mathbb{E}\!\left[F(\overline{x}^{S})\right]-F(x^{*})\leq\frac{2(m\!+\!1)}{(\beta\!-\!5)mS}[F(\widetilde{x}^{0})-F(x^{*})]+\frac{\beta(\beta\!-\!1)L}{2(\beta\!-\!5)mS}\|\widetilde{x}^{0}-x^{*}\|^{2}.
\end{equation*}

When $F(\widetilde{x}^{S})\!\leq\!F(\overline{x}^{S})$, then $\widehat{x}^{S}\!=\!\widetilde{x}^{S}$, and
\begin{equation*}
\begin{split}
\mathbb{E}\!\left[F(\widehat{x}^{S})\right]-F(x^{*})\leq\,\frac{2(m\!+\!1)}{(\beta\!-\!5)mS}[F(\widetilde{x}^{0})-F(x^{*})]+\frac{\beta(\beta\!-\!1)L}{2(\beta\!-\!5)mS}\|\widetilde{x}^{0}-x^{*}\|^{2}.
\end{split}
\end{equation*}
Alternatively, if $F(\widetilde{x}^{S})\!\geq\!F(\overline{x}^{S})$, then $\widehat{x}^{S}\!=\!\overline{x}^{S}$, and the above inequality still holds.

This completes the proof.
 \end{proof}

\section*{Appendix F: Proof of Theorem~\ref{the7}}
\begin{proof}
Since each $f_{i}(\cdot)$ is convex and $L$-smooth, then we have
\begin{equation*}
\begin{split}
&\left(1-\frac{2}{\beta\!-\!1}\right)\mathbb{E}\!\left[F(\widetilde{x}^{s})-F(x^{*})\right]+\frac{2}{m(\beta\!-\!1)}\mathbb{E}\!\left[F(x^{s+1}_{0})\!-\!F(x^{*})\right]+\frac{L\beta}{2m}\mathbb{E}\!\left[\|x^{s+1}_{0}\!-\!x^{*}\|^2\right]\\
\leq&\,\frac{2}{\beta\!-\!1}\mathbb{E}\!\left[F(\widetilde{x}^{s-1})\!-\!F(x^{*})\right]+\frac{2}{m(\beta\!-\!1)}\mathbb{E}\!\left[F(x^{s}_{0})-F(x^{*})\right]+ \frac{L\beta}{2m}\mathbb{E}\!\left[\|x^{s}_{0}\!-\!x^{*}\|^2\right]\\
\leq&\,\frac{2}{\beta\!-\!1}\mathbb{E}\!\left[F(\widetilde{x}^{s-1})\!-\!F(x^{*})\right]+\left(\frac{2}{m(\beta\!-\!1)}+\frac{L\beta}{m\mu}\right)\mathbb{E}\!\left[F(x^{s}_{0})-F(x^{*})\right]\\
\leq&\,\left(\frac{2(m\!+\!C)}{m(\beta\!-\!1)}+\frac{CL\beta}{m\mu}\right)\mathbb{E}\!\left[F(\widetilde{x}^{s-1})-F(x^{*})\right]
\end{split}
\end{equation*}
where the first inequality follows from Lemma~\ref{lemm6}; the second inequality holds due to the fact that $\|x^{s}_{0}-x^{*}\|^2\leq({2}/{\mu})[F(x^{s}_{0})-F(x^{*})]$; and the last inequality follows from Assumption~\ref{assum3}.

Due to the definition of $\beta\!=\!1/(L\eta)$, the above inequality is rewritten as follows:
\begin{equation*}
\begin{split}
\frac{1\!-\!3L\eta}{1\!-\!L\eta}\mathbb{E}\!\left[F(\widetilde{x}^{s})-F(x^{*})\right]\leq\left(\frac{2L\eta(m\!+\!C)}{m(1\!-\!L\eta)}+\frac{C}{m\mu\eta}\right)\mathbb{E}\!\left[F(\widetilde{x}^{s-1})-F(x^{*})\right].
\end{split}
\end{equation*}
Dividing both sides of the above inequality by $(1\!-\!3L\eta)(1\!-\!L\eta)\!>\!0$, we arrive at
\begin{equation*}
\begin{split}
\mathbb{E}\!\left[F(\widetilde{x}^{s})-F(x^{*})\right]\leq\left(\frac{2L\eta(m\!+\!C)}{m(1\!-\!3L\eta)}+\frac{C(1\!-\!L\eta)}{m\mu\eta(1\!-\!3L\eta)}\right)\mathbb{E}\!\left[F(\widetilde{x}^{s-1})-F(x^{*})\right].
\end{split}
\end{equation*}
This completes the proof.
\end{proof}

\bibliographystyle{IEEEtran}
\bibliography{IEEEabrv,nips17}

\end{document}